%% file: main.tex
\documentclass[10pt]{article}
\usepackage[preprint]{tmlr}

\input{preamble}

\title{Hidden Activations are not Enough~I:\\ Knowledge Matrices as Higher Representations}
\author{\name Marco Armenta \email marco.armenta@usherbrooke.ca \\
      \addr Institut quantique, Universit\'e de Sherbrooke}
\date{\today}

\begin{document}
\maketitle

\input{sections/abstract}

\input{sections/introduction}        %
\input{sections/fig1_germ_identity}  %
\input{sections/previous_work}       %
\input{sections/section_germ}        %
\input{sections/section_geometry}    %
\input{sections/empirical_setup}     %
\input{sections/study1_invariance}   %
\input{sections/study2_coherence}    %
\input{sections/study3_cross_arch}    %
\input{sections/study4_negatives}     %
\input{sections/limitations}         %
\input{sections/conclusion}          %
\input{sections/statements}          %

\appendix
\input{sections/appendix_proofs}      %
\input{sections/appendix_engineering} %

\input{sections/appendix_teleport_exact}

\input{tables/table_coherence_max}
\input{tables/s3_table}

\input{sections/appendix_A_geometric} %

\input{sections/appendix_remarks}          %
\input{sections/appendix_limitations_more} %
\input{sections/appendix_setup_details}    %
\input{sections/appendix_mechanism_pilot}  %
\input{sections/appendix_kendall}          %
\input{sections/appendix_negatives_detail} %

\providecommand{\vocabdir}{sections/vocab}
\input{sections/vocab/appendix_vocab_I}

\bibliographystyle{tmlr}
\bibliography{references}

\end{document}

%% file: preamble.tex
\usepackage[utf8]{inputenc}
\usepackage[T1]{fontenc}
\usepackage{lmodern}
\usepackage{amsmath,amssymb,amsthm}
\usepackage{booktabs}
\usepackage{longtable}
\usepackage{array}
\usepackage{graphicx}
\usepackage{enumitem}
\usepackage{xcolor}
\usepackage[colorlinks=true,linkcolor=blue!70!black,citecolor=green!50!black,urlcolor=blue!70!black]{hyperref}
\usepackage{tikz}
\usetikzlibrary{arrows.meta,decorations.pathreplacing,calc,positioning}
\usepackage[most]{tcolorbox}

\newtheorem{theorem}{Theorem}
\newtheorem{proposition}[theorem]{Proposition}
\newtheorem{definition}[theorem]{Definition}
\newtheorem{remark}[theorem]{Remark}
\newtheorem{lemma}[theorem]{Lemma}
\newtheorem{corollary}[theorem]{Corollary}
\newtheorem{conjecture}[theorem]{Conjecture}
\numberwithin{theorem}{section}

\newcommand{\R}{\mathbb{R}}
\newcommand{\M}{\mathrm{M}}
\newcommand{\Mfx}{\M(W\!,f)(x)}
\newcommand{\Mfxp}{\M(W\!,f)(x')}
\newcommand{\Weff}{W_{\text{eff}}}
\newcommand{\beff}{b_{\text{eff}}}
\newcommand{\KM}{\textsc{KM}}

\newcommand{\Net}{\Psi(W\!,f)}          %
\newcommand{\Netx}{\Psi(W\!,f)(x)}      %
\newcommand{\Nety}{\Psi(W\!,f)(y)}
\newcommand{\Netxp}{\Psi(W\!,f)(x')}

\newcommand{\KMwf}{\M(W\!,f)}
\newcommand{\KMwfx}{\M(W\!,f)(x)}
\newcommand{\KMarg}[2]{\M(#1\!,f)(#2)}
\newcommand{\Jwf}{J(W\!,f)}
\newcommand{\cwf}{c(W\!,f)}

\newcommand{\kmapx}{\phi(W\!,f)(x)}

%% file: sections/abstract.tex
\begin{abstract}
We study the knowledge matrix of a trained feedforward network as a
\emph{higher} representation of its inputs. A network is a pair $(W\!,f)$, a
thin representation $W$ of its quiver and an activation $f$; \citet{armenta2021representation} showed that its
function factorizes through the space of quiver representations, each input
$x$ inducing a representation $\kmapx$, and the knowledge matrix
$\Mfx\in\R^{C\times(d+1)}$ is the contraction of that representation to one
matrix whose rows sum exactly to the logits \citep{leblanc2024hidden}. At one
trained network we ask what determines it, what it is invariant to, what it
determines, and what its geometry measures. Under (LCS), a locally constant
slope diagonal, as for ReLU, the matrix at a regular input is a function of the
realized germ; its stabilizer among encodings regular there is exactly the germ
stabilizer at inputs with no vanishing coordinate, neuron permutation a special
case; and it recovers the germ,
whereas hidden activations, gauge-covariant and germ-incomplete, are not
enough. Under (LCS) it equals per-class gradient$\times$input plus an exact
aggregate bias attribution, grounding it in attribution theory and computing it
by $C$ vector--Jacobian products instead of probing. The fixed shape gives an
alignment-free per-sample distance between ResNet-152, DenseNet-121 and
GoogLeNet; the row-sum identity gives an exact visible/invisible displacement
decomposition whose unit-free coherence $A=(d_\Psi/d_M)^2$ puts adversarial
germ motion at median $A\le0.23$, with an attack-family ordering, an
ordering of how far past the decision boundary each attack pushes, concordant
across six architectures (Kendall $W=0.921$; $0.97$ on the three networks at
full scale). Two honest negatives: on AlexNet/CIFAR-10 penultimate features win
$5$ of $6$ detectors and all $16$ attacks, and a matrix-direction
counterfactual fails $0/54$.
\end{abstract}

%% file: sections/introduction.tex
\section{Introduction}
\label{sec:introduction}

A trained neural network is often used not for its final logits but for an
intermediate representation of its inputs: the activations of a hidden layer,
typically the penultimate one. That representation drives transfer learning, anomaly
detection \citep{reiss2022anomaly,roth2022towards}, model comparison through
representation-similarity analyzes \citep{kornblith2019similarity},
federated-learning aggregation, and out-of-distribution detection
\citep{lee2018simple,papernot2018deep}. All of these uses assume that the
hidden representation is a \emph{canonical} object attached to the trained
network: something that captures what the network has learned about each
sample, intrinsic to the network's function.

The assumption is wrong. A network's function is preserved under a
non-trivial group of weight-space transformations: at minimum the permutation
symmetries of each layer's neurons, and more generally the quiver-isomorphism
group of \citet{armenta2021representation,armenta2022double}. The hidden
activations of every layer change under these transformations while the
network computes exactly the same function. They are invariants of the network
\emph{plus a choice of parameterization}, not of the network. Two networks that
compute the same function can have hidden activations that differ by an
orthogonal change of basis, or worse, so any application built on a distance
between hidden activations is sensitive to weight-space conventions with no
functional consequence. The optimization literature has used this weakness to
disprove flat-minima generalization arguments \citep{dinh2017sharp}; the
analysis pipelines that take hidden features as a primitive do not, to our
knowledge, address it. Everything we prove about hidden activations holds for
the activations of any hidden layer; in the experiments we compare against the
penultimate layer, the customary choice, and nothing in the theory depends on
that choice.

Following \citet{armenta2021representation}, a neural network is a pair
$(W\!,f)$: $W$ is a thin representation of the network quiver $Q$, the directed
graph with one vertex per neuron and one arrow per weight, and $f$ is the
activation function attached to the hidden vertices. The pair realizes the
network function $\Net:\R^d\to\R^C$ on $d$ input coordinates and $C$ classes,
and Theorem~6.4 of \citet{armenta2021representation} shows that this function
factorizes through the space of thin representations of $Q$: each input $x$
induces a representation $\kmapx$, whose arrows leaving a hidden vertex carry
the weights of $W$ multiplied by that vertex's activation-to-pre-activation
quotient in the forward pass of $x$, those leaving an input vertex the weight
multiplied by the input coordinate, and those leaving a bias vertex the
weight, and the network function is recovered from $\kmapx$ alone, by
feeding it the all-ones vector. The induced representation is a linear object, so it can
be contracted to a single matrix by multiplying out its layers
\citep{armenta2022double}. That contraction is the \emph{knowledge matrix}
$\Mfx\in\R^{C\times(d{+}1)}$, drawn in Figure~\ref{fig:germ}, with one column
per input coordinate and one for the biases, and the identity the contraction
preserves is the \emph{row-sum identity} of \citet{leblanc2024hidden}: the rows
of $\Mfx$ sum to the logits, $\Mfx\,\mathbf 1_{d+1}=\Netx$, row
$i\in\{1,\ldots,C\}$ recording how the input contributes to class $i$'s logit
through every path of the network. So a network and an input sample together
produce a quiver representation, and that representation contracts to the
knowledge matrix. This paper studies the matrix at one trained network, where
the input moves and the weights are fixed, and answers four questions about
it: what determines it, what it is invariant to, what it determines, and what
its geometry measures. It then asks what those answers buy that hidden
activations cannot provide.

The row-sum identity is completeness in the attribution sense of the parts
summing to the output, and it is not a new pointwise invariant. When the slope
diagonal is locally constant ((LCS), Definition~\ref{def:lcs}; the ReLU family,
and our setting throughout, by Proposition~\ref{prop:lcs}), the knowledge
matrix at almost every input equals per-class gradient$\times$input plus an
exact aggregate bias attribution (Theorem~\ref{thm:weff}). Its invariance,
completeness, fixed shape and exact row sums are therefore shared with that
gradient-based data. What is canonical, in the precise sense hidden
activations fail, is the \emph{arrangement} --- an arrangement that per-class
gradient$\times$input with FullGrad's bias column shares, the contrast class
being hidden activations --- and what it enables. First, the
knowledge matrix is invariant under the entire stabilizer of the realized germ
at $x$, among encodings at which $x$ remains regular
(Theorem~\ref{thm:maxinv}, Section~\ref{sec:germ}; shared with per-class
gradient$\times$input under (LCS), Theorem~\ref{thm:weff}). It is the same for
any two parameterizations, indeed any two architectures, that compute the same
function on a neighborhood of an input at which both are regular; neuron
permutation is a special case (Corollary~\ref{prop:perm}), automatic for
anything built, as the knowledge matrix is, from the induced representation
of the quiver-representation approach to neural networks
\citep{armenta2021representation} by a construction that respects relabelings
of the hidden vertices and changes of basis at them. Second,
the arrangement carries an exact partitioned row sum into a fixed
$C\times(d{+}1)$ shape, comparable across architectures without an alignment
step; hidden activations lack that property, and Study~3 demonstrates it as
a standalone capability. Third, building on the row sums,
Theorem~\ref{thm:pyth} (Section~\ref{sec:geometry}) splits any
knowledge-matrix displacement into a logit-visible part and a logit-invisible
part; the latter encodes real changes of the local linearization that the
endpoint logits cannot see. The unit-free coherence $A=(d_\Psi/d_M)^2$ of
Definition~\ref{def:coherence}, with $d_M$ and $d_\Psi$ the knowledge-matrix
and logit displacements of a pair of inputs, is a geometric descriptor read
against the single-pixel reference line $A{=}1$.

Hidden activations are thus gauge-covariant and germ-incomplete as a
representation of neural-network behavior (Theorem~\ref{thm:complete}), and
knowledge matrices are a canonical alternative. We give three independent
demonstrations on pretrained ImageNet networks, ResNet-152, DenseNet-121 and
GoogLeNet, chosen to span the residual, dense and inception families. Study~1
(Section~\ref{sec:study1}) asks what the standard representational-similarity
machinery does with a transformation that provably changes nothing about the
function. The invariance of the knowledge matrix under the two symmetries at
issue, neuron permutation and neural teleportation
\citep{armenta2023teleportation}, needs no experiment: teleportation is a
change of basis at the hidden vertices, for which Theorem~4.13 of
\citet{armenta2021representation} gives $\Psi(W,f)=\Psi(V,g)$, batch
normalization in evaluation mode included, a permutation is a relabeling of
the hidden vertices, which preserves the function trivially, and under both
the induced representations contract to the same matrix
(Lemma~\ref{lem:quiver-inv}). The effect on the penultimate features has a closed form in $h$ and the
transform, and we report the algebra rather than a measurement. What is open is
whether the standard similarity measures --- designed to be invariant to
orthogonal changes of basis and isotropic scaling, and deliberately not to
invertible linear maps \citep[\S2.3]{kornblith2019similarity} --- absorb this
transformation, and by how much; Section~\ref{sec:study1c} settles that on
pairs whose ground truth is known exactly, with the random-network and
shuffled-pair controls: none of the eight measures that returned a value
recovers the exact invariance, and the two that would, raw CCA and PWCCA, are
the ones \citet{kornblith2019similarity} set aside for being invariant to
every invertible linear map. Study~2 (Section~\ref{sec:study2})
applies the visible/invisible decomposition of Theorem~\ref{thm:pyth} to
adversarial pairs on all three architectures and six adversarial-pair
generators (FGSM, PGD, CW, DeepFool, APGD, Square). The coherence $A$ of
adversarial germ motion sits well below the single-pixel line, its
attack-family ordering is concordant across the six rater architectures of the
$200$-pair set (Kendall $W=0.921$) --- an ordering set by how far past the
decision boundary each attack pushes, and on that pair set no more concordant
than the logit displacement alone --- and the population-matched full-scale
panel on the three networks reproduces the same family-level grouping
($W=0.97$). Study~3 (Section~\ref{sec:crossarch}) compares architectures
directly: the knowledge matrix has the same shape for every feedforward
network on a given input, so its Frobenius distance compares ResNet-152,
DenseNet-121 and GoogLeNet with no alignment step. We then report two
\emph{honest negatives} (Section~\ref{sec:negatives}). In a detector bake-off
on AlexNet/CIFAR-10, penultimate features win $5$ of the $6$ detectors, all
$16$ of the attacks and every SVD rank from $16$ to $512$, so the
adversarial-detection claim of \citet{leblanc2024hidden} is not made here; and
a linear-program counterfactual fails structurally on the three pretrained
ImageNet networks. Section~\ref{sec:limitations} catalogs the limitations of
the paper, each with a forward pointer, and Section~\ref{sec:conclusion}
concludes.

What this paper inherits is set out in Section~\ref{sec:previous}: the object,
its row-sum identity and its invariance under quiver isomorphisms
\citep{armenta2021representation,leblanc2024hidden}, and the algebraic split of
a locally affine network into gradient$\times$input plus a bias remainder
\citep{srinivas2019full,balestriero2018spline,ancona2018towards}, which
Theorem~\ref{thm:weff} states as an equivalence on which we claim no priority.
What we add is the reading of that object as a function invariant, and the
geometry that follows from its row sums. In summary, our work makes the
following key contributions:
\begin{enumerate}
\item The germ identity (Theorem~\ref{thm:germ}) reads the matrix as
\emph{parameterization-independent}: it depends on the weights only through the
germ of the realized function at $x$.
\item The converse, \emph{germ recovery}, is Theorem~\ref{thm:complete}(i):
the matrix determines the local function, not merely its value. Every prior
notion of completeness we are aware of says that the parts sum to the output;
this one says the object pins down the germ, and it is what makes the title
claim a theorem rather than a phrase.
\item The invariance is \emph{maximal} (Theorem~\ref{thm:maxinv},
Corollary~\ref{cor:stab-eq}): for a fixed architecture, among encodings
regular at $x$ with every $x_i\neq0$, the stabilizer of the matrix in any
group acting on the parameters is exactly the germ stabilizer at $x$ --- the
stabilizer of the germ itself, hence the largest group under which a
germ-determining observable can be invariant --- and across architectures the
matrix and the germ are functions of each other. We have not found this
stated. Among encodings regular at $x$, invariance under the germ stabilizer
implies Implementation Invariance \citep{sundararajan2017axiomatic} at $x$ and
is stronger, since it needs only local agreement; off the regular set two
globally equivalent encodings can differ (the identity-versus-$g$ example of
Section~\ref{sec:invariance}), so the matrix satisfies Implementation
Invariance in the everywhere-quantified sense of
\citet{sundararajan2017axiomatic} only almost everywhere, whereas integrated
gradients satisfies it everywhere. Contributions 2 and 3 follow from
contribution 1 and the injectivity of $(J,c)\mapsto[\,J\,\mathrm{diag}(x)\mid c\,]$
off $\{x_i=0\}$; their content is the statement, not the proof.
\item The distance geometry of Section~\ref{sec:geometry} is new in its
entirety: the exact visible/invisible split and its unit-free coherence $A$
(Theorem~\ref{thm:pyth}), the within-region and wall-crossing anatomy
(Theorems~\ref{thm:within} and~\ref{thm:anatomy}), and the impossibility of a
$C^0$ bound in the other direction (Theorem~\ref{thm:nogo}).
\item The $C$-row fixed-width arrangement, one row per class, $d$ input
columns and a single scalar offset column, is what makes the width
architecture-independent, and it is what Study~3 spends: no alignment step,
and a per-sample distance whose logit-visible component is exactly the logit
displacement divided by $\sqrt{d{+}1}$ (Theorem~\ref{thm:pyth}).
\end{enumerate}

Two words are used with care throughout. ``Completeness'' has two
incompatible senses in this literature: where we mean the attribution sense,
the entries of a row sum to that class's logit, we say \emph{row-sum} or
\emph{summation-to-output}, and that property is inherited; where we mean that
the matrix \emph{determines the germ} we say \emph{germ recovery},
Theorem~\ref{thm:complete}(i), and only that is a contribution of this paper.
And ``canonical'' never means superior: we do not claim that knowledge matrices
beat penultimate features at any task, the bake-off of
Section~\ref{sec:negatives} finding the opposite, nor that they are the only
computable function invariant, since per-class gradient$\times$input and
FullGrad share the invariance (Theorem~\ref{thm:weff}). Nothing in this paper
concerns populations of training runs: every comparison is between fixed
pretrained networks.

The paper is organized as follows. Section~\ref{sec:previous} places the
object in the literature. Section~\ref{sec:germ} proves that under (LCS) the
knowledge matrix is the germ of the network function, with its invariance,
completeness and attribution consequences, and Section~\ref{sec:geometry}
works out the distance geometry that the row-sum identity forces.
Section~\ref{sec:setup} fixes the empirical apparatus,
Sections~\ref{sec:study1}--\ref{sec:crossarch} report the three studies,
Section~\ref{sec:negatives} the two negatives, and
Sections~\ref{sec:limitations}--\ref{sec:conclusion} the limitations and the
conclusion. Proofs of all numbered statements are collected in the first
appendix, Appendix~\ref{app:proofs}; the later appendices hold the computation
of the matrix in software, the implementation checks, and the supplementary
remarks and tables that the main text refers to.

%% file: sections/fig1_germ_identity.tex
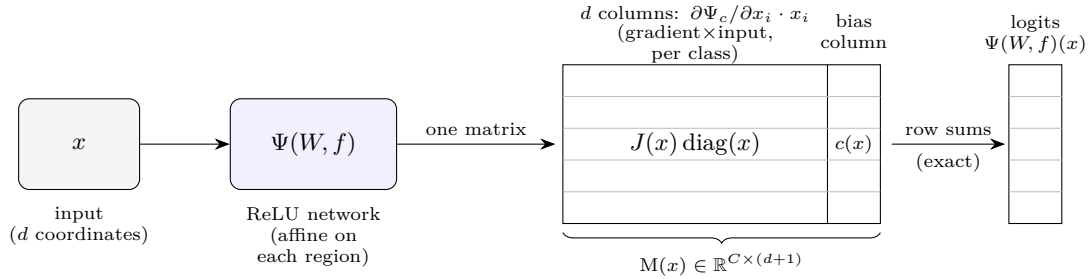
\begin{figure}[!ht]
\centering
\begin{tikzpicture}[x=1cm,y=1cm,font=\small,>={Stealth[length=2mm]}]
  \node[draw,rounded corners,minimum width=1.6cm,minimum height=1.2cm,fill=gray!8] (x) at (0,0) {$x$};
  \node[below=1mm of x,font=\scriptsize,text width=2.2cm,align=center] {input\\($d$ coordinates)};
  \node[draw,rounded corners,minimum width=2.2cm,minimum height=1.2cm,fill=blue!6] (net) at (3.1,0) {$\Net$};
  \node[below=1mm of net,font=\scriptsize,text width=2.6cm,align=center] {ReLU network\\(affine on each region)};
  \draw[->] (x) -- (net);
  \begin{scope}[shift={(6.4,-1.05)}]
    \draw (0,0) rectangle (4.2,2.1);
    \draw (3.5,0) -- (3.5,2.1);
    \foreach \yy in {0.42,0.84,1.26,1.68} \draw[gray!50] (0,\yy) -- (4.2,\yy);
    \node[font=\scriptsize,align=center,text width=3.3cm] at (1.75,2.5) {$d$ columns: $\partial \Psi_c/\partial x_i \cdot x_i$\\[-1pt](gradient$\times$input, per class)};
    \node[font=\scriptsize,align=center,text width=1.6cm] at (3.85,2.55) {bias\\column};
    \node at (1.75,1.05) {$J(x)\,\mathrm{diag}(x)$};
    \node[font=\scriptsize] at (3.85,1.05) {$c(x)$};
    \draw[decorate,decoration={brace,mirror,amplitude=4pt}] (0,-0.12) -- (4.2,-0.12)
      node[midway,below=4pt,font=\scriptsize] {$\mathrm{M}(x)\in\mathbb{R}^{C\times(d+1)}$};
  \end{scope}
  \draw[->] (net) -- node[above,font=\scriptsize]{one matrix} (6.3,0);
  \begin{scope}[shift={(12.3,-1.05)}]
    \draw (0,0) rectangle (0.7,2.1);
    \foreach \yy in {0.42,0.84,1.26,1.68} \draw[gray!50] (0,\yy) -- (0.7,\yy);
    \node[font=\scriptsize,align=center,text width=2.4cm] at (0.35,2.5) {logits\\$\Netx$};
  \end{scope}
  \draw[->] (10.75,0) -- node[above,font=\scriptsize]{row sums}
            node[below,font=\scriptsize]{(exact)} (12.25,0);
\end{tikzpicture}
\caption{\textbf{The knowledge matrix encodes the germ of the network function.} At almost every
input, the network computes an affine map $x' \mapsto J x' + c$ on the surrounding linear
region --- its germ (Definition~\ref{def:germ}). The knowledge matrix arranges this germ as
one fixed-shape matrix
$\mathrm{M}(x)=[\,J(x)\,\mathrm{diag}(x) \mid c(x)\,]$: its first $d$ columns are per-class
gradient$\times$input contributions, its last column collects every bias term, and each row
sums \emph{exactly} to the corresponding logit. Because column $i$ is scaled by $x_i$, the
germ is \emph{recoverable} from the matrix at coordinates where $x_i\neq0$
(Remark~\ref{rem:zero-pixels}); at $x_i=0$ that column vanishes. Everything in this paper
follows from this picture: invariance (the germ does not change under reparameterizations that
preserve the function), completeness (the germ is recoverable from the matrix where the input
is nonzero), and the distance geometry
(the row-sum constraint splits any matrix change into a function-visible part and a
function-invisible remainder).}
\label{fig:germ}
\end{figure}

%% file: sections/previous_work.tex
\section{Previous Work}
\label{sec:previous}

\citet{armenta2021representation} define a neural network as a pair $(W\!,f)$:
a thin representation $W$ of the network quiver, whose vertices are the neurons
and whose arrows carry the weights, together with an activation $f$ at the
hidden vertices. They show that each input $x$ induces a second thin
representation $\kmapx$ of the same quiver, in which every arrow leaving an
input vertex carries its weight multiplied by that input coordinate, every
arrow leaving a bias vertex its weight, and every arrow leaving a hidden
vertex its weight multiplied by that vertex's activation-to-pre-activation
quotient $f(z)/z$ (their Definition~6.2). Their Theorem~4.13 states that an
isomorphism of neural networks --- a change of basis at the hidden vertices,
the input, bias and output vertices fixed, with the activation of a vertex
rescaled by $\tau$ carried as $z\mapsto\tau f(z/\tau)$ (their
Definitions~4.9 and~4.12) --- preserves the realized function; a relabeling
of the neurons is an automorphism of the quiver and is not among these
isomorphisms. Their Theorem~6.4 states that, under their standing assumption
that no pre-activation vanishes (their Remark~6.3), the network function
factorizes through the induced representation, $\Netx$ being the output of
$\kmapx$ on the all-ones vector, for an arbitrary activation; its proof
treats the hidden vertices that apply an activation to a weighted sum, and
their Section~5 records that average pooling sits inside the framework
while max-pooling, whose vertices take a maximum instead, breaks the
algebraic structure. Their Remark~5.4 places batch normalization at test
time inside the framework, its running statistics being ordinary weights.
\citet{armenta2022double} study the moduli spaces of double-framed quiver
representations, in which the isomorphism classes of networks live, prove that
the output of a network depends only on the corresponding point of that moduli
space, and contract a thin representation such as $\kmapx$ to a single matrix
by multiplying out its layers. Neural teleportation
\citep{armenta2023teleportation} realizes the change-of-basis group of the
representation $W$ as an operation on a trained network, a per-neuron
rescaling that leaves the function exactly unchanged and moves every hidden
activation, and studies its effect on the loss landscape and on optimization;
it is the transform Study~1 applies. \citet{leblanc2024hidden} call the map
$x\mapsto\kmapx$ the knowledge map and its contraction the knowledge matrix,
prove that the matrix reproduces the logits by its row sums and is invariant
under isomorphisms of neural networks, and propose it as the input to an
adversarial-example detector. From this strand we take the object itself and
its two inherited facts, the row-sum identity and invariance under changes of
basis, both verbatim for networks without max-pooling; the extension of the
contraction to max-pooling, through a locally constant selection matrix, and
to quivers with skip connections and parallel branches, through path sums,
is stated in Section~\ref{sec:germ} and is ours. The detection claim of
\citet{leblanc2024hidden} is not made here (Section~\ref{sec:negatives}).

The same quotient $f(z)/z$ is known to the attribution literature.
Gradient$\times$input \citep{shrikumar2016not} attributes a class score to the
input coordinates by the product of the input gradient with the input;
\citet{ancona2018towards} showed that $\varepsilon$-LRP \citep{bach2015lrp} is,
in the $\varepsilon\to0$ limit and with the bias included in the denominators,
a modified backward pass in which each unit's derivative is replaced by the
quotient $f(z)/z$, the diagonal from which the knowledge matrix is built, and
placed it, with DeepLIFT and deep Taylor decomposition
\citep{montavon2017deeptaylor}, in one family. FullGrad \citep{srinivas2019full}
proves that a network with biases decomposes exactly into an input-gradient
term and a bias-gradient term (their Proposition~3), that no saliency map
valued in $\R^{d}$ can in general be both complete --- the output recoverable
from the map and the input --- and weakly dependent on the input --- constant
on each activation region (their Proposition~1) --- and that an implicit bias
term extends the decomposition to
arbitrary nonlinearities (their Section~4). The affine-spline view
\citep{balestriero2018spline,balestriero2018madmax} reads a piecewise-linear
network as a per-region affine operator, which is the object the knowledge
matrix arranges, and Jacobian analyzes of bias-free denoisers
\citep{mohan2020robust} work with the same operator.
\citet{sundararajan2017axiomatic} state the axioms of Sensitivity,
Implementation Invariance and Completeness for integrated gradients;
Implementation Invariance is quantified over networks that agree globally and
is asked at every input. Among encodings regular at $x$, invariance under the
germ stabilizer of Theorem~\ref{thm:maxinv} implies it at $x$ and is stronger,
since it needs only local agreement; off the regular set two globally
equivalent encodings can give different matrices
(Section~\ref{sec:invariance}), so the knowledge matrix satisfies
Implementation Invariance in that everywhere-quantified sense only almost
everywhere, whereas integrated gradients satisfies it everywhere. \citet{crabbe2023evaluating} prove the
permutation-equivariance law that Proposition~\ref{prop:equivar}(iii) restates
for gradient-based explanations of invariant models. The sanity checks of
\citet{adebayo2018sanity} and the remove-and-retrain protocol of
\citet{hooker2019roar} are the faithfulness tests an attribution method must
face; we run neither, and this paper makes no attribution-quality claim.
\citet{ghorbani2019fragile} and \citet{dombrowski2019explanations} showed that
gradient-based maps, gradient$\times$input and integrated gradients among
them, move substantially under input perturbations that barely move the
output, and traced the motion to the walls between the affine regions of a
ReLU network; Study~2 measures that motion exactly, through the
visible/invisible split of Theorem~\ref{thm:pyth}. From
this strand we take the identification itself: Theorem~\ref{thm:weff}
identifies the knowledge matrix, for (LCS) networks, with per-class
gradient$\times$input plus an exact aggregate bias attribution, and we claim no
priority on that split.

The representation-similarity literature supplies the measures of Study~1.
Linear CKA \citep{kornblith2019similarity} compares two representations of the
same inputs through the Hilbert--Schmidt independence criterion of their Gram
matrices and is invariant to orthogonal transformations and isotropic scaling;
the unbiased estimator of \citet{song2012feature}, as used in the minibatch CKA
of \citet{nguyen2021minibatch}, removes its $O(1/n)$ upward bias. Orthogonal
Procrustes shape distance \citep{williams2021generalized} and Bures similarity
\citep{harvey2023bures} are metrics on the orthogonal-quotient shape space,
soft-matching distance \citep{khosla2024soft} quotients permutations only,
representational similarity analysis \citep{kriegeskorte2008representational}
compares distance matrices by rank, Gromov--Wasserstein distance
\citep{memoli2011gromov} compares metric-measure spaces up to isometry,
distance correlation \citep{szekely2007measuring} measures dependence, and
SVCCA \citep{raghu2017svcca} truncates before a canonical correlation analysis;
the ReSi benchmark \citep{klabunde2025resi} catalogs these measures by
invariance class. Two controls and two critiques govern how such numbers are
read: the random-network control of \citet{cui2022deconfounded} and the
shuffled-pair control of \citet{murphy2024debiased} separate learned structure
from input confounding and estimator bias, \citet{ding2021grounding} ground the
measures in statistical tests and document sensitivities that differ from
measure to measure, \citet{davari2023reliability} show that CKA can be
manipulated, and \citet{bansal2021revisiting} set the measures against model
stitching. From this strand we take the nine-measure panel and the two controls
of Study~1, whose question is whether any of these measures quotients out a
transformation that provably changes nothing about the function, and the
observation that the knowledge-matrix distance, being invariant under the whole
quiver-isomorphism group, needs no alignment step at all.

The symmetry that Study~1 exercises has a long history of its own. That a
positive per-neuron rescaling leaves a ReLU network's function unchanged while
moving every hidden activation is the positive-homogeneity symmetry on which
\citet{neyshabur2015path} build Path-SGD, which \citet{dinh2017sharp} use to
show that the sharpness of a minimum is reparameterization-dependent, and
which \citet{lyu2020gradient} assume in their analysis of gradient descent on
homogeneous networks. \citet{hashimoto2024unification} read such parametric
redundancy as a gauge symmetry, the usage we adopt, and
\citet{gorokhovik2016positively} survey the positively homogeneous functions
that Proposition~\ref{prop:lcs} classifies on the line. Permutation symmetry is
the other half of the quiver-isomorphism group: \citet{entezari2022permutation}
conjecture that it accounts for the barriers of linear mode connectivity, and
\citet{ainsworth2023gitrebasin} align networks modulo it. The identifiability
literature asks how much of a network its function determines:
\citet{phuong2020functional} characterize functional equivalence for ReLU
networks of non-increasing widths as permutation and positive rescaling,
\citet{rolnick2020reverse} recover a ReLU network's architecture and weights
from queries up to those symmetries, \citet{grigsby2022functional} measure the
functional dimension of ReLU networks, and \citet{flinth2026fibers} study the
fibers of the ReLU neuromanifold. \citet{chu2018exact} treat the whole
piecewise-linear family exactly by carrying a per-neuron slope and intercept,
\citet{lakshminarayanan2020npf} isolate the gating pattern as a feature in its
own right, and \citet{novak2018sensitivity} relate input--output Jacobian norms
to generalization. What this paper adds to the strand is stated as a theorem:
for a fixed architecture, among encodings regular at $x$ with no vanishing
coordinate, the stabilizer of the knowledge matrix in any group acting on the
parameters equals the stabilizer of the germ (Theorem~\ref{thm:maxinv},
Corollary~\ref{cor:stab-eq}), a group that contains the quiver-isomorphism
group and the global function stabilizer; across architectures the matrix is
a function of the germ, so any two encodings that realize the same germ at
$x$ and are regular there give the same matrix. We have not found a prior
result identifying a representation's stabilizer as maximal in this sense.

Finally, the knowledge matrix sits beside the objects of interpretability
without replacing any of them. Like sparse autoencoders
\citep{bricken2023monosemanticity,templeton2024scaling,cunningham2023sparse} and
the lens methods \citep{nostalgebraist2020logitlens,belrose2023tunedlens}, the
knowledge matrix expresses the output as an additive decomposition; unlike a
sparse autoencoder, whose overcomplete dictionary is learned to defeat
superposition, its basis is given by the network quiver, rows indexing classes
and columns indexing input coordinates and a bias slot, and it makes no
monosemanticity claim, so it occupies the opposite end of that design space. A
partial row sum is structurally the lens operation of projecting an intermediate
object through the unembedding \citep{elhage2021mathematical}, and the knowledge
matrix is defined for feedforward convolutional networks rather than transformer
residual streams; we run no such row-lens experiment here. Saliency maps
\citep{simonyan2013deep,selvaraju2017gradcam,smilkov2017smoothgrad} are
per-input like the knowledge matrix, whereas feature visualization
\citep{olah2017feature,olah2020zoom}, network dissection \citep{bau2017network}
and concept activation vectors \citep{kim2018tcav} interpret units and
directions across a population of inputs. We position the knowledge matrix as a
canonical complement to these methods and to CKA, not as a replacement for any
of them.

%% file: sections/section_germ.tex
\section{The knowledge matrix is the germ of the network function}
\label{sec:germ}

In this section, we show that the knowledge matrix of an input is the germ of
the network function at that input, and we draw the consequences: invariance
under every reparameterization that preserves the germ, completeness, and an
exact relation to gradient$\times$input attribution.
Table~\ref{tab:notation} collects the notation of this section and the next.

\input{sections/notation}

Throughout, a feedforward network is a pair $(W\!,f)$ in the sense of
\citet{armenta2021representation}: $W$ is a thin representation of the network
quiver $Q$, the graph whose vertices are the neurons --- $d$ input vertices,
the hidden vertices, bias vertices that feed the constant $1$, and $C$ output
vertices --- and whose arrows carry one weight each, and $f$ is the activation
function at the hidden vertices. The pair realizes the network function
$\Net:\R^d\to\R^C$ on $d$ input coordinates and $C$ output classes. Each input
$x$ induces a second thin representation $\kmapx$ of $Q$, the knowledge map of
\citet{leblanc2024hidden} and the representation $W^f_x$ of
\citet{armenta2021representation}: an arrow leaving input vertex $i$ carries
its weight multiplied by $x_i$, an arrow leaving a bias vertex carries its
weight, and an arrow leaving a hidden vertex $q$ carries its weight multiplied
by the activation-to-pre-activation quotient of $q$ at $x$. Theorem~6.4 of
\citet{armenta2021representation} states that, under their standing
assumption that no pre-activation vanishes (our $x\in X_{\mathrm{nz}}$ below),
the network function factorizes through it: $\Netx$ is the output of
$\kmapx$, read as a network with identity activations, on the all-ones
vector. The induced representation is linear, so its paths sum to a single
linear map from the input and bias vertices to the output vertices, and that
contraction \citep{armenta2022double} is the knowledge matrix $\KMwfx$
defined below, with one column per input vertex and the bias sources summed
into one column --- which is why it has exactly $d+1$ columns. The
construction is theirs for networks whose hidden vertices apply an activation
to a weighted sum; max-pooling vertices, which take a maximum instead and
which every network in our experiments contains, are outside the proof of
their Theorem~6.4 and, as they note, break the algebraic structure, and the
paragraph after \eqref{eq:km-def} says how we extend the contraction to
them. Write $z^{(\ell)}_q(x)$
for the pre-activation of hidden unit $q$ in layer $\ell$ at input $x$,
$h^{(\ell)}_q(x)=f\bigl(z^{(\ell)}_q(x)\bigr)$ for its activation, and
\begin{equation}
D^{(\ell)}(x)\;=\;\mathrm{diag}\!\left(\frac{f\bigl(z^{(\ell)}_q(x)\bigr)}{z^{(\ell)}_q(x)}\right)_{\!q},
\qquad\text{with the guard } 0/0\mapsto0,
\label{eq:D-def}
\end{equation}
for the diagonal matrix of that layer's \emph{activation-to-pre-activation
quotients}. This is the definition for an arbitrary activation $f$, not a
device for a special case: $D^{(\ell)}(x)$ is exactly the diagonal that the
induced representation $\kmapx$ carries on the arrows leaving layer $\ell$,
and Theorem~6.4 of \citet{armenta2021representation} gives the identity below
for any activation function. It is a $0/1$ \emph{mask}, recording the active
units, when $f$ is ReLU, since $\mathrm{ReLU}(z)/z=\mathbb 1[z>0]$ off $z=0$. The same
quotient is already familiar from the attribution literature:
\citet[Prop.~1]{ancona2018towards} show that $\varepsilon$-LRP is, in the
$\varepsilon\to0$ limit and with the bias included in the denominators, a
modified backward pass in which each unit's derivative is replaced by
$f(z)/z$, which is exactly the diagonal of \eqref{eq:D-def}, and they record the same
unboundedness caveat near $z=0$ that Remark~\ref{rem:xnz} records. The knowledge matrix
differs from $\varepsilon$-LRP in one structural respect: $\varepsilon$-LRP
carries no bias column, and the amount by which its per-class attributions fall
short of the logit is exactly the column $\cwf(x)$ below.
Write $X_{\mathrm{nz}}$ for the set of inputs at which no hidden
pre-activation vanishes. Define $\Jwf(x)\in\R^{C\times d}$ and
$\cwf(x)\in\R^{C}$ by path sums in $\kmapx$: the $(k,i)$ entry of $\Jwf(x)$
is the sum, over the directed paths from input vertex $i$ to output vertex
$k$, of the product of the weights along the path times the quotients
$f(z_q)/z_q$ of the hidden vertices the path passes through, and the $k$-th
entry of $\cwf(x)$, the accumulated bias, is the same sum over the paths from
the bias vertices to output vertex $k$. For a layered chain the paths
multiply out to $\Jwf(x)=W^{(L)}D^{(L-1)}(x)\cdots D^{(1)}(x)W^{(1)}$, the
product of the layers of $\kmapx$, and
$\cwf(x)=b^{(L)}+\sum_{\ell<L}W^{(L)}D^{(L-1)}(x)\cdots D^{(\ell)}(x)\,b^{(\ell)}$;
skip connections, concatenations and parallel branches, which make the
network quiver a directed acyclic graph rather than a chain, are covered by
the path sums, and every statement of this paper is about the path-sum
objects, the displayed products being their chain case. The knowledge
matrix, the contraction of $\kmapx$, is
\begin{equation}
\begin{gathered}
\KMwfx\;=\;\bigl[\,\Jwf(x)\,\mathrm{diag}(x)\;\big|\;\cwf(x)\,\bigr]
\;\in\;\R^{C\times(d+1)},\\[3pt]
\KMwfx\,\mathbf 1_{d+1}\;=\;\Netx
\quad\text{(exact on $X_{\mathrm{nz}}$, for \emph{every} activation $f$)} .
\end{gathered}
\label{eq:km-def}
\end{equation}

We impose no condition on $f$, and in particular none on $f(0)$; the one
hypothesis is $x\in X_{\mathrm{nz}}$, which no guard can remove, and its
exceptional set is empty exactly when $f(0)=0$ (Remark~\ref{rem:xnz},
Appendix~\ref{app:remarks-germ}).

Max-pooling enters through one more locally constant object. For a
max-pooling layer $\ell$ write $S^{(\ell)}(x)$ for the $0/1$ \emph{selection
matrix} that routes each pooling window to its stored argmax (the lowest flat
index on a tie, Remark~\ref{rem:pooling-ties},
Appendix~\ref{app:remarks-germ}). A pooling vertex carries the identity
activation, so its quotient is $1$, and in $\kmapx$ the arrows into it carry
the entries of $S^{(\ell)}(x)$: the layer enters the path sums, and the chain
product, as the linear map $S^{(\ell)}(x)$, and the row-sum identity holds
for it because the selected entry passes through unchanged. Off the finitely
many tie hyperplanes, on which two entries of one window have equal
pre-activations, $S^{(\ell)}$ is locally constant. For a network with
max-pooling, (LCS) at $x$ (Definition~\ref{def:lcs}) asks every $D^{(\ell)}$
\emph{and} every $S^{(\ell)}$ to be constant on a neighborhood of $x$, and
the regular set $X_{\mathrm{reg}}$ defined below is the set where this holds;
the tie hyperplanes are the extra walls that Lemma~\ref{lem:null} adds, so
``no pooling ties'' in Theorem~\ref{thm:weff} is what $X_{\mathrm{reg}}$ asks
of the pooling layers. This is the one point at which we extend the
construction of \citet{armenta2021representation}: for the networks of this
paper ``the contraction of $\kmapx$'' means the path sum with $S^{(\ell)}(x)$
inserted at the pooling layers, which coincides with their contraction on
the pooling-free part. The same layer can be re-encoded inside the (LCS)
class, since $\max(a,b)=b+\mathrm{ReLU}(a-b)$ writes a window maximum as
ReLU units with fixed $\pm1$ weights on an enlarged quiver, the tie wall
being the zero set of the pre-activation $a-b$; nothing in
Sections~\ref{sec:germ}--\ref{sec:geometry} depends on which encoding is
used, and we keep the selection matrix because it is what the implementation
stores.

We separate two things at the outset, since the rest of this section turns on
the distinction. Equation~\eqref{eq:km-def} is a \emph{definition} plus an
algebraic identity: $h_q=D_{qq}z_q$ holds by
construction, so the row sums reproduce the logits on $X_{\mathrm{nz}}$ for
every activation, with no genericity hypothesis beyond that and no appeal to
differentiability. What is special about the locally-constant-slope case
(Definition~\ref{def:lcs}) is not the definition but a \emph{theorem}: there,
and not in general otherwise (Remark~\ref{rem:lcs-strict}), the product
$\Jwf(x)$ coincides with the Jacobian $D\Net(x)$,
so the matrix records the germ (Theorem~\ref{thm:germ}). ``Piecewise linear''
is not the same hypothesis and does not suffice
(Remark~\ref{rem:lcs-strict}). For a smooth activation the quotient $f(z)/z$
is not $f'(z)$ --- at $z=0.105$ a $\tanh$ unit has quotient $0.9963$ against
derivative $0.9891$ --- and the germ reading is unavailable, though the row-sum
identity is untouched. Everything in Section~\ref{sec:geometry} that speaks of
regions, walls and crossings therefore belongs to the (LCS) case; the object
itself does not.

The letter $f$ does two different jobs in the literature, and only one of
them here. Throughout this paper $f$ denotes the \emph{activation} function and
$W$ the collection of weights, while $\Net$ is the \emph{network} function that
the pair realizes. This is why we write the knowledge
matrix as $\KMwf$, with local data $\Jwf$ and $\cwf$: all three are
built from the same $W$ and the same $f$, in the quiver-representation notation
of \citet{armenta2021representation}. The parameter collection is $W$
throughout; a transformed one is written $\tilde W$, the network it realizes
$\tilde\Psi=\Psi(\tilde W\!,f)$ and its matrix $\tilde{\M}=\M(\tilde W\!,f)$;
where a proof must tell a neuron permutation $\pi$ from a rescaling $\tau$ we
write $W^\pi$ and $W^\tau$ for the two. We stress that $f$ is
\emph{not} the network: writing the row-sum identity as $M(W{,}f)(x)\mathbf
1=f(x)$, as is sometimes done, uses $f$ in two incompatible senses.

\begin{definition}[Locally constant slope diagonal]\label{def:lcs}
A network has a \emph{locally constant slope diagonal} at $x$ if every slope
diagonal $D^{(\ell)}$ of \eqref{eq:D-def} and, for a network with max-pooling,
every selection matrix $S^{(\ell)}$ is constant on a neighborhood of $x$. We
abbreviate the condition (LCS).
\end{definition}

\begin{proposition}[What (LCS) is, and what it is not]\label{prop:lcs}
Let $\Net$ be the feedforward network fixed above, with activation $f$ applied
neuron-wise.
\emph{(i)} If (LCS) holds at $x$ and $x\in X_{\mathrm{nz}}$, then $\Net$ is
affine on a neighborhood of $x$, with $\Netxp=\Jwf(x)\,x'+\cwf(x)$ there.
\emph{(ii)} Let $f$ be continuous and let $E\subset\R$ be discrete, a set
with no accumulation point in $\R$. If $z\mapsto f(z)/z$ is locally constant
on $\R\setminus(E\cup\{0\})$, then
\[
f(z)\;=\;a_{+}\max(z,0)+a_{-}\min(z,0)
\qquad\text{for constants }a_{\pm};
\]
in particular $f(0)=0$, and one may take $E\subseteq\{0\}$. Conversely every
$f$ of that form has $f(z)/z$ locally constant off $\{0\}$, so a network built
from it satisfies (LCS) off the finitely many hyperplanes of
Lemma~\ref{lem:null}.
\emph{(iii)} For $f:\R\to\R$ the quotient $f(z)/z$ is locally constant on
$\R\setminus\{0\}$ if and only if $f$ is positively homogeneous of degree one
on each open half-line, $f(\lambda z)=\lambda f(z)$ for $\lambda>0$ and
$z\neq0$; for continuous $f$ this is again the family of (ii).
\end{proposition}

The second hypothesis of part~(i) is automatic when $f(0)=0$. Part~(ii) is a
statement about $f$ alone, and the family it exhibits is the two-parameter
leaky-ReLU family, which contains $\mathrm{ReLU}$ $(a_-{=}0)$, $|\cdot|$
$(a_-{=}{-}1)$ and the identity $(a_+{=}a_-{=}1)$; $f(0)=0$ follows rather
than being assumed. Parts~(ii) and~(iii) restate the classical classification
of positively homogeneous functions on the line, on which we claim no novelty
(the attribution note is in Appendix~\ref{app:remarks-germ}).
Part~(ii) needs $E$ discrete rather than merely Lebesgue-null --- a
Cantor-function counterexample shows the weakening is false --- and the bridge
from a network satisfying (LCS) off finitely many hyperplanes to an activation
in this family, together with the attribution of (ii)--(iii) to the classical
classification of positively homogeneous functions, is in
Appendix~\ref{app:remarks-germ}. (LCS) is strictly stronger than piecewise
linearity: hard-tanh is piecewise linear, satisfies (LCS) nowhere at $|z|>1$,
and the germ identity fails for it by $7.8$ on the network of
Appendix~\ref{app:remarks-germ} (Remark~\ref{rem:lcs-strict}).

(LCS) is the hypothesis that the germ results actually use, and stating
it on $D$ rather than on $f$ is what makes them sharp. Piecewise linearity of
the activation is neither the condition nor a synonym for it: it delivers a
germ, since the network is still piecewise affine, but it does not deliver
$\Jwf(x)=D\Net(x)$, which is what makes the knowledge matrix \emph{be} that
germ (Remark~\ref{rem:lcs-strict}).

Unless a statement says otherwise, we work throughout this section and
Section~\ref{sec:geometry} with networks that satisfy (LCS), and for a
\emph{network} the phrase means three things: the activation $f$ is
continuous, (LCS) holds at every input outside a finite union of affine
hyperplanes, and some first-layer unit has a nonzero weight row. By the bridge
of Appendix~\ref{app:remarks-germ}, $f$ then lies in the leaky-ReLU family of
Proposition~\ref{prop:lcs}(ii), so $f(0)=0$, the exceptional set of
Remark~\ref{rem:xnz} is empty, $X_{\mathrm{nz}}=\R^d$, and
Lemma~\ref{lem:null} applies. Every
architecture in our experiments does: ReLU throughout, with affine layers,
batch normalization in evaluation mode and average pooling, and max-pooling,
which enters through the selection matrices $S^{(\ell)}(x)$ above and whose
tie hyperplanes are the extra walls of Lemma~\ref{lem:null}. We keep ``PL'' as
descriptive shorthand for that setting; the load-bearing hypothesis is (LCS).

We first fix the object that the knowledge matrix records. Write
$X_{\mathrm{reg}}\subseteq\R^d$ for the \emph{regular set}: the inputs at which
(LCS) holds, i.e.\ at which every slope diagonal $D^{(\ell)}$ of
\eqref{eq:D-def} and every selection matrix $S^{(\ell)}$ is constant on a
neighborhood. This is the only definition of
$X_{\mathrm{reg}}$ used in the paper. It is a condition on the slope diagonal,
not on the $0/1$ activation pattern, and the two differ for
activations that are piecewise linear without satisfying (LCS)
(Remark~\ref{rem:lcs-strict}); under the hypothesis of
Proposition~\ref{prop:lcs}(ii) they agree. Lemma~\ref{lem:null} then puts the
complement of $X_{\mathrm{reg}}$ inside a finite union of affine hyperplanes, so
$X_{\mathrm{reg}}$ is open, dense and of full Lebesgue measure; on it the
network coincides with a single affine map on a neighborhood of each point
(Proposition~\ref{prop:lcs}(i)).

\begin{definition}[Germ of a network satisfying (LCS)]\label{def:germ}
Let $\Net:\R^d\to\R^C$ satisfy (LCS) and let $x\in X_{\mathrm{reg}}$. Two
functions agreeing on \emph{some} neighborhood of $x$
define the same \emph{germ} at $x$; the germ of $\Net$ at $x$ is the equivalence
class of $\Net$ under this relation. By Proposition~\ref{prop:lcs}(i) the
network is affine on a neighborhood of $x$, so this class is represented by the
local affine pair $\bigl(\Jwf(x),\,\cwf(x)\bigr)$ with
$\Netxp=\Jwf(x)\,x'+\cwf(x)$ for all $x'$ near $x$.
\end{definition}

Throughout, ``germ'' means this local affine pair. It depends on $W$ only
through the function $\Net$ realizes locally, not through how $W$ encodes it.
That this pair is also the first-order Taylor data
$\bigl(D\Net(x),\,\Netx\bigr)$ of the realized function is
Theorem~\ref{thm:germ} and not part of the definition. The distinction is the
whole point of the hypothesis: piecewise linearity of $f$ delivers the germ, so
Definition~\ref{def:germ} would survive it, but it does not deliver
$\Jwf(x)=D\Net(x)$ (Remark~\ref{rem:lcs-strict}).

\begin{theorem}[Germ identity, under (LCS)]\label{thm:germ}
Let the network satisfy (LCS) (Definition~\ref{def:lcs}). For every
$x\in X_{\mathrm{reg}}$, $\Net$ is affine on a neighborhood of $x$ and
\[
\Jwf(x)=D\Net(x),\qquad \cwf(x)=\Netx-D\Net(x)\,x,
\]
\[
\KMwfx=\bigl[\,D\Net(x)\,\mathrm{diag}(x)\;\big|\;\Netx-D\Net(x)\,x\,\bigr].
\]
In particular $\KMwfx$ depends on $W$ only through the germ of the
realized function at $x$. For $x\notin X_{\mathrm{reg}}$ the identity
$\KMwfx\mathbf 1=\Netx$ still holds, since a network satisfying (LCS) has
$f(0)=0$ and hence $X_{\mathrm{nz}}=\R^d$.
\end{theorem}

At a point outside $X_{\mathrm{reg}}$ the guard $0/0\mapsto0$ of
\eqref{eq:D-def} assigns slope $0$ to a unit whose pre-activation vanishes
there. The per-region affine operator the theorem exhibits is the object of the
spline view of deep networks \citep{balestriero2018spline}; the two readings of
the boundary case (the row-sum half holds for every activation, the slope-$0$
germ-selection half is ReLU-specific), and why \citet[Prop.~1]{srinivas2019full},
which concerns saliency maps valued in $\R^{d}$, does not bear on the
$C\times(d{+}1)$ matrix, are in Appendix~\ref{app:remarks-germ}. Column $i$ of $\M(x)$ is $J_{:,i}\,x_i$ and
vanishes identically on $\{x_i=0\}$ --- a Lebesgue-null set of positive
probability under raw image data --- so germ \emph{recovery}
(Theorem~\ref{thm:complete}(i)) needs $x_i\neq0$ while invariance
(Theorem~\ref{thm:maxinv}) does not (Remark~\ref{rem:zero-pixels},
Appendix~\ref{app:remarks-germ}).

\subsection{Invariance: the largest possible group}
\label{sec:invariance}

The knowledge matrix is built from the quiver-representation approach to
neural networks \citep{armenta2021representation,armenta2022double} in two
steps, each of which respects relabelings of the hidden vertices and changes
of basis at them: the network $(W\!,f)$ induces the representation $\kmapx$ on
the input $x$, and the matrix is its contraction. This fixes the invariance
picture before any analysis. We keep two sources of invariance apart, because
they need entirely different machinery.

\emph{Quiver isomorphisms} act on the network $(W\!,f)$ through two kinds of
generator: relabelings of the hidden vertices of a layer, which are
automorphisms of $Q$ acting on $W$ by pull-back, and changes of basis at the
hidden vertices by nonzero per-neuron factors, which are isomorphisms of thin
representations of the fixed quiver $Q$ fixing the framed input, bias and
output vertices --- the isomorphisms of neural networks of
\citet{armenta2021representation}. Under either, the representations induced
on $x$ by the two networks are related by the same relabeling or change of
basis, and thin representations related by an isomorphism fixing the framed
vertices contract to the same matrix, so the knowledge matrix is invariant by
construction. This is one telescoping lemma, stated next. Theorem~4.13 of
\citet{armenta2021representation} covers the changes of basis, of which
positive rescalings are the ReLU case of Theorems~4.1--4.2 of
\citet{leblanc2024hidden}; invariance under relabelings is relabeling
invariance, covered by the same telescoping
(Appendix~\ref{app:proof-quiver-inv}). A relabeling is not an isomorphism of
representations of $Q$ --- the pulled-back representation is in general not
isomorphic to $W$ as a representation of $Q$ --- and we do \emph{not}
attribute permutation invariance to those theorems: they are about changes of
basis, and permutation invariance is the automorphism case of the same
construction, not a corollary of them.

\emph{Function-preserving changes that are not quiver isomorphisms} --- dead-unit
insertion, neuron splitting, and above all moving to a \emph{different
architecture} that realizes the same function near $x$ --- lie in the full
stabilizer of the germ but outside the quiver-isomorphism group. Invariance
under these is exactly what is \emph{not} automatic, and it is the genuinely
novel content of the germ identity (Theorem~\ref{thm:maxinv}): the matrix
depends on $W$ only through the germ of Definition~\ref{def:germ}, so any
two encodings of the same local function --- however different their graphs ---
produce the same matrix.

The invariance is with respect to weight-space conventions at fixed input
coordinates. An input shift $x\mapsto x+m$ absorbed into
$b^{(1)}\mapsto b^{(1)}-W^{(1)}m$ --- the same function of the raw pixels ---
gives $[\,J\,\mathrm{diag}(x+m)\mid c-Jm\,]\neq[\,J\,\mathrm{diag}(x)\mid c\,]$:
the mean-shift non-invariance that \citet{kindermans2019unreliability} proved
for gradient$\times$input, integrated gradients and DeepLIFT, which is why the
$x$ in every statement is the preprocessed input the network sees
(Remark~\ref{rem:zero-pixels}, Appendix~\ref{app:remarks-germ}).

\begin{lemma}[Quiver-isomorphism invariance]\label{lem:quiver-inv}
Let $(\tilde W\!,\tilde f)$ be obtained from $(W\!,f)$ by a quiver
isomorphism: a relabeling of the hidden vertices of a layer, a change of
basis at the hidden vertices by nonzero factors, or any composite thereof,
with the per-neuron activations carried with their units --- a unit rescaled
by $\tau\neq0$ carrying $\tilde f_q=f_\tau$, $f_\tau(z)=\tau\,f(z/\tau)$.
Then $\M(\tilde W\!,\tilde f)(x)=\KMwfx$ for every $x\in\R^d$.
\end{lemma}

For ReLU and $\tau>0$ one has $f_\tau=f$, so $\tilde f=f$ and the rescaled
network is again a ReLU network; for $\tau<0$ the unit carries
$f_\tau(z)=\min(z,0)$ and the network has left the ReLU class, which is why
Theorem~\ref{thm:complete}(iii) needs $\tau>0$ while the lemma does not. The
scope of the lemma is per-vertex scalar activations: a channel permutation
and a positive channel rescaling commute with spatial max-pooling and its
lowest-index tie-break, so the lemma covers the pooling networks of our
experiments for those transforms, whereas a negative $\tau$ upstream of a
max-pooling window carries the maximum to a minimum
\citep[Remark~4.14]{armenta2021representation} and is not covered there.
The conclusion holds at every input and not only on $X_{\mathrm{nz}}$: the
guard $0/0\mapsto0$ of \eqref{eq:D-def} is applied entrywise on both sides,
and on $X_{\mathrm{nz}}$, which is the same set for both networks since
$\tau z=0$ if and only if $z=0$, no guard is invoked.
The lemma rests on one substitution: in place of the derivative $f'(z_q)$, each
hidden unit's slope is its \emph{secant} (chord) $a_q=f(z_q)/z_q$, with
$a_q=0$ when $z_q=0$, so that $h_q=a_qz_q$ holds exactly and the diagonal of
chords is precisely $D^{(\ell)}(x)$ of \eqref{eq:D-def}. The chord discussion
--- its identity with the modified gradient of $\varepsilon$-LRP, the two
consequences used repeatedly (exactness at every input when $f(0)=0$, and the
smooth-activation case in which the chord absorbs FullGrad's implicit-bias
term), and why the rescaling half holds for every $\tau\neq0$ while
Theorem~\ref{thm:complete}(iii) needs $\tau>0$ --- is in
Appendix~\ref{app:remarks-germ}.

\begin{theorem}[Maximal invariance: across architectures]\label{thm:maxinv}
Let $N_1,N_2$ be networks satisfying (LCS) that realize the same function on a
neighborhood of $x\in X_{\mathrm{reg}}(N_1)\cap X_{\mathrm{reg}}(N_2)$. Then
$\M_1(x)=\M_2(x)$: the matrix is a function of the germ
(Definition~\ref{def:germ}), so any two encodings --- architecture changes
included --- that realize the same germ at $x$ and are regular there give the
same matrix. Consequently, for a fixed architecture, every transformation
$W\mapsto\tilde W$ of its parameters that preserves the germ at $x$ and keeps
$x\in X_{\mathrm{reg}}$ preserves $\M(x)$: $\M(\cdot)(x)$ is invariant under
the stabilizer of the germ at $x$ --- the \emph{germ stabilizer} --- in any
group acting on the parameters, among encodings at which $x$ remains regular.
\end{theorem}

We place no restriction on $N_1$ and $N_2$ beyond (LCS): they may
differ in architecture, width, depth and parameters, and that
cross-architecture case is the proper payload of the theorem --- it is
\emph{not} a quiver isomorphism (Lemma~\ref{lem:quiver-inv}) and could not
follow from one. What the theorem proves is, across architectures, that
$\M(x)$ is a function of the germ, and, for a fixed architecture and any group
acting on its parameters, one inclusion,
$\mathrm{Stab}(\text{germ at }x)\subseteq\mathrm{Stab}(\M(x))$ among encodings
regular at $x$: every germ-preserving change that keeps $x$ regular preserves
the matrix. The reverse inclusion --- which is
what would make ``maximal'' literal rather than rhetorical --- does not follow
from this theorem and is not automatic; it is Corollary~\ref{cor:stab-eq}, and
it needs the extra hypothesis that every $x_i\neq0$.

\begin{corollary}[Permutation invariance, at every input]\label{prop:perm}
Let the activations be elementwise, so that a hidden-layer permutation matrix
$P$ commutes with the nonlinearity applied neuron-wise. Then for hidden-layer
permutations (weights conjugated, neuron-wise activations carried with the
units), $\KMarg{W^\pi}{x}=\KMwfx$ for \emph{every} $x\in\R^d$, region
boundaries included.
\end{corollary}

This is the permutation case of Lemma~\ref{lem:quiver-inv}, so what the
corollary contributes is its \emph{scope}. The conclusion holds at region
boundaries because the slope diagonal \eqref{eq:D-def} is the chord $f(z)/z$
with the guard $0/0\mapsto0$ rather than $f'(z)$ --- the secant construction
described after that lemma and in Appendix~\ref{app:remarks-germ} --- which is
what makes it hold there as well as on the regular set. The telescoping needs
elementwise (neuron-wise equivariant) activations;
a non-equivariant \emph{vector} activation --- a softmax or a normalization
taken across a layer --- is excluded, because $P$ then fails to commute with
it. Within that scope the conclusion holds at every input, with no genericity
hypothesis at all. When $f(0)\neq0$ the \emph{row-sum} identity
\eqref{eq:km-def} needs $x\in X_{\mathrm{nz}}$, but that requirement falls on
both networks equally and is no limit on the scope just described: at an
exactly-zero pre-activation $\KMarg{W^\pi}{x}$ and $\KMwfx$ fall short of the
logits by the \emph{same} amount, so the shortfall is a property of the
definition rather than of the permutation, and $P$ does commute with the
constant offset $f(0)\mathbf 1$. The proof records the measurement.

So permutations and positive rescalings telescope identically; neither is the
``gap'' --- both are quiver isomorphisms. \emph{Function-preserving surgery}
(dead-unit insertion, neuron splitting, architecture change) is the part that
genuinely needs the germ identity, and is covered at generic $x$ by
Theorem~\ref{thm:maxinv}.

Genericity is not removable: the
identity map and $g(x)=\mathrm{ReLU}(x-1)-\mathrm{ReLU}(1-x)+1\equiv x$ have
$\M_{\mathrm{id}}(1)=[\,1\mid 0\,]\neq[\,0\mid 1\,]=\M_g(1)$ at the
breakpoint $x=1$ (with equal row sums). The point $x=1$ is the common
breakpoint of the two ReLU pieces out of which $g$ is built, since each of
$\mathrm{ReLU}(x-1)$ and $\mathrm{ReLU}(1-x)$ is non-differentiable exactly
there. So $x=1$ lies on a wall of $g$'s region decomposition,
$1\notin X_{\mathrm{reg}}(g)$, even though the function $g$ realizes --- the
identity --- is smooth at $1$, and $1$ is an interior point of the
identity network's single region. That is the entire content of the example:
the germ of Definition~\ref{def:germ} is defined on the regular set, and at a
wall the mask convention $\mathbb 1[z>0]$ picks a one-sided germ; two
encodings of one function can sit on walls of \emph{different} decompositions
and pick different sides. Hence the genericity hypothesis in
Theorem~\ref{thm:maxinv}. The disagreement is confined to a Lebesgue-null set,
and even there it does not touch the row-sum identity: $[\,1\mid0\,]$ and
$[\,0\mid1\,]$ both sum to $g(1)=1$.

Conversely, the stabilizer is strictly larger than the isomorphism
group (Proposition~\ref{prop:contraction}), so $\M$ cannot determine a network
up to isomorphism --- a feature for robustness claims, a boundary for inverse
problems.

\subsection{Completeness: the title claim, as a theorem}
\label{sec:completeness}

This subsection proves the claim in the title.

\begin{theorem}[Hidden activations are not enough]\label{thm:complete}
Let the network satisfy (LCS) (Definition~\ref{def:lcs}).
\emph{(i)} For $x\in X_{\mathrm{reg}}$ with all $x_i\neq0$, $\M(x)$ determines
the germ $(J,c)$ of $\Net$ at $x$.
\emph{(ii)} The field $x\mapsto\M(x)$ determines $\Net$ pointwise on whatever
set it is given ($\Net=\M\mathbf 1$), hence globally when that set has full
measure.
\emph{(iii)} Hidden activations are gauge-covariant: under a positive
per-neuron rescaling, the activation $h_v$ of hidden unit $v$ moves,
$h_v\mapsto\tau_v h_v$, while the function is unchanged
--- any activation statistic that is not gauge-invariant is not a function of
$(\Net,x)$. \emph{(iv)} Hidden activations are germ-incomplete. Let the network have at
least one hidden layer, write $h(x')\in\R^{n}$ for the input to its output
layer, so that $\Netxp=W^{(L)}h(x')+b^{(L)}$, and let $x\in X_{\mathrm{reg}}$
have every hidden pre-activation nonzero, with $G=\partial h/\partial x\,(x)$.
Suppose either \emph{(a)} $D\Net(x)\neq0$ and the first layer is free, its
weight rows and biases being independent parameters, or \emph{(b)} some column
of $G$ is not a multiple of $h(x)$. Then there are arbitrarily small parameter
changes that leave \emph{every} hidden pre- and post-activation at $x$ and the
output $\Netx$ unchanged while changing the germ --- hence changing $\Net$ on
every neighborhood of $x$ --- and changing $\M(x)$: under \emph{(a)} in every
column $i$ with $x_i\neq0$, and under \emph{(b)} in every column $i$ with
$x_i\neq0$ and $G_{:,i}\notin\mathrm{span}\,h(x)$.
\end{theorem}

Part~(iv) has two proofs, and they cover different architectures. The first
perturbs one first-layer weight row and its bias, $w_q\mapsto w_q+t\mathbf 1_d$
and $b_q\mapsto b_q-t\mathbf 1_d^{\top}x$, which holds that unit's
pre-activation at $x$ fixed and moves every column with $x_i\neq0$; it needs
the row to be a free parameter, which is the case for a multilayer perceptron
and not for a convolutional first layer, whose rows are shifted copies of one
filter and cannot be changed one at a time. The second perturbs the output
layer alone, $W^{(L)}\mapsto W^{(L)}+t\,e_ca^{\top}$ with $a\perp h(x)$,
touches no hidden unit at any input, and is the proof that applies to
ResNet-152, DenseNet-121 and GoogLeNet, whose output layer is a free linear
map on the pooled features. Its hypothesis (b) fails only when the range of
$G$ lies in $\mathrm{span}\,h(x)$, so it holds whenever $\mathrm{rank}\,G\ge2$,
and it needs $n\ge2$ when $h(x)\neq0$.

Part~(ii) is sharp: on a mere open set the field determines $\Net$ on that
set and, in general, on no larger one --- not even on the closure of the
region that the open set sits inside in any particular candidate network. A
one-dimensional pair shows what ``sharp'' attaches to. Take
$\Psi_1=\mathrm{ReLU}(x-1)$ and $\Psi_2=\mathrm{ReLU}(x)$ on $\R$. Both have
the constant field $\equiv[\,0\mid0\,]$ on the open set $(-1,0)$, and they
differ at $x=0.75$, which lies in the closed
region of $\Psi_1$ that contains $(-1,0)$, namely $(-\infty,1]$. So the field
on an open set does not determine $\Net$ even on the closure of the region it
sits inside. For a \emph{fixed, known} network the germ at $x$ does extend
$\Net$ to $\overline{S(x)}$ (proof of Theorem~\ref{thm:germ}), but that
extension uses the region decomposition, which the field alone does not
reveal --- two networks with the same field on an open set can have different
regions, and here they do. This is why (ii) claims determination on the given
set and nothing beyond it.

Germ recovery, part~(i), holds for pointwise objects of the form
$[\,J(x)\,\mathrm{diag}(x)\mid c(x)\,]$, with a bias column, and fails for path
or neighborhood averages: integrated gradients
\citep{sundararajan2017axiomatic} and SmoothGrad \citep{smilkov2017smoothgrad}
average the Jacobian along a path or over a neighborhood and do not determine
the endpoint germ, which is the sense in which germ recovery is a different
property from summation-to-output.

\begin{corollary}[The two stabilizers coincide]\label{cor:stab-eq}
Let $N_1,N_2$ satisfy (LCS) and let $x\in X_{\mathrm{reg}}(N_1)\cap
X_{\mathrm{reg}}(N_2)$ have every coordinate $x_i\neq0$. Then $\M_1(x)=\M_2(x)$
if and only if $N_1$ and $N_2$ realize the same germ at $x$. Consequently,
for a fixed architecture, the stabilizer of $\M(x)$ in any group acting on
its parameters equals the stabilizer of the germ, among encodings regular at
$x$; across architectures, $\M(x)$ and the germ are functions of each other.
\end{corollary}

For a fixed architecture and among encodings regular at $x$, Theorem~\ref{thm:maxinv} supplies
$\mathrm{Stab}(\text{germ})\subseteq\mathrm{Stab}(\M(x))$ and
Theorem~\ref{thm:complete}(i) the reverse, since $\M(x)$ then determines
$(J,c)$. With both inclusions in hand, ``maximal'' in
Theorem~\ref{thm:maxinv} is an equality rather than a figure of speech. The hypothesis
$x_i\neq0$ is used for the reverse inclusion only, and is not removable
(Remark~\ref{rem:zero-pixels}).

Parts (iii) and (iv) together are the precise form of the title: activations
are neither invariant (two networks, same function, different activations
everywhere) nor complete (two networks, identical activations at $x$,
different functions near $x$). $\M(x)$ alone is both --- a complete,
gauge-invariant record of the computation's germ, invariant by
Theorem~\ref{thm:maxinv} and complete by part~(i), with the completeness
hypothesis ($x_i\neq0$) stated by Remark~\ref{rem:zero-pixels}. The slope
diagonal $D(x)$ is \emph{not} part of that record: a hidden-layer permutation
reindexes it exactly as it reindexes the activations.

\subsection{Relation to attribution methods}
\label{sec:attribution}

This subsection places the knowledge matrix beside the attribution methods it
coincides with.

\begin{theorem}[Knowledge matrix $=$ gradient$\times$input $\oplus$ bias attribution]\label{thm:weff}
For a network satisfying (LCS) at $x\in X_{\mathrm{reg}}$ with no pooling
ties, the first $d$ columns of $\M(x)$ are the per-class gradient$\times$input
attributions, $\M_{:,i}(x)=(\partial\Net/\partial x_i)(x)\,x_i$, and the last
column is the class-resolved aggregate of the bias attributions,
$\cwf(x)=\sum_\ell(\partial\Net/\partial b^{(\ell)})(x)\,b^{(\ell)}$, where
$W^{(\ell)},b^{(\ell)}$ denote the composed affine map producing the
pre-activation $z^{(\ell)}$, normalization layers in evaluation mode folded
into it. Hence
\[
\M(x)=\bigl[\,\nabla_x\Net(x)\,\mathrm{diag}(x)\;\big|\;\nabla_b\Net(x)\cdot b\,\bigr],
\qquad
\Netx=\nabla_x\Net(x)\cdot x+\nabla_b\Net(x)\cdot b .
\]
\end{theorem}

The second identity is Proposition~3 of \citet{srinivas2019full}, the FullGrad
bias identity, read class by class, so $\M(x)$ is per-class
gradient$\times$input augmented by an exact aggregate bias column; the
per-region affine operator it arranges is the object of the spline view of
deep networks \citep[\S5.2]{balestriero2018spline}, and the same matrix is
computed by three equivalent routes (masked product, the probe construction,
and autograd), detailed in Appendix~\ref{app:engineering}.
Theorem~\ref{thm:weff} also has a consequence for training that we record but
do not test: penalizing the Frobenius size of the knowledge matrix, bias column
aside, is the input-gradient$\times$input penalty
$\|\nabla_x\Net\odot x\|_F^2$, a class-resolved, input-scaled member of the
double-backpropagation family (Remark~\ref{rem:doublebackprop},
Appendix~\ref{app:remarks-germ}).

The ``no pooling ties'' hypothesis (Remark~\ref{rem:pooling-ties}) is not
new: local constancy of the selection matrices $S^{(\ell)}$ is part of the
definition of $X_{\mathrm{reg}}$, the tie hyperplanes are the extra walls
that Lemma~\ref{lem:null} adds for max-pooling, so it is what
$x\in X_{\mathrm{reg}}$ already asks of the pooling layers, and
Remark~\ref{rem:pooling-ties} (Appendix~\ref{app:remarks-germ}) records the
tie-breaking convention the implementation inherits. For a batch-normalized
layer in evaluation mode the folded offset $b^{(\ell)}$ is
$\beta-\gamma\mu/\sqrt{\sigma^{2}+\epsilon}$, plus any convolution bias scaled
by $\gamma/\sqrt{\sigma^{2}+\epsilon}$, the running statistics $\mu,\sigma^2$
being ordinary weights (Appendix~\ref{app:teleport-why}); $\nabla_\beta\Net\cdot\beta$
alone is not the bias column.

Under (LCS) at almost every
input, the knowledge matrix coincides, column for column, with per-class
gradient$\times$input \citep{shrikumar2016not,ancona2018towards} augmented by a
per-class aggregate bias attribution in the sense of FullGrad
\citep{srinivas2019full}; the per-region affine operator itself is the object of
the spline view of deep networks \citep{balestriero2018spline} and of Jacobian
analyzes of bias-free denoisers \citep{mohan2020robust}. Everything below rests
on this premise: \emph{every function-level property of
$\M$ under (LCS) --- invariance, completeness, fixed shape, exact row sums --- is
shared by per-class gradient$\times$input plus the aggregate bias column, not
specific to the knowledge matrix.} The equivalence grounds the construction in
established attribution theory, supplies its exact completeness identity, and is
what makes the vector--Jacobian-product computation of Appendix~\ref{app:engineering}
possible. What the knowledge matrix adds over the same data is not a new
pointwise invariant but a \emph{canonical arrangement}: one fixed-shape matrix
per input with exact row-sum accounting (the basis of
Section~\ref{sec:geometry}), comparable across architectures without neuron
alignment, with the induced representation $\kmapx$ above it, of which the
matrix is the contraction
\citep{armenta2021representation,armenta2022double}, and, for non-PL
activations, a secant form that needs no extra term to close --- the
definition \eqref{eq:D-def} itself, not an extension of it.

That secant form is the construction described after
Lemma~\ref{lem:quiver-inv} and in Appendix~\ref{app:remarks-germ}: each unit's
slope is its chord $f(z)/z$ rather
than its derivative, which keeps $\M\mathbf 1=\Net$ exact for smooth
activations. Gradient
attributions are \emph{not} without an exact identity there.
\citet[\S4]{srinivas2019full} extend FullGrad to
arbitrary non-linearities by appending an implicit bias
$b_f=f(z)-f'(z)z$ to the bias vector, and the resulting
decomposition is exact (verified on a three-layer sigmoid network,
Appendix~\ref{app:remarks-germ}). What separates the two routes is the cost
of that exactness.
FullGrad's
route carries one additional, separately accounted implicit-bias term per unit;
the chord absorbs the same discrepancy into the slope it already carries, so
the arrangement --- $C\times(d{+}1)$, one column per input coordinate plus one
bias column --- is unchanged as the activation changes.

Computation of $\M$ --- the probe construction, the $C$
vector--Jacobian-product route, the measured speed-up, and numerical validation
--- is deferred to Appendix~\ref{app:engineering}.

\subsection{Input symmetries, and what the contraction forgets}
\label{sec:equivariance}

Sections~\ref{sec:invariance}--\ref{sec:attribution} moved $W$ and held $x$
fixed; geometric deep learning \citep{bronstein2021geometric} does the opposite,
fixing $W$ and letting a group act on the input. For a $G$-equivariant network,
$\Net(\pi(g)x)=\rho(g)\,\Netx$ with $\pi$ and $\rho$ linear,
Proposition~\ref{prop:equivar} (Appendix~\ref{app:equivariance}) states that at
inputs with both $x$ and $\pi(g)x$ regular the germ intertwines the two
representations, the bias column is equivariant for the output action alone, and
the matrix is an equivariant $C\times(d{+}1)$ tensor,
$\KMarg{W}{\pi(g)x}=\rho(g)\,\KMwfx\,(\pi(g)\oplus1)^{-1}$, when $\pi(g)$ is a
permutation matrix --- with a converse, and with $\Weff=[\,J\mid c\,]$ obeying
the same law for every invertible linear $\pi(g)$. Its remarks --- the pair
hypothesis, the geometric models covered, the two meanings of ``gauge'', why the
correctness checks permute channels, the quiver lift as scaffolding --- are in
Appendix~\ref{app:remarks}.

The knowledge matrix is the contraction of the representation $\kmapx$ that
the network induces on $x$ \citep{armenta2021representation,armenta2022double},
and what the contraction forgets is exactly what the following proposition
measures. The \emph{path value} of a directed path from an input or bias
vertex to an output vertex is the product of the weights along it, and the
\emph{path-value multiset} of a network is the multiset of these values over
all such paths; an isomorphism of neural networks preserves it, since the
changes of basis at the hidden vertices cancel along every path and a
relabeling of the hidden vertices permutes the paths.

\begin{proposition}[The per-input contraction is exactly what is lost]
\label{prop:contraction}
\emph{(i)} There exist ReLU networks $A$ and $B$ of the same architecture that
realize the same function, are not isomorphic, and have distinct path-value
multisets, for which the induced representations $\phi(A,f)(x)$ and
$\phi(B,f)(x)$ are non-isomorphic for every $x$ in a nonempty open set, while
their contractions satisfy $\M_A(x)=\M_B(x)$ for every $x\in\R^{d}$.
\emph{(ii)} On the identifiable architecture class of
\citet{phuong2020functional}, architectures of non-increasing widths, and for
general networks of the same architecture in their sense, equality of the
knowledge-matrix fields is equivalent to equivalence under permutation and
positive rescaling.
\end{proposition}

The witness for (i) is the $1$--$2$--$1$ pair of Appendix~\ref{app:proofs}.
For it the induced representations are non-isomorphic at every $x>0$, where
both hidden units are active and the path-value multisets $\{x,2x\}$ and
$\{1.5x,1.5x\}$ of the two induced representations differ, and they coincide
for $x\le0$, where both units are inactive and every arrow leaving them
carries $0$: the open set of the statement is $(0,\infty)$, not all of $\R$.
Functional identity does not on its own give $\M_A(x)=\M_B(x)$ at every $x$:
Theorem~\ref{thm:maxinv} delivers equality only on
$X_{\mathrm{reg}}(A)\cap X_{\mathrm{reg}}(B)$, and the identity-versus-$g$
pair after that theorem is a same-function pair that disagrees off the
intersection. For the $1$--$2$--$1$ witness the two networks share their single
wall, so equality holds at every $x$, and we verify $\M_A(x)=\M_B(x)$ directly
on a $401$-point grid through the breakpoint. That witness has increasing
widths and so lies outside the class of (ii). For general architectures the
fiber question is open and connects to functional dimension
\citep{grigsby2022functional} and the neuromanifold fibers of
\citet{flinth2026fibers}.

%% file: sections/notation.tex
\begin{table}[t]
\centering\small
\caption{Notation of Sections~\ref{sec:germ}--\ref{sec:geometry}, with the place
each symbol is introduced.}
\label{tab:notation}
\smallskip
\begin{tabular}{@{}l>{\raggedright\arraybackslash}p{7.2cm}l@{}}\toprule
Symbol & Meaning & Introduced \\ \midrule
$Q$ & network quiver: vertices are the neurons, arrows carry the weights & Section~\ref{sec:germ} \\
$(W\!,f)$ & the network: a thin representation $W$ of $Q$ (one weight per arrow) and the activation $f$ & Section~\ref{sec:germ} \\
$\Net:\R^d\to\R^C$ & network function realized by $(W,f)$; $d$ input coordinates, $C$ classes & Section~\ref{sec:germ} \\
$\kmapx$ & the thin representation of $Q$ induced by the input $x$; its contraction is $\KMwfx$ & Section~\ref{sec:germ} \\
$z^{(\ell)}_q(x)$, $h^{(\ell)}_q(x)$ & pre-activation and activation of hidden unit $q$ in layer $\ell$ & Section~\ref{sec:germ} \\
$D^{(\ell)}(x)$ & slope diagonal of layer $\ell$, $\mathrm{diag}\bigl(f(z_q)/z_q\bigr)_q$ & Eq.~\eqref{eq:D-def} \\
$\Jwf(x)$, $\cwf(x)$ & path sums of $\kmapx$ from the input and from the bias vertices; for a chain, the slope product $W^{(L)}D^{(L-1)}\cdots D^{(1)}W^{(1)}$ and the accumulated bias & before Eq.~\eqref{eq:km-def} \\
$S^{(\ell)}(x)$ & $0/1$ selection matrix of max-pooling layer $\ell$ (stored argmax) & after Eq.~\eqref{eq:km-def} \\
$\KMwfx$ & knowledge matrix $[\,\Jwf(x)\,\mathrm{diag}(x)\mid\cwf(x)\,]\in\R^{C\times(d+1)}$ & Eq.~\eqref{eq:km-def} \\
$X_{\mathrm{nz}}$ & inputs at which no hidden pre-activation vanishes & before Eq.~\eqref{eq:km-def} \\
(LCS) & locally constant slope diagonal (and selection matrices) & Definition~\ref{def:lcs} \\
$X_{\mathrm{reg}}$ & regular set: the inputs at which (LCS) holds & after Proposition~\ref{prop:lcs} \\
$\tilde W$, $\tilde\Psi$, $\tilde{\M}$, $h_{\tilde W}$ & a transformed weight collection, the network function, matrix and hidden activations (penultimate in the experiments) it realizes & Section~\ref{sec:germ} \\
$\mathrm{Stab}(\cdot)$ & stabilizer of an object among encodings regular at $x$ & after Theorem~\ref{thm:maxinv} \\
$d_M$, $d_\Psi$ & knowledge-matrix and logit displacement of a pair of inputs & Eq.~\eqref{eq:displacements} \\
$P$, $Q$ & visible and invisible parts of a displacement $\Delta\M$ (local to that theorem) & Theorem~\ref{thm:pyth} \\
$A=(d_\Psi/d_M)^2$ & coherence of a displacement & Definition~\ref{def:coherence} \\
\bottomrule\end{tabular}
\end{table}

%% file: sections/section_geometry.tex
\section{Distance geometry in matrix space}
\label{sec:geometry}

In this section, we move from a single input to a pair of them, and we work out
what the row-sum identity forces on distances in matrix space.

Every network in this section satisfies (LCS)
(Definition~\ref{def:lcs}), so that Theorem~\ref{thm:germ} applies and $\M(x)$
records the germ; every statement below that speaks of regions, walls or
crossings depends on it. Theorem~\ref{thm:pyth} is the one exception: it is
linear algebra about matrices with a prescribed row sum and needs no hypothesis
on the network at all. For a pair of inputs $x,y$ we write
\begin{equation}
d_M \;=\; \|\M(y)-\M(x)\|_F,
\qquad
d_\Psi \;=\; \|\Nety-\Netx\|_2
\label{eq:displacements}
\end{equation}
for the knowledge-matrix and logit displacements. These are the two stored
scalars from which every statistic in this section is built.

The row-sum identity $\M\mathbf 1=\Net$ gives knowledge-matrix perturbations an
exact visible/invisible accounting. This accounting is shared with any
representation carrying the same fixed row-sum constraint --- in particular
per-class gradient$\times$input (Theorem~\ref{thm:weff}); what it excludes is
any representation by \emph{hidden activations}, of the penultimate layer or
any other, for which no fixed parameter-independent vector plays the role of
$\mathbf 1$ (Proposition~\ref{prop:dichotomy}).

\begin{theorem}[Visible/invisible decomposition]\label{thm:pyth}
For any $\Delta\M\in\R^{C\times(d+1)}$ with logit displacement
$\Delta\Net=\Delta\M\,\mathbf 1$, let $P=\Delta\Net\,\mathbf 1^{\!\top}/(d{+}1)$
and $Q=\Delta\M-P$. Then $P$ is the orthogonal projection of $\Delta\M$ onto
constant-row matrices and the unique minimum-norm matrix with row sums
$\Delta\Net$; $Q\mathbf 1=0$; and
\[
\|\Delta\M\|_F^2=\frac{\|\Delta\Net\|_2^2}{d+1}+\|Q\|_F^2,
\qquad\text{hence}\qquad
\|\Delta\M\|_F\;\ge\;\frac{\|\Delta\Net\|_2}{\sqrt{d+1}},
\]
with equality iff $\Delta\M$ has constant rows.
\end{theorem}

\begin{definition}[Visible fraction and coherence]\label{def:coherence}
For a displacement $\Delta\M$ with logit displacement
$\Delta\Net=\Delta\M\,\mathbf 1$, the \emph{visible fraction} is
$\rho_{\mathrm{vis}}=\|\Delta\Net\|_2^2/\bigl((d{+}1)\|\Delta\M\|_F^2\bigr)$ and
the \emph{coherence} is
$A:=(d{+}1)\,\rho_{\mathrm{vis}}=(d_\Psi/d_M)^2$.
\end{definition}

Both quantities are functions of the two stored scalars $(d_M,d_\Psi)$ alone,
and since $A=1/r^2$ for $r=d_M/d_\Psi$, with $t\mapsto1/t^2$ strictly
decreasing, any rank-based statistic transfers between $r$ and $A$ with its
direction reversed and nothing else changed: rank statistics of $d_M/d_\Psi$
are reverse rank statistics of $A$. This is what makes the attack-family
ordering one and the same finding whichever of the two we tabulate; it is
unpacked, together with why the coherence is the squared ratio and not the
ratio itself (Remark~\ref{rem:why-squared}), in
Appendix~\ref{app:remarks-geometry}.

$Q$ is invisible \emph{to the evaluated logit displacement}: it encodes real
changes of the local linearization that the endpoint logits cannot see. The
split is an exact orthogonal decomposition --- short to prove but load-bearing,
since the entire descriptive geometry below rests on it. The
decomposition is invariant under output rescaling, under the gauge group, and
under per-coordinate input reparameterizations (pixel-unit changes move
nothing); it is \emph{not} invariant under general input rotations --- the
statistic is tied to the pixel basis, which for images is the natural one, and
we make no claim of rotation invariance --- nor under an input shift absorbed
into the first-layer bias, which changes $\M$ itself, the mean-shift
non-invariance of \citet{kindermans2019unreliability}, so that $x$ is always
the preprocessed input the network sees (Section~\ref{sec:invariance}).

This theorem replaces the amplification framing of
\citet{leblanc2024hidden}, which is a unit artifact: raw Frobenius
comparisons inflate with mass-spreading and per-coordinate (RMS) comparisons
deflate with dimension, and both are monotone transforms of the same
unit-free coherence $A=(d_\Psi/d_M)^2$ of Definition~\ref{def:coherence} (a
monotone transform of the $d_M/d_\Psi$ ratio). We use $A$ as a \emph{geometric descriptor} of a perturbation, read
against the two theorem-given reference lines $A=1$ (the one-pixel law,
Theorem~\ref{thm:within}(ii)) and $A\le d$ (the within-region cap,
Theorem~\ref{thm:within}(i)). We do \emph{not} tie $A$ to detectability or to
any decision-relevant outcome: it is descriptive geometry, not a validated
statistic, and the standalone empirical finding it is reported alongside is the
attack-family rank ordering, which does not depend on $A$ having any such link.

For a displacement $\delta$ within a region write $e_i=\delta_i^2\|J_{:,i}\|_2^2$
for the energy that column $i$ contributes to $d_M^2=\sum_ie_i$; the
column-energy \emph{participation ratio} of the displacement is
$\mathrm{PR}=\bigl(\sum_ie_i\bigr)^2\big/\sum_ie_i^2$.

\begin{theorem}[Within-region anatomy]\label{thm:within}
For a network satisfying (LCS): if $x$ and $y=x+\delta$ share the strict mask
pattern, then
$\M(y)-\M(x)=[\,J\,\mathrm{diag}(\delta)\mid 0\,]$ --- the bias column cancels
exactly --- and $d_M^2=\sum_i\delta_i^2\|J_{:,i}\|_2^2$. Consequently:
\emph{(i)} (coherence cap) $A\le d$, with the sharp maximum iff all weighted
columns $\delta_iJ_{:,i}$ are equal and nonzero; \emph{(ii)} (one-pixel law)
for $\delta=s\,e_{i_0}$ with $J_{:,i_0}\neq0$: $d_M=d_\Psi$ exactly, i.e.\
$A=1$, independent of pixel and magnitude; \emph{(iii)} ($k$-sparse bound) a
perturbation supported on $k$ pixels has at most $k$ nonzero columns,
$A\le k$, and column-energy participation ratio at most $k$.
\end{theorem}

The participation ratio of part (iii) is the standard
inverse-participation-ratio measure of how many columns actually carry the
energy. It equals $k$ exactly when $k$ columns share the energy equally and
the rest are silent, and $1$ when a single column carries all of it, so it
reads as an \emph{effective number of active columns}. Part (iii) then says a
$k$-sparse perturbation cannot have more than $k$ of them, which is immediate:
only the $k$ columns in the support have $e_i\neq0$.

The one-pixel law of part (ii) requires comment, since $A=1$ is easily misread
as ``nothing happened''.
The algebra is two lines. Within a region
$\M(y)-\M(x)=[\,J\,\mathrm{diag}(\delta)\mid0\,]$, so
$d_M^2=\sum_i\delta_i^2\|J_{:,i}\|^2$; for $\delta=s\,e_{i_0}$ this is
$d_M=|s|\,\|J_{:,i_0}\|$, while $\Delta\Net=J\delta=s\,J_{:,i_0}$ gives
$d_\Psi=|s|\,\|J_{:,i_0}\|$ as well. Hence $d_M=d_\Psi$ and $A=1$, for every pixel
and every magnitude.

What is pinned is the \emph{ratio}, not the effect. Both $d_M$ and $d_\Psi$ grow
linearly in $|s|$, so a large single-pixel edit moves the matrix and the
logits a great deal; it is their quotient that cannot move. $A$ measures
interference among the weighted columns $\delta_iJ_{:,i}$, and a one-pixel
displacement switches on exactly one of them --- there is nothing to reinforce
and nothing to cancel, so all of the matrix motion is logit-visible by
default. $A=1$ is the no-interference baseline, not a claim that the
perturbation is inert.

The objection that editing an image one pixel at a time would keep $A$ at $1$
is answered in two steps.
\emph{First}, $A$ is a
functional of the displacement $\delta=y-x$, not of a path. Editing $k$ pixels
one at a time produces $k$ single-pixel displacements each with $A=1$, but the
object this theorem scores is the \emph{cumulative} $\delta$, which is
$k$-sparse and not one-sparse; $A$ is not additive along a path and does not
accumulate. For the cumulative $\delta$, part (iii) allows any value up to
$k$: above $1$ when the $\delta_iJ_{:,i}$ reinforce, below $1$ when they
cancel, and near $1$ when they are mutually incoherent. \emph{Second}, a long
walk leaves the region. Once $x$ and $y$ no longer share a mask pattern the
hypothesis here fails and Theorem~\ref{thm:anatomy} governs instead: the
crossing dyads satisfy $K\mathbf 1=0$, so they contribute exactly nothing to
$d_\Psi$ while moving $\M$; that they raise $d_M$ --- $K$ is not orthogonal to
the smooth part, so this is not automatic --- and push $A$ down is the
heuristic behind Conjecture~\ref{conj:mechanism}, not a consequence of the
theorem. This is
the regime the experiments live in --- full-image adversarial displacements,
crossing many walls, are measured at median $A\le0.23$
(Section~\ref{sec:study2}).

For two inputs $x,y$ the endpoint \emph{mask-Hamming distance} $H$ is the
number of hidden units whose $0/1$ activation mask differs between them,
$\sum_{\ell,k}\bigl|D^{(\ell)}_k(x)-D^{(\ell)}_k(y)\bigr|$ for ReLU, plus the
number of max-pooling windows whose selected entry differs.

\begin{theorem}[Smooth $+$ crossing anatomy]\label{thm:anatomy}
Let the network satisfy (LCS) with slopes $a_\pm$ and slope jump
$\kappa=a_+-a_-$, and let $(x,y)$, $\delta=y-x$, be a pair whose segment
crosses region walls transversally, one at a time, at $t_1<\dots<t_N$, each
crossing a flip of one unit or a switch of one max-pooling window. With
$\bar J=\int_0^1J(x+t\delta)\,dt$,
\[
\M(y)-\M(x)=\bigl[\,\bar J\,\mathrm{diag}(\delta)\mid 0\,\bigr]
\;+\;K,\qquad
K=\sum_{j}K_j,\qquad K_j\mathbf 1=0,
\]
where the $j$-th crossing contributes the rank-one dyad of
Lemma~\ref{lem:dyad}:
$K_j=\kappa\,\sigma_j\,u_{k_j}\bigl[\,v_{k_j}^{\top}\mathrm{diag}(z_j)\mid
\gamma_{k_j}\,\bigr]$ for a flip of unit $k_j$, and
$K_j=u_j\bigl[\,(v_{b_j}-v_{a_j})^{\top}\mathrm{diag}(z_j)\mid
\gamma_{b_j}-\gamma_{a_j}\,\bigr]$ for a switch of a window from entry $a_j$
to entry $b_j$; its row sums vanish because the flipped unit's
pre-activation, or the difference of the two entries' pre-activations,
vanishes at the crossing point $z_j$; hence $\Nety-\Netx=\bar J\,\delta$. The
bias column of $\M(y)-\M(x)$ equals $c(y)-c(x)$, the sum of the last columns
of the $K_j$ ($\sum_j\kappa\,\sigma_j\gamma_{k_j}u_{k_j}$ when every
crossing is a unit flip), so $c(y)\neq c(x)$ implies $N\ge1$. The endpoint
mask-Hamming distance satisfies $H\le N$, and $H\equiv N\pmod 2$ when the
network has no max-pooling.
\end{theorem}

The transversality hypothesis is an assumption and not something we prove.
The theorem asks that the segment meet the walls transversally, finitely
often, and one unit or window at a time. For a single-hidden-layer network
the walls form a finite hyperplane arrangement and the standard
transversality argument does put the bad pairs in a null set. At depth this
is no longer a hyperplane arrangement: a deeper unit's wall is cut only
within the region where the frozen pattern below it is realized, so the wall
complex is piecewise --- a finite union of relatively open polyhedral pieces
--- and the transversality argument has to be run piece by piece, with the
pieces themselves depending on the parameters. We have not carried that out,
and we state the hypothesis as a hypothesis. Two concrete conditions are
needed and are worth naming, because both are checkable: no two units may
share a wall piece (exactly duplicated units, as in pruned or weight-tied
networks, violate it and should be screened), and no crossing may be
tangential. Our experiments screen for the first and rely on the second
holding for randomly drawn pairs.

For ReLU, $\kappa=1$. The theorem says that the crossing part $K$ moves the
germ and not the function: $K$ contributes nothing to the endpoint logit
displacement, which the smooth part accounts for in full. A nonzero bias
column certifies at least one crossing, with no false positives; the converse
fails, since a bias-free network has a zero bias column whatever it crosses.
The symbols in $K$ --- the rank-one dyad $u_{k_j}v_{k_j}^{\top}$ of the unit
$k_j$ that flips at the $j$-th crossing (Lemma~\ref{lem:dyad}), the crossing
point $z_j$, the flip direction $\sigma_j$, and the input-side data
$(v_{k_j},\gamma_{k_j})$ whose wall $\{v_{k_j}^{\top}x'+\gamma_{k_j}=0\}$ is
exactly why the dyad's row sums vanish; for a pooling switch, the two
entries' input-side data, whose tie wall
$\{(v_{b_j}-v_{a_j})^{\top}x'+\gamma_{b_j}-\gamma_{a_j}=0\}$ plays the same
role --- are unpacked in Appendix~\ref{app:remarks-geometry}.

The smooth block $\bar J\,\mathrm{diag}(\delta)$ is, entry for entry, the
integrated-gradients attribution of $y$ against baseline $x$
\citep{sundararajan2017axiomatic}, and the crossing term $K$ is the remainder
that integrated gradients has no name for
(Remark~\ref{rem:ig-smooth}, Appendix~\ref{app:remarks-geometry}).

The reading ``the crossing part moves the germ and not the function'' does not
contradict Definition~\ref{def:germ}. The germ is a \emph{local} object
attached to a \emph{single} point, and it is determined there by the function
(Theorem~\ref{thm:germ}); nothing in this theorem makes the germ at a fixed
point ambiguous. The sentence is about the \emph{displacement between two
different points}, and its ``function'' means the endpoint logit displacement
$\Delta\Net=\Nety-\Netx$ --- not the function as a whole. The claim is that
$\Delta\Net$ is accounted for in full by the smooth term, since
$\Delta\Net=\bar J\delta$ while $K\mathbf 1=0$, so the crossings contribute
nothing to it. They are nevertheless exactly what changes the local affine
data between the endpoints:
$J(y)-J(x)=\sum_j\kappa\,\sigma_ju_{k_j}v_{k_j}^{\top}$
and $c(y)-c(x)=\sum_j\kappa\,\sigma_j\gamma_{k_j}u_{k_j}$, both carried by $K$ and
neither legible in $\Delta\Net$. The long-hand reading is therefore: between $x$
and $y$ the germ moves by more than the endpoint output values can reveal.
Both germs remain determined by the same single function, at two different
points --- which is Theorem~\ref{thm:germ}, not a violation of it.

The endpoint mask-Hamming distance $H$ lower-bounds
the number of wall crossings along the segment, for a pooling-free network
with equality of parity: a
unit that ends up flipped must have crossed an odd number of times, a unit
that ends up unflipped an even number --- possibly zero, possibly two, and a
unit that crosses and crosses back is invisible to the endpoints. A pooling
window whose selected entry differs between the endpoints has switched at
least once, but a window with three or more entries can switch twice and end
at a third entry, which is why the parity statement excludes max-pooling.
Equality $H=N$ holds iff no unit flips twice and no window switches twice.
The reason it is worth stating is cost:
$H$ needs only the two mask records that evaluating the network at $x$ and at
$y$ already produces --- one forward pass per endpoint, two in total --- with
no need to trace the segment, locate the crossing times $t_j$, or identify
which units flipped. It is the cheapest available certificate that two inputs
lie in different activation regions, obtained at one forward pass per
endpoint.

The crossing dyads suggest a mechanism for the attack-family ordering of
Study~2 --- iterative small-step attacks cross fewer walls and align with
high-energy Jacobian columns, raising $A$ --- which we state as
Conjecture~\ref{conj:mechanism} in Appendix~\ref{app:mechanism-pilot}, together
with the size-controlled VGG pilot that tests its rank-correlation consequence.

\begin{proposition}[Hidden-activation distances carry no germ-level accounting]
\label{prop:dichotomy}
Let the activation be positively homogeneous, $f(\lambda z)=\lambda f(z)$ for
$\lambda>0$, let $1\le\ell<L$ index a hidden layer with activations
$h^{(\ell)}$, and for a pair of inputs $x,y$ write
$d_h=\|h^{(\ell)}(y)-h^{(\ell)}(x)\|_2$, and let $W^{(\ell)},b^{(\ell)}$
denote the composed affine map producing $z^{(\ell)}$. For $\lambda>0$ let
$W_\lambda$ be obtained from $W$ by $W^{(\ell)}\mapsto\lambda W^{(\ell)}$,
$b^{(\ell)}\mapsto\lambda b^{(\ell)}$ and
$W^{(\ell+1)}\mapsto\lambda^{-1}W^{(\ell+1)}$. Then
$\Psi(W_\lambda,f)=\Net$, so that $d_\Psi$ and $d_M$ are unchanged, while
$h^{(\ell)}\mapsto\lambda h^{(\ell)}$ and $d_h\mapsto\lambda d_h$.
Consequently every function of $(d_h,d_\Psi,d_M)$ that is invariant under
this action is independent of $d_h$.
\end{proposition}

The decomposition of Theorem~\ref{thm:pyth} exists because
$\M\mathbf 1=\Net$ holds with a fixed, parameter-independent vector
$\mathbf 1$, and it is shared by any representation with the same row-sum
constraint, gradient$\times$input among them (Theorem~\ref{thm:weff}). Hidden
activations admit no analog: the relation that returns the activations of
layer $\ell$ to the logits pairs them against the parameter-dependent map
downstream of that layer, and the proposition turns this into an invariance
statement. For a convolution--normalization block in evaluation mode the
rescaled parameters are the normalization's affine pair,
$(\gamma,\beta)\mapsto(\lambda\gamma,\lambda\beta)$, which rescales the
composed map; rescaling the convolution alone would not, since
$\mathrm{BN}(\lambda z)\neq\lambda\,\mathrm{BN}(z)$ in evaluation mode. As
$\lambda$ ranges over $(0,\infty)$ the stored distance $d_h$
sweeps the whole ray, so no non-constant function of $(d_h,d_\Psi,d_M)$ that
is invariant under this action can depend on $d_h$, at the penultimate layer
that the experiments compare or at any other. The proposition's scope ---
whole-layer versus single-unit rescaling, and functions of the three stored
distances only --- and its provenance
in the positive-homogeneity symmetry of ReLU networks
\citep{neyshabur2015path,dinh2017sharp}, of which it draws only the
consequence for the accounting of Theorem~\ref{thm:pyth}, are set out in
Appendix~\ref{app:remarks-geometry}.

\subsection{Teleportation: what is bounded and what must be measured}
\label{sec:teleport-claims}

\begin{theorem}[No $C^0$ bound]\label{thm:nogo}
For every $\varepsilon,B>0$ and every $x$ with some coordinate $x_i\neq0$
there exist two networks of the same architecture, realizing functions $\Net$
and $g$ with $\sup_{x'}\|\Netxp-g(x')\|\le\varepsilon$ and identical masks at
$x$, whose matrices $\M_\Psi(x)$ and $\M_g(x)$ satisfy
$\|\M_g(x)-\M_\Psi(x)\|_F\ge B$. For $x=0$ no such pair exists: there
$\M_g(0)-\M_\Psi(0)=[\,0\mid g(0)-\Psi(0)\,]$ has Frobenius norm at most
$\varepsilon$.
\end{theorem}

In words, function-space closeness does not control germ drift. $C^0$
closeness is closeness in \emph{value}, uniformly, and the theorem says it
gives no control of the first-order data the matrix records --- the smooth
illustration $g=\Net+\varepsilon\sin(\cdot/\varepsilon^{2})$ stays within
$\varepsilon$ of $\Net$ while its derivative differs by $1/\varepsilon$, and
the theorem's witness is the two-unit ReLU ramp of its proof, a
same-architecture network --- which is why
agreement of two networks' logits says nothing, by itself, about their
matrices (Appendix~\ref{app:remarks-geometry}). For a transform that is only
approximately function-preserving, Proposition~\ref{prop:visdrift}
(Appendix~\ref{app:remarks-geometry}) separates three tiers --- zero drift
under an exact isomorphism, visible drift pinned to the logit gate as an
identity, invisible drift that the gate does not bound. Neural teleportation
is an exact isomorphism, batch normalization in evaluation mode included
(Appendix~\ref{app:teleport-why}), so the teleportation check of
Appendix~\ref{sec:study1b} sits at the first tier.

%% file: sections/empirical_setup.tex
\section{Empirical setup}
\label{sec:setup}

In this section we fix the shared apparatus --- networks, data, the
knowledge-matrix construction, the distance metric, and the adversarial-attack
suite --- used by all three studies and by the honest negatives of
Section~\ref{sec:negatives}. The per-study specifics --- sample sizes, transform
counts, attack budgets, similarity-measure settings --- are stated step by step
in each study section.

\paragraph{Setting.}
We work with pretrained feedforward networks $\Net \colon \R^d \to \R^C$,
where $W$ is the collection of weights and $f$ the activation function.
Inputs are ImageNet-normalized RGB tensors of shape $(3, 224, 224)$,
so $d = 3 \cdot 224 \cdot 224 = 150{,}528$, and $C = 1000$ for the
standard ImageNet head. The hidden activations that the experiments compare
against are those of the penultimate layer, denoted $h(x) \in \R^D$, where
$D = 2048$ for ResNet-152 and $D = 1024$ for DenseNet-121 and GoogLeNet. This
is a choice of representative hidden layer, the one most similarity methods
compare; every statement of Sections~\ref{sec:germ}--\ref{sec:geometry} about
hidden activations applies to any hidden layer.

\paragraph{Architectures.}
The studies of this paper are scoped to three pretrained
\texttt{torchvision} networks --- ResNet-152, DenseNet-121, and GoogLeNet
(with \texttt{aux\_logits=False}), loaded from the standard pretrained
weights and never trained or fine-tuned here
(Section~\ref{sec:limitations}, L8). They span the
residual, dense and inception families common in
vision-interpretability work, with different penultimate dimensionalities
and different concatenation topologies, both of which interact with the
constructions below. The attack-family ordering cross-check
(Section~\ref{sec:study2}, Table~\ref{tab:ordering}) additionally enrolls
three further networks --- ResNet-18, AlexNet, and VGG --- as extra
\emph{ordering-only} raters; these enter no coherence-magnitude statistic
and no cross-architecture distance, and serve only to test whether the
attack-family ranking is stable beyond those three networks.

\paragraph{The knowledge matrix.}
For a network $(W\!,f)$ the knowledge matrix
$\Mfx \in \R^{C \times (d+1)}$ is the per-sample matrix of
Equation~\ref{eq:km-def}, the contraction of the quiver representation
$\kmapx$ that the network induces on the input $x$ (Section~\ref{sec:germ}),
which under (LCS) is built from the local affine map (Jacobian and
bias) that $\Net$ realizes on the activation region of $x$; its row sums
reproduce the logits, $\M(W{,}f)(x)\,\mathbf 1_{d+1} = \Netx$. The
construction is uniform across architectures: for any feedforward network
on $224\times224$ ImageNet inputs with a 1000-class head, $\Mfx$ has the
fixed shape $1000 \times 150{,}529$, which is what makes the
cross-architecture comparison of Study~3 alignment-free.

\paragraph{Distances, versions and reproduction.}
Knowledge-matrix distances are Frobenius, reported RMS-per-coordinate where
architectures are compared, and penultimate distances are $\ell_2$
(RMS-per-dimension across architectures; Section~\ref{sec:limitations}, L6);
the metric conventions, library versions, hardware and the reproduction
details (sample sets, seeds and scripts) are in
Appendix~\ref{app:setup-details} and the supplementary material.

\paragraph{Adversarial-attack suite (Study~2 and the permutation check of Appendix~\ref{app:teleport-checks}).}
Adversarial pairs are generated with the \texttt{torchattacks} library
(\citealt{kim2020torchattacks}) using six attacks that span the standard
families: FGSM~\citep{goodfellow2015explaining} (one-shot sign),
PGD~\citep{madry2018towards} and APGD~\citep{croce2020reliable} (iterative
$\ell_\infty$), CW~\citep{carlini2017towards} ($\ell_2$
margin-optimizing), DeepFool~\citep{moosavi2016deepfool} (minimal-norm
boundary), and Square~\citep{andriushchenko2020square} (gradient-free,
score-based). We keep \emph{all} clean/adversarial pairs (including
attack-failure pairs). The per-cell medians of Study~2
(Section~\ref{sec:study2}) are taken over the pairs passing a division
guard, $d_\Psi > 10^{-6}\max_i d_{\Psi,i}$, a per-cell relative threshold that
removes only pairs whose logits did not move; the stricter attack-success
filter ($d_\Psi \ge 1$, $d_M > 0$) is applied in the two analyzes that keep
per-pair records --- the mechanism pilot
(Table~\ref{tab:mechanism-pilot}) and the appendix ordering panel
(Table~\ref{tab:s3_ordering_panel}).
The hyperparameters that depart from the \texttt{torchattacks}~3.5.1 defaults
(DeepFool and APGD step counts, the APGD loss, the Square query budget), the two
ordering-only raters --- AlexNet and VGG --- that ran the library defaults
throughout and therefore face weaker attack budgets, and the full-scale pair
set's own budgets are listed in Appendix~\ref{app:setup-details}.

\paragraph{Signal and noise: the two reference scales.}
Several studies below report a drift ``relative to the adversarial signal'':
the \emph{adversarial signal} of a representation is its mean
clean-to-adversarial distance over the pairs of one (architecture, attack)
cell, and the \emph{permutation noise floor} is the residual that survives a
neuron permutation --- for the knowledge matrix zero in exact arithmetic
(Lemma~\ref{lem:quiver-inv}), so that what the pipeline records is its own
resolution, for the penultimate features a genuine motion
--- and the two scales, together with the ``$x \mid y$'' ratio cells of
Table~\ref{tab:signal-relative}, are defined in full in
Appendix~\ref{app:setup-details}.

\paragraph{Bootstrap confidence intervals.}
Two $95\%$ bootstrap intervals appear in this paper --- a percentile bootstrap
of the mean over the $T=50$ teleportation draws of an architecture
($B=10{,}000$; Table~\ref{tab:s1_table}) and a seeded bias-corrected and
accelerated (BCa) bootstrap of the median over the adversarial pairs of one
cell ($B=20{,}000$; Table~\ref{tab:s3_ordering_panel}) --- and each quantifies
the sampling variability of its statistic over its own resampled unit and
nothing else; both are specified in Appendix~\ref{app:setup-details}.

\paragraph{Samples, features, and the CKA estimator.}
The similarity panel of Section~\ref{sec:study1c} evaluates its HSIC on the
full $N=25{,}000$-sample Gram matrix against penultimate dimensions
$D\in\{1024,2048\}$, so it sits in the low-dimensional regime $n\gg p$; it uses
the unbiased HSIC$_1$ $U$-statistic all the same, since that removes the biased
estimator's $O(1/n)$ upward bias at any ratio of $n$ to $p$, and the Murphy
shuffled-pair control confirms that no upward bias remains (debiased CKA
$\le 5.4\times10^{-4}$; details in Appendix~\ref{app:setup-details}).

%% file: sections/study1_invariance.tex
\section{Study 1: what the standard similarity measures do with a gauge transformation}
\label{sec:study1}

In this section we ask what the representational-similarity literature's
standard measures do with a transformation that leaves the network's function
unchanged. Section~\ref{sec:study1-analytic} settles by algebra what algebra
settles, for the knowledge matrix and for the penultimate features alike.
Section~\ref{sec:study1c} runs the nine-measure panel on the one question that
is left.

\subsection{Why the invariance itself needs no experiment}
\label{sec:study1-analytic}

Both symmetries this paper appeals to preserve the function by construction.
Neural teleportation \citep{armenta2023teleportation}, a per-neuron
rescaling, is an element of the change-of-basis group of
\citet{armenta2021representation}, whose Theorem~4.13 states that an
isomorphism of neural networks $\tau:(W,f)\to(V,g)$ leaves the realized function
unchanged, $\Psi(W,f)=\Psi(V,g)$; a hidden-layer neuron permutation is a
relabeling of the quiver's hidden vertices. The knowledge matrix is invariant
under both (Lemma~\ref{lem:quiver-inv}), batch normalization in evaluation mode included
(Appendix~\ref{app:teleport-why}), so there is nothing here for an experiment
to settle. The effect on the hidden activations of any layer is equally
explicit, and we state it for the penultimate features $h$ that the
experiments compare: a hidden-layer permutation $\pi$ reindexes $h$, with
$\|Ph-h\|_2^2=2\bigl(\|h\|^2-\langle Ph,h\rangle\bigr)$, and a teleportation
with per-neuron factors $\tau$ rescales it coordinatewise, $h\mapsto\tau\odot h$,
with drift $\|(\tau-1)\odot h\|_2/\sqrt D$ --- closed forms in $h$ and $\pi$ or
$\tau$ alone (Appendix~\ref{app:s1-analytic-detail}). What algebra does not settle is
whether the representational-similarity measures absorb this motion as a
coordinate artifact, and Section~\ref{sec:study1c} settles that on a pair of
networks whose ground truth is known exactly: they compute the same function.

\subsection{The nine-measure panel}
\label{sec:study1c}
\label{sec:A1}

The teleportation check of Appendix~\ref{app:teleport-checks}
(Appendix~\ref{sec:study1b}) reports substantial penultimate-feature drift
under teleportation. A natural objection is that this drift is a
coordinate-frame artifact, which the standard representational-similarity
machinery (CKA, Procrustes and the rest) would quotient out. If so, the
``hidden activations are not enough'' claim would collapse to ``hidden
activations require the right metric''. The objection is partly right: raw
CCA and PWCCA quotient this transform out exactly
(Appendix~\ref{app:s1-scale}). What defeats its stronger form --- that hidden
activations, given the right metric, carry what the matrix carries --- is
Theorem~\ref{thm:complete}(iv): at a point no statistic of the activations
determines the germ. What this experiment measures is narrower: how much
drift the measures in common use register on a pair whose ground truth is
exact, under $\tau\sim U[0,2]$. It runs the canonical-measure panel of the
representation-similarity literature on the same teleportation pairs and asks
which measures absorb the transformation, and by how much.

\paragraph{Setup.}
For each of the $T = 50$ random teleportations per architecture (seed
indices $0,\ldots,49$) we compute a panel of nine representational-similarity
measures between the penultimate features $h_W(x)$ and $h_{\tilde W}(x)$ on
$N = 25{,}000$ ImageNet-validation samples (the first $25{,}000$ by sorted
filename, the same set Study~3 uses; Study~2's adversarial pairs are drawn
differently --- see Section~\ref{sec:setup}). The measures span
the dominant invariance classes cataloged by~\citet{klabunde2025resi}:
orthogonal-plus-isotropic-scaling (debiased linear CKA, angular CKA, Bures
similarity, distance correlation), orthogonal-only (Procrustes shape
distance), permutation-only (soft-matching), monotone-of-distance
(RSA-Spearman), isometry (Gromov--Wasserstein on a 5{,}000-sample
subsample for tractability), and a function-level baseline (square-root
output JSD). Nine measures were run; eight returned a value, and
Gromov--Wasserstein did not (Table~\ref{tab:s1_table}), so every count in
this paper is over the eight. Of those, angular CKA is a reparameterization
of debiased CKA, its arccosine, and output JSD is a functional baseline on
the logits rather than a penultimate measure, so the panel has six distinct
penultimate measures across five invariance classes. The Gromov--Wasserstein
column was computed with the entropic solver of the POT library
\citep{flamary2021pot} (\texttt{entropic\_gromov\_wasserstein2}, square loss,
$\varepsilon=10^{-2}$, $100$ iterations) on the unnormalized Euclidean
distance matrices of the subsample, and it returned a degenerate plan
($0.0$) on every cell; the failure was not diagnosed, and with centered Gram
traces of $10^{6}$--$10^{7}$ (Appendix~\ref{app:crossarch-detail}) an
unnormalized cost matrix under a fixed $\varepsilon$ of that size is the
likely cause. For CKA we use the unbiased HSIC$_1$ estimator
of~\citet{song2012feature} as plugged into the minibatch CKA framework
of~\citet{nguyen2021minibatch}, which removes the $O(1/n)$ finite-sample upward
bias documented by~\citet{murphy2024debiased}. This choice does not depend on
the sample-to-dimension ratio, and these runs are not in the high-dimensional
regime that usually motivates it: $N=25{,}000$ samples against penultimate
widths $D\in\{1024,2048\}$, so $n\gg p$ (Section~\ref{sec:setup}).
We accompany every measure with
the random-network control of~\citet{cui2022deconfounded} (each measure
between trained $h_W$ and an identically-initialized but untrained $h_R$)
and the shuffled-pair control of~\citet{murphy2024debiased} (sample
alignment permuted between $X$ and $Y$).
Three scope facts about the controls --- they were run separately on a smaller
$n=2{,}048$ input population, they are bare point estimates with no confidence
intervals or permutation $p$-values, and the pipeline's automated gate checks
only three of the measures --- are stated in
Appendix~\ref{app:controls-full}; the panel's own numbers depend on none of
the three.

\paragraph{Procedure.}
\begin{enumerate}
\item For each teleportation $\tau$ collect the matched penultimate pairs
$(h_W(x), h_{\tilde W}(x))$ over the $N=25{,}000$ samples.
\item Compute each of the nine measures on the pair, with the estimator
settings above.
\item Compute the same measures for the Cui random-network and Murphy
shuffled-pair controls.
\item Aggregate to a mean and 95\% bootstrap CI over the $T=50$
teleportations per architecture.
\end{enumerate}

\paragraph{What we measure and the claim it licenses.}
Each measure recovers the invariance only insofar as the teleportation
change-of-basis lies in its invariance class. Teleportation acts on the penultimate
vector as $h\mapsto\tau\odot h$, an invertible but \emph{anisotropic} diagonal
map, and that lies outside the orthogonal-plus-isotropic-scaling class this
panel's most permissive members quotient out --- neither class contains the
other. (It lies outside by design of the measures, not because of any
batch-normalization inconsistency: the transform is exact on these networks,
running statistics included --- Appendix~\ref{app:teleport-why}.) The
prediction, which follows from the measures' invariance classes together with
Theorem~\ref{thm:complete}(iii), is sharp \emph{for the measures of this
panel}: none of them quotients out the full teleportation drift.
This is a statement about the eight measures that returned a value: the
linear-invariant CCA family (raw CCA, PWCCA), which is not in the panel,
quotients $h\mapsto\tau\odot h$ out exactly, whereas SVCCA, whose energy
truncation runs before the CCA, does not (Appendix~\ref{app:s1-scale}).

The knowledge matrix is \emph{not} entered as a similarity column: it is not
a member of this panel's family, and its invariance under this transform is
Lemma~\ref{lem:quiver-inv}, realized in software in
Appendix~\ref{app:teleport-checks}, rather than a quantity any of these
measures adjudicates. What the panel adjudicates is the penultimate-feature
drift.

\paragraph{Result.}
\input{tables/s1_table}
\input{sections/fig_panel}
The empirical pattern (Table~\ref{tab:s1_table}) matches the
invariance-class prediction. Figure~\ref{fig:panel-teleportation} shows the
same panel against the identity line and the controls. The orthogonal-class and permutation-class
measures register the teleportation drift directly: Procrustes shape distance
reads $632$--$1471$ across the three architectures and soft-matching distance
$22$--$61$.
Both are unnormalized distances whose scale is the Procrustes bound
$\sqrt{\|\tilde X\|_F^2+\|\tilde Y\|_F^2}$, the largest value the statistic
can take: measured against it, the penultimate features move by $0.23$--$0.30$
of the bound in Procrustes distance (ResNet-152 $776.6/2599$, DenseNet-121
$1471.4/5071$, GoogLeNet $632.1/2727$) and by $0.8$--$1.2\%$ of it in
soft-matching distance, and both figures scale as $\sqrt n$
(Appendix~\ref{app:s1-scale}).
The bounded
isotropic-scaling-invariant similarities quotient much of the drift and
report high agreement (debiased CKA $0.90$--$0.94$, Bures $0.92$--$0.96$,
distance correlation $0.91$--$0.97$, RSA-Spearman $0.84$--$0.96$), while
the square-root output JSD is $\approx 0$ on all three architectures,
confirming the teleportation is function-preserving. Crucially, no measure
\emph{of this panel} reports the exact invariance ($1$ for similarities, $0$ for
distances) that the knowledge matrix has --- every one of the eight registers
some residual drift, since the anisotropic rescale is outside every invariance
class in this panel, and the penultimate fractions above are, for Procrustes,
a quarter to a third of the largest value the statistic can take. The gap is
a property of the invariance classes: an anisotropic diagonal map lies
outside every class in this panel and inside the quiver-isomorphism group
under which the knowledge matrix is invariant by construction. Every
magnitude in Table~\ref{tab:s1_table} is conditional on the change-of-basis
distribution, $\tau\sim U[0,2]$ (\texttt{cob\_range}$=1$); the penultimate
drift is $\approx\|h\|_2/\sqrt{3D}$ in closed form
(Appendix~\ref{app:s1-analytic-detail}) and would shrink toward zero for
$\tau$ concentrated near $1$ and grow for a wider range. The analytic
prediction is ``not exactly $1$'', nothing more.

Both controls behave as intended on the HSIC/CKA family: the Cui random-network
control gives debiased linear CKA $0.010$ (ResNet-152), $0.058$ (DenseNet-121)
and $0.052$ (GoogLeNet), far below both the pipeline's gate threshold of $0.5$
for this control and the $0.90$--$0.94$ trained-vs-teleported agreement. For
debiased CKA the three readings available in this paper are therefore
$0.90$--$0.94$ (a network and its teleported copy, Table~\ref{tab:s1_table}),
$0.36$--$0.54$ (two different architectures, Table~\ref{tab:s2_table}) and
$0.01$--$0.06$ (a network and its random initialization,
Table~\ref{tab:controls-full}): the measure places the exact-function pair
far above every other pair, which is the discriminability criterion of
Appendix~\ref{app:vocab}, while registering the drift; what the control does
not test is whether the panel value would differ for an untrained network and
its own teleported image. The Murphy shuffled-pair
control returns debiased CKA $2.4\times10^{-5}$, $2.6\times10^{-4}$ and
$5.4\times10^{-4}$, distance correlation $\le0.024$ and RSA-Spearman within
$0.007$ of $0$, so the unbiased HSIC$_1$ estimator is correctly implemented and
carries no upward bias. Bures similarity, however, does not return to zero
under either control --- $0.264$--$0.294$ on misaligned pairs, and $0.383$ for
a randomly-initialized GoogLeNet against its trained counterpart --- which caps
what the cross-architecture Bures values of Study~3 ($0.478$--$0.595$, within
$1.6$--$2.3\times$ of the shuffled-pair floor) can be read to mean
(Appendix~\ref{app:limitations-more}, L11); all eight measures of both controls,
computed on $n=2{,}048$ inputs as bare point estimates, are tabulated in
Appendix~\ref{app:controls-full} (Table~\ref{tab:controls-full}).

\paragraph{Reading.}
The panel shows that all eight of the similarity measures that returned a
value register some teleportation drift --- none of them quotients out the
full change of basis --- whereas the knowledge matrix does not move. The scope
of ``none'' is the panel: raw CCA and PWCCA \emph{would} quotient this
particular transform out, being invariant to any invertible linear map, and
they are not members of this panel.
This is a statement about \emph{which invariance class each representation
belongs to}, not a performance ranking. What it does not show is the knowledge
matrix's invariance itself: the panel adjudicates the penultimate drift and
nothing else, and the matrix's invariance under teleportation is
Lemma~\ref{lem:quiver-inv}, with the software check that realizes it in
Appendix~\ref{app:teleport-checks}.

%% file: tables/s1_table.tex
\begin{table}[htbp]
\centering
\caption{No measure of the panel returns the exact invariance that the transformation
guarantees: the distances and the bounded similarities alike register the drift, as their
invariance classes predict.
Representational-similarity panel between penultimate features $h_W(x)$ and their
neural-teleportation image $h_{\tilde W}(x)$, mean $[95\%$ CI$]$ over $T=50$ teleportations
per architecture ($N=25{,}000$ samples). Distances (Procrustes, soft-matching) grow with the
teleportation drift; bounded similarities (CKA, Bures, dCor, RSA) report high agreement;
square-root output JSD $\approx 0$ confirms the transformation is function-preserving. Every
magnitude is conditional on the change-of-basis distribution $\tau\sim U[0,2]$
(\texttt{cob\_range}$=1$); the analytic prediction is ``not exactly $1$'', nothing more. GW
omitted: the entropic solver returned a degenerate plan ($0.0$) on every cell
(Section~\ref{sec:study1c}). \emph{Angular CKA is neither} (marked $\downarrow$: smaller
is more similar): it is
$\arccos(\text{CKA})\in[0,\pi/2]$ in radians, so \emph{small} means similar --- the
$0.36$--$0.45$ entries are the arccosines of the debiased-CKA row above
($\arccos 0.908 = 0.432$), not a report of low agreement.
The knowledge matrix is not a column here: teleportation is an \emph{exact}
function-preserving isomorphism (Appendix~\ref{app:teleport-exact}), so the knowledge
matrix does not move; the implementation check that realizes this is
Appendix~\ref{sec:study1b}.
\emph{What the intervals do and do not cover}: the resampled unit is the $T=50$ teleportation
draws, with the $N=25{,}000$ images held fixed (percentile bootstrap of the mean, $B=10{,}000$,
seeded; Section~\ref{sec:setup}). They are therefore intervals on the variability of the
change-of-basis draw alone, which is why several are narrower than the printing precision; they
say nothing about how a cell would move on a different image set or a different architecture.
Cui random-network and Murphy shuffled-pair controls reported in the text.}
\label{tab:s1_table}
\small\setlength{\tabcolsep}{3.5pt}
\begin{tabular}{llccc}
\toprule
Measure & Invariance class & ResNet-152 & DenseNet-121 & GoogLeNet \\
\midrule
debiased CKA  & orth + iso-scale     & $0.908\,[0.906,0.910]$ & $0.902\,[0.900,0.904]$ & $0.935\,[0.935,0.936]$ \\
angular CKA$^{\downarrow}$ & orth + iso-scale     & $0.432\,[0.428,0.436]$ & $0.445\,[0.441,0.450]$ & $0.361\,[0.359,0.363]$ \\
Procrustes    & orth                 & $776.6\,[773.8,779.4]$ & $1471.4\,[1459.1,1484.2]$ & $632.1\,[629.2,634.9]$ \\
Bures         & orth + iso-scale     & $0.920\,[0.919,0.921]$ & $0.925\,[0.924,0.927]$ & $0.956\,[0.956,0.957]$ \\
soft-matching & permutation          & $21.93\,[21.85,22.01]$ & $60.73\,[60.23,61.24]$ & $32.22\,[32.09,32.34]$ \\
RSA           & rotation + monotone  & $0.841\,[0.837,0.845]$ & $0.908\,[0.906,0.910]$ & $0.958\,[0.957,0.958]$ \\
output JSD    & (none, functional)   & $0.001\,[0.001,0.001]$ & $0.000\,[0.000,0.000]$ & $0.001\,[0.001,0.001]$ \\
dCor          & translation + orth + iso-scale & $0.908\,[0.907,0.909]$ & $0.954\,[0.953,0.955]$ & $0.970\,[0.970,0.970]$ \\
\bottomrule
\end{tabular}
\end{table}

%% file: sections/fig_panel.tex
\begin{figure}[t]
\centering
\includegraphics[width=\linewidth]{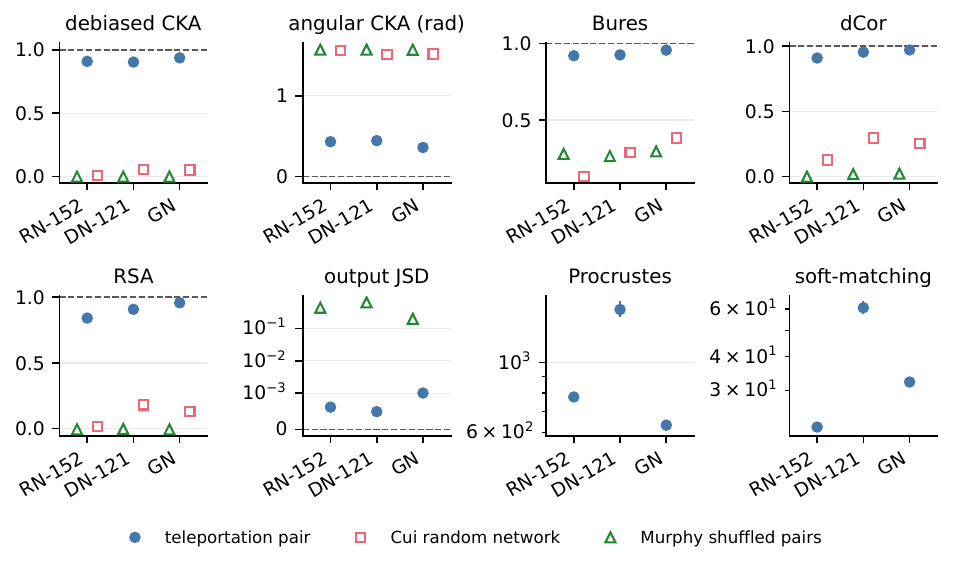}
\caption{On every bounded measure the teleportation pairs sit between the identity line and
the controls, and on none of them at the identity: no measure of the panel absorbs a
transformation that changes nothing about the function. The similarity panel of
Table~\ref{tab:s1_table} as a picture: each panel is one
measure, each column one architecture. Filled dots are the mean over $T=50$ neural
teleportations of the value between the penultimate features $h_W(x)$ and their teleported
image $h_{\tilde W}(x)$ on $N=25{,}000$ images, with the 5th--95th percentile range of the
$50$ draws as a bar (narrower than the marker on most panels). The dashed line is the value
each measure returns for identical representations ($1$ for the bounded similarities, $0$
for angular CKA, output JSD and the two distances). Hollow squares are the Cui random-network
control and hollow triangles the Murphy shuffled-pair control of
Table~\ref{tab:controls-full}; they are omitted for Procrustes and soft-matching, whose
control values were computed on $n=2{,}048$ inputs and are not comparable in magnitude to
the panel.}
\label{fig:panel-teleportation}
\end{figure}

%% file: sections/study2_coherence.tex
\section{Study 2: the coherence of adversarial germ motion}
\label{sec:study2}
\label{sec:A2}

Theorem~\ref{thm:pyth} splits any knowledge-matrix displacement into a
logit-\emph{visible} part (what the function does differently) and a
logit-\emph{invisible} part (how the germ restructures without moving the
output). This study reads adversarial perturbations through that split. It
asks how much of an attack's matrix motion is visible to the logits, and
whether the answer organizes by attack family. The exact row-sum geometry of
Theorem~\ref{thm:pyth} is what makes the question precise.

\paragraph{Setup.}
Architectures: ResNet-152, DenseNet-121 and GoogLeNet for the
headline coherence, plus the three extra ordering-only raters (ResNet-18,
AlexNet, VGG) for the cross-architecture concordance, six in total. Attacks:
the six-attack suite of Section~\ref{sec:setup}. Inputs: for each
(architecture, attack) cell we generate clean/adversarial \emph{pairs} and
record the matrix and logit displacements $(d_M, d_\Psi)$ per pair; we report
the median over the pairs passing a division guard,
$d_\Psi > 10^{-6}\max_i d_{\Psi,i}$ --- a \emph{per-cell relative} threshold that
discards only pairs on which the attack left the logits numerically
unmoved, not an attack-success criterion. Each cell holds $200$ nominal
pairs, of which $110$--$200$ pass it: FGSM and PGD keep all $200$ on every
network, while CW, DeepFool, APGD and Square keep $149$--$161$ on
ResNet-152, $150$--$152$ on DenseNet-121, $120$--$123$ on GoogLeNet,
$140$--$141$ on ResNet-18, $110$ on AlexNet and $198$--$200$ on VGG, so those
cells should be read as low-$n$. The appendix panel
(Table~\ref{tab:s3_ordering_panel}) recomputes the ordering of those three
networks on the full-scale pair set (a different image set), with the cells matched
on a common image population because one cell was truncated. The coherence statistic is
$A = (d_\Psi/d_M)^2$ (Definition~\ref{def:coherence}), read against two
theorem-given reference lines:
$A = 1$, the exact coherence of any single-pixel within-region
perturbation (Theorem~\ref{thm:within}(ii)), and $A \le d$, the
within-region maximum.

\paragraph{Procedure.}
\begin{enumerate}
\item For each (architecture, attack) generate clean/adversarial pairs.
\item For each pair compute $\M(x)$, $\M(x')$, the matrix displacement
$d_M = \|\M(x')-\M(x)\|_F$, and the logit displacement
$d_\Psi = \|\Netxp-\Netx\|$.
\item Apply the division guard $d_\Psi > 10^{-6}\max_i d_{\Psi,i}$; compute the
per-pair coherence $A = (d_\Psi/d_M)^2$. (The stricter attack-success filter
$d_\Psi \ge 1$, $d_M > 0$ is a far more aggressive cut, used only in the two
analyzes that hold per-pair records; Appendix~\ref{app:population-filtering}.)
\item Report the per-cell median $A$ (Table~\ref{tab:coherence}) and the
per-architecture attack-family ordering by median $d_M/d_\Psi$
(Table~\ref{tab:ordering}; the three-network cross-check is
Table~\ref{tab:s3_ordering_panel}).
\end{enumerate}

\input{tables/table_coherence}
\input{tables/table_ordering}
\input{sections/fig_coherence}
\paragraph{What we measure and the claim it licenses.}
Coherence $A$ is a \emph{geometric descriptor} of how an attack moves the
germ, read against the single-pixel line $A = 1$ and the within-region
cap $A \le d$. We do not tie it to any decision-relevant quantity: it is
not a detectability statistic and is not invariant to input rotation (it
is tied to the pixel basis). The attack-family \emph{ordering} of $A$,
by contrast, is a genuine standalone finding (a concordance across raters)
and does not borrow its credibility from the magnitude of $A$.

\paragraph{Result.}
Three findings. \emph{First}, on the three deep ImageNet networks every attack
family moves the germ with less than a quarter of single-pixel coherence
(median $A\le0.23$): read descriptively against the reference lines,
adversarial perturbations are overwhelmingly function-invisible germ
restructuring, in the precise sense of Theorem~\ref{thm:pyth}.
``Median $A\le0.23$'' is the largest of the $18$ cell medians of those three
networks ($3$ architectures $\times$ $6$ attacks), attained at GoogLeNet/APGD ($0.228$,
Table~\ref{tab:coherence}), and never a median of pooled data
(Appendix~\ref{app:population-filtering}). The phenomenon is not new:
\citet{ghorbani2019fragile} and \citet{dombrowski2019explanations} showed that
gradient$\times$input and integrated-gradients maps move substantially under
perturbations that barely move the output, and located the cause at the walls
between the affine regions of a ReLU network --- the crossing dyads of
Theorem~\ref{thm:anatomy}. The invisible part $Q$ and ``median $A\le0.23$''
are a quantitative, per-class, whole-map, unit-free re-measurement of that
fact; what is new is the exact decomposition and the reference lines, not the
phenomenon.
Figure~\ref{fig:coherence-family} shows the per-pair distribution of $A$ on the
full-scale set with the $200$-pair medians overlaid.

These medians are point estimates without bootstrap CIs
(Appendix~\ref{app:limitations-more}, L12) and are sensitive to the filtering
convention --- on the three VGG cells with per-pair records the stricter
attack-success filter moves the median $d_M/d_\Psi$ by $-12\%$ to $-39\%$ and
swaps VGG's top two, while the family-level grouping is unchanged and the
appendix panel's ranking is identical under all three filtering conventions
--- so we report the tables under one convention and state the sensitivity
(Appendix~\ref{app:population-filtering}; Section~\ref{sec:limitations}, L7).

\emph{Second}, the attack-family ordering is stable across architectures.
Across all six raters it is concordant at Kendall $W=0.921$ (permutation $p<10^{-6}$;
leave-one-out Spearman $0.77$--$0.99$, Table~\ref{tab:ordering}), with the
family-level structure
$\{$DeepFool, CW, Square$\}>$ FGSM $\gg\{$PGD, APGD$\}$ in $d_M/d_\Psi$. What
separates the families is \emph{how far past the decision boundary an attack
pushes}, not any per-unit-displacement property of the germ motion and not the
density of $\delta$. The minimal-norm and early-stopping attacks --- DeepFool,
CW and Square --- halt as soon as the class flips and leave the logits barely
moved; the budget-exhausting iterative attacks --- PGD and APGD --- run to the
$\varepsilon$-ball limit and drive the logits far past it; single-step FGSM sits
between them. On this pair set ordering by $d_\Psi$ alone is more concordant
($W=0.975$); the case for the ratio rests on the full-scale inversion, the
cross-image-set reproducibility and the partial concordance
(Appendix~\ref{app:kendall}). The knowledge-matrix displacement $d_M$ is markedly
\emph{more attack-invariant} than the logit displacement it is divided by: across
the six cell medians of one architecture, $\mathrm{sd}(\log d_M)$ runs
$0.24$--$0.79$ over the six raters while $\mathrm{sd}(\log d_\Psi)$ runs
$0.57$--$1.43$, a $1.8$--$2.9\times$ wider spread in the denominator.
The two gaps differ in size: the top family leads
FGSM by $1.13$--$1.51\times$, while FGSM leads $\max\{$PGD, APGD$\}$ by
$1.57$--$2.09\times$. The appendix panel
(Table~\ref{tab:s3_ordering_panel}), which recomputes the ordering of the three
networks on the full-scale pair set, reproduces this family-level grouping on
all three (Kendall $W=0.97$ over the three, exact permutation
$p=3.1\times10^{-5}$; the same three rows of Table~\ref{tab:ordering} give $0.91$).
We do not read the panel as agreeing \emph{better} than those rows, and its
agreement depends on matching the cells to a common image population --- the
intersection of the six cells' stored pair indices, since the ResNet-152
DeepFool run stopped at $512$ of $1000$ pairs --- within which a seeded
bootstrap over pairs reads DeepFool $\approx$ Square $\approx$ CW $\approx$
FGSM $>$ PGD $>$ APGD on ResNet-152 while resolving every adjacent gap at
$95\%$ on DenseNet-121 and GoogLeNet
(Appendix~\ref{app:population-filtering}).

Two qualifications attach to that sentence --- the unresolved ResNet-152 set
$\{$DeepFool, Square, CW, FGSM$\}$ straddles the family boundary, so on that
architecture only the point estimate orders the families, and the panel's $15$
adjacent-gap tests are uncorrected, so ``every adjacent gap is resolved'' on
DenseNet-121 and GoogLeNet is a nominal-$95\%$ statement --- and both are set
out in Appendix~\ref{app:kendall} (Appendix~\ref{app:limitations-more}, L12).

We read the panel as a cross-check of the family ordering rather than as a
separate per-architecture ranking.

The ordering, which \citet{leblanc2024hidden} read as an amplification effect, is real, and it
is an ordering of \emph{invisible fractions}.

\emph{Third}, the mechanism: the
smooth$+$crossing anatomy (Theorem~\ref{thm:anatomy}) predicts that
crossing more region walls reduces coherence
(Conjecture~\ref{conj:mechanism}). A per-pair pilot on VGG
(Table~\ref{tab:mechanism-pilot}) finds the predicted sign once two
disciplines are applied: attack-failure pairs must be excluded ($29$--$33\%$
of the $200$ stored pairs --- the source of the unbounded cells of
Table~\ref{tab:coherence-max}, in which some pair has $d_M=0$ exactly: same-region
mathematics, not numerical underflow), and
perturbation size must be controlled. Writing $H$ for the endpoint mask-Hamming
distance (ReLU flips plus max-pool argmax mismatches), the partial rank
correlation of $H$ with coherence given $d_\Psi$ is negative on all three attacks
tested: $\rho_S(H,A\mid d_\Psi)=-0.343$ (APGD), $-0.304$ (DeepFool), $-0.341$
(Square), with seeded percentile bootstrap $95\%$ intervals $[-0.497,-0.177]$,
$[-0.450,-0.120]$ and $[-0.474,-0.187]$ ($B=10{,}000$, seed $0$;
\texttt{scripts/vgg\_mechanism\_pilot.py}) --- all three excluding zero, though the
three attacks share the same $200$ base images, so the intervals are not
independent of one another.
The mechanism therefore remains a conjecture (Conjecture~\ref{conj:mechanism}),
supported but not established: this is a VGG pilot, on a network outside the
three of Section~\ref{sec:setup}, with $d_\Psi$
as a proxy for $\|\delta\|$, and the two refinements of its reading, the direct
statistic $\rho_S(H,d_M\mid d_\Psi)$ and its caveats are recorded in
Appendix~\ref{app:mechanism-pilot}.

\paragraph{Kendall's \texorpdfstring{$W$}{W}.}
The concordance $W$ is worked through on this panel --- the rank matrix, the
column sums, the tie correction, the Monte-Carlo and exact permutation nulls,
and why a high $W$ says only that the ordering is architecture-independent ---
in Appendix~\ref{app:kendall}.

\paragraph{The ratio and its denominator.}
The obvious objection --- that the ordering of $d_M/d_\Psi$ is a shadow of its
denominator's --- largely lands on the six-rater $200$-pair panel, where
ordering the same $36$ cells by median $d_\Psi$ alone gives $W=0.9746$ against
the ratio's $0.9206$, but it inverts at full scale: on the three networks at full scale the
ratio is the \emph{most} concordant of the three ($W(d_M/d_\Psi)=0.9746$
against $W(d_\Psi)=0.8222$ and $W(d_M)=0.7968$), and a statistic that was
merely a function of its denominator could not be more concordant across
raters than that denominator itself is. The ratio ordering is also the one
that reproduces across the two image sets, and its concordance survives
partialling the denominator out (partial $W=0.8857$ over the four full-budget
raters, with the residual ordering led by Square); the full analysis is in
Appendix~\ref{app:kendall}.

\paragraph{Reading.}
The coherence medians describe \emph{where} adversarial germ motion sits
relative to the single-pixel line; they do not establish that $A$ predicts
any outcome, and they lack CIs. The attack-family ordering is the robust
finding here ($W=0.921$ over six raters), with the population-matched
full-scale appendix panel reproducing the same family-level grouping on all
three of its architectures \emph{in point estimate}. What the appendix panel's
confidence intervals do \emph{not} support on ResNet-152 is any ranking among
$\{$DeepFool, Square, CW, FGSM$\}$ --- a set that crosses the family boundary,
since FGSM is not in the leading family. On that architecture, therefore, the
family grouping is unresolved at $95\%$ and only the point estimate orders it;
the grouping is resolved on DenseNet-121 and GoogLeNet, at an uncorrected
nominal $95\%$ across $15$ simultaneous gap tests. What the ordering is
\emph{not} is a claim that the ratio beats its own denominator as a descriptor:
on the $200$-pair set $d_\Psi$ alone is more concordant, and the case for the
ratio rests on the full-scale inversion, the cross-image-set reproducibility and
the partial concordance reported in Appendix~\ref{app:kendall}. The crossing
mechanism that would \emph{explain} the ordering is supported only by an
underpowered, single-architecture VGG pilot whose residual
$\|\delta\|$ confound points the same way as the predicted effect, and stays a
conjecture.

\paragraph{What the within-region theorem does not explain.}
A density reading of the ordering --- that Theorem~\ref{thm:within}(iii) ranks
attacks by the sparsity of $\delta$ --- is tempting and wrong: the theorem
places sparse perturbations at the \emph{top} of the scale ($A=1$ exactly for
a $1$-sparse within-region perturbation) and its $k$-sparse bound is vacuous at
$k=d$, while FGSM, PGD and APGD run at the identical $\varepsilon=8/255$ and
still spread $3$--$4\times$ in $A$ on ResNet-152. Appendix~\ref{app:kendall}
records why; Theorem~\ref{thm:within} is used in this study only as the pair of
reference lines $A=1$ and $A\le d$ against which the magnitudes are read.

%% file: tables/table_coherence.tex
\begin{table}[t]
\centering\small
\begin{tabular}{lcccccc}
\toprule
attack & ResNet-152 & DenseNet-121 & GoogLeNet & ResNet-18 & AlexNet & VGG \\
\midrule
FGSM & 0.0303 & 0.0606 & 0.05 & 0.0932 & 0.453 & 0.191 \\
PGD & 0.0998 & 0.216 & 0.203 & 0.338 & 1.72 & 1.25 \\
CW & 0.0239 & 0.0272 & 0.037 & 0.0547 & 0.168 & 0.0619 \\
DeepFool & 0.0164 & 0.01 & 0.0108 & 0.015 & 0.099 & 0.0255 \\
APGD & 0.113 & 0.192 & 0.228 & 0.229 & 1.28 & 0.837 \\
Square & 0.0146 & 0.0273 & 0.0183 & 0.0373 & 0.228 & 0.0834 \\
\bottomrule
\end{tabular}
\caption{On the three deep ImageNet networks every attack family moves the germ at median coherence
$A\le0.23$ --- less coherent than a single within-region pixel --- with the iterative-PGD family
consistently the most coherent. Median coherence $A=(d_\Psi/d_M)^2$ of adversarial
knowledge-matrix motion ($d{+}1=150{,}529$; $\rho_{\mathrm{vis}}=A/(d{+}1)$); $A=1$ is the
single-pixel coherence line and $A\le d$ the within-region cap. On the two default-budget
raters the PGD family crosses the one-pixel line (AlexNet: PGD $1.72$, APGD $1.28$;
VGG: PGD $1.25$).
\emph{Caveat (median interpolation):} the per-pair median is exact only for odd valid-pair counts; only
$9/36$ cells are odd-count/exact, so even-count cells (the DenseNet-121 column in particular) are
interpolation-approximate. \emph{Caveat (cell sizes):} each cell holds $200$ nominal pairs, of which $110$--$200$
are valid; FGSM and PGD keep all $200$, while the CW, DeepFool, APGD and Square cells keep $149$--$161$
(ResNet-152), $150$--$152$ (DenseNet-121), $120$--$123$ (GoogLeNet), $140$--$141$ (ResNet-18), $110$ (AlexNet)
and $198$--$200$ (VGG), so read those cells as low-$n$.
Each entry is a median over the valid pairs of \emph{that one cell}; nothing here is
pooled across attacks or across architectures. Accordingly the ``$A\le0.23$'' headline in
Section~\ref{sec:study2} is the largest of the $18$ cell medians of the three networks (GoogLeNet/APGD, $0.228$),
not the median of a pooled population.}
\label{tab:coherence}
\end{table}

%% file: tables/table_ordering.tex
\begin{table}[t]
\centering\small
\begin{tabular}{llc}
\toprule
architecture & ordering by median $d_M/d_\Psi$ (desc.) & LOO Spearman \\
\midrule
ResNet-152 & Square $>$ DeepFool $>$ CW $>$ FGSM $>$ PGD $>$ APGD & 0.771 \\
DenseNet-121 & DeepFool $>$ CW $>$ Square $>$ FGSM $>$ APGD $>$ PGD & 0.943 \\
GoogLeNet & DeepFool $>$ Square $>$ CW $>$ FGSM $>$ PGD $>$ APGD & 0.928 \\
ResNet-18 & DeepFool $>$ Square $>$ CW $>$ FGSM $>$ APGD $>$ PGD & 0.986 \\
AlexNet & DeepFool $>$ CW $>$ Square $>$ FGSM $>$ APGD $>$ PGD & 0.943 \\
VGG & DeepFool $>$ CW $>$ Square $>$ FGSM $>$ APGD $>$ PGD & 0.943 \\
\bottomrule
\end{tabular}
\caption{Six architectures rank the attack families in the same order: Kendall $W=0.921$
(6 raters $\times$ 6 attacks; permutation $p<10^{-6}$, $10^6$ Monte-Carlo draws). The last
column is the leave-one-architecture-out Spearman of each architecture's ranking against
the mean consensus rank of the other five (midrank ties). The six raters are the paper's three networks
(ResNet-152, DenseNet-121, GoogLeNet) plus three \emph{ordering-only} cross-check raters (ResNet-18,
AlexNet, VGG) that enter no coherence-magnitude statistic and no cross-architecture distance. The stable
structure is family-level: $\{$DeepFool, CW, Square$\}>$ FGSM $\gg\{$PGD, APGD$\}$ in $d_M/d_\Psi$. The two
gaps are of different sizes: the top family leads FGSM by $1.13$--$1.51\times$ across the six raters, while
FGSM leads $\max\{$PGD, APGD$\}$ by $1.57$--$2.09\times$. \emph{What separates the families is the
denominator.} Ordering the same $36$ cells by median logit displacement $d_\Psi$ alone is
\emph{more} concordant than the ratio ordering ($W=0.975$ against $0.921$), and on four of the
six raters (DenseNet-121, ResNet-18, AlexNet, VGG) the ratio ordering is the exact reversal of
the $d_\Psi$ ordering (rank Spearman $-1.00$; $-0.94$ on GoogLeNet, $-0.77$ on ResNet-152). The
families are separated by how far past the decision boundary an attack pushes --- minimal-%
perturbation attacks stop at it, large-step attacks drive well beyond it --- and not by any
per-unit-displacement property of the germ motion. Attack budgets are those of
Section~\ref{sec:setup}: the overrides listed there apply to ResNet-152, DenseNet-121,
GoogLeNet and ResNet-18, whereas AlexNet and VGG ran the \texttt{torchattacks}~3.5.1 defaults
throughout. Six raters do not settle what that asymmetry costs: the two default-budget raters
(AlexNet, VGG) produce rank rows \emph{identical} to DenseNet-121's, which ran the overrides,
while the most deviant rater is ResNet-152 (LOO $0.771$) --- the one rater with a DeepFool
budget of its own (\texttt{steps}$=100$, against $200$ on the other three override
architectures). Budget-insensitivity and a budget-driven deviation are both consistent with
$m=6$; we state the fact and do not adjudicate it.
Restricting the concordance to the four raters that ran the override budgets (ResNet-152, DenseNet-121, GoogLeNet, ResNet-18) gives $W=0.92$ ($p=3.8\times10^{-6}$, exact permutation test over the $(6!)^{3}$ null), the six-rater value to two decimals, so dropping the two default-budget raters leaves the concordance where it is.
The restriction of this concordance to the three networks of Section~\ref{sec:setup}
is $W=0.91$ ($p=4.9\times10^{-4}$, exact permutation test). The appendix panel
(Table~\ref{tab:s3_ordering_panel}) recomputes the three networks' ordering on the full-scale reduce, on
a population-matched pair set and a different image set, and reproduces the family-level grouping on all
three architectures ($W=0.97$).}
\label{tab:ordering}
\end{table}

%% file: sections/fig_coherence.tex
\begin{figure}[t]
\centering
\includegraphics[width=\linewidth]{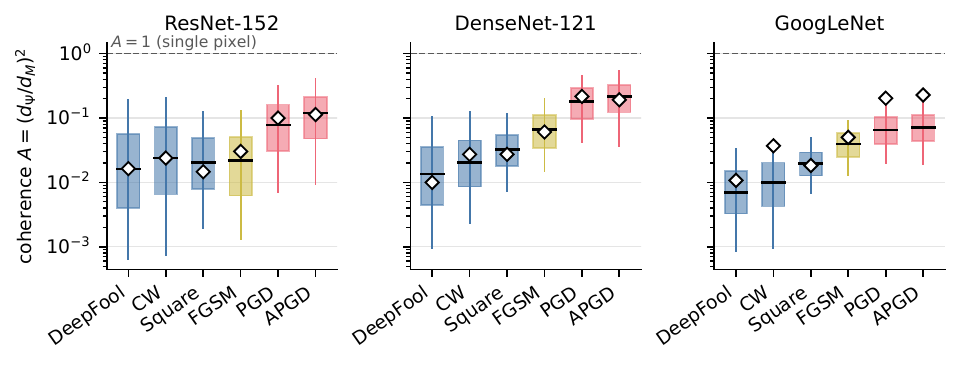}
\caption{Every cell sits below the single-pixel line $A=1$, and the iterative-PGD family is
the most coherent on all three architectures. Coherence $A=(d_\Psi/d_M)^2$ of adversarial
knowledge-matrix motion, per attack family and architecture, on the full-scale pair set
(boxes: median and interquartile range; whiskers: 5th--95th percentiles; $n=449$--$1000$
pairs per cell after the attack-success filter $d_\Psi\ge1$, $d_M>0$). Colors group the
attack families of the ordering: minimal-norm attacks (DeepFool, CW, Square), FGSM, and the
iterative-PGD family (PGD, APGD). The dashed line is the single-pixel reference $A=1$ of
Theorem~\ref{thm:within}. Hollow diamonds are the $200$-pair headline medians of Table~\ref{tab:coherence}, computed on
different images, with the library-default attack budgets and the division-guard convention;
the two sets agree on the family ordering and differ most on GoogLeNet's PGD family, whose
$200$-pair cells are the smallest ($120$--$123$ valid pairs) and whose full-scale generator
uses different PGD and APGD budgets (Section~\ref{sec:setup}).}
\label{fig:coherence-family}
\end{figure}

%% file: sections/study3_cross_arch.tex
\section{Study 3: alignment-free cross-architecture comparison}
\label{sec:crossarch}
\label{sec:A3}

Studies~1 and~2 establish knowledge-matrix invariance under
\emph{within-architecture} isomorphisms and the visible/invisible
decomposition of adversarial-pair displacement \emph{within} a single
architecture. Study~3 turns to the one standalone capability that the function
determination of Theorem~\ref{thm:maxinv} (at almost every $x$) demonstrably
buys: \emph{alignment-free} cross-architecture comparison. The knowledge
matrix is built on the input and label spaces alone (Equation~\ref{eq:km-def}),
and under (LCS) at regular inputs it is determined by the realized function
germ rather than by any particular parameterization (Theorem~\ref{thm:germ}).
The same construction therefore returns an object of \emph{the same shape} for
any feedforward network on a fixed input and label space, and a
knowledge-matrix distance between two architectures is well defined with no
correspondence to learn. The penultimate features carry no such structure: a
comparison between two architectures' features is a statement about a whole
sample population, not about one input. What we demonstrate is a capability,
not a claim that the knowledge-matrix number is a better quality ranking than
the penultimate ones. The three penultimate widths differ
($D_{\text{RN-152}} = 2048$, $D_{\text{DN-121}} = D_{\text{GN}} = 1024$) and
the panel is computed on them as they stand. Why its measures return
population-level scores in their own metric spaces, while the knowledge-matrix
distance is a per-sample quantity --- the uniform
shape, the exact row-sum accounting and the germ-stabilizer invariance of
Section~\ref{sec:germ}, which the gradient$\times$input data of
Theorem~\ref{thm:weff} share --- is set out in
Appendix~\ref{app:crossarch-detail}. The knowledge-matrix distance is one
Frobenius distance per input, reported RMS-per-coordinate throughout
(Section~\ref{sec:setup}), whose logit-visible component is exactly
$\|\Psi_A(x)-\Psi_B(x)\|_2/\sqrt{d{+}1}$ in Frobenius norm, with
$\Psi_A,\Psi_B$ the two network functions (Theorem~\ref{thm:pyth}); neither
its raw nor its RMS value is a logit distance. This study places the
knowledge-matrix distance beside the similarity measures on the same
architecture pairs, not to argue that one number is better than another, but
to show that it is a per-sample quantity where the others are
population-level scores in their own metric spaces.

\paragraph{Setup.}
We compare the three pretrained ImageNet networks of this paper pairwise ---
ResNet-152 $\leftrightarrow$ DenseNet-121, ResNet-152 $\leftrightarrow$
GoogLeNet, and DenseNet-121 $\leftrightarrow$ GoogLeNet --- on the same
$N = 25{,}000$ ImageNet-validation samples (the first $25{,}000$ by sorted
filename, the same sample set used in Study~1; Study~2 draws its adversarial
pairs from different, smaller sets --- see Section~\ref{sec:setup}).
For each architecture pair we report six penultimate similarity measures
--- debiased linear CKA, angular CKA, Bures similarity, soft-matching
distance, RSA-Spearman and distance correlation --- and one functional
baseline, square-root output JSD on the logits; two honest omissions
(entropic Gromov--Wasserstein, whose solver returned a degenerate plan on
every cross-architecture pair as it did in Study~1, with the settings and the
likely cause recorded in Section~\ref{sec:study1c}, and cross-architecture
Procrustes, omitted because our chunked pipeline's per-chunk PCA target
dimension collapses to the chunk sample count --- the shape distance itself
is defined for unequal widths by zero-padding the narrower representation
\citep{williams2021generalized} and needs no PCA, so this is a pipeline
limitation, not a property of the pairs --- and which is therefore reported
within-architecture in Study~1 only); and the per-sample knowledge-matrix
Frobenius distance $\|\M_{W_A}(x) - \M_{W_B}(x)\|_F$.
We do not apply CKA, Procrustes or the other measures to knowledge matrices:
the canonical metric on them is the Frobenius distance itself, whose
displacement decomposes exactly via Theorem~\ref{thm:pyth}, and applying
rotation-invariant similarity measures to a representation that is already
invariant, at regular inputs, under the germ stabilizer at $x$ (which contains
the global function-stabilizer) would weaken, not strengthen, the comparison.
The random-network control of~\citet{cui2022deconfounded} and the
shuffled-pair control of~\citet{murphy2024debiased} that accompany the panel
in Section~\ref{sec:study1c} were not run in this reduce
(Section~\ref{sec:limitations}, L2).

\paragraph{Procedure.}
\begin{enumerate}
\item For each architecture pair and each of the $N = 25{,}000$ samples,
compute the per-sample knowledge matrices $\M_{W_A}(x)$, $\M_{W_B}(x)$
(uniform $1000 \times 150{,}529$) and the penultimate features
$h_{W_A}(x)$, $h_{W_B}(x)$ (unequal dimension).
\item Compute the knowledge-matrix Frobenius distance per sample directly
(no alignment), and average to a per-pair RMS-per-coordinate distance.
\item Compute the six penultimate similarity measures and the functional
baseline directly on the unequal-dimension features --- no PCA, no alignment
--- recording Gromov--Wasserstein and cross-architecture Procrustes as honest
omissions (Table~\ref{tab:s2_table}).
\item Compare the orderings the different measures induce over the three
pairs.
\end{enumerate}

\paragraph{What we measure and the claim it licenses.}
The per-pair distances and the orderings they induce. The claim is one of
\emph{availability}: the knowledge-matrix distance is a per-sample quantity,
one Frobenius distance per input, whereas each
penultimate measure is a population-level score in its own metric space, and
the penultimate measures, belonging to different invariance classes, induce
different orderings. This is not
a claim that the knowledge-matrix number is a better measure of
architectural similarity.

\paragraph{Result.}
\input{tables/s2_table}

The knowledge-matrix Frobenius distance (RMS-per-coordinate,
Table~\ref{tab:s2_table}) is available directly for all three pairs and
ranges from $0.0107$ (ResNet-152 $\leftrightarrow$ GoogLeNet, the closest
pair) to $0.0180$ (DenseNet-121 $\leftrightarrow$ GoogLeNet, the farthest),
with ResNet-152 $\leftrightarrow$ DenseNet-121 at $0.0169$. Per sample the
distance is right-skewed, the mean exceeding the median on every pair: median
$0.0146$ with $5$th--$95$th percentiles $0.0083$--$0.0328$ on ResNet-152
$\leftrightarrow$ DenseNet-121, median $0.0096$ with $5$th--$95$th percentiles
$0.0054$--$0.0195$ on ResNet-152 $\leftrightarrow$ GoogLeNet, and median
$0.0162$ with $5$th--$95$th percentiles $0.0092$--$0.0327$ on DenseNet-121
$\leftrightarrow$ GoogLeNet (interquartile ranges in the caption of
Table~\ref{tab:s2_table}) --- the first per-sample numbers of this study,
available because the distance is one number per input, where the panel
measures, being population-level scores, have no per-sample counterpart. The
penultimate measures also separate the pairs ---
soft-matching distance from $63.6$ to $124.0$, RSA-Spearman from $0.33$ to
$0.52$, square-root output JSD from $0.34$ to $0.39$ --- but each lives in
its own metric space and the orderings they induce are themselves
measure-dependent: the knowledge-matrix Frobenius distance and soft-matching \emph{agree} that
ResNet-152 $\leftrightarrow$ GoogLeNet is the closest pair, whereas
RSA-Spearman ranks it the \emph{least} similar of the three. The
penultimate measures, which belong to different invariance classes, induce
different orderings on the three pairs, as the literature leads one to
expect \citep{ding2021grounding,klabunde2025resi}; the knowledge-matrix
column is one Frobenius distance and induces one ordering --- a statement
about the metric chosen, not about its quality.

We do not read the agreement between two of these columns as corroboration:
three architectures give three pairs (zero residual degrees of freedom, so no
interval is put on the knowledge-matrix column); soft-matching is
unnormalized while DenseNet-121's penultimate representation carries
$3.5$--$3.8\times$ the energy of the other two, so that dividing each pair's
soft-matching distance by $\sqrt{\|\tilde X\|_F^2+\|\tilde Y\|_F^2}$ keeps the
closest pair but swaps the top two; and the knowledge-matrix column has the
same exposure, which the reduce does not let us quantify
(Section~\ref{sec:limitations}, L6, and Appendix~\ref{app:limitations-more},
L10; the numbers are in Appendix~\ref{app:crossarch-detail}). What no common scale affects is that the
penultimate measures disagree among themselves, since RSA-Spearman is a rank
statistic.

The CKA family (debiased CKA, angular CKA, Bures, distance correlation;
Table~\ref{tab:s2_table}) is likewise computed directly on the
unequal-dimension features. Gromov-Wasserstein and cross-architecture
Procrustes are omitted as honest negatives (the entropic solver returned a
degenerate plan on all three pairs, as in Study~1; a per-chunk PCA target
dimension that collapses to the chunk sample count, a limitation of our
pipeline and not a property of the pairs). The Cui/Murphy controls were
not run in \emph{this} reduce and are pending; what the separate control reduce
of Section~\ref{sec:study1c} does establish already constrains one column here.
Bures similarity does not return to zero under either control ---
$0.264$--$0.294$ on shuffled pairs, and $0.383$ for a randomly-initialized
GoogLeNet against its trained counterpart (Table~\ref{tab:controls-full}) ---
so the cross-architecture Bures values of $0.478$--$0.595$ sit within
$1.6$--$2.3\times$ of a null and within $1.2$--$1.6\times$ of an untrained
network; the shuffled null is a fidelity between two positive semidefinite
kernels and depends on $n$ as well as on the spectrum, and it was computed at
$n=2{,}048$ against panel values at $N=25{,}000$, so the ratio is indicative
only. We read no architectural conclusion from that column.

\paragraph{Reading.}
The result demonstrates a \emph{capability}: a single cross-architecture
distance available with no alignment, where each penultimate measure lives
in its own metric space and where measures of different invariance classes
order the three pairs differently. It does \emph{not} show that the knowledge-matrix number is a
better measure of architectural similarity --- there is no ground-truth
ordering to be better against, and the supporting Cui/Murphy controls are
still pending. What is licensed is availability, not a quality ranking.

\paragraph{Cross-architecture comparison without alignment.}
The per-pair knowledge-matrix Frobenius distances --- $0.0169$ (RN--DN),
$0.0107$ (RN--GN) and $0.0180$ (DN--GN) --- are computed directly on the
uniform $1000\times150{,}529$ matrices with no PCA projection, dimension
matching or learned alignment (the dimension-matched block of the reduce is
empty for all three pairs, and no penultimate \emph{distance} $d_h$ is reported
across architectures at all), illustrating that the knowledge matrix supports
an alignment-free cross-architecture metric where the penultimate features do
not; the same displacement is subject to the exact accounting of
Theorem~\ref{thm:pyth}, which we read as descriptive geometry and not as a
quality ranking (Appendix~\ref{app:crossarch-detail}).

Whether independently trained networks converge to the same
knowledge matrix is a question about populations of training runs that this
paper does not ask; this study only establishes that the comparison is
well-posed.

%% file: tables/s2_table.tex
\begin{table}[h]
\centering
\caption{The knowledge-matrix distance is available for all three architecture pairs with no
alignment step, while the panel measures return population-level scores and two of them are
unavailable. Alignment-free cross-architecture comparison (Study~3): KM Frobenius distance
(RMS-per-coordinate, mean over $N=25{,}000$ ImageNet-val samples) is computed directly between
architectures and is a distance \emph{per sample}, whose logit-visible component is exactly
the logit displacement divided by $\sqrt{d{+}1}$ (Theorem~\ref{thm:pyth}); the panel measures
beside it return population-level similarity scores in their own metric spaces. GW omitted
(the entropic solver returned a degenerate plan on all pairs; Section~\ref{sec:study1c});
cross-architecture Procrustes omitted (a pipeline limitation, see text). Cui/Murphy controls pending a
controls-enabled reduce. Per sample, the knowledge-matrix distance has median $0.0146$
(interquartile range $0.0113$--$0.0200$, $5$th--$95$th percentiles $0.0083$--$0.0328$) on RN--DN;
$0.0096$ (interquartile range $0.0074$--$0.0126$, $5$th--$95$th percentiles $0.0054$--$0.0195$) on
RN--GN; $0.0162$ (interquartile range $0.0126$--$0.0214$, $5$th--$95$th percentiles
$0.0092$--$0.0327$) on DN--GN ($N=25{,}000$ samples each; the mean exceeds the median on every
pair). \emph{Read the two unnormalized columns with care.} KM
Frobenius$_{\mathrm{RMS}}$ and soft-matching carry the units of the representations they
compare, and the three architectures are not on a common scale: DenseNet-121's centered
penultimate Gram carries $3.5$--$3.8\times$ the squared Frobenius energy of the other two
($2.91\times10^{6}$, $1.10\times10^{7}$, $3.18\times10^{6}$ for ResNet-152, DenseNet-121,
GoogLeNet), so on both of those columns the two DenseNet-containing pairs are the two largest
and the table partly orders the pairs by whether DenseNet takes part. Dividing soft-matching by
$\sqrt{\|\tilde X\|_F^2+\|\tilde Y\|_F^2}$ leaves the order intact ($0.0332$, $0.0289$,
$0.0258$ for RN--DN, DN--GN, RN--GN) but shrinks the spread: RN--GN rises from $51\%$ to $78\%$
of RN--DN. We give no confidence intervals: three pairs generated by three architectures leave
no residual degrees of freedom, and an interval would decorate the confound rather than
address it. Read the row as a demonstration that the comparison is \emph{computable} without
alignment, not as a ranking of architectural similarity.}
\label{tab:s2_table}
\small\setlength{\tabcolsep}{4pt}
\begin{tabular}{lcccccccc}
\toprule
Pair & KM Frob$_{\mathrm{RMS}}$ & deb.\ CKA & ang.\ CKA & Bures & soft-match & RSA & out-JSD & dCor \\
\midrule
RN--DN & $0.0169$ & $0.442$ & $1.113$ & $0.488$ & $123.98$ & $0.419$ & $0.342$ & $0.661$ \\
RN--GN & $0.0107$ & $0.358$ & $1.205$ & $0.478$ & $63.63$  & $0.334$ & $0.388$ & $0.594$ \\
DN--GN & $0.0180$ & $0.539$ & $1.002$ & $0.595$ & $109.11$ & $0.524$ & $0.354$ & $0.751$ \\
\bottomrule
\end{tabular}
\end{table}

%% file: sections/study4_negatives.tex
\section{Honest negatives}
\label{sec:negatives}
\label{sec:A4}

The studies above establish what the knowledge matrix \emph{is}: a
function-determined object with an exact displacement accounting. This section
reports two places where it does \emph{not} buy what one might hope. The scope of each
negative is narrow. Knowledge matrices are insufficient \emph{as invariants for
the specific tasks tested here}; that is not the claim that they are worse for
every downstream use.

\subsection{Adversarial detection: penultimate features win the bake-off}
\label{sec:study3-bakeoff}

\citet{leblanc2024hidden} proposed the knowledge matrix as the input to
an adversarial-example detector, and that claim has to be settled against
standard baselines. The bake-off that settles it is a $6$-detector $\times$
$3$-representation grid on \emph{one} configuration --- AlexNet trained on
CIFAR-10, evaluated against $16$ \texttt{torchattacks} attack families, and
not the three ImageNet networks of Sections~\ref{sec:study1}--\ref{sec:crossarch} ---
with the penultimate features, the concatenation of \emph{all} hidden-layer
activations and the knowledge matrix as the three arms, six off-the-shelf
detectors (Mahalanobis, $k$-NN, KDE, GMM, one-class SVM, Isolation Forest)
fitted on clean samples and scored by AUROC averaged over the $16$ attacks, and
the multi-layer Mahalanobis detector of \citet{lee2018simple} as an external
baseline. \emph{Penultimate features win $5$ of the $6$ detector
configurations} (Table~\ref{tab:bakeoff}, Appendix~\ref{app:negatives-bakeoff});
the sixth is not a knowledge-matrix win in any useful sense, one-class SVM
scoring $0.550$ for the knowledge matrix against $0.512$ for the penultimate
features, both close to the $0.5$ chance line on the one detector
that fails for every representation. Broken out per attack, taking each
representation's best of the six detectors, the penultimate features beat the
knowledge matrix on \emph{all $16$} attacks (average $0.918$ against $0.798$)
and the Lee et al.\ baseline beats both at $0.937$; an SVD rank ablation on the
Mahalanobis detector has the penultimate features ahead at \emph{every} rank
from $16$ to $512$ ($0.925$ against $0.817$ at rank $16$), with the knowledge
matrix getting monotonically \emph{worse} as more rank is restored
($0.817\to0.590$). The adversarial-detection claim of \citet{leblanc2024hidden} is therefore
\emph{not made here}. The negative is consistent with the theory rather than in
tension with it: detection is a statistical-power question about a particular
discriminator, and nothing in function-determination
(Theorem~\ref{thm:maxinv}) or in the displacement decomposition
(Theorem~\ref{thm:pyth}) predicts that the canonical representation should
also be the most \emph{separable} one for an off-the-shelf detector
(Appendix~\ref{app:negatives-bakeoff}).

\subsection{Single-region LP-counterfactual: \texorpdfstring{$0/54$}{0/54}}
\label{sec:study3-lp}

A second hope was that the linearization $\Weff(x)$ underlying the knowledge
matrix could be turned into a clean, matrix-direction counterfactual: for a
source class $s$ and target class $t\neq s$, the $\ell_1$-minimum input
perturbation $\delta$ that, \emph{within the linearization around $x$}, pushes
the target--source logit gap past a margin $m>0$, subject to a per-coordinate
box that keeps the pixel-space image in $[0,1]$ (Equation~\ref{eq:lp-cf}; its
closed-form solution without the box, the saturation greedy with it, and the
exchange argument for the greedy's $\ell_1$-optimality are in
Appendix~\ref{app:negatives-lp}). It cannot, on pretrained ImageNet networks,
and the failure is structural. For each architecture in $\{$ResNet-152,
DenseNet-121, GoogLeNet$\}$ we run $3$ source images $\times\,3$ source classes
$\times\,2$ target classes ($18$ LPs per architecture, $54$ in all) with margin
$m=0.1$, recording a Boolean \texttt{region\_ok} flag that is True only when
$x$ and $x+\delta$ realize identical activation patterns at every ReLU and
identical argmax indices at every pooling. \emph{On $54$ of $54$ LPs,
\texttt{region\_ok} is False}, and accordingly the actual (non-linearized)
logits at $x+\delta$ never flip ($0/54$); the perturbations sit far outside any
standard budget (every $\|\delta\|_\infty$ at the per-channel saturation cap
$1/\sigma_c\approx4.4$). The perturbation derived from $\Weff(x)$ is large enough to flip
many ReLUs and pooling argmaxes downstream, landing $x+\delta$ in a
\emph{different} linear region governed by a different $\Weff$ --- the geometry
the theory predicts, since Theorem~\ref{thm:within} holds \emph{within} a
region and nothing extends it across walls --- so the negative is a corollary,
not a surprise; one worked instance is Figure~\ref{fig:lp-pathology}
(Appendix~\ref{app:negatives-lp}), and a multi-region or continuation-based
reformulation is left as future work (Section~\ref{sec:limitations}, L3).

%% file: sections/limitations.tex
\section{Limitations}
\label{sec:limitations}

Table~\ref{tab:limitations} lists the concrete limitations that the experiments
here \emph{do not} resolve, each with its consequence and a forward pointer.
L1--L8 run in decreasing order of importance for a reader of this paper. Four
further caveats of measurement convention, L9--L12 --- one on the attack
suite, two on the similarity panel, one on the coherence statistics --- apply
across several studies at once and so belong to none of them; they are
collected in Table~\ref{tab:limitations-more} (Appendix~\ref{app:limitations-more}).

\begingroup\small
\setlength{\LTpre}{6pt}\setlength{\LTpost}{6pt}
\begin{longtable}{@{}>{\raggedright\arraybackslash}p{0.24\linewidth}p{0.73\linewidth}@{}}
\caption{Eight limitations that the experiments of this paper do not resolve, in
decreasing order of importance, each with its consequence and where it is
taken up.}\label{tab:limitations}\\
\toprule
Limitation & Consequence, and forward pointer \\
\midrule
\endfirsthead
\toprule
Limitation & Consequence, and forward pointer \\
\midrule
\endhead
\bottomrule
\endlastfoot
\textbf{(L1)} Both invariance arms are quiver isomorphisms; the cross-architecture claim has no transform-based evidence. &
The permutation and teleportation checks confirm an identity the theory already guarantees (Lemma~\ref{lem:quiver-inv}), whereas the non-automatic cross-architecture invariance of Theorem~\ref{thm:maxinv} is exercised by no same-architecture transform, and Study~3 compares networks computing \emph{different} functions. The missing experiment --- two genuinely different architectures realizing the same function on a neighborhood, by dead-unit insertion, neuron splitting or a width-changing re-encoding --- is the highest-priority follow-up; the permutation check runs on ResNet-152 alone because a post-pool channel permutation does not compose through DenseNet's concatenations or GoogLeNet's Inception branches, a software limit that neural teleportation \citep{armenta2024neural}, acting on a more general change-of-basis structure, does not share. \\
\addlinespace
\textbf{(L2)} The cross-architecture study omits the Cui/Murphy controls. &
Study~3 (Section~\ref{sec:crossarch}) is reported without the random-network control of \citet{cui2022deconfounded} or the shuffled-pair control of \citet{murphy2024debiased} that accompany the same measures in Section~\ref{sec:study1c}; those controls were not run in the cross-architecture reduce. Its numbers establish \emph{availability} --- a per-sample alignment-free distance whose logit-visible component is the logit displacement divided by $\sqrt{d{+}1}$ (Theorem~\ref{thm:pyth}) --- and not a control-calibrated quality ranking; running the two controls is mechanical follow-up. \\
\addlinespace
\textbf{(L3)} Single-region LP-counterfactual fails on ImageNet. &
Section~\ref{sec:negatives} reports \emph{region\_ok} $=0/54$ on the box-constrained LP, which closes off the claim that knowledge matrices admit a clean adversarial-style counterfactual; per-nonlinearity constraints keeping $\delta$ in the source region, or a homotopy/continuation method across regions, are the natural extensions and are not pursued here. The per-LP records behind the count are not preserved in the supplementary material, so the count and the $\|\delta\|$ magnitudes cannot be recomputed from stored data; the surviving record is Figure~\ref{fig:lp-pathology}, with its embedded $\|\delta\|_1=128.0$ and $\|\delta\|_\infty=4.10$. \\
\addlinespace
\textbf{(L4)} Procrustes, Bures and soft-matching benchmarks deferred. &
Section~\ref{sec:study1c} runs the nine-measure panel with the random-network \citep{cui2022deconfounded} and shuffled-pair \citep{murphy2024debiased} controls, and Section~\ref{sec:crossarch} extends it across architectures. Deeper benchmarking of the shape-space metrics \citep{williams2021generalized,harvey2023bures,khosla2024soft} and of Gromov--Wasserstein \citep{memoli2011gromov} --- multi-architecture sweeps including transformer families, the full ReSi protocol \citep{klabunde2025resi}, direct comparison to model stitching \citep{bansal2021revisiting} --- is deferred to future work. \\
\addlinespace
\textbf{(L5)} CNN-only. &
We cover the residual, dense and inception families (ResNet-152, DenseNet-121, GoogLeNet); the quiver-representation framework \citep{armenta2021representation,armenta2022double} applies to attention-based architectures, but the \texttt{knowledgematrix} library does not yet support attention layers. Extension to transformers is library engineering rather than theory, and is left as future work. \\
\addlinespace
\textbf{(L6)} The cross-architecture penultimate metric is confounded by feature scale. &
Wherever penultimate distances are shown across widths ($D=2048$ for ResNet-152, $D=1024$ for DenseNet-121 and GoogLeNet) we use the RMS-per-dimension form $\|h(x)-h(x')\|_2/\sqrt D$ (Section~\ref{sec:setup}), which is not normalized by $\|h(x)\|_2$, so an absolute cross-architecture ranking built on it reflects feature scale and we read none. Within-architecture comparisons and the two implementation checks of Appendix~\ref{app:teleport-checks} --- the permutation arm in absolute $\ell_2$, the teleportation arm through the panel (Procrustes $632$--$1471$, soft-matching $22$--$61$) --- are unaffected. \\
\addlinespace
\textbf{(L7)} The coherence $A$ is a pixel-basis descriptor with underpowered cells. &
$A=(d_\Psi/d_M)^2$ (Definition~\ref{def:coherence}; Theorem~\ref{thm:pyth}) is a geometric descriptor, not a validated statistic: it is not invariant under rotation of the input; its cells hold $110$--$200$ valid pairs (only $9$ of $36$ have odd counts, so even-count medians are interpolation-approximate); on the three VGG cells with per-pair records the stricter attack-success filter moves the median $d_M/d_\Psi$ by $-12\%$ to $-39\%$ and carries VGG/APGD across $A=1$; and the medians carry no interval (L12). The Kendall-$W$ ordering inherits neither the sample-size nor the basis caveat and survives the filter recut at the family level on every rater we can test (the recut reorders VGG's top two, CW ahead of DeepFool); the per-cell magnitudes inherit all four. \\
\addlinespace
\textbf{(L8)} Fixed pretrained networks only. &
All comparisons are between fixed pretrained networks; population statistics across training runs, seed-to-seed baselines and trajectories are out of scope. \\
\end{longtable}
\endgroup

%% file: sections/conclusion.tex
\section{Conclusion}
\label{sec:conclusion}

We have studied the knowledge matrix at one trained network: what determines
it, what it is invariant to, what it determines, and what its geometry
measures. Three studies on pretrained ImageNet networks and two honest
negatives witness one thesis: the knowledge matrix is a \emph{canonical
arrangement} of a trained network's per-sample behavior that hidden
activations cannot supply. Study~1 (Section~\ref{sec:study1}) runs the
nine-measure similarity panel on pairs whose ground truth is exact, a network
and its teleported copy, and none of the eight measures that returned a value
recovers the exact invariance: Procrustes reads $632$--$1471$ and soft-matching $22$--$61$, the
bounded similarities $0.84$--$0.97$, and output JSD alone $\approx0$.
Teleportation is an exact isomorphism, batch normalization included, so the
knowledge matrix does not move; the implementation checks of
Appendix~\ref{app:teleport-checks} confirm this in software on all three
architectures, and confirm on ResNet-152 that a random neuron permutation
leaves the matrix fixed while the penultimate features' adversarial signal is
at most $3.05\times$ their permutation drift and falls below it in $16$ of
$24$ cells --- a statistic in raw $\ell_2$, which a permutation-equivariant
penultimate statistic would not register. Both checks confirm an identity the theory guarantees; the
non-automatic invariance of Theorem~\ref{thm:maxinv}, across architectures, is
exercised by neither (L1). Study~2 (Section~\ref{sec:study2}) applies the
visible/invisible decomposition of Theorem~\ref{thm:pyth} to adversarial pairs
and reads the result as descriptive geometry: on the three networks every
attack family moves the germ at median coherence $A\le0.23$, well below the
single-pixel line $A=1$, and the attack-family ordering --- an ordering set
by how far past the decision boundary each attack pushes, and on that pair
set no more concordant than the logit displacement alone --- is concordant
across six rater architectures (Kendall $W=0.921$), with the population-matched
full-scale panel reproducing the family-level grouping on all three
(Table~\ref{tab:s3_ordering_panel}). Study~3 (Section~\ref{sec:crossarch})
shows that the fixed-shape matrix compares architectures directly, with no
alignment step. The two honest negatives (Section~\ref{sec:negatives}) are the
detector bake-off on AlexNet/CIFAR-10, where penultimate features win $5$ of
the $6$ detectors, all $16$ attacks and every SVD rank from $16$ to $512$, so
that the detection claim of \citet{leblanc2024hidden} is not made here, and
the LP-counterfactual, which fails structurally on the three ImageNet networks
(\texttt{region\_ok} $=0/54$, Limitation~L3).

Hidden activations are gauge-covariant and germ-incomplete
(Theorem~\ref{thm:complete}): they change under network isomorphisms that
preserve the function exactly, and no exact accounting ties their
displacement to the change in the network's output. The knowledge matrix
avoids both. Under (LCS), at every regular input, it is an invariant of the
function germ (Theorem~\ref{thm:maxinv}), and every displacement splits
exactly into a logit-visible and a logit-invisible part whose balance is the
unit-free coherence $A$ (Theorem~\ref{thm:pyth}). For the ReLU networks
studied here these function-level properties are shared with per-class
gradient$\times$input plus an aggregate bias attribution
(Theorem~\ref{thm:weff}). What the knowledge matrix contributes is the
canonical arrangement of that data: one fixed-shape matrix per input, the
contraction of the quiver representation that the network induces on that
input, which makes the alignment-free cross-architecture comparison possible
and which the $C$-VJP computation of Appendix~\ref{app:engineering} exploits.
The contribution is foundational: a canonical per-sample object, an exact
characterization of the transformations that leave it invariant (for a fixed
architecture, the germ stabilizer at $x$, which contains the global function
stabilizer and of which neuron permutation is one special case; across
architectures, equality of the matrices at a regular input with no vanishing
coordinate is equality of the germs), and the alignment-free
cross-architecture comparison that follows. We make no application claim.

The limitations of Section~\ref{sec:limitations} name the follow-ups. Two
genuinely different architectures that realize the same function on a
neighborhood, by dead-unit insertion, neuron splitting or a width-changing
re-encoding, would supply the transform-based evidence for cross-architecture
invariance that no same-architecture symmetry can give (L1). A multi-region or
continuation reformulation of the matrix-direction counterfactual would turn
the structural negative of Section~\ref{sec:negatives} into a positive one
(L3). The application track, federated-learning aggregation in matrix space,
model comparison without CKA-style alignment, and $k$-nearest-neighbor
classification on knowledge matrices, rests on the convexity of the class
regions collected in Appendix~\ref{app:geometric} and is future work, together
with the similarity-measure benchmarking of L4. Sparse autoencoders trained on
knowledge-matrix rows \citep{bricken2023monosemanticity,templeton2024scaling}
and partial row sums in analogy with the logit lens
\citep{nostalgebraist2020logitlens} are separate future work; attention-layer
support in the library would close L5.

Hidden activations are not enough: a trained network's behavior on a sample
is a property of its \emph{function}, and the knowledge matrix, the contraction
of the quiver representation that the network induces on that sample, is, at
almost every input, an invariant of that function's germ with an exact
accounting of how it moves. The three studies and two negatives of this paper
are the first witnesses of that thesis. This paper is the first of a series
that continues \citet{leblanc2024hidden}, which defined the object; here the
object is characterized at one trained network, and the second paper of the
series asks whether independently trained networks learn the same knowledge
matrix, and how the answer depends on width, learning rate and
parameterization. Nothing in this paper depends on that answer.

%% file: sections/statements.tex
\section*{Broader Impact Statement}
This paper characterizes a per-sample representation of fixed, pretrained image
classifiers and studies its invariances and its distance geometry. It trains no
model, collects no data, and proposes no deployment. The adversarial-detection use
proposed for this representation by \citet{leblanc2024hidden} is not claimed here; the paper reports
a detector bake-off as a negative result. The knowledge matrix is an analysis
tool: it exposes no attack capability that per-class gradient$\times$input, to which
it is equivalent for the networks studied (Theorem~\ref{thm:weff}), does not already
provide. We foresee no direct societal risk from this work.

\section*{Reproducibility Statement}
Every number printed in this paper is traced to a stored artifact and, with the
two exceptions stated below, to the script that produces it from that artifact;
both ship in the anonymized supplementary material, whose \texttt{REPRODUCING.md}
maps each printed number to its artifact, its script and its runtime.
The networks are the pretrained \texttt{torchvision} ResNet-152, DenseNet-121 and
GoogLeNet weights named in Section~\ref{sec:setup}, in evaluation mode throughout.
The attack budgets are listed in full in Appendix~\ref{app:setup-details}, including
every departure from the \texttt{torchattacks}~3.5.1 defaults. The theorem-level
statements have numerical witnesses that run on a CPU in seconds
(\nolinkurl{scripts/verify_theory_audit_2026_09_11.py},
\nolinkurl{scripts/verify_lcs_vs_pl.py} and
\nolinkurl{scripts/verify_km_claim_checks.py}), and so do the numerical checks
cited by name in the appendices (\nolinkurl{scripts/verify_float64_rowsum.py},
\nolinkurl{scripts/verify_cca_invariance.py} and
\nolinkurl{scripts/verify_dimension_bias_probe.py}); the teleportation checks
of Appendix~\ref{app:teleport-exact} are reproduced on a CPU in minutes
(\nolinkurl{scripts/verify_teleportation_exactness.py} and
\nolinkurl{scripts/verify_teleportation_km_exactness.py}). The ImageNet-scale
knowledge matrices behind Studies~1--3 were computed on H100 GPUs; the per-pair
and per-cell JSON reductions those runs produced are included, and the ordering
tables, the coherence tables, the mechanism-pilot table, both data figures,
the derived panel numbers and the per-sample Study~3 quantiles are regenerated
from them by
\nolinkurl{scripts/regen_ordering_tables.py},
\nolinkurl{scripts/make_study2_tables.py},
\nolinkurl{scripts/vgg_mechanism_pilot.py},
\nolinkurl{scripts/make_paper_figures.py},
\nolinkurl{scripts/print_panel_derived_numbers.py} and
\nolinkurl{scripts/print_study3_km_quantiles.py}, each of which aborts rather
than print a caption that disagrees with its numbers. The two exceptions are the
honest negatives of Section~\ref{sec:negatives}: the detector bake-off is
reported from the rendered tables of the retired detection pipeline, which ship
in the supplementary material, because the per-detector AUROC files behind them
were not preserved; and the LP-counterfactual is reported from its count and its
surviving figure, because the per-program dumps were not preserved
(Limitation~L3). The code pins the knowledge-matrix library by commit and the
unit tests pass in a fresh virtual environment created from
\texttt{requirements-local.txt}.

%% file: sections/appendix_proofs.tex
\section{Proofs}
\label{app:proofs}

\paragraph{Verification protocol.} Every identity proved below was also
verified numerically by two independent implementations (including
an independent implementation of the segment traversal for
Theorem~\ref{thm:anatomy}); worst-case measured errors are quoted where
informative. The scripts are in the supplementary material: the checks
specific to Section~\ref{sec:germ}'s hypotheses are in
\texttt{scripts/verify\_lcs\_vs\_pl.py} and
\texttt{scripts/verify\_km\_claim\_checks.py}, and
\texttt{scripts/verify\_theory\_audit\_2026\_09\_11.py} builds the knowledge
matrix only through the secant products of \eqref{eq:D-def}--\eqref{eq:km-def}
and exercises every statement of Sections~\ref{sec:germ}--\ref{sec:geometry}
that admits a finite-dimensional witness.

Every proof in this appendix is written the same way: a subsection naming the
statement it proves, the proof itself, and where a numerical check exists a
closing \emph{Measured} clause before the end of the proof.

\subsection{Proof of Proposition~\ref{prop:lcs} (what (LCS) is)}
\label{app:proof-lcs}

\begin{proof}
\emph{(i)} Since $x\in X_{\mathrm{nz}}$, no $z^{(\ell)}_q$ vanishes on a
neighborhood $U$ of $x$ --- shrinking $U$ layer by layer, since each
pre-activation is affine, hence continuous, on $U$ once the layers below it
are --- so $h^{(\ell)}=D^{(\ell)}z^{(\ell)}$ holds exactly on $U$ (it reads
$f(z_q)=\bigl(f(z_q)/z_q\bigr)z_q$); with each $D^{(\ell)}$ constant on $U$
this relation is linear on $U$ (and each pooling layer is the fixed linear
map $S^{(\ell)}(x)$ there), composing the affine pre-activation maps layer
by layer gives an affine map on $U$, and unrolling identifies it with
$\Jwf(x)x'+\cwf(x)$. The hypothesis $x\in X_{\mathrm{nz}}$ is not removable
when $f(0)\neq0$: a unit with $z_q\equiv0$ has $D_{qq}=0$ by the guard and
$h_q=f(0)$, and contributes $f(0)$ to $\Net$ but nothing to $\Jwf x'+\cwf$.

\emph{(iii)}, first, since (ii) uses it. Each of $(0,\infty)$ and $(-\infty,0)$
is connected, so a locally constant function on it is constant: if $f(z)/z$ is
locally constant on $\R\setminus\{0\}$ it equals some $a_+$ throughout
$\{z>0\}$ and some $a_-$ throughout $\{z<0\}$, i.e.\ $f(z)=a_+\max(z,0)+
a_-\min(z,0)$ for $z\neq0$, which is positive homogeneity of degree one on each
half-line. Conversely $f(\lambda z)=\lambda f(z)$ for $\lambda>0$ gives
$f(z)=f(1)z$ on $\{z>0\}$ and $f(z)=-f(-1)z$ on $\{z<0\}$, so $f(z)/z$ is
constant on each half-line. If $f$ is continuous then
$f(0)=\lim_{z\to0^{+}}a_+z=0$ and the formula holds at $z=0$ as well.

\emph{(ii)} Let $E\subset\R$ be discrete --- no accumulation point in $\R$ ---
and let $f(z)/z$ be locally constant on $\R\setminus E$, where we may and do
assume $0\in E$, since $E\cup\{0\}$ is again discrete. Because $E$ has no
accumulation point, $\R\setminus E$ is open and each of its connected
components is an open interval whose endpoints lie in $E\cup\{\pm\infty\}$;
this is the step that fails for a merely null $E$, and it is where the
counterexample recorded in the main text bites. On each such component
$f(z)/z$ is locally constant on a connected set, hence constant, say $=a_I$; so
$f(z)=a_Iz$ there --- every linear piece passes through the origin. Now let
$z_0\in E$ with $z_0\neq0$, and let $I_1=(\,\cdot\,,z_0)$ and
$I_2=(z_0,\,\cdot\,)$ be the two components meeting it, which exist because $E$
is discrete. Continuity at $z_0$ gives $a_{I_1}z_0=f(z_0)=a_{I_2}z_0$, and
$z_0\neq0$ forces $a_{I_1}=a_{I_2}$; the common value is also $f(z_0)/z_0$, so
$f(z)/z$ is in fact locally constant \emph{at} $z_0$ too. Every point of
$E\setminus\{0\}$ is therefore removable, and $f(z)/z$ is locally constant on
$\R\setminus\{0\}$. Part (iii) now gives the displayed form, with $f(0)=0$
following from continuity rather than being assumed, and with $E$ replaceable
by $\{0\}$ (by $\emptyset$ when $a_+=a_-$). The converse is immediate: $f(z)/z$
takes the value $a_+$ on $\{z>0\}$ and $a_-$ on $\{z<0\}$, so a network built
from such an $f$ has every entry of every $D^{(\ell)}$ locally constant at each
input where that unit's pre-activation is nonzero. That the remaining inputs lie
in finitely many affine hyperplanes is Lemma~\ref{lem:null}, whose proof takes
the displayed form of $f$ as an input and does not re-use the present part ---
so there is no circle. Since $f(0)=0$ for a network satisfying (LCS) in the sense fixed after
Remark~\ref{rem:lcs-strict} (continuity and the bridge), $X_{\mathrm{nz}}=\R^d$
and \eqref{eq:km-def} holds at every input.
\emph{Measured:} the null-set version of (ii) is false, and the Cantor witness
of the main text is evaluated at $f(1)=1$, $f(1.5)=2.25$, $f(2)=4$
(\texttt{scripts/verify\_lcs\_vs\_pl.py}); the strictness of (LCS) over
piecewise linearity is measured in Remark~\ref{rem:lcs-strict}.
\end{proof}

\subsection{Preliminaries}

For the reader's convenience, we gather here the vocabulary used in the
statement and the proof of the lemma below. Fix a network whose
activation satisfies the hypothesis of Proposition~\ref{prop:lcs}(ii), so that
by that part $f(z)=a_+\max(z,0)+a_-\min(z,0)$.

The \emph{pattern} at $x$ is the assignment to
each hidden unit $(\ell,q)$ of the side of the origin its pre-activation falls
on, equivalently of which of the two slopes $a_\pm$ its entry of $D^{(\ell)}(x)$
takes. It is a vector of bits, one per hidden unit --- for ReLU exactly the
familiar $0/1$ mask; for a network with max-pooling it also records, for each
pooling window, which entry the selection matrix $S^{(\ell)}(x)$ selects. A
\emph{candidate} pattern is written $\sigma$: any one of the
$2^{\#\mathrm{units}}$ bit assignments (times the finitely many selections),
whether or not some input realizes it. There are finitely many, which is what
makes the union below finite.

A \emph{breakpoint} of $f$ is a point
where its slope changes, i.e.\ where the two linear pieces meet. Under (LCS)
the only breakpoint is $z=0$ (Proposition~\ref{prop:lcs}(ii)): away from $z=0$
the function is linear and $D$ is locally constant, so the only way a unit's
entry of $D$ can change is for its pre-activation to cross zero. ``The
activation breaks at $z=0$'' means exactly this.

The \emph{frozen pre-activations} $z^{(\ell)}_{\sigma,q}$ are obtained as
follows. Fix $\sigma$ and pretend
every unit below layer $\ell$ has the slope $\sigma$ prescribes, regardless of
the input. Every layer below $\ell$ is then a fixed linear map followed by
multiplication by a fixed diagonal, so the composite is affine and
$z^{(\ell)}_{\sigma,q}(x)=\alpha^\top x+\beta$ for some $\alpha\in\R^d$,
$\beta\in\R$. This is a polynomial of degree at most $1$, and it is
\emph{affine}, not homogeneous: $\beta\neq0$ in general, because the biases
contribute to it. It agrees with the network's realized pre-activation exactly
on the inputs whose pattern below $\ell$ is $\sigma$.

The \emph{zero set} of
$z^{(\ell)}_{\sigma,q}$ is $\{x\in\R^d: \alpha^\top x+\beta=0\}$ --- the zeros
of that degree-$1$ polynomial, and nothing to do with higher-degree varieties.
When $\alpha\neq0$ this set is an \emph{affine hyperplane}: a translate of a
linear subspace of dimension $d-1$, i.e.\ $\{x:\alpha^\top x=-\beta\}$. It is
closed and Lebesgue-null. When $\alpha=0$ the set is either empty ($\beta\neq0$)
or all of $\R^d$ ($\beta=0$), which is the case the proof handles separately.

A function $g$ is \emph{locally constant at}
$x$ if there is \emph{some} open $U\ni x$ on which $g$ is constant. Its
\emph{non-local-constancy set} is the set of $x$ where no such $U$ exists. Negating
the quantifiers: $x$ lies in that set iff for \emph{every} open $U\ni x$ there
\emph{exists} $y\in U$ with $g(y)\neq g(x)$ --- the inner quantifier over $y$
is existential, not ``for all''. (Reading it as ``$g(y)\neq g(x)$ for all
$y\in U$'' inverts that quantifier and makes the set look open; it is closed.) We write
$\mathcal B_\ell$ for the non-local-constancy set of the pattern truncated to
layers $1,\dots,\ell$, and $\mathcal B=\mathcal B_L$ for the whole pattern.

\begin{lemma}[Null switching set]\label{lem:null}
Let the network's activation satisfy the hypothesis of
Proposition~\ref{prop:lcs}(ii), $f$ continuous with $f(z)/z$ locally constant
off a discrete set, so that $f(z)=a_+\max(z,0)+a_-\min(z,0)$ and the only
breakpoint of $f$ is $z=0$. Let $\mathcal B$ be the set of inputs at which
some slope diagonal $D^{(\ell)}$ or, for a network with max-pooling, some
selection matrix $S^{(\ell)}$ is not locally constant. Then
$\mathcal B$ is closed and contained in
a finite union of affine hyperplanes; its complement $X_{\mathrm{reg}}$ is
open, dense, and of full Lebesgue measure.
\end{lemma}

\subsection{Proof of Lemma~\ref{lem:null} (null switching set)}
\label{app:proof-null}

\begin{proof}
\emph{$X_{\mathrm{reg}}$ is open, hence $\mathcal B$ is closed.} Let
$x\in X_{\mathrm{reg}}$, so the pattern is constant on some open $U\ni x$. That
same $U$ witnesses local constancy at every one of its points: for $y\in U$ the
set $U$ is an open neighborhood of $y$ on which the pattern is constant, so
$y\in X_{\mathrm{reg}}$. Hence $U\subseteq X_{\mathrm{reg}}$, and
$X_{\mathrm{reg}}$ is open; $\mathcal B$ is its complement and so is closed.
(The point worth noticing is that one does not construct a new neighborhood
for each $y$ --- the witness $U$ is reused unchanged.)

\emph{The hyperplane bound.} Induct over layers. For each pattern $\sigma$ and unit
$(\ell,q)$ let $z^{(\ell)}_{\sigma,q}$ be the affine function of $x$ obtained
by freezing the masks below layer $\ell$ to $\sigma$. Let $\mathcal B_{\ell}$
be the non-local-constancy set of the pattern up to layer $\ell$;
$\mathcal B_0=\emptyset$. If $x\notin\mathcal B_{\ell-1}$, the realized
pattern below $\ell$ is constant on a neighborhood $U$ of $x$, on which the
realized $z^{(\ell)}_q$ equals $z^{(\ell)}_{\sigma,q}$. If
$z^{(\ell)}_{\sigma,q}\equiv0$ then the unit's mask bit is locally constant on
$U$ (an affine function vanishing on an open set vanishes identically, so
there is no partial-vanishing case). Otherwise its zero set is a hyperplane,
off which the bit is locally constant. Hence
$\mathcal B_\ell\subseteq\mathcal B_{\ell-1}\cup\bigcup_{\sigma,q}
\{z^{(\ell)}_{\sigma,q}=0,\ z^{(\ell)}_{\sigma,q}\not\equiv0\}$, a finite
union (patterns that are never realized only enlarge it). For a max-pooling
layer $\ell$ the same step applies to $S^{(\ell)}$: on $U$ the frozen
pre-activations $z^{(\ell)}_{\sigma,a}$ of the entries of one window are
affine, so $S^{(\ell)}$ is locally constant off the tie sets
$\{z^{(\ell)}_{\sigma,a}=z^{(\ell)}_{\sigma,b}\}$ over pairs $a,b$ of entries
in a common window; each is an affine hyperplane when
$z^{(\ell)}_{\sigma,a}-z^{(\ell)}_{\sigma,b}$ is not constant, empty when it
is a nonzero constant, and all of $\R^d$ when the two frozen
pre-activations coincide identically, in which case the lowest-index rule
selects the same entry throughout $U$ and $S^{(\ell)}$ is locally constant
there. So finitely many tie hyperplanes join the union, and nothing else.

\emph{Reading off the three conclusions.} Unwinding the induction to
$\ell=L$ puts $\mathcal B=\mathcal B_L$ inside a finite union of sets of the
form $\{z^{(\ell)}_{\sigma,q}=0\}$ with $z^{(\ell)}_{\sigma,q}\not\equiv0$;
each is an affine hyperplane, and the union is finite because there are
finitely many candidate patterns $\sigma$ and finitely many units $q$. That
gives the containment. A hyperplane has Lebesgue measure zero and a finite
union of null sets is null, so $\lambda(\mathcal B)=0$ and
$X_{\mathrm{reg}}=\R^d\setminus\mathcal B$ has full measure. Density follows
from the containment as well: a finite union of hyperplanes has empty interior,
so every ball meets its complement, and $X_{\mathrm{reg}}$ is dense. Openness
was the first paragraph.
\end{proof}

Observe that the lemma needs (LCS) while the definition does not. The
knowledge matrix is defined for \emph{any} activation through the
quotient diagonal \eqref{eq:D-def}, and its row-sum identity needs only that no
pre-activation vanish ($x\in X_{\mathrm{nz}}$) --- nothing about $f(0)$, and
nothing about differentiability. This lemma is a different matter: it is about the set
where the \emph{pattern} is locally constant, and the proof uses that each
$z^{(\ell)}_{\sigma,q}$ is affine, so its zero set is a hyperplane. Both fail
for a strictly nonlinear smooth activation. There $D^{(\ell)}(x)$ takes a
continuum of values and varies continuously with $x$, so the set of inputs with
a locally constant pattern is not of full measure but generically \emph{empty},
and there is no switching set for the lemma to bound. The lemma is therefore
not a technical convenience one could remove by a better definition of
$D^{(\ell)}$: it exists to support the germ identity
(Theorem~\ref{thm:germ}), and it is the germ reading, not the matrix, that
carries the hypothesis. A direct check makes the boundary concrete. At
$x_0=1$ a single $\tanh$ unit with unit weights and the affine map
$x\mapsto f'(1)\,x+\bigl(f(1)-f'(1)\bigr)$ have the \emph{same} germ,
$\bigl(0.419974,\,0.761594\bigr)$, and the same row sum, yet knowledge matrices
$[\,0.761594\mid0\,]$ and $[\,0.419974\mid0.341620\,]$; repeating the
construction with ReLU returns $[\,1\mid0\,]$ for both, as
Theorem~\ref{thm:germ} requires. So for activations outside the (LCS) class the
matrix is not a function of the germ, and Theorems~\ref{thm:germ},
\ref{thm:maxinv} and~\ref{thm:complete} are statements about the (LCS) case
specifically --- which, by Remark~\ref{rem:lcs-strict}, is strictly narrower
than the piecewise-linear one. Whether
some weaker function-determination survives outside it is open; we do not claim it.

\subsection{Proof of Theorem~\ref{thm:germ} (germ identity)}
\label{app:proof-germ}

\begin{proof}
On a neighborhood $U$ of $x\in X_{\mathrm{reg}}$ every slope diagonal, and
every selection matrix, is constant, so each pooling layer is the fixed
linear map $S^{(\ell)}(x)$ on $U$ and
$h^{(\ell)}=D^{(\ell)}(x)\,z^{(\ell)}$ holds identically on $U$
(under (LCS) the entry is $a_+$ where $z_q>0$ and $a_-$ where $z_q<0$, and the
sign does not change on $U$; for ReLU that reads $z_q>0\mapsto z_q$,
$z_q<0\mapsto0$). Unrolling,
$\Netxp=\Jwf(x)\,x'+\cwf(x)$ on $U$; differentiate at $x$ and solve for
$c$. For the boundary complement: the same pointwise computation shows
$h^{(\ell)}(x')=D^{(\ell)}(x)\,z^{(\ell)}(x')$ for every $x'$ in the closed
region $\overline{S(x)}$ of $x$'s realized pattern, so the affine map
$A_x(x')=Jx'+c$ agrees with $\Net$ on $\overline{S(x)}$; since $x\in S(x)$
always, $\M(x)\mathbf 1=Jx+c=\Netx$ at every $x$. The germ-selection reading of
the guard is ReLU-specific and we state it as such: for ReLU the guard
$0/0\mapsto0$ assigns slope $0$ to a unit exactly on its wall, which is the
one-sided slope taken from the inactive side, so for $v$ in the tangent cone of
$S(x)$ at $x$ one has $\partial^+_v\Net(x)=Jv$. For a general (LCS) activation
with $a_-\neq0$ the guard still assigns $0$, which is neither $a_+$ nor $a_-$;
there the assigned slope is not a one-sided germ of anything, and only the
row-sum half of the previous sentence survives --- it survives for every
activation, since $h_q=a_qz_q$ reads $0=0$ at $z_q=0$.
\emph{Measured:} probe vs.\ masked product $\le 7\times10^{-15}$; vs.\
finite-difference Jacobian $\le 4\times10^{-11}$; row-sum identity at an
engineered exact-boundary point: $\le 5\times10^{-16}$
(\texttt{scripts/verify\_km\_claim\_checks.py}, check F).
\end{proof}

\subsection{Proof of Theorem~\ref{thm:maxinv} (maximal invariance)}
\label{app:proof-maxinv}

\begin{proof}
Immediate from Theorem~\ref{thm:germ}: both matrices equal
$[\,D\Net(x)\,\mathrm{diag}(x)\mid \Netx-D\Net(x)x\,]$, a function of the shared germ
alone. This is the cross-architecture statement, an equality of fibers: two
encodings regular at $x$ with the same germ have the same matrix. The
stabilizer statement is its case for a fixed architecture: a transformation of
the parameters that preserves the germ at $x$ and keeps $x$ regular preserves
$\M(x)$, and every function-preserving transformation preserves the germ at
every $x\in X_{\mathrm{reg}}$ of both realizations. Sharpness at boundaries: the
identity map and $g(x)=\mathrm{ReLU}(x-1)-\mathrm{ReLU}(1-x)+1\equiv x$
realize the same function, yet at the breakpoint $x=1$ the conventions give
$\M_{\mathrm{id}}(1)=[\,1\mid0\,]$ and $\M_g(1)=[\,0\mid1\,]$ (equal row
sums).
\end{proof}

\subsection{Proof of Lemma~\ref{lem:quiver-inv} (quiver-isomorphism invariance) and Corollary~\ref{prop:perm} (permutation scope)}
\label{app:proof-quiver-inv}

\begin{proof}
A quiver isomorphism factors into a per-layer neuron permutation $P_\ell$ and an
invertible-diagonal rescaling $T_\ell$ on each hidden layer; we treat the two
generators and compose.

\emph{Permutations.} Let $P_\ell$ permute hidden layer $\ell$, with
$W^{(1)}\!\to P_1W^{(1)}$, $b^{(\ell)}\!\to P_\ell b^{(\ell)}$,
$W^{(\ell)}\!\to P_\ell W^{(\ell)}P_{\ell-1}^{\top}$,
$W^{(L)}\!\to W^{(L)}P_{L-1}^{\top}$, and neuron-wise activations carried with
the units. By induction $z^{(\ell)}(W^{\pi},x)=P_\ell z^{(\ell)}(W,x)$
for every $x$, so the diagonal secant matrices \eqref{eq:D-def} satisfy
$D^{(\ell)}_{\pi}(x)=P_\ell D^{(\ell)}(x)P_\ell^{\top}$ (entrywise ratios
permute; the $0/0\to0$ guard is applied entrywise, hence equivariantly ---
this is where non-equivariant \emph{vector} activations are excluded). The
slope product and the bias accumulation then telescope:
$W^{(L)}P_{L-1}^{\top}\cdot P_{L-1}D^{(L-1)}P_{L-1}^{\top}\cdot
P_{L-1}W^{(L-1)}P_{L-2}^{\top}\cdots=W^{(L)}D^{(L-1)}W^{(L-1)}\cdots$, so
$\KMarg{W^\pi}{x}=\KMwfx$.

\emph{Rescalings by nonzero factors.} For $T_\ell=\mathrm{diag}(\tau^{(\ell)})$
with every $\tau^{(\ell)}_q\neq0$, conjugate
$W^{(\ell)}\!\to T_\ell W^{(\ell)}T_{\ell-1}^{-1}$,
$b^{(\ell)}\!\to T_\ell b^{(\ell)}$, with the per-neuron activation rescaled to
$f_{\tau}(z)=\tau\,f(z/\tau)$. The sign of $\tau$ is immaterial to
the secant: $z^{(\ell)}\!\to T_\ell z^{(\ell)}$, and the rescaled unit's chord
is
\[
\frac{f_\tau(\tau z)}{\tau z}
=\frac{\tau\,f\bigl((\tau z)/\tau\bigr)}{\tau z}
=\frac{f(z)}{z},
\]
with the guard preserved because $\tau z=0$ iff $z=0$. Hence
$D^{(\ell)}_{\tau}(x)=T_\ell D^{(\ell)}(x)T_\ell^{-1}$ --- the diagonals commute,
so this is just $D^{(\ell)}(x)$ --- and the inserted factors $T_\ell^{-1}T_\ell=I$
telescope by the identical argument, giving $\M(W^\tau\!,f_\tau)(x)=\KMwfx$.
Composing the two generators gives the lemma. What \emph{does} depend on the
sign is which activation the rescaled unit carries: for ReLU and $\tau>0$,
$f_\tau=f$, so the rescaled network is again a ReLU network,
whereas for $\tau<0$, $f_\tau(z)=\tau\max(z/\tau,0)=\min(z,0)$, and the
layer has left the ReLU class. The positive case recovers the ReLU case of
Theorems~4.1--4.2 of \citet{leblanc2024hidden}, and it is that case,
not the lemma, that Theorem~\ref{thm:complete}(iii) uses.

\emph{Permutation scope (Corollary~\ref{prop:perm}).} The corollary is the
permutation case above, so invariance is immediate; the content is the scope.
The telescoping used that $P_\ell$ commutes with $f$ applied
neuron-wise, which requires elementwise (equivariant) activations; a
non-equivariant \emph{vector} activation breaks the equivariance of the secant
guard above, and that is the whole of the restriction. It is worth recording
what is \emph{not} a restriction on it. When $f(0)\neq0$ the row-sum
identity \eqref{eq:km-def} needs $x\in X_{\mathrm{nz}}$, but that requirement
falls on the two networks equally: at an exactly-zero pre-activation both
$\KMwfx$ and $\KMarg{W^\pi}{x}$ fall short of the logits by the same amount,
$P$ does commute with the constant offset $f(0)\mathbf 1$, and the
invariance conclusion is untouched.
\emph{Measured:} a permutation at a bitwise-exact zero pre-activation of a
sigmoid unit gives $\max|\M(W^\pi\!,f)(x)-\KMwfx|=2.2\times10^{-16}$ with an
identical row-sum shortfall of $0.7235$ on both networks and
$\|Pf(0)\mathbf 1-f(0)\mathbf 1\|=0$
(\texttt{scripts/verify\_km\_claim\_checks.py}, check~(B)); a rescaling by
nonzero factors of mixed sign with the carried $f_\tau$ leaves $\M$ and $\Net$
invariant to $\le1.4\times10^{-12}$ on sigmoid, $\tanh$, ReLU and leaky-ReLU
networks, and the permutation case at an exact-zero pre-activation returns
$\M(W^\pi\!,f)(x)=\KMwfx$ to $4\times10^{-17}$ for ReLU and sigmoid with the
same shortfall on both sides
(\texttt{scripts/verify\_theory\_audit\_2026\_09\_11.py}, checks C15--C16).
\end{proof}

\subsection{Proof of Theorem~\ref{thm:complete} (completeness)}
\label{app:proof-complete}

\begin{proof}
\emph{(i)} $J_{:,i}=\M_{:,i}/x_i$ for $x_i\neq0$ and $c=\M_{:,d+1}$.
\emph{(ii)} $\Net=\M\mathbf 1$ pointwise gives determination on whatever set the
field is given; on a full-measure set, density plus continuity of $\Net$ extends
it globally. That is the positive half, and it is all of it. For a
\emph{fixed, known} network the germ at $x$ also extends $\Net$ to
$\overline{S(x)}$ (proof of Theorem~\ref{thm:germ}) --- but that extension is
not recoverable from the field, because the field does not reveal the region
decomposition. Concretely, $\Psi_1=\mathrm{ReLU}(x-1)$ and
$\Psi_2=\mathrm{ReLU}(x)$ have identical fields ($\equiv[\,0\mid0\,]$) on the
open set $(-1,0)$ yet differ at $0.75$, which lies in
$\overline{S_{\Psi_1}}=(-\infty,1]$, the closed region of $\Psi_1$ containing
$(-1,0)$. So the field on an open set determines $\Net$ on that set, and not on
the closure of the region it sits inside --- neither in a candidate network of
one's own choosing nor in the network that actually produced the field. The
statement is therefore sharp as written, in the sense recorded after it in
the main text.
\emph{(iii)} For $\tau_q>0$ per hidden unit with weights conjugated,
$z^{(\ell)}\mapsto T_\ell z^{(\ell)}$ and an (LCS) activation commutes with
positive scalars ($f(\tau z)=\tau f(z)$ for $\tau>0$, which is
Proposition~\ref{prop:lcs}(iii)), so the activation $h_v$ of hidden unit $v$
moves, $h_v\mapsto\tau_v h_v$, while $\Net$ is unchanged. Positivity is used here and is not removable: by
Lemma~\ref{lem:quiver-inv} a negative $\tau$ also leaves $\M$ fixed, but it
changes the activation the unit carries and so leaves the architecture.
\emph{(iv)} We give one construction for each hypothesis. \emph{Free first
layer, hypothesis (a).} Write $J=\sum_q D^{(1)}_{qq}\,U_{:,q}\,w_q^{\top}$ with
$U=W^{(L)}D^{(L-1)}\cdots W^{(2)}$. Since $D\Net(x)=J\neq0$, some first-layer
unit $q$ has $D^{(1)}_{qq}\neq0$ (for ReLU: is active) and $u:=U_{:,q}\neq0$.
Perturb along $v=\mathbf 1_d$: $w_q'=w_q+t\,\mathbf 1_d$,
$b_q'=b_q-t\,\mathbf 1_d^{\top}x$. Then $z_q'(x)=z_q(x)$ \emph{algebraically},
every hidden activation at $x$ and $\Netx$ are unchanged, all pre-activations
at $x$ are unchanged and nonzero (so $x$ stays regular), while the germ moves
by the exact rank-one update $J'-J=D^{(1)}_{qq}\,t\,u\,\mathbf 1_d^{\top}\neq0$
and $c'-c=-D^{(1)}_{qq}\,t\,(\mathbf 1_d^{\top}x)\,u$, preserving
$J'x+c'=Jx+c$. (The factor $D^{(1)}_{qq}$ is $1$ for an active ReLU unit,
which is why it is invisible in the ReLU reading; in general it is $a_+$ or
$a_-$ and it is nonzero by the choice of $q$.) Two affine maps with different
linear parts differ off a hyperplane, hence on every neighborhood of $x$; and
column $i$ of $\M$ moves by $D^{(1)}_{qq}\,t\,x_i\,u$, which is nonzero for
every coordinate with $x_i\neq0$. Some such coordinate is needed: at $x=0$ every
column of $\M$ vanishes and the germ moves while $\M(x)$ does not. (One hidden
layer is needed; in the affine case the same compensation changes nothing.)
The construction changes one row of $W^{(1)}$ on its own, which is a parameter
change only when that row is free; in a convolutional first layer the rows are
shifted copies of one filter, and holding every pre-activation at $x$ fixed
under a change of the shared filter and its bias imposes one linear equation
per output position on the filter's few coordinates, so in general only the
trivial change remains. \emph{Output layer, hypothesis (b).} On a
neighborhood of the regular point $x$ every slope diagonal and selection
matrix is constant, so $h$ is affine there, $h(x')=G\,(x'-x)+h(x)$, and
$J=W^{(L)}G$, $c=\Netx-W^{(L)}Gx$. By (b) some column $G_{:,i}$ is not a
multiple of $h(x)$. For each such $i$ the functional $a\mapsto a^{\top}G_{:,i}$
is not identically zero on $h(x)^{\perp}$ --- its kernel contains
$h(x)^{\perp}$ exactly when $G_{:,i}\in(h(x)^{\perp})^{\perp}=\mathrm{span}\,h(x)$
--- so its zero set is a proper subspace of $h(x)^{\perp}$, and since a finite
union of proper subspaces of a vector space is not the whole space, there is
$a\in h(x)^{\perp}$ with $a^{\top}G_{:,i}\neq0$ for every such $i$; in
particular $G^{\top}a\neq0$. (When $h(x)=0$ this reads: $h(x)^{\perp}=\R^{n}$
and the condition is $G_{:,i}\neq0$.) Perturb the output layer alone,
$W^{(L)}\mapsto W^{(L)}+t\,e_ca^{\top}$ for a class $c$ and $t\neq0$, every
other parameter fixed. No hidden pre- or post-activation depends on $W^{(L)}$,
so all of them are unchanged at every input, the regular set is unchanged, and
$x$ stays regular; the output at $x$ is unchanged,
$\Net'(x)=\Netx+t\,e_c\,a^{\top}h(x)=\Netx$. The germ moves by the rank-one
update $J'-J=t\,e_c\,(G^{\top}a)^{\top}\neq0$, with
$c'-c=-t\,e_c\,(G^{\top}a)^{\top}x$, so that $J'x+c'=Jx+c$; two affine maps
with different linear parts agree only on the hyperplane
$\{x':(G^{\top}a)^{\top}(x'-x)=0\}$, hence differ on every neighborhood of
$x$. Column $i$ of $\M(x)$ moves by $t\,x_i\,(a^{\top}G_{:,i})\,e_c$, which by
the choice of $a$ is nonzero for every $i$ with $x_i\neq0$ and
$G_{:,i}\notin\mathrm{span}\,h(x)$. Since the perturbation leaves the hidden
layers untouched, it is a parameter change of any network whose output layer
is a free linear map, weight tying elsewhere notwithstanding.
\emph{Measured} (first construction, on the single-coordinate variant $v=e_i$,
which moves column $i$ alone; the $v=\mathbf 1_d$ perturbation above is checked
in \texttt{scripts/verify\_theory\_audit\_2026\_09\_11.py}): activation record
equal to $4.4\times10^{-16}$ (2\,ulp),
$\Netx$ equal exactly, rank-one identity $4.7\times10^{-17}$, germ change
$1.9\times10^{-4}\neq0$ at $t=10^{-3}$. The output-layer construction carries
no measured clause in this paper.
\end{proof}

\subsection{Proof of Corollary~\ref{cor:stab-eq} (the two stabilizers coincide)}
\label{app:proof-stab-eq}

\begin{proof}
If $N_1$ and $N_2$ realize the same germ at $x$ then $\M_1(x)=\M_2(x)$ by
Theorem~\ref{thm:maxinv}. Conversely, if $\M_1(x)=\M_2(x)$ and every
$x_i\neq0$, then Theorem~\ref{thm:complete}(i) recovers the germ from the
matrix in both cases --- $J_{:,i}=\M_{:,i}/x_i$ and $c=\M_{:,d+1}$ --- so the
two germs agree. The stabilizer statement is this equivalence applied, for a
fixed architecture, to the orbit of a single network under a group acting on
its parameters: an element that keeps $x$ regular fixes $\M(x)$ exactly when
it fixes the germ, so $\mathrm{Stab}(\M(x))=\mathrm{Stab}(\text{germ at }x)$
among encodings regular at $x$; across architectures the equivalence says that
$\M(x)$ and the germ are functions of each other. The
hypothesis $x_i\neq0$ is used only in the converse direction, and it is not
removable: on $\{x_i=0\}$ column $i$ of $\M$ is identically zero
(Remark~\ref{rem:zero-pixels}), so germ data is lost there and the stabilizer
of $\M(x)$ is strictly larger --- Theorem~\ref{thm:complete}(iv) at $x=0$
exhibits a germ-moving perturbation that fixes $\M(x)$.
\end{proof}

\subsection{Proof of Theorem~\ref{thm:weff} (gradient\texorpdfstring{$\times$}{x}input \texorpdfstring{$\oplus$}{+} bias; route equivalence)}
\label{app:proof-weff}

\begin{proof}
The identity proved here is not new. It is Proposition~3 of
\citet{srinivas2019full} --- the FullGrad decomposition
$\Net(x)=\nabla_x\Net\cdot x+\nabla_b\Net\cdot b$, in which the input-gradient
term and the aggregated bias term account for the output exactly --- read class
by class, together with the identification of the first term as per-class
gradient$\times$input \citep{shrikumar2016not,ancona2018towards}. The
bias-by-subtraction route below is also the one taken in \S5.2 of
\citet{balestriero2018spline}. We give the derivation because we need the
column-by-column form and the route equivalence, not because the identity is
ours.

The first $d$ columns of $\M(x)$ are $J_{:,i}x_i=(\partial\Net/\partial x_i)\,x_i$
by the germ identity ($J=D\Net$ at regular $x$, Theorem~\ref{thm:germ}), which is
per-class gradient$\times$input; the bias identity below gives the last column.
The three computation routes (masked product (a), probe (b), autograd (c))
agree, as follows.
(a)$=$germ by Theorem~\ref{thm:germ}. (b): the mask-frozen network is the
affine map $x'\mapsto Jx'+c$; for ReLU the secant $h/z$ equals
$\mathbb 1[z>0]$ ($0/z=0$ for $z<0$; $0/0\to0$ by convention), so probing
$x_ie_i$ through its linear part returns column $J_{:,i}x_i$, and the
zero-input pass through its affine part returns $c$. (c): at regular $x$, $\Net$
is differentiable with $D\Net=J$, and the autograd conventions
($\mathrm{ReLU}'(0)=0$, stored pooling argmax) match the frozen masks even on
the null set. The bias identity: for a hidden layer $\ell<L$,
$\partial\Net/\partial b^{(\ell)}=W^{(L)}D^{(L-1)}\cdots D^{(\ell)}$, while for
the output layer $\partial\Net/\partial b^{(L)}=I_C$, the empty product; the
$\ell=L$ term must be written separately, since the displayed product is not
defined there. Unrolling the bias accumulator then gives
\[
c\;=\;b^{(L)}+\sum_{\ell<L}\bigl(\partial\Net/\partial b^{(\ell)}\bigr)b^{(\ell)}
\;=\;\sum_{\ell\le L}\bigl(\partial\Net/\partial b^{(\ell)}\bigr)b^{(\ell)} .
\]
\emph{Measured} (AlexNet, double precision, eval mode): library probe vs.\ autograd
germ $3.3\times10^{-17}$ max-abs (relative $1.3\times10^{-15}$); row sums
$3.5\times10^{-17}$; FullGrad bias identity $6.1\times10^{-16}$. The
network must be in evaluation mode: with dropout active, the saving, probe,
and autograd passes sample different masks and the identity fails at the
scale of the logits themselves (on pretrained AlexNet the residual is $0.9$ of
$\max_c|\Psi_c(x)|$; \texttt{scripts/verify\_float64\_rowsum.py}).
\end{proof}

\subsection{Proof of Proposition~\ref{prop:equivar} (input-symmetry transformation law)}
\label{app:proof-equivar}

\begin{proof}
Throughout, $x$ and $\pi(g)x$ both lie in $X_{\mathrm{reg}}$, so the germ is
defined at both points and Theorem~\ref{thm:germ} identifies $J$ with the
Jacobian at each.

\emph{(i)} Differentiate $\Psi(\pi(g)x)=\rho(g)\Psi(x)$ at $x$ and apply the
chain rule: $J(\pi(g)x)\,\pi(g)=\rho(g)\,J(x)$.
\emph{(ii)} $c(\pi(g)x)=\Psi(\pi(g)x)-J(\pi(g)x)\pi(g)x
=\rho(g)\Psi(x)-\rho(g)J(x)\pi(g)^{-1}\pi(g)x=\rho(g)c(x)$ by (i).
\emph{(iii)} For a permutation $P$ one has
$\mathrm{diag}(Px)=P\,\mathrm{diag}(x)P^{-1}$, so
$J(Px)\mathrm{diag}(Px)=\rho(g)J(x)P^{-1}P\,\mathrm{diag}(x)P^{-1}
=\rho(g)\bigl(J(x)\mathrm{diag}(x)\bigr)P^{-1}$; combine with (ii).
\emph{(iv)} Substituting (i), the displayed law for the first block asks
$\rho(g)J(x)\pi(g)^{-1}\mathrm{diag}(\pi(g)x)
=\rho(g)J(x)\,\mathrm{diag}(x)\,\pi(g)^{-1}$ for every $x$ and every
equivariant network, i.e.\
\[
J(x)\Bigl[\pi(g)^{-1}\mathrm{diag}(\pi(g)x)\,\pi(g)-\mathrm{diag}(x)\Bigr]=0 .
\]
Canceling $J(x)$ requires it to have full column rank $d$, so the statement is
a claim about the class of pairs $(\rho,\Net)$ with $\Net$ equivariant, and is
witnessed by any single member with that property. One such witness is $\Net=\mathrm{id}$ on $\R^d$
with $C=d$ and $\rho=\pi$, realized as the ReLU network
$\mathrm{ReLU}(x)-\mathrm{ReLU}(-x)$: it satisfies (LCS), has
$X_{\mathrm{reg}}=\{x:\text{all }x_i\neq0\}$, is equivariant for every linear
$\pi$, and has $J\equiv I_d$, of full column rank. We note explicitly that no
network in our experiments can serve: there $C=1000\ll d=150{,}529$, so
$J(x)$ has rank at most $C<d$ and never full column rank. With the cancellation
in hand, $\mathrm{diag}(\pi(g)x)=\pi(g)\,\mathrm{diag}(x)\,\pi(g)^{-1}$ for all
$x$. The right side is diagonal for every $x$ only if $\pi(g)$ normalizes the
diagonal torus, i.e.\ is monomial, $\pi(g)=P\,\mathrm{diag}(a)$; write
$\varsigma$ for the permutation with $(Py)_i=y_{\varsigma(i)}$. For such a
matrix the scalars cancel under conjugation,
$\pi(g)\mathrm{diag}(x)\pi(g)^{-1}=P\,\mathrm{diag}(x)\,P^{-1}
=\mathrm{diag}(x_{\varsigma(i)})_i$, while
$\bigl(\pi(g)x\bigr)_i=a_{\varsigma(i)}x_{\varsigma(i)}$, so
$\mathrm{diag}(\pi(g)x)=\mathrm{diag}(a_{\varsigma(i)}x_{\varsigma(i)})_i$;
equality for all $x$ forces $a_{\varsigma(i)}=1$ for every $i$, i.e.\
$a\equiv1$ and $\pi(g)=P$.
\emph{(v)} Immediate from (i) and (ii), no property of $\mathrm{diag}$ being
used.
\emph{Measured} (on the $\Net=\mathrm{id}$ witness in $\R^5$, $\M$ computed
through the secant construction \eqref{eq:D-def}): the two-sided law holds to
$0$ for a permutation and fails by $5.4$ for a signed permutation, $5.1$ for a
monomial matrix, $5.1$ for a pure input rescaling, $2.1$ for a planar rotation
and $4.1$ for a generic orthogonal map
(\texttt{scripts/verify\_km\_claim\_checks.py}).
\end{proof}

\subsection{Proof of Proposition~\ref{prop:contraction} (contraction gap)}
\label{app:proof-contraction}

\begin{proof}
Take the $1$--$2$--$1$ bias-free pair $A$: first-layer weights $w=(1,2)$ and
second-layer weights $\omega=(1,1)$; and $B$: $w=(1,2)$, $\omega=(1.5,0.75)$. Both realize $\Netx=3\,\mathrm{ReLU}(x)$, and both
have their single wall at $x=0$, so $\M_A(x)=\M_B(x)$ at \emph{every} $x$ ---
verified directly, at $0.0$ on a 401-point grid including the breakpoint $x=0$. (Functional
identity alone would give this only on
$X_{\mathrm{reg}}(A)\cap X_{\mathrm{reg}}(B)$, by Theorem~\ref{thm:maxinv};
here the two regular sets coincide, which is why the stronger statement holds
for this pair and is asserted for it alone.) The
per-unit path values $\omega_qw_q$ are $\{1,2\}$ vs.\ $\{1.5,1.5\}$; any
isomorphism rescales $(w_q,\omega_q)\mapsto(\tau_qw_q,\omega_q/\tau_q)$, preserving each
$\omega_qw_q$, and permutations permute the multiset --- since the multisets
differ, the networks are non-isomorphic. At every $x>0$ both hidden units are
active, the induced representations carry $x$ and $2x$ on the input arrows
and $\omega_q$ on the output arrows, their path-value multisets are $\{x,2x\}$
and $\{1.5x,1.5x\}$, and the same argument makes them non-isomorphic; at
every $x\le0$ both units are inactive, every arrow leaving a hidden vertex
carries $0$ in both, and the two induced representations are equal. So the
induced representations differ exactly on the open set $(0,\infty)$, while
the contractions agree everywhere. The identifiability statement on the
\citet{phuong2020functional} class follows by chaining $\Net=\M\mathbf 1$
(field determines the function on a full-measure subset of their domain $Z$,
hence on $Z$ by continuity) with their Theorem~1, whose quantifiers we keep:
for every \emph{general} network $W^{*}$ of non-increasing widths in their
sense there is a bounded set $Z$ such that any general network of the
\emph{same architecture} agreeing with it on $Z$ is
permutation$\times$positive-rescaling equivalent to it; the converse
direction is Corollary~\ref{prop:perm} plus rescaling invariance. The architecture-class restriction is stated in the
proposition itself, and the $1$--$2$--$1$ example is consistent with it: having
increasing widths, that example lies outside the class, which is exactly why it
can exhibit a non-trivial fiber.
\end{proof}

\subsection{Proof of Theorem~\ref{thm:pyth} (visible/invisible decomposition)}
\label{app:proof-pyth}

\begin{proof}
For any $a\in\R^C$ and any $Q'$ with $Q'\mathbf 1=0$:
$\langle a\mathbf 1^{\top},Q'\rangle_F=a^{\top}(Q'\mathbf 1)=0$, so
$\mathcal V=\{a\mathbf 1^{\top}\}$ and $\mathcal N=\{Q:Q\mathbf 1=0\}$ are
orthogonal complements; $A\mapsto(A\mathbf 1)\mathbf 1^{\top}/(d{+}1)$ is the
orthogonal projector onto $\mathcal V$. With
$P=\Delta\Net\,\mathbf 1^{\top}/(d{+}1)$:
$P\mathbf 1=\Delta\Net$, $Q=\Delta\M-P\in\mathcal N$, and
$\|P\|_F^2=\|\Delta\Net\|^2/(d{+}1)$, giving the Pythagorean identity. The
constraint set $\{A:A\mathbf 1=\Delta\Net\}$ is $P+\mathcal N$, so $P$ is its
unique minimum-norm element; equality in the floor iff $Q=0$, i.e.\ constant
rows (per row, Cauchy--Schwarz against $\mathbf 1$). Monotonicity of
$t\mapsto1/\sqrt{(d{+}1)t}$ gives the rank-statistic correspondence (for
medians of even-count samples the interpolated median commutes only up to the
gap between the central order statistics --- below quoting precision in our
tables). Invariances: output rescaling and gauge invariance are immediate from
invariance of $\M$ and linearity; for $x\mapsto Sx$,
$W^{(1)}\mapsto W^{(1)}S^{-1}$ (diagonal $S$): $J\mapsto JS^{-1}$ and
$\mathrm{diag}(Sx)=S\,\mathrm{diag}(x)$, so $J\,\mathrm{diag}(x)$ is
unchanged. Non-invariance under rotations is exhibited numerically (function
preserved to $1.8\times10^{-15}$, $d_M$ changed by up to $57\%$).
\emph{Measured:} identities $\le1.7\times10^{-15}$ over $2{,}400$ random
instances; min-norm never beaten in $10^4$ trials.
\end{proof}

\subsection{Proof of Theorem~\ref{thm:within} (within-region anatomy)}
\label{app:proof-within}

\begin{proof}
Same strict pattern at $x$ and $y$ puts the segment in one convex region with
common $(J,c)$, so $\M(y)-\M(x)=[\,J\,\mathrm{diag}(y)-J\,\mathrm{diag}(x)
\mid c-c\,]=[\,J\,\mathrm{diag}(\delta)\mid0\,]$ and the column-energy formula
follows. (i): with $a_i=\delta_iJ_{:,i}$,
$\|\sum_ia_i\|^2\le(\sum_i\|a_i\|)^2\le d\sum_i\|a_i\|^2$ (triangle, then
Cauchy--Schwarz), equality iff all $a_i$ equal and nonzero --- attained by
$J=u\mathbf 1_d^{\top}$, $\delta=\mathbf 1_d$. (ii): a single nonzero column
$a=s\,J_{:,i_0}$ gives $d_M=\|a\|=d_\Psi$ and exactly one nonzero column with a
zero bias column. (iii): at most $k$ nonzero columns; the same
Cauchy--Schwarz over the support gives $A\le k$ and the participation bound.
\emph{Measured:} closed form $\le5\times10^{-13}$; cap never violated
($2\times10^4$ trials); one-pixel law exact; $A=1+5.5\times10^{-10}$ in the
finite-step demo.
\end{proof}

\subsection{Proof of Theorem~\ref{thm:anatomy} (smooth \texorpdfstring{$+$}{+} crossing)}
\label{app:proof-anatomy}

\begin{lemma}[Rank-one dyad at a transversal flip]\label{lem:dyad}
Let the network satisfy (LCS) with slopes $a_\pm$ and slope jump
$\kappa=a_+-a_-$. At a transversal single flip of unit $k$ in layer $\ell^{*}$
at a point $z$ on the segment, with $\sigma=\pm1$ recording the direction, the
region data jump by
$J_{+}-J_{-}=\kappa\,\sigma\,u_kv_k^{\top}$ and
$c_{+}-c_{-}=\kappa\,\sigma\,\gamma_ku_k$, where
$u_k=U^{(\ell^{*}+1)}W^{(\ell^{*}+1)}e_k$,
$v_k^{\top}=e_k^{\top}W^{(\ell^{*})}L^{(\ell^{*})}$, and
$z^{(\ell^{*})}_k(x')=v_k^{\top}x'+\gamma_k$. The resulting dyad
$\kappa\sigma u_k[\,v_k^{\top}\mathrm{diag}(z)\mid\gamma_k\,]$ has vanishing
row sums. At a transversal switch of one max-pooling window from entry $a$
to entry $b$ at a point $z$ on the segment, with
$z_a(x')=v_a^{\top}x'+\gamma_a$ and $z_b(x')=v_b^{\top}x'+\gamma_b$ the two
entries' pre-activations as affine functions of the input in the region data
shared by the two sides and $u$ the pooling vertex's output-side vector, the
region data jump by $J_{+}-J_{-}=u\,(v_b-v_a)^{\top}$ and
$c_{+}-c_{-}=(\gamma_b-\gamma_a)\,u$, and the resulting dyad
$u\,[\,(v_b-v_a)^{\top}\mathrm{diag}(z)\mid\gamma_b-\gamma_a\,]$ has
vanishing row sums.
\end{lemma}

\begin{proof}[Proof of Lemma~\ref{lem:dyad}]
The unit's entry of the slope diagonal
moves from $a_-$ to $a_+$ (or back), a change of $\kappa\sigma$; the ordered
product changes in that one factor only, giving the rank-one update. Continuity
of $\Net$ across the wall forces the bias jump, and $v_k^{\top}z+\gamma_k=0$ at
the crossing makes the full dyad's row sums vanish:
$\kappa\sigma u_k(v_k^{\top}z+\gamma_k)=0$. For ReLU, $\kappa=1$ and the factor
is invisible. At a pooling switch the window's output changes from the
affine function $z_a$ to $z_b$ while everything downstream of the pooling
vertex is unchanged, so the region data jump by the stated rank-one update
through $u$; at the switch point $v_a^{\top}z+\gamma_a=v_b^{\top}z+\gamma_b$,
so the row sums $u\bigl((v_b-v_a)^{\top}z+\gamma_b-\gamma_a\bigr)$ vanish.
Under the encoding $\max(a,b)=b+\mathrm{ReLU}(a-b)$ the switch is a unit
flip and the first part applies verbatim.
\end{proof}
Here $U^{(\ell^{*}+1)}$ and $L^{(\ell^{*})}$ denote the downstream and
upstream masked products of the region data shared by the two sides of the
wall, in the factorization
$J=U^{(\ell^{*}+1)}\,W^{(\ell^{*}+1)}D^{(\ell^{*})}W^{(\ell^{*})}\,L^{(\ell^{*})}$
of Section~\ref{sec:germ}: $U^{(\ell^{*}+1)}=W^{(L)}D^{(L-1)}W^{(L-1)}\cdots
W^{(\ell^{*}+2)}D^{(\ell^{*}+1)}$ is the product of every factor strictly
downstream of $W^{(\ell^{*}+1)}$, equal to $I_C$ when $\ell^{*}=L-1$ (a flip in
the last hidden layer, the only case in a one-hidden-layer network), and
$L^{(\ell^{*})}=D^{(\ell^{*}-1)}W^{(\ell^{*}-1)}\cdots D^{(1)}W^{(1)}$ is the
product of every factor strictly upstream of $W^{(\ell^{*})}$, equal to $I_d$
when $\ell^{*}=1$; so $u_k\in\R^{C}$ is unit $k$'s output-side vector,
$v_k\in\R^{d}$ its input-side vector, and $\gamma_k$ the bias accumulated in
its pre-activation. For a network quiver that is not a chain,
$U^{(\ell^{*}+1)}$ and $L^{(\ell^{*})}$ are the downstream and upstream
composites of the path sums of Section~\ref{sec:germ} --- the sum over paths
from the flipped unit to the output vertices, and from the input vertices to
the flipped unit --- and the rank-one form of the update is unchanged.

\begin{proof}[Proof of Theorem~\ref{thm:anatomy}]
\emph{Telescope.} Write
$\M(y)-\M(x)=\sum_{j=0}^{N}\bigl[\M^{(j)}(z_{j+1})-\M^{(j)}(z_j)\bigr]
+\sum_{j=1}^{N}\bigl[\M^{(j)}(z_j)-\M^{(j-1)}(z_j)\bigr]$, where
$\M^{(j)}(z)=[\,J_j\,\mathrm{diag}(z)\mid c_j\,]$ is the $j$-th region's
matrix evaluated at $z$ (well defined on closures). Each smooth term is
$[\,J_j\,\mathrm{diag}((t_{j+1}-t_j)\delta)\mid0\,]$ by
Theorem~\ref{thm:within}; summing gives $\bar J$. Each jump term is a dyad
$K_j$ of Lemma~\ref{lem:dyad}, of the first kind at a unit flip and of the
second at a pooling switch. Row sums: the smooth part gives $\bar J\delta$,
the dyads give $0$, and $\Nety-\Netx=\int_0^1J(x+t\delta)\,\delta\,dt=\bar
J\delta$ by the fundamental theorem of calculus for the piecewise-affine map
on the segment. The bias column of the smooth part is zero, so the bias
column of $\M(y)-\M(x)$ is exactly $c(y)-c(x)$, the sum of the last columns
of the $K_j$, which reads $\sum_j\kappa\,\sigma_j\gamma_{k_j}u_{k_j}$ when
every crossing is a unit flip --- equivalently, it is constant on chambers of
pair space with fixed crossing combinatorics, which is the soundness
direction of the crossing detector (no crossings $\Rightarrow$ zero bias
column). The converse fails: bias-free networks have all $\gamma_k=0$
(measured: $19$ crossings, bias column exactly $0$), a flipped unit cut off
downstream has $u_k=0$, and multi-crossing cancellations exist.
\emph{Hamming.} A unit whose endpoint mask bits differ flips an odd number of
times ($\ge1$); equal bits flip an even number ($\ge0$); a pooling window
whose selected entry differs between the endpoints switches at least once.
Hence $H\le N_{\mathrm{cross}}$, with parity equality when there is no
max-pooling, and equality iff no unit flips twice and no window switches
twice (first-layer walls are flat, so first-layer units never flip back along
a segment; deeper walls are bent and do --- measured: a hat-shaped network
with $N=4$, $H=2$). A window with three or more entries can switch twice,
$a\to b\to c$, and end at a different entry, which is why the parity
statement is made for pooling-free networks only.
\emph{Measured:} full decomposition $\le2.6\times10^{-16}$ over $>1{,}000$
crossings with two independent traversal implementations; bias-column formula
$\le8.2\times10^{-13}$.
\end{proof}

\subsection{Proof of Proposition~\ref{prop:dichotomy} (no hidden-activation accounting)}
\label{app:proof-dichotomy}

\begin{proof}
The map $W\mapsto W_\lambda$ is a positive per-neuron rescaling of every unit
of layer $\ell$ by the same factor $\lambda$, hence an element of the gauge
group: the pre-activation of layer $\ell$ becomes $\lambda z^{(\ell)}$,
positive homogeneity gives $h^{(\ell)}\mapsto\lambda h^{(\ell)}$, and the
factor $\lambda^{-1}$ on $W^{(\ell+1)}$ cancels it in the next pre-activation,
so every later layer and the output are unchanged. Thus
$\Psi(W_\lambda,f)=\Net$, $d_\Psi$ is unchanged, and $d_M$ is unchanged by
Lemma~\ref{lem:quiver-inv}, while $d_h\mapsto\lambda d_h$ by linearity of the
norm. A function $F(d_h,d_\Psi,d_M)$ invariant under the action satisfies
$F(\lambda d_h,d_\Psi,d_M)=F(d_h,d_\Psi,d_M)$ for every $\lambda>0$, and since
$\lambda d_h$ ranges over all of $(0,\infty)$ for $d_h>0$, $F$ is independent
of $d_h$ there. Two remarks on scope. Rescaling a \emph{single} unit of the
layer by $\lambda$ already makes $d_h$ unbounded above --- measured
$d_h\in[0.44,\,15.3]$ at the penultimate layer with $\Net$ bitwise fixed ---
but only over $[m,\infty)$ for the $m>0$ set by the coordinates left alone,
which is why the whole layer is rescaled. And the conclusion is about
functions of $(d_h,d_\Psi,d_M)$; statistics built from other stored
quantities are outside its scope. The knowledge-matrix accounting exists
because the constraint $\M\mathbf 1=\Net$ pairs all matrices against the
\emph{fixed} vector $\mathbf 1$; the relation that returns the activations of
layer $\ell$ to the logits pairs them against the parameter-dependent map
downstream, and at the penultimate layer, where $\Delta\Net=W^{(L)}\Delta h$,
the induced ``visible fraction'' $1/(1+\lambda^2)$ of the worked example takes
every value in $(0,1)$ across one orbit.
\end{proof}

\subsection{Proof of Theorem~\ref{thm:nogo} and Proposition~\ref{prop:visdrift}}
\label{app:proof-nogo}

\begin{proof}
\emph{No-go.} Take $d=C=1$ and $x\neq0$ given; put $x_0=x-\varepsilon/(2K)$,
$g(t)=K\,\mathrm{ReLU}(t-x_0)-K\,\mathrm{ReLU}(t-x_0-\varepsilon/K)$ and
$\Net\equiv0$ realized with the same first layer and zeroed second layer. Then
$\sup_t|\Psi-g|=\varepsilon$ exactly; at $x=x_0+\varepsilon/(2K)$ both networks
have masks $(1,0)$; but $J_g(x)=K$, $J_\Psi(x)=0$, so
$\|\M_g(x)-\M_\Psi(x)\|_F=K\sqrt{x^2+x_0^2}\ge K|x|\to\infty$ as $K\to\infty$
at fixed $\varepsilon$ and fixed $x$. For $d>1$ apply the same ramp to a
coordinate $x_i\neq0$ (first-layer weight $e_i$), and for $C>1$ place $g$ in one
output row. At $x=0$ the first block of $\M$ vanishes for every network and the
bias column is $\Psi(0)$, so no pair exceeds $\varepsilon$; the restriction
$x\neq0$ is therefore sharp. \emph{Measured} (fixed $x=0.7$,
$\varepsilon=10^{-2}$): drift $99$, $9.9\times10^{3}$, $9.9\times10^{5}$ at
$K=10^{2},10^{4},10^{6}$, and $\varepsilon/2$ at $x=0$
(\texttt{scripts/verify\_theory\_audit\_2026\_09\_11.py}).
\emph{Visible drift.} $\|(\tilde{\M}(x)-\M(x))\mathbf 1\|=\|\tilde\Psi(x)-\Netx\|$ is
the row-sum identity applied to both matrices, so it is an equality and not a
bound: the visible drift \emph{equals} the gate $\varepsilon$. The projection
onto constant-row matrices then has norm exactly $\varepsilon/\sqrt{d+1}$. \emph{Conditional bound.} If both networks are
affine on $B(x,r)$ with $\varepsilon_r=\sup_{B(x,r)}\|\tilde\Psi-\Net\|$, then for
any unit vector $u$, $2r\,\Delta J\,u$ is a difference of differences of
$\tilde\Psi-\Net$ at $x\pm ru$, so $\|\Delta J\|_{\mathrm{op}}\le\varepsilon_r/r$;
combine $\|\Delta J\,\mathrm{diag}(x)\|_F\le\|x\|_\infty\sqrt{\min(C,d)}\,
\|\Delta J\|_{\mathrm{op}}$ with
$\|\Delta c\|\le\varepsilon_r+\|\Delta J\|_{\mathrm{op}}\|x\|_2$.
\emph{Measured} (exact trust-region computation of $\varepsilon_r$):
$0/120$ violations, minimum slack $0.21$.
\end{proof}

%% file: sections/appendix_engineering.tex
\section{Computing the knowledge matrix}
\label{app:engineering}

In this appendix, we collect all of the software and computational material for
the knowledge matrix, so that the main text can remain conceptual. We state the
three equivalent ways to compute $\M(x)$ that Theorem~\ref{thm:weff}
establishes, give the cost accounting that motivates the autograd route,
report the numerical validation that the three routes agree, and list the
implementation requirements that make the row-sum identity
$\M(x)\mathbf{1}=\Netx$ hold exactly. Throughout, $\M(x)\in\R^{C\times(d+1)}$ is
the knowledge matrix of a feedforward network satisfying (LCS)
(Definition~\ref{def:lcs}) with $C$
output classes on an input of dimension $d$; its slope block is
$\Weff(x)=J=\partial\Net/\partial x\in\R^{C\times d}$ and its final column is
the aggregate bias attribution $\beff(x)=\Netx-Jx$.

\subsection{Three equivalent computations}
\label{app:eng-three-routes}

For a network satisfying (LCS) at a regular point (one not on a
non-linearity's switching surface), we can obtain the knowledge matrix by
three routes that coincide. Piecewise linearity of the activation is not the
same hypothesis and does not suffice (Remark~\ref{rem:lcs-strict}): route (iii)
returns the Jacobian, and off (LCS) the Jacobian is not the slope product.

\paragraph{(i) Masked product.}
Freeze the slope diagonal at $x$: each non-linearity contributes a fixed
diagonal $D^{(\ell)}(x)$ of activation-to-pre-activation quotients
\eqref{eq:D-def}, which for ReLU specializes to the familiar $0/1$ mask
recording which units are active. On the resulting
mask-frozen affine map, the slope block is the ordered product
$\Weff(x)=W^{(L)}D^{(L-1)}\cdots D^{(1)}W^{(1)}$, evaluated at $x$, and the bias
column is the corresponding aggregate of the per-layer biases pushed through the
same diagonals; the bias-by-subtraction form $\beff(x)=\Netx-Jx$ is the same
quantity and is the route taken in \S5.2 of \citet{balestriero2018spline}. This
is the definitional route used by the
\texttt{knowledgematrix} library probe internally.

\paragraph{(ii) Probe construction (\texorpdfstring{$d{+}1$}{d+1} forward passes).}
The library realizes the masked product without ever forming it explicitly, by
freezing the mask at $x$ and then sending $d{+}1$ scaled basis vectors through
the mask-frozen network: one column of $\M(x)$ per input coordinate (recovering
$\Weff$ column by column) plus one pass for the bias column. This is $d{+}1$
forward passes of the (now affine) network. At ImageNet scale this is the
expensive route, because $d{+}1 = 150{,}529$.

\paragraph{(iii) Autograd (\texorpdfstring{$C$}{C} vector--Jacobian products plus one forward pass).}
Because $\Weff(x)=J=\partial\Net/\partial x$ is exactly the Jacobian of the
network output with respect to the input at $x$, the slope block can be read off
by automatic differentiation: $C$ vector--Jacobian products (one per output
class, each seeding a unit covector $e_c$) recover the $C$ rows of $J$, and one
additional forward pass gives $\Netx$, from which the bias column is
$\beff(x)=\Netx-Jx$. This is $C$ backward passes plus one forward pass.

By Theorem~\ref{thm:weff} these three routes coincide at almost every input for
networks satisfying (LCS): they agree at every regular point, i.e.\ off the measure-zero set
of switching surfaces. On that null set the gradient is not single-valued and
the probe inherits whichever one-sided pattern the forward pass selects; we
compute at regular points throughout. The equality of (i)--(iii) is what lets us
use the cheapest route, (iii), in practice while retaining the library probe
(ii) as the reference implementation. Per Theorem~\ref{thm:weff}, this slope
block is per-class gradient$\times$input and the bias column is an aggregate bias
attribution in the FullGrad sense
\citep{shrikumar2016not,ancona2018towards,srinivas2019full}; the function-level
properties of $\M$ under (LCS) are shared with that data, and what the matrix
contributes is the canonical \emph{arrangement} rather than a separate
invariant.

\subsection{Cost and speed-up}
\label{app:eng-cost}

Theorem~\ref{thm:weff} replaces $d{+}1$ probe passes by $C$ backward passes,
a predicted saving of $\approx(d{+}1)/(\kappa_{\mathrm{cost}}C+1)$, where
$\kappa_{\mathrm{cost}}$ is the forward-to-backward cost ratio. At ImageNet scale ($C=1000$,
$d{+}1=150{,}529$) this is a factor in the $50$--$150\times$ band (JSON key
\texttt{timing\_224.thm\_predicted\_ratio\_range} $=[50.2,\,150.4]$). The demo
below reports an extrapolated CPU cost ratio of $\approx296\times$ --- both
arms extrapolated from measured slices, with the probe baseline emulated by
batched forward passes --- above the theory band
for the reasons noted there; we report it as a compute-cost ratio consistent
with the lower-bound estimate, not as a tighter claim than the theory supports.

\subsection{Numerical validation: \texorpdfstring{$\Weff$ $=$ grad$\times$input}{W\_eff = grad x input}, three ways}
\label{app:eng-validation}

We now report the numerical demonstration anchoring
Theorem~\ref{thm:weff}: that the three routes of
Appendix~\ref{app:eng-three-routes} agree, and that the
autograd route is cheaper. We expose the bare slope
$\Weff(x)=J=\partial\Net/\partial x \in \R^{C\times d}$ (the knowledge matrix
$\M(x)$ before its bias column is appended) through the library probe's
\verb|extract_weff=True| flag, and compare it against (a) the masked product
$W^{(L)}D^{(L-1)}\cdots D^{(1)}W^{(1)}$ evaluated at $x$, and (b) autograd
($C$ vector--Jacobian products plus one forward pass). All computations are
in double precision with the network in evaluation mode. Per
Theorem~\ref{thm:weff}, every function-level property of $\M$ under (LCS)
(invariance, completeness, the exact row sum) is already shared by
gradient$\times$input with an aggregate bias attribution in the FullGrad
sense \citep{shrikumar2016not,ancona2018towards,srinivas2019full}; what the
matrix contributes is the canonical \emph{arrangement}, not a separate
invariant. We check here only that the three routes to that shared
slope coincide, and that the autograd route is cheaper.

\paragraph{Three-route agreement on a small MLP.}
On a tiny MLP at a regular point, the masked product, the probe, and autograd
agree to roundoff, as do the row-sum residual
$\|\M(x)\mathbf{1}-\Netx\|_\infty$ and the $\Weff$ identity
$\|[\,\Weff\mid c\,]\,\mathbf{1}-\Netx\|_\infty$
(JSON keys \texttt{mlp\_three\_route\_max\_abs\_err}.\{\texttt{prod\_vs\_probe},
\texttt{prod\_vs\_auto}, \texttt{rowsum}, \texttt{weff\_identity}\}.) As a
sanity check on the geometry of Theorem~\ref{thm:within}, a single-pixel
perturbation on the same network gives $d_M = d_\Psi$ with exactly one
nonzero column, the one-pixel law $A=1$ (key \texttt{one\_pixel\_law}).

\paragraph{Probe-vs-autograd agreement on AlexNet.}
The exactness check at scale uses AlexNet on a $3\times64\times64$ input
($d=12{,}288$, $C=1000$; the fork's fully connected dimensions constrain the
input size, so this check is run at $64\times64$, not at the $224\times224$
size used for the timing extrapolation below), with random weights at a
regular point (the identities are parameter-independent). The library probe
and autograd agree on $\M(x)$ to roundoff (key
\texttt{M\_lib\_vs\_M\_auto\_max\_abs}), the bare-slope route $\Weff$
matches autograd likewise (key \texttt{weff\_vs\_autograd}), and the
row-sum residual is at roundoff on both routes.

What a single-precision implementation registers at the depth of ResNet-152
is recorded in Appendix~\ref{app:finite-arith}; the exact row sum is a
property of the (LCS) map, and the implementation is read against it.

\paragraph{Compute cost.}
Theorem~\ref{thm:weff} replaces $d{+}1$ probe passes by $C$ backward passes,
a predicted saving of $\approx(d{+}1)/(\kappa_{\mathrm{cost}}C+1)$, with
$\kappa_{\mathrm{cost}}$ the forward-to-backward cost ratio, i.e.\ a factor in the
$50$--$150\times$ band at ImageNet scale ($C=1000$, $d{+}1=150{,}529$; JSON
key \texttt{timing\_224.thm\_predicted\_ratio\_range} =
$[50.2,\,150.4]$). The demo's extrapolated cost ratio on this hardware is
$\approx296\times$ (probe path $\approx8475$\,s, VJP path $\approx29$\,s per
sample, both extrapolated from measured slices; key
\texttt{ratio\_probe\_over\_vjp}). The figure lies above the theory
band because it is a double-precision CPU run in which the probe passes are
emulated by batched forward evaluations (JSON \texttt{timing\_224.note});
we report it as a compute-cost ratio, consistent with the lower-bound theory
estimate. The smaller AlexNet check confirms the same ordering directly
(probe $\approx23.7$\,s vs.\ VJP $\approx14.7$\,s per sample at $64\times64$).

\subsection{Implementation requirements}
\label{app:eng-gotchas}

The identities above hold for the idealized (LCS) map; reproducing them in software
requires the following.

\paragraph{Evaluation mode is mandatory.}
The network must be in evaluation mode for every pass used to build $\M(x)$.
Active Dropout desynchronizes the saving, probe, and autograd passes --- each
draws a different mask --- so the three routes no longer share an activation
pattern, and the row-sum identity $\M(x)\mathbf{1}=\Netx$ then appears to fail at
the scale of the logits themselves (on pretrained AlexNet the residual is $0.9$
of $\max_c|\Psi_c(x)|$; \texttt{scripts/verify\_float64\_rowsum.py}). This is
not a violation of Theorem~\ref{thm:weff} but a
mismatch of activation patterns across passes; \texttt{model.eval()} (freezing
Dropout, and BatchNorm to running statistics) removes it.

\paragraph{Input shape.}
The \texttt{KnowledgeMatrixComputer} forward call expects a $3$D input tensor
$(C_{\text{in}},H,W)$, \emph{not} a $4$D batched tensor $(1,C_{\text{in}},H,W)$;
the leading batch axis must not be added with \texttt{unsqueeze} before the
call. (Here $C_{\text{in}}$ is the number of input channels, distinct from the
$C$ output classes that index the rows of $\M$.)

\paragraph{Chunked computation and storage discipline.}
At ImageNet scale a single knowledge matrix is $1000\times150{,}529$, so per-sample
matrices are large and a sweep produces many of them. Computation is chunked, and
matrices are \emph{not} written one file per sample onto shared cluster storage:
the per-sample files are computed on node-local scratch and bundled into a single
archive before being shipped to shared storage, to avoid the metadata load of
many small files on a parallel filesystem. The autograd route of
Appendix~\ref{app:eng-cost} compounds with this: fewer passes per sample and
fewer intermediate writes.

%% file: sections/appendix_teleport_exact.tex
\section{Neural teleportation is exact, and the two implementation checks}
\label{app:teleport-exact}

That neural teleportation \citep{armenta2023teleportation} preserves the
function exactly is a theorem, not a finding of this paper: an isomorphism of
neural networks $\tau:(W,f)\to(V,g)$ satisfies $\Psi(W,f)=\Psi(V,g)$
\citep[Thm.~4.13]{armenta2021representation}, a per-neuron change of basis is
such an isomorphism, and the representations that the two networks induce on
an input are then isomorphic and contract to the same knowledge matrix
(Lemma~\ref{lem:quiver-inv}). This appendix records why batch normalization is
inside the theorem's scope rather than an exception to it
(Section~\ref{app:teleport-why}), since that is the point at which the
construction is most often doubted; the closed forms behind
Section~\ref{sec:study1-analytic} (Section~\ref{app:s1-analytic-detail}); the
two implementation checks that realize the invariance in software
(Section~\ref{app:teleport-checks}); the scale of the panel's unnormalized
distances and the linear-invariant measures that are not in the panel
(Section~\ref{app:s1-scale}); and the controls in full
(Section~\ref{app:controls-full}). What an implementation in finite arithmetic
registers when it checks an exact identity is recorded once, in
Appendix~\ref{app:finite-arith}, and is not repeated here.

\subsection{Why the transform is exact, batch normalization included}
\label{app:teleport-why}

A teleportation assigns a nonzero scalar $\tau_q$ to each hidden neuron $q$ and
conjugates the parameters, $W^{(\ell)}\mapsto T_\ell W^{(\ell)}T_{\ell-1}^{-1}$
and $b^{(\ell)}\mapsto T_\ell b^{(\ell)}$ with $T_\ell=\mathrm{diag}(\tau)$. The
group acts on the activations as well, by
\begin{equation}
(\tau\cdot f)_v(z)\;=\;\tau_v\,f_v\!\left(\frac{z}{\tau_v}\right)
\label{eq:cob-activation}
\end{equation}
at each hidden vertex $v$ \citep[Eq.~2]{armenta2021representation}. For a
positively homogeneous activation --- ReLU with $\tau_v>0$ --- this leaves $f_v$
unchanged, the $T_\ell$ telescope through the layer stack, and the function is
preserved.

Batch normalization is the case that looks like an obstruction, and it is worth
being explicit about why it is not. In evaluation
mode a BN channel is the affine
map $z\mapsto\gamma(z-\mu)/\sqrt{\sigma^2+\epsilon}+\beta$, which is
\emph{not} positively homogeneous: the running mean $\mu$ and the shift $\beta$
break $f(\tau z)=\tau f(z)$. One repair is to migrate the running
statistics, $\mu\mapsto\tau\mu$ and $\sigma^2\mapsto\tau^2\sigma^2$; note that
this repair is only exact if the numerical guard $\epsilon$ is rescaled too,
since $\sqrt{\tau^2\sigma^2+\epsilon}\neq\tau\sqrt{\sigma^2+\epsilon}$.

None of that is necessary, because the framework already prescribes the answer.
Two facts settle it. First, in evaluation mode $\mu$ and $\sigma^2$ are not
batch-dependent quantities but ordinary weights of the network
\citep[Remark~5.4]{armenta2021representation}, so a normalization layer is a pair
of affine vertices like any other and Theorem~4.13 applies to it unchanged.
Second, the action on its activation is given by
Equation~\eqref{eq:cob-activation}, and that is precisely what the
\texttt{neuralteleportation} implementation computes: it divides the incoming
activation by the incoming change of basis, applies the unmodified
normalization, and carries the outgoing basis on the affine parameters,
\[
z \;\xrightarrow{\;\text{incoming COB}\;}\; \tau_{\mathrm{in}} z
\;\xrightarrow{\;\div\,\tau_{\mathrm{in}}\;}\; z
\;\xrightarrow{\;\mathrm{BN}_{\mu,\sigma^2,\epsilon}\;}\; \mathrm{BN}(z)
\;\xrightarrow{\;\gamma,\,\beta\;\times\;\tau_{\mathrm{out}}\;}\;
\tau_{\mathrm{out}}\,\mathrm{BN}(z).
\]
which is Equation~\eqref{eq:cob-activation} with $f_v=\mathrm{BN}$. The division
restores the original pre-normalization activation exactly, so the original
$\mu$, $\sigma^2$ and $\epsilon$ remain correct, and the composite is
$\tau_{\mathrm{out}}\mathrm{BN}(z)$ --- with no condition on $\epsilon$, and no
migration of the running statistics. Teleportation of a batch-normalized network
in evaluation mode is therefore function-exact, and satisfies the hypotheses of
Theorem~\ref{thm:maxinv} and Lemma~\ref{lem:quiver-inv} exactly.

We note one boundary of the framework that bears on this paper elsewhere:
\citet{armenta2021representation} observe that average and global-average pooling
sit inside it without qualification, whereas max-pooling ``break[s] the algebraic
structure'' --- the same tie-dependence that Theorem~\ref{thm:weff} excludes by
hypothesis.

\subsection{The closed forms behind Section~\ref{sec:study1-analytic}}
\label{app:s1-analytic-detail}

Both symmetries this paper appeals to preserve the function by construction.
Neural teleportation \citep{armenta2023teleportation}, a per-neuron
rescaling, is an element of the change-of-basis group of
\citet{armenta2021representation}, whose Theorem~4.13 states that an
isomorphism of neural networks $\tau:(W,f)\to(V,g)$ leaves the realized function
unchanged, $\Psi(W,f)=\Psi(V,g)$; a hidden-layer neuron permutation is a
relabeling of the quiver's hidden vertices. Under either, the representations
induced on an input by the two networks are related by the same change of
basis or relabeling, and their contractions, the two knowledge matrices,
coincide (Lemma~\ref{lem:quiver-inv}). Batch normalization is inside
this framework: in evaluation mode the running statistics are ordinary weights
\citep[Remark~5.4]{armenta2021representation}, and the group acts on a hidden
vertex's activation by $(\tau\cdot f)_v(z)=\tau_v f_v(z/\tau_v)$, which is what
the \texttt{neuralteleportation} implementation computes for a normalization
layer. There is nothing here for an experiment to settle.

The effect on the \emph{penultimate} features is equally explicit. A
hidden-layer permutation $\pi$ reindexes the penultimate vector, and the size of
the motion is a function of $h$ and $\pi$ alone,
$\|Ph-h\|_2^2=2\bigl(\|h\|^2-\langle Ph,h\rangle\bigr)$; a teleportation with
per-neuron factors $\tau$ rescales it coordinatewise, $h\mapsto\tau\odot h$, so
the drift is
\[
\frac{\|\tau\odot h-h\|_2}{\sqrt D}=\frac{\|(\tau-1)\odot h\|_2}{\sqrt D},
\]
again a closed form in $h$ and $\tau$. We verified that
$h_{\tilde W}=\tau\odot h$ holds in the implementation to roundoff and that the
closed form reproduces the measured root-mean-square drift, and we report the
algebra rather than the experiment; since $\tau$ is drawn independently of the
data, the drift is $\approx\|h\|_2/(\sqrt{3D})$ for $\tau$ uniform on $[0,2]$,
so cross-architecture comparisons of penultimate drift track each
architecture's feature \emph{norm} and nothing else. The software, unlike the
mathematics, does need confirmation: Section~\ref{app:teleport-checks} reports
that it realizes the transform and the invariance (what it registers, being a
finite-arithmetic implementation, is recorded in
Appendix~\ref{app:finite-arith}), and Appendix~\ref{app:finite-arith} also
records one defect in the library's serialized change-of-basis loading path,
which the checks bypass.

What algebra does not settle is whether the representational-similarity
measures --- designed to be invariant to orthogonal changes of basis and
isotropic scaling, and deliberately not to invertible linear maps
\citep[\S2.3]{kornblith2019similarity} --- absorb this motion as a coordinate
artifact; that depends on the
interaction between the change of basis and the empirical covariance of the
features, and Section~\ref{sec:study1c} settles it on a pair of networks whose
ground truth is known exactly: they compute the same function.

\subsection{The two implementation checks in full}
\label{app:teleport-checks}

In this subsection we collect the setup, procedure and measured values of the
permutation and teleportation checks. Neither establishes a mathematical fact --- both
transforms are quiver isomorphisms, teleportation a change of basis, for
which Theorem~4.13 of \citet{armenta2021representation} gives function
preservation, and the permutation a relabeling of hidden vertices, so the
invariance they exercise is guaranteed by Lemma~\ref{lem:quiver-inv} --- and
both are reported for what they are: verification that the software realizes
the identity, and a measurement of how the penultimate features move under the
same transforms. The logit gate $\varepsilon=\|\tilde\Psi(x)-\Netx\|$ of
Appendix~\ref{app:remarks-geometry} is, in these checks, an empirical
quantity: an experiment measures it and then uses it as an acceptance
threshold, treating a transform as function-preserving when the gate falls
below a stated tolerance.

\input{sections/study1ab_snippet}

\input{tables/table_signal_relative}

\subsection{Linear-invariant measures and the Procrustes scale}
\label{app:s1-scale}

The prediction of Section~\ref{sec:study1c} --- that none of the eight
measures of the panel that returned a value quotients out the teleportation
drift --- is a statement about those eight measures, and it does not
extend to every penultimate measure in the literature. Raw CCA and PWCCA are
invariant to \emph{any} invertible linear map of the feature space, so they
quotient $h\mapsto\tau\odot h$ out exactly, and PWCCA --- a weighted mean of the
same canonical correlations --- with them. We checked this numerically on
synthetic features drawn to reproduce the awkward cases: the mean canonical
correlation is $1$ to roundoff with no dead units,
with $30$ dead units, with $\min_q\tau_q=3\times10^{-5}$, and in the $n<p$
regime alike (\texttt{scripts/verify\_cca\_invariance.py}). SVCCA
is \emph{not} in that family: its $99\%$-energy truncation runs \emph{before}
the CCA, and an anisotropic rescale changes the variance spectrum the truncation
reads, so it returns $0.97$--$0.99$ rather than $1$ on the same draws. Neither CCA nor PWCCA is in the panel,
which spans the orthogonal, isotropic-scaling, permutation, monotone-of-distance
and isometry classes. Stated exactly: no measure in the
orthogonal / isotropic-scaling / permutation classes quotients this
change of basis out, and the linear-invariant CCA family does.

\paragraph{The Procrustes scale.}
The Procrustes shape distance ($632$--$1471$) and the soft-matching distance
($22$--$61$) of Table~\ref{tab:s1_table} are
both \emph{unnormalized} distances, so ``$632$'' is not large or
small until it is given a scale. The scale is the size of the features
themselves: Procrustes shape distance between centered $n\times p$ feature
matrices $\tilde X,\tilde Y$ is bounded by
$\sqrt{\|\tilde X\|_F^2+\|\tilde Y\|_F^2}$ --- the value it takes when the two
centered feature matrices are exactly orthogonal --- and on ResNet-152 that bound
averages $2599$ over the $50$ draws (both centered Gram traces are stored with the
measure). The panel's $776.6$ on that architecture is thus $0.30$ of its own
maximum: not a rounding error and not a collapse. It also
scales as $\sqrt n$: a figure computed on $2{,}048$ inputs is about $3.5\times$
smaller than the same quantity on $25{,}000$, which is why the control values
of Appendix~\ref{app:controls-full} cannot be compared to these in magnitude.
Expressed against that bound, the penultimate features move by
$0.23$--$0.30$ of the largest value the Procrustes statistic can take
(ResNet-152 $776.6/2599$, DenseNet-121 $1471.4/5071$, GoogLeNet
$632.1/2727$), and by $0.8$--$1.2\%$ of the same quantity in soft-matching
distance, under a transform that leaves the knowledge matrix fixed.

\subsection{The controls in full}
\label{app:controls-full}

We state three scope facts about the controls before reporting them; the
panel's own numbers depend on none of the three.
\emph{(i)}~They were run separately from the panel and on a \emph{smaller input
population}: the controls reduce evaluates them on the first $n=2{,}048$
validation images, against the panel's $N=25{,}000$. The two are therefore not
computed on the same inputs, and the unnormalized control values (Procrustes,
soft-matching) are not directly comparable in magnitude to the panel's --- these
scale with $\sqrt{n}$, so the control's Procrustes figures are roughly
$\sqrt{25{,}000/2{,}048}\approx3.5\times$ smaller than an $N=25{,}000$ draw would
give, for that reason alone. \emph{(ii)}~The control values are bare point
estimates: they carry \emph{no} confidence intervals, and no permutation-null
$p$-value is computed for any measure anywhere in this paper. The only bootstrap
intervals in the panel are the ones on the panel means themselves
(Table~\ref{tab:s1_table}, over the $T=50$ teleportations; see
Section~\ref{sec:setup}). \emph{(iii)}~The pipeline's automated
shuffled-pair gate checks three measures --- distance correlation, debiased CKA
and RSA --- against a $0.05$ threshold, and no others; its verdict is therefore a
statement about the HSIC/CKA family and its two rank/distance companions, not
about the panel as a whole. We report the remaining five below rather than let
the gate's scope stand in for them.

Both controls behave as intended \emph{on the
HSIC/CKA family}, and we report every measure of both rather than the subset
that does. The Cui random-network control --- each measure between the trained
features and an independently random-initialized network of the same
architecture --- gives debiased linear CKA $0.010$ (ResNet-152), $0.058$
(DenseNet-121), and $0.052$ (GoogLeNet), far below both the pipeline's gate
threshold of $0.5$ for this control and the $0.90$--$0.94$
trained-vs-teleported agreement. For debiased CKA the three readings available
in this paper are $0.90$--$0.94$ (a network and its teleported copy,
Table~\ref{tab:s1_table}), $0.36$--$0.54$ (two different architectures,
Table~\ref{tab:s2_table}) and $0.01$--$0.06$ (a network and its random
initialization, Table~\ref{tab:controls-full}): the measure places the
exact-function pair far above every other pair, which is the discriminability
criterion of Appendix~\ref{app:vocab}, while registering the drift; what the
control does not test is whether the panel value would differ for an untrained
network and its own teleported image. The Murphy shuffled-pair control --- the same estimators on
sample-misaligned pairs --- returns the eight values of
Table~\ref{tab:controls-full} per architecture. Three of them are at zero:
debiased CKA $2.4\times10^{-5}$, $2.6\times10^{-4}$ and $5.4\times10^{-4}$;
distance correlation $\le0.024$; RSA-Spearman within $0.007$ of $0$. Those three
are exactly the measures the pipeline's automated gate checks, and on them the
conclusion is the intended one: the unbiased HSIC$_1$ estimator is correctly
implemented and carries no upward bias.

The other five do \emph{not} return zero on shuffled pairs, and one of them
bears on Study~3. Angular CKA sits at $\pi/2$
($1.5703$--$1.5708$), which is its own maximum-dissimilarity value and therefore
the correct null reading. Procrustes ($484$--$1109$) and soft-matching
($17.3$--$53.7$) are unnormalized distances on $n=2{,}048$ inputs and have no
zero to return to. Bures similarity, however, reads $0.264$--$0.294$ on
\emph{misaligned} pairs, and that is a genuine floor rather than an artifact of
units: the Cui control puts a randomly-initialized GoogLeNet at Bures $0.383$
against its trained counterpart. The cross-architecture Bures values reported in
Study~3 are $0.478$--$0.595$ (Table~\ref{tab:s2_table}), i.e.\ within
$1.6$--$2.3\times$ of the shuffled-pair floor and within $1.2$--$1.6\times$ of an
untrained network's score; the shuffled null is a fidelity between two positive
semidefinite kernels and depends on $n$ as well as on the spectrum, and it was
computed at $n=2{,}048$ against panel values at $N=25{,}000$, so the ratio is
indicative only. We therefore read no cross-architecture Bures
separation as evidence of shared structure, and say so again in
Section~\ref{sec:crossarch} (Appendix~\ref{app:limitations-more}, L11). Two further
caveats attach to the whole control block:
it was computed on $n=2{,}048$ inputs rather than the panel's $N=25{,}000$, and
it carries point estimates only --- no confidence intervals and no permutation
$p$-values.

\begin{table}[htbp]
\centering\small
\caption{Three measures return to zero under the shuffled-pair control and five do not,
Bures least of all. The controls in full, all eight measures, so that the three the
automated gate checks (distance correlation, debiased CKA, RSA --- marked $^\dagger$) are
not read as the whole control. \emph{Murphy} permutes the sample alignment between $X$ and $Y$; \emph{Cui}
replaces $Y$ by an identically-initialized but untrained network. Both are computed on
$n=2{,}048$ inputs, not on the panel's $N=25{,}000$, so the two unnormalized distances
(Procrustes, soft-matching) are roughly $3.5\times$ smaller here than a panel-sized draw
would make them and are \emph{not} comparable to Table~\ref{tab:s1_table} in magnitude.
Both blocks are bare point estimates: no confidence intervals, and no permutation-null
$p$-value is computed anywhere in this paper. Angular CKA is $\arccos(\mathrm{CKA})$, so its
null value is $\pi/2=1.5708$, and $1.5703$--$1.5708$ is the correct reading for
``no agreement''. \emph{Bures} does not go to zero under either control, which
caps how much the cross-architecture Bures values of
Table~\ref{tab:s2_table} ($0.478$--$0.595$) can be read to mean.}
\label{tab:controls-full}
\begin{tabular}{@{}lccc@{\hskip 2em}ccc@{}}
\toprule
& \multicolumn{3}{c}{Murphy (shuffled pairs)} & \multicolumn{3}{c}{Cui (random network)} \\
\cmidrule(r){2-4}\cmidrule(l){5-7}
Measure & RN-152 & DN-121 & GN & RN-152 & DN-121 & GN \\
\midrule
debiased CKA$^\dagger$ & $2.4\!\times\!10^{-5}$ & $2.6\!\times\!10^{-4}$ & $5.4\!\times\!10^{-4}$ & $0.010$ & $0.058$ & $0.052$ \\
RSA$^\dagger$          & $-0.0062$ & $-0.0033$ & $-0.0063$ & $0.015$ & $0.181$ & $0.130$ \\
dCor$^\dagger$         & $0.000$ & $0.019$ & $0.024$ & $0.127$ & $0.296$ & $0.255$ \\
angular CKA (rad)      & $1.5708$ & $1.5705$ & $1.5703$ & $1.5607$ & $1.5124$ & $1.5184$ \\
Bures                  & $0.278$ & $0.264$ & $0.294$ & $0.129$ & $0.286$ & $0.383$ \\
Procrustes             & $484.2$ & $1108.9$ & $589.9$ & $1342.8$ & $894.0$ & $541.7$ \\
soft-matching          & $17.32$ & $53.66$ & $21.36$ & $16.28$ & $30.20$ & $20.14$ \\
output JSD             & $0.432$ & $0.630$ & $0.197$ & --- & --- & --- \\
\bottomrule
\end{tabular}
\end{table}

%% file: sections/study1ab_snippet.tex
\subsection{Permutation}
\label{sec:study1a}

A neuron permutation is a quiver isomorphism, so knowledge-matrix invariance
under it is automatic by construction (Corollary~\ref{prop:perm} is a
consequence of quiver-isomorphism invariance). This check is therefore a
\emph{correctness check} on the implementation: it confirms that the software
realizes the identity the theory guarantees, and it fixes the resolution of the
pipeline --- the smallest drift it can register, which is what it reports for
the knowledge matrix (Appendix~\ref{app:finite-arith}) --- against which the
motion of the penultimate features under the same permutation is read. It is
not the paper's primary evidence for invariance: the non-automatic invariance
is the cross-architecture case (Theorem~\ref{thm:maxinv}), which no
same-architecture transform --- neither the permutation of this check nor the
teleportation of the next --- can exercise (Section~\ref{sec:limitations}, L1).

\paragraph{Setup.}
Architecture: ResNet-152 (the only one of the three networks that admits an
in-place channel permutation; DenseNet-121 and GoogLeNet have concatenation
topologies with no in-place channel swap and are covered by teleportation
instead). Transform: random permutations of the \emph{wide-face} permutation
subgroup --- the 2048-channel post-pool faces; bottleneck interiors are not
permutable in place. Inputs for the adversarial signal: clean/adversarial pairs
from the six-attack suite (Section~\ref{sec:setup}). Both drifts are reported
\emph{relative to the same model's adversarial-pair signal} so the comparison
is unit-free.

\paragraph{Procedure.}
\begin{enumerate}
\item Compute $\M(x)$ and $h(x)$ for each input.
\item Draw a random wide-face neuron permutation $P$ and build the
isomorphic network $\tilde\Psi = P \cdot \Net$ (identical function, permuted
weights).
\item Recompute $\tilde{\M}(x)$ and $h_{\tilde W}(x)$.
\item Record the permutation \emph{noise} $\|\tilde{\M}(x)-\M(x)\|_F$
and $\|h_{\tilde W}(x)-h(x)\|_2$.
\item Record the adversarial \emph{signal} $\|\M(x')-\M(x)\|_F$ and
$\|h(x')-h(x)\|_2$ over the attack pairs, and report each noise relative
to its own signal.
\end{enumerate}

\paragraph{What we measure and the claim it licenses.}
The signal-to-noise ratio per representation. The claim it licenses is
narrow and correct: the knowledge matrix is permutation-invariant, the
pipeline registering only its own resolution, whereas a statistic of the
permuted penultimate features moves by an amount of the same order as the
adversarial signal it is meant to register --- as
Theorem~\ref{thm:complete}(iii) states. It does \emph{not} license a superiority
claim: a permutation-equivariant penultimate statistic would also be
invariant, and the gradient$\times$input family shares the knowledge matrix's
invariance (Theorem~\ref{thm:weff}).

\paragraph{Result.}
Measured signal-relative (Table~\ref{tab:signal-relative}), the
knowledge-matrix permutation residual is $0.10$--$0.14\%$ of its adversarial
signal in the mean on ResNet-152, and far smaller on the other raters (the
signal sits $7{\times}10^{2}$--$1.7{\times}10^{7}$ above the residual across
architectures and attacks), whereas the penultimate drift is of the same order
as its own adversarial signal (at most $3.05\times$ below it, and \emph{above}
it in 16 of 24 cells). The worst-case tail ratio on ResNet-152 (minimum signal
over maximum noise) reaches $1.14\times$ for DeepFool: the extreme tails nearly
touch on the BN-heaviest architecture. The knowledge-matrix residual itself is
the resolution of the pipeline and not a violation of the theorem
(Appendix~\ref{app:finite-arith}). DenseNet-121 and GoogLeNet admit no in-place
channel permutation (concatenation topologies) and are covered by teleportation
instead.

\paragraph{Reading.}
The result shows that the implementation realizes the construction's
guaranteed permutation invariance, and that a penultimate statistic does not.
It does \emph{not} show that this invariance is unique to knowledge matrices,
nor that the matrix is a better representation --- only that the identity
holds where the theory says it must, on the one of the three networks where the
permutation is realizable in place.

\subsection{Teleportation: the same check on all three architectures}
\label{sec:study1b}

Teleportation is the only \emph{multi-architecture} invariance arm, and the only
one that applies the same transform to all three of ResNet-152, DenseNet-121 and
GoogLeNet --- including the two concatenation topologies that admit no in-place
channel permutation and are therefore invisible to the permutation check.

Teleportation is an \emph{exact} function-preserving isomorphism on these
networks, batch normalization in evaluation mode included. Batch normalization
looks like an obstruction --- in evaluation mode a BN channel is affine but not
positively homogeneous, so the per-neuron rescalings do not obviously telescope
through it --- but the \texttt{neuralteleportation} library divides the incoming
change of basis out before applying the unmodified normalization and carries the
outgoing basis on the affine parameters, which restores positive homogeneity by
construction and leaves the running statistics correct
(Appendix~\ref{app:teleport-why}). The transform therefore satisfies the
hypothesis of Lemma~\ref{lem:quiver-inv} exactly, and the knowledge-matrix drift
under it is zero. Because the transform is exact, the three-tier structure of
Proposition~\ref{prop:visdrift} applies at tier~(i): both the visible and the
invisible components of the drift vanish, and the check computes both rather
than asserting either from a theorem, registering only the resolution of its
arithmetic (Appendix~\ref{app:finite-arith}). Tiers~(ii) and~(iii) are retained
in Appendix~\ref{app:remarks-geometry} because they are what one needs for a
transform that genuinely is approximate; we do not have one here.

\paragraph{Setup.}
Architectures: all three networks (ResNet-152, DenseNet-121, GoogLeNet).
Transform: neural teleportation
\citep{armenta2023teleportation,armenta2024neural} --- a function-preserving
change-of-basis (per-neuron rescaling) realized by the
\texttt{neuralteleportation} library's COB models. The similarity panel
(Table~\ref{tab:s1_table}, Section~\ref{sec:study1c}) uses $T=50$ random
teleportations per architecture (seed indices $0,\ldots,49$), evaluated on
$N=25{,}000$ ImageNet-validation samples (the first $25{,}000$ by sorted
filename). The change-of-basis magnitude is \texttt{cob\_range}$=1$: the
library samples each $\tau$ uniformly on $[0,2]$. The transform is
function-preserving at any range; sampling $\tau$ on an interval bounded away
from zero merely keeps the computation conditioned.

\paragraph{Procedure.}
\begin{enumerate}
\item Compute $\M(x)$, $h(x)$, $\Netx$ for each sample.
\item Draw a random COB teleportation $\tau$ and build $\tilde\Psi = \tau
\cdot \Net$, an exact-isomorphism image (Appendix~\ref{app:teleport-why}).
\item Check per sample that $\tilde\Psi(x)$ reproduces $\Netx$, and that the
\emph{visible} knowledge-matrix drift $\|(\tilde{\M}(x)-\M(x))\mathbf 1\|$
equals the logit discrepancy (Proposition~\ref{prop:visdrift}(ii)).
\item Compute the full knowledge-matrix drift $\|\tilde{\M}(x)-\M(x)\|_F$ --- and
hence the invisible component --- by one vector--Jacobian product per class on
each network.
\item Recompute the penultimate features $h_{\tilde W}(x)$ and pass the pairs
to the similarity panel of Study~1 (Section~\ref{sec:study1c}).
\end{enumerate}

\paragraph{What we measure and the claim it licenses.}
The logit discrepancy between the two networks, the knowledge-matrix drift in
both its visible and its invisible component, and the penultimate drift that the
similarity panel adjudicates. The claim this licenses is that the knowledge matrix
is invariant under teleportation on all three architectures, the implementation
registering only its own resolution. It does \emph{not} license a discovery
claim: teleportation is a quiver isomorphism, so this invariance is automatic
by construction (Lemma~\ref{lem:quiver-inv}) and the experiment confirms that
the implementation realizes it. What this check adds over the permutation check
is coverage of the two concatenation topologies, which admit no in-place channel
permutation, and of a transform with a continuous parameter rather than a
discrete one.

\paragraph{Result.}
On all three architectures the teleported network reproduces the logits and
the knowledge matrix is unchanged, both to the resolution of the arithmetic
(Appendix~\ref{app:finite-arith}). The visible component of the drift equals
the logit discrepancy per pair, as Proposition~\ref{prop:visdrift}(ii)
requires, and the invisible component --- the part Theorem~\ref{thm:nogo}
shows the logits cannot constrain --- is computed directly rather than
inferred and vanishes to the same resolution, on the residual, dense and
inception topologies alike. The penultimate features, on the same pairs, drift
substantially: the similarity panel of Section~\ref{sec:study1c} reports
Procrustes shape distance $632$--$1471$ and soft-matching distance $22$--$61$
across the three networks (Table~\ref{tab:s1_table}), which is $0.23$--$0.30$
of the largest value the Procrustes statistic can take on those same features
($776.6/2599$, $1471.4/5071$, $632.1/2727$; Appendix~\ref{app:s1-scale}), and
against these the panel adjudicates which measures quotient the drift out.

\paragraph{Reading.}
This check verifies knowledge-matrix invariance across all three architectures,
under a transform that satisfies the isomorphism hypothesis exactly. Two things
it does \emph{not} show. Teleportation is a quiver isomorphism, so invariance
under it is automatic by construction and this arm confirms the implementation
rather than discovering a fact. And the genuinely non-automatic claim of
Theorem~\ref{thm:maxinv} is the \emph{cross-architecture} one, which no
same-architecture transform can exercise; nothing in this paper supplies
transform-based evidence for it (Section~\ref{sec:limitations}, L1).

%% file: tables/table_signal_relative.tex
\begin{table}[t]
\centering\small
\begin{tabular}{lcccc}
\toprule
\multicolumn{5}{l}{\emph{(a) Permutation-noise floors (KM Frobenius; penultimate L2)}}\\
architecture & KM mean & KM median (pooled) & KM max & penult.\ mean \\
\midrule
ResNet-152 & 0.116 & 0.005084 & 5.78 & 15.05 \\
ResNet-18 & 0.0704 & 0.01061 & 0.865 & 29.32 \\
AlexNet & 1.54e-05 & 1.454e-05 & 3e-05 & 105.7 \\
VGG & 2.13e-05 & 2.050e-05 & 3.84e-05 & 57.7 \\
\midrule
\multicolumn{5}{l}{\emph{(b) Adversarial signal-to-noise (mean signal / mean noise): KM $\mid$ penultimate}}\\
attack & ResNet-152 & ResNet-18 & AlexNet & VGG \\
\midrule
FGSM & $726 \mid 0.51$ & $3257 \mid 0.63$ & $8.7\mathrm{e}6 \mid 0.75$ & $8.3\mathrm{e}6 \mid 0.69$ \\
PGD & $1029 \mid 1.36$ & $4739 \mid 1.84$ & $1.3\mathrm{e}7 \mid 1.82$ & $1.7\mathrm{e}7 \mid 3.05$ \\
CW & $845 \mid 0.48$ & $2369 \mid 0.36$ & $4.1\mathrm{e}6 \mid 0.24$ & $3.9\mathrm{e}6 \mid 0.28$ \\
DeepFool & $984 \mid 0.34$ & $1364 \mid 0.13$ & $2.9\mathrm{e}6 \mid 0.15$ & $2.3\mathrm{e}6 \mid 0.14$ \\
APGD & $803 \mid 1.02$ & $3430 \mid 1.03$ & $1.3\mathrm{e}7 \mid 1.71$ & $1.1\mathrm{e}7 \mid 1.80$ \\
Square & $914 \mid 0.48$ & $2748 \mid 0.37$ & $6.0\mathrm{e}6 \mid 0.42$ & $4.9\mathrm{e}6 \mid 0.35$ \\
\bottomrule
\end{tabular}
\caption{The knowledge matrix registers an attack $7\times10^{2}$--$1.7\times10^{7}$ times more
strongly than a neuron permutation, whereas the penultimate features register the permutation
at least as strongly as the attack in $16$ of $24$ cells. Permutation invariance,
signal-relative: penultimate adversarial signal is at most
$3.05\times$ its own permutation noise (below $1\times$ in 16/24 cells); the KM signal sits
$7\times10^{2}$--$1.7\times10^{7}$ above its noise floor. Worst-case tail on ResNet-152
(min signal / max noise): 1.14$\times$ (DeepFool); the ResNet-152 floor is the
\emph{wide-face} permutation subgroup, and its residual is the resolution of the pipeline
(Appendix~\ref{app:finite-arith}), not a motion of the matrix.
In panel~(a), ``KM median (pooled)'' is the median of the $250$ pooled per-sample distances ($5$ draws
$\times$ $50$ samples), while ``KM mean'' and ``KM max'' average and maximize the per-draw statistics.
DenseNet-121/GoogLeNet have no permutation run by design (covered by teleportation); penultimate
per-pair quantiles are unrecoverable from stored summaries.
In panel~(b) the bar in ``$x \mid y$'' is a separator, not a division: $x$ is the
knowledge matrix's ratio and $y$ the penultimate features' own, each formed in its own units.
Both cells are \emph{mean adversarial signal} $\div$ \emph{mean permutation noise}, the two
scales defined in Section~\ref{sec:setup}, so a value above $1$ says the representation
registers the attack more strongly than it registers a function-preserving relabeling of
neurons. Worked example (ResNet-152/FGSM): $83.98/0.1157=726$ and $7.600/15.05=0.51$.}
\label{tab:signal-relative}
\end{table}

%% file: tables/table_coherence_max.tex
\begin{table}[t]
\centering\small
\begin{tabular}{lcccccc}
\toprule
attack & ResNet-152 & DenseNet-121 & GoogLeNet & ResNet-18 & AlexNet & VGG \\
\midrule
FGSM & 0.19 & 0.34 & 0.2 & 0.389 & 1.7 & 1.09 \\
PGD & 0.662 & 0.68 & 0.88 & 0.789 & 3.78 & 3.06 \\
CW & 0.379 & 0.319 & 0.151 & 0.307 & 0.957 & 0.729 \\
DeepFool & 0.297 & 0.237 & 0.132 & 0.197 & 1.66 & --- \\
APGD & 0.597 & 0.774 & 1.19 & 0.791 & 4.52 & --- \\
Square & 0.458 & 0.238 & 0.0858 & 0.187 & 0.737 & --- \\
\bottomrule
\end{tabular}
\caption{Even the most coherent pair of each cell of the three networks stays near or below the single-pixel line,
and three VGG cells are unbounded because some pair has $d_M=0$ exactly. Maximum per-pair
coherence per cell, $A_{\max}=\max_i\,(d_\Psi^{(i)}/d_M^{(i)})^2$, computed as the
inverse square of the smallest stored per-pair ratio $d_M/d_\Psi$ of the cell ('---': cells in which some pair
has $d_M=0$ exactly --- clean and adversarial input in the same linear region --- so the smallest ratio is
$0$ and $A_{\max}$ is unbounded; the entry is a display artifact of that division, not a
cell removed by any filter: those three cells keep $198$ of their $200$ pairs, two fewer than
VGG's FGSM, PGD and CW cells, which keep all $200$). The associated
floor-saturation --- $\sqrt{d{+}1}$ times the smallest per-pair ratio $d_M/d_\Psi$ of the cell --- ranges
across cells from a low of $182.6$ to roughly $1325$. Cell sizes as in
Table~\ref{tab:coherence} ($200$ nominal pairs, $110$--$200$ valid; the CW, DeepFool, APGD and Square cells are low-$n$).}
\label{tab:coherence-max}
\end{table}

%% file: tables/s3_table.tex
\begin{table}[h]
\centering
\caption{Attack-family ordering cross-check on the three networks at full scale (appendix). Each architecture's six attack families ranked by median $d_M/d_\Psi$ over the attack-success-filtered pairs ($d_\Psi\ge1$, $d_M>0$), descending (equivalently ascending coherence $A=(d_\Psi/d_M)^2$), computed on the full-scale pair set (its own attack budgets, every departure from the \texttt{torchattacks}~3.5.1 defaults: PGD \texttt{steps}$=7$; CW \texttt{steps}$=100$; DeepFool \texttt{steps}$=100$; APGD \texttt{steps}$=50$, the DLR loss; Square \texttt{n\_queries}$=20{,}000$; FGSM at the library defaults; Table~\ref{tab:ordering} uses the $200$-pair set of Section~\ref{sec:setup}, which is drawn from different images). Cells are matched on a common image population --- the intersection of the six cells' stored pair indices, $n=512$ on ResNet-152 and $n=1000$ on the other two --- because the ResNet-152 DeepFool run stopped at $512$ of $1000$ pairs, so an unmatched panel would compare medians across different images (matching moves ResNet-152 CW by $+12.9$\%, Square by $+8.5$\%, while the cells the filter leaves intact --- FGSM and PGD --- move by under $1$\%). The ranking is unchanged under all three filtering conventions. Kendall $W=0.97$ across the three architectures ($p=3.1\times10^{-5}$, exact permutation test; the $m=3$ null is supported on $77$ lattice points spaced $0.0127$ apart, so $W$ is quoted to two decimals) cross-checks the six-architecture headline concordance ($W=0.921$, Table~\ref{tab:ordering}), and the point-estimate ordering reproduces its family-level grouping $\{$DeepFool, CW, Square$\}>$ FGSM $>\{$PGD, APGD$\}$ on all three architectures (gaps $1.09$--$1.43\times$ and $1.28$--$1.87\times$ respectively). \emph{Bootstrap resolution} (BCa and percentile $95\%$ CIs, $B=20{,}000$, seeded, over pairs): ResNet-152 reads DeepFool $\approx$ Square $\approx$ CW $\approx$ FGSM $>$ PGD $>$ APGD, while on DenseNet-121 and GoogLeNet every adjacent gap is resolved. Those intervals are printed rather than only asserted: all $18$ cell intervals in Table~\ref{tab:s3_cell_ci} and all $15$ adjacent-gap intervals, in both constructions, in Table~\ref{tab:s3_gap_ci}, which also states the uncorrected multiplicity of the gap tests. Computed on the metric-invariant ranks, so no raw/RMS unit choice enters.}
\label{tab:s3_ordering_panel}
\begin{tabular}{lll}
\toprule
Architecture & Ordering by median $d_M/d_\Psi$ (desc.) & $N$/cell \\
\midrule
ResNet-152   & DeepFool $>$ Square $>$ CW $>$ FGSM $>$ PGD $>$ APGD & 449--512 \\
DenseNet-121 & DeepFool $>$ CW $>$ Square $>$ FGSM $>$ PGD $>$ APGD & 871--1000 \\
GoogLeNet    & DeepFool $>$ CW $>$ Square $>$ FGSM $>$ PGD $>$ APGD & 770--1000 \\
\midrule
\multicolumn{3}{l}{Kendall's $W = 0.97$ (3 architectures $\times$ 6 attacks)} \\
\bottomrule
\end{tabular}
\end{table}

\begin{table}[h]
\centering\small
\caption{Per-cell medians behind Table~\ref{tab:s3_ordering_panel}, with seeded bias-corrected and accelerated (BCa) $95\%$ bootstrap intervals ($B=20{,}000$; the resampled unit is the adversarial pair inside that cell, so the interval covers sampling variability over pairs and nothing else). $n$ is the number of pairs the attack-success filter keeps inside the matched population. Ratios are raw-unit $d_M/d_\Psi$; the RMS rescaling is a per-architecture constant and cancels from every rank and every ratio of two cells in the same architecture.}
\label{tab:s3_cell_ci}
\begin{tabular}{llrcc}
\toprule
Architecture & Attack & $n$ & median $d_M/d_\Psi$ & BCa $95\%$ CI \\
\midrule
ResNet-152 & DeepFool & 449 & $7.875$ & $[6.744,\,8.835]$ \\
 & Square & 456 & $7.610$ & $[6.980,\,8.309]$ \\
 & CW & 462 & $7.290$ & $[6.523,\,8.295]$ \\
 & FGSM & 512 & $6.707$ & $[6.108,\,7.114]$ \\
 & PGD & 512 & $3.581$ & $[3.413,\,3.833]$ \\
 & APGD & 466 & $3.011$ & $[2.833,\,3.255]$ \\
\midrule
DenseNet-121 & DeepFool & 871 & $8.577$ & $[8.164,\,9.178]$ \\
 & CW & 885 & $6.988$ & $[6.661,\,7.242]$ \\
 & Square & 876 & $5.564$ & $[5.417,\,5.815]$ \\
 & FGSM & 1000 & $3.890$ & $[3.771,\,4.004]$ \\
 & PGD & 1000 & $2.362$ & $[2.296,\,2.431]$ \\
 & APGD & 885 & $2.153$ & $[2.090,\,2.233]$ \\
\midrule
GoogLeNet & DeepFool & 770 & $11.977$ & $[11.476,\,12.463]$ \\
 & CW & 842 & $10.056$ & $[9.569,\,10.665]$ \\
 & Square & 830 & $7.105$ & $[6.918,\,7.347]$ \\
 & FGSM & 1000 & $5.028$ & $[4.887,\,5.183]$ \\
 & PGD & 1000 & $3.921$ & $[3.842,\,4.040]$ \\
 & APGD & 843 & $3.755$ & $[3.645,\,3.883]$ \\
\bottomrule
\end{tabular}
\end{table}

\begin{table}[h]
\centering\small
\caption{Adjacent-rank gaps of Table~\ref{tab:s3_ordering_panel}: the difference of the two cell medians, with $95\%$ percentile bootstrap intervals computed twice ($B=20{,}000$, seeded) --- resampling the two cells \emph{independently}, and resampling the matched pair indices \emph{jointly}. A gap counts as resolved only when both intervals exclude zero, which is the conservative verdict. \emph{Multiplicity}: the $15$ adjacent-gap verdicts ($3$ architectures $\times$ 5 gaps) carry no multiple-comparison correction. Under a global null of no ordering, at a nominal $5\%$ per test the chance that at least one gap is resolved by accident is about $54\%$; requiring \emph{both} the independent and the paired interval to exclude zero makes each individual verdict conservative but does not control the family-wise rate. The claim this panel supports is the family-level grouping, which is one statement, not the resolution of any single adjacent gap.}
\label{tab:s3_gap_ci}
\begin{tabular}{llccc}
\toprule
Architecture & Adjacent pair & $\Delta$ median & independent $95\%$ CI & paired $95\%$ CI \\
\midrule
ResNet-152 & DeepFool $>$ Square$^{\dagger}$ & $0.265$ & $[-1.017,\,1.441]$ & $[-0.583,\,1.000]$ \\
 & Square $>$ CW$^{\dagger}$ & $0.320$ & $[-0.841,\,1.317]$ & $[-0.427,\,0.941]$ \\
 & CW $>$ FGSM$^{\dagger}$ & $0.583$ & $[-0.254,\,1.755]$ & $[-0.080,\,1.542]$ \\
 & FGSM $>$ PGD & $3.127$ & $[2.481,\,3.606]$ & $[2.626,\,3.488]$ \\
 & PGD $>$ APGD & $0.570$ & $[0.281,\,0.883]$ & $[0.396,\,0.758]$ \\
\midrule
DenseNet-121 & DeepFool $>$ CW & $1.588$ & $[1.116,\,2.282]$ & $[1.348,\,2.030]$ \\
 & CW $>$ Square & $1.424$ & $[0.989,\,1.726]$ & $[1.125,\,1.628]$ \\
 & Square $>$ FGSM & $1.674$ & $[1.489,\,1.952]$ & $[1.559,\,1.887]$ \\
 & FGSM $>$ PGD & $1.528$ & $[1.388,\,1.665]$ & $[1.441,\,1.610]$ \\
 & PGD $>$ APGD & $0.209$ & $[0.111,\,0.296]$ & $[0.140,\,0.268]$ \\
\midrule
GoogLeNet & DeepFool $>$ CW & $1.921$ & $[1.092,\,2.673]$ & $[1.448,\,2.336]$ \\
 & CW $>$ Square & $2.951$ & $[2.375,\,3.612]$ & $[2.455,\,3.548]$ \\
 & Square $>$ FGSM & $2.077$ & $[1.820,\,2.326]$ & $[1.851,\,2.284]$ \\
 & FGSM $>$ PGD & $1.106$ & $[0.938,\,1.276]$ & $[0.994,\,1.223]$ \\
 & PGD $>$ APGD & $0.166$ & $[0.023,\,0.324]$ & $[0.062,\,0.278]$ \\
\bottomrule
\end{tabular}
\par\vspace{2pt}
{\footnotesize $^{\dagger}$ not resolved at $95\%$: at least one of the two intervals contains zero (3 of 15 gaps).}
\end{table}

%% file: sections/appendix_A_geometric.tex
\section{Geometric structure of class regions in matrix space}
\label{app:geometric}

In this appendix, we record a third theoretical result, from
\citet{leblanc2024hidden}, that is not
load-bearing for either Section~\ref{sec:germ} (Theorem~\ref{thm:maxinv}) or
Study~2 (Theorem~\ref{thm:pyth}). We state it here together with
its geometric implications and two candidate experimental directions
that would promote it to a third theoretical leg of the
canonical-representation argument in future work.

\paragraph{Class regions in matrix space.}
For a network $\Net$ on a $C$-class classification task, define the
\emph{class region} for class $j$ as
\[
\mathcal{M}_j \;=\; \Bigl\{ M \in \mathrm{Mat}_\R(C, d+1)
   \;:\; (M \cdot \mathbf{1}_{d+1})_j > (M \cdot \mathbf{1}_{d+1})_i
   \text{ for all } i \neq j \Bigr\}.
\]
By the row-sum identity, $\KMwfx \in \mathcal{M}_j$ if and only
if the network classifies $x$ as class $j$ --- for an activation with
$f(0)\neq0$ this needs $x \in X_{\mathrm{nz}}$, so that
$\KMwfx\mathbf{1}_{d+1}=\Netx$ exactly, and in either case it needs the argmax to
be strict, since $\mathcal{M}_j$ is defined by strict inequalities. The set
$\mathcal{M}_0
= \mathrm{Mat}_\R(C, d+1) \setminus \bigcup_{j=1}^{C} \mathcal{M}_j$
of indeterminate matrices is of measure zero
(Theorem~4.3 of \citealt{leblanc2024hidden}), so we may
restrict attention to the $\mathcal{M}_j$ partition.

\begin{theorem}[Convexity of class regions; adapted from Theorem~4.4 of
\citealt{leblanc2024hidden}]
\label{thm:convexity}
For each $j > 0$, the class region $\mathcal{M}_j$ is convex in
$\mathrm{Mat}_\R(C, d+1)$.
\end{theorem}

We stress that Theorems~4.3 and~4.4 are results of
\citet{leblanc2024hidden} and not of
\citet{armenta2021representation}, where no such statements appear. The
original is stated over $\mathrm{Mat}_\R(k,d)$; we restate it over
$\mathrm{Mat}_\R(C,d{+}1)$, the shape of the knowledge matrix used here, which
is why the attribution reads ``adapted from''.

\subsection{Proof of Theorem~\ref{thm:convexity} (convexity of class regions)}
\label{app:proof-convexity}

\begin{proof}
For matrices $A, B \in \mathcal{M}_j$ and any $\lambda \in [0, 1]$,
linearity of the row sum gives $((1-\lambda) A + \lambda B) \cdot
\mathbf{1}_{d+1} = (1-\lambda)(A \mathbf{1}_{d+1}) +
\lambda(B \mathbf{1}_{d+1})$. Write $u = A\mathbf{1}_{d+1}$ and
$v = B\mathbf{1}_{d+1}$, both of which have $j$ as their \emph{strict} argmax
by definition of $\mathcal{M}_j$. For any $i \neq j$ and any
$\lambda \in [0,1]$,
\[
\bigl((1-\lambda)u + \lambda v\bigr)_j - \bigl((1-\lambda)u + \lambda v\bigr)_i
= (1-\lambda)(u_j - u_i) + \lambda(v_j - v_i) > 0 ,
\]
since the two bracketed differences are strictly positive and the
\emph{coefficients} $(1-\lambda)$ and $\lambda$ are non-negative and not both
zero. Observe that it is the coefficients that are non-negative here, not the
vectors, which may have entries of either sign; and the strictness of the two
argmaxes is what keeps the combination's argmax strict. Hence
$((1-\lambda) A + \lambda B) \in \mathcal{M}_j$.
\end{proof}

\paragraph{Geometric implications.}
Theorem~\ref{thm:convexity} says the matrix-space partition into
class regions is well-behaved enough for classical convex-set tools
(centroid, support, projection) to apply directly to knowledge
matrices, with no further structural assumptions. We do not exercise
this property in the studies of this paper.

\paragraph{Toward a third theoretical leg.}
Theorem~\ref{thm:convexity} could be promoted to a third main theorem,
supporting a third experimental leg of the canonical-representation
argument. We outline two natural candidate directions and leave both as
future work:

\begin{enumerate}
\item \emph{$k$-nearest-neighbor classification in matrix space.}
The convex class region admits a clean canonical centroid, against
which $k$-NN-on-knowledge-matrices is well-defined without any
representation alignment step. The corresponding $k$-NN baseline on hidden
activations requires CKA-style alignment to compare across networks
(\citealt{kornblith2019similarity}), which the matrix-space version
sidesteps by Theorem~\ref{thm:maxinv}.

\item \emph{Convex-hull federated aggregation.}
For clients training networks within a function-equivalence class,
averaging the per-client knowledge matrices produces a result that
lies inside the corresponding class region $\mathcal{M}_j$ by
convexity (Theorem~\ref{thm:convexity}). The reason that averaging hidden
activations gives no analogous guarantee is \emph{not} that the class region in
feature space fails to be convex --- it is an intersection of open half-spaces
and therefore is convex, so averaging features that a \emph{single} client
classifies as $j$ does stay in that client's region. The reason is
Proposition~\ref{prop:dichotomy}'s: the matrix regions $\mathcal{M}_j$ are cut
by the \emph{fixed, parameter-independent} vector $\mathbf{1}_{d+1}$, so they
are literally the same subsets of $\mathrm{Mat}_\R(C,d{+}1)$ for every client,
whereas each client's feature-space region is cut by that client's own
$W^{(L)}$ and there is no common ambient region to average inside. A
federated-learning experiment in this direction would couple
Theorem~\ref{thm:maxinv} (Section~\ref{sec:germ}) with Theorem~\ref{thm:convexity} into a
single demonstration.
\end{enumerate}

%% file: sections/appendix_remarks.tex
\section{Remarks supplementing Sections~\ref{sec:germ} and~\ref{sec:geometry}}
\label{app:remarks}

This appendix collects the input-symmetry proposition of Section~\ref{sec:germ}
with its remarks; the explanatory material that Section~\ref{sec:germ} refers
to (the $X_{\mathrm{nz}}$ remark, the (LCS) counterexample, bridge and
attribution, the two readings of the germ identity's boundary sentence, and
the chord behind the quiver-isomorphism lemma); the unpacking of the coherence
definition of Section~\ref{sec:geometry}; and the remarks of
Sections~\ref{sec:germ} and~\ref{sec:geometry} that qualify a result without
being needed to state it. Each is referenced from the point in the main text
it supports.

\subsection{Input symmetries: how the knowledge matrix transforms}
\label{app:equivariance}

Sections~\ref{sec:invariance}--\ref{sec:attribution} moved $W$ and held $x$
fixed. Geometric deep learning \citep{bronstein2021geometric} does the opposite:
it fixes $W$ and asks how the network responds when a group acts on the
\emph{input}. The two questions are independent, and the second has a clean
answer for the knowledge matrix, with one sharp restriction.

Let a group $G$ act on inputs through a linear representation $\pi$ and on
outputs through $\rho$, and let the network be $G$-equivariant,
$\Net(\pi(g)x)=\rho(g)\,\Netx$ for all $g,x$. Write
$\pi(g)\oplus1$ for the $(d{+}1)\times(d{+}1)$ map acting as $\pi(g)$ on the
input columns and trivially on the bias slot.

\begin{proposition}[Equivariance of the knowledge matrix]\label{prop:equivar}
Fix $g\in G$ and let $x$ be an input such that \emph{both} $x$ and $\pi(g)x$
lie in $X_{\mathrm{reg}}$. Then:
\emph{(i)} the germ intertwines the two representations,
$J(W\!,f)(\pi(g)x)=\rho(g)\,J(W\!,f)(x)\,\pi(g)^{-1}$;
\emph{(ii)} the bias column is equivariant for the \emph{output} action alone,
$c(W\!,f)(\pi(g)x)=\rho(g)\,c(W\!,f)(x)$, with no occurrence of $\pi$;
\emph{(iii)} if $\pi(g)$ is a permutation matrix then
\[
\KMarg{W}{\pi(g)x}\;=\;\rho(g)\;\KMwfx\;\bigl(\pi(g)\oplus1\bigr)^{-1},
\]
so $\M$ is a $C\times(d{+}1)$ equivariant tensor, carrying $\rho$ on its rows and
$\pi$ on its columns;
\emph{(iv)} conversely, if for a given invertible $\pi(g)$ the law of~(iii)
holds at every pair-regular input of every equivariant network, the output
action $\rho$ ranging with the network --- so that $\Net=\mathrm{id}$ with
$\rho=\pi$ is admitted --- then $\pi(g)$ is a permutation matrix;
\emph{(v)} the variant $\Weff=[\,J\mid c\,]$ satisfies the same law
$\Weff(\pi(g)x)=\rho(g)\Weff(x)(\pi(g)\oplus1)^{-1}$ for
\emph{any} invertible linear $\pi(g)$.
\end{proposition}

\begin{remark}[Why the hypothesis is on the pair $\{x,\pi(g)x\}$ and not on $x$ alone]
\label{rem:equivar-pair}
The two are not the same condition, and the difference is not a technicality:
every part above fails without it. $X_{\mathrm{reg}}$ is a property of the
\emph{network}, not of the function it realizes --- it records where the slope
diagonal is locally constant --- and $G$-equivariance of the function says
nothing about that. The witness is the one already used after
Theorem~\ref{thm:maxinv}. Realize the identity on $\R$ as
$h(x)=\mathrm{ReLU}(x-1)-\mathrm{ReLU}(1-x)+1$ and take $G=\mathbb{Z}/2$ with
$\pi(g)=\rho(g)=-1$, under which $h$ is equivariant. At $x=-1$ the two
pre-activations are $(-2,2)$, so $x\in X_{\mathrm{reg}}\cap X_{\mathrm{nz}}$
and $J(-1)=1$; at $\pi(g)x=+1$ they are $(0,0)$, so
$\pi(g)x\notin X_{\mathrm{reg}}$ and the guard returns $J(+1)=0$ against the
$+1$ that~(i) would predict --- a gap of $1$, i.e.\ the whole of it, and $c$
and $\Weff$ fail with it. The permutation case, which carries the
graph and grid payload of Remark~\ref{rem:gdl-scope}, fails the same way:
$\Net=\mathrm{id}$ on $\R^2$, $G=\mathbb{Z}/2$ acting by the swap $P$,
coordinate $1$ realized by the piecewise-linear identity $h$ above and coordinate $2$ by
$\mathrm{ReLU}(x_2+5)-5$; at $x=(2,1)$ the pre-activations are $(1,-1,6)$, so
$x\in X_{\mathrm{reg}}$, while at $Px=(1,2)$ they are $(0,0,7)$, so
$Px\notin X_{\mathrm{reg}}$, and (iii) predicts
$\M(Px)=\left[\begin{smallmatrix}1&0&0\\0&2&0\end{smallmatrix}\right]$ against
the actual
$\left[\begin{smallmatrix}0&0&1\\0&2&0\end{smallmatrix}\right]$.
This is the phenomenon described after
Theorem~\ref{thm:maxinv}: two encodings of one function can sit on walls of
\emph{different} decompositions. For a finite group the two-sided hypothesis
costs nothing. Each $\pi(g)$ is a linear isomorphism, so
$\pi(g)^{-1}X_{\mathrm{reg}}$ is again open, dense and of full measure, and
$\bigcap_{g\in G}\pi(g)^{-1}X_{\mathrm{reg}}$ is a finite intersection of such
sets, hence open, dense and of full measure itself. Note also that $X_{\mathrm{nz}}$
plays no role: none of (i)--(v) uses it, and we do not assume it.
\end{remark}

Part~(iii) is \emph{not} new and we
do not claim it: \citet[Prop.~D.6]{crabbe2023evaluating} prove exactly this
equivariance, under the same permutation hypothesis and by the same
$\mathrm{diag}(Px)=P\,\mathrm{diag}(x)P^{\top}$ step, for a family of
gradient-based explanations that includes gradient$\times$input
(their Remark~D.7, at $\varphi=\delta(t{-}1)$) and integrated gradients
(at $\varphi\equiv1$), where $\varphi$ is their path-weighting kernel and not
an activation. Part~(i) is older still and essentially the chain rule.
Their Remark~D.3 already notes that permutation representations are orthogonal
but not conversely, and their Appendix~H places Spherical CNNs outside their
scope as future work.

What we add is (ii), (iv) and (v). Their model is assumed \emph{invariant}, a
scalar or label output with $\rho$ trivial; the two-sided law of~(iii) with a
non-trivial output action, and the bias-column identity~(ii) --- which has no
counterpart in their setting, their explanation being a Hadamard product with no
bias term --- do not appear there. Nor does the converse~(iv): sufficiency of
the permutation hypothesis is proved, necessity is not, and we have not found it
stated elsewhere. We verify~(iv) numerically as well as algebraically in
Appendix~\ref{app:proofs}, on the
cheapest witness of the necessity direction: $\Net=\mathrm{id}$ on $\R^5$,
realized as the ReLU network $\mathrm{ReLU}(x)-\mathrm{ReLU}(-x)$ --- so that
(LCS) holds and $X_{\mathrm{reg}}$ is the full-measure set of inputs with no
vanishing coordinate --- and equivariant for \emph{every} linear $\pi$ with
$\rho=\pi$. Computing $\M$ through the secant construction \eqref{eq:D-def},
the two-sided law holds exactly for a permutation and fails by
$\mathcal O(1)$ rather than marginally for a signed permutation, a
monomial matrix, a pure input rescaling, a planar rotation
and a generic orthogonal map.

\begin{remark}[Which geometric models this covers, and which it does not]
\label{rem:gdl-scope}
Part~(iii) applies to every symmetry that acts on the input by permuting
coordinates: node relabeling in graph networks and DeepSets,
translations of a discrete grid \emph{under periodic boundary conditions} ---
without them a translation is not a permutation of the grid, since mass leaves
one edge and does not re-enter at the other --- and the finite
rotation--reflection groups of a square lattice. It does not apply to
continuous rotations. On $\mathrm{SO}(2)$- and $\mathrm{SO}(3)$-equivariant
architectures the first $d$ columns of $\M$ are not equivariant, because the
$\mathrm{diag}(x)$ weighting is tied to a distinguished basis; the discrepancy
is of the same order as the matrix itself, not a small error. $E(n)$-equivariant
architectures sit outside the setting altogether: $E(n)$ contains translations,
which are not linear, so its action is not a linear representation $\pi$ of the
kind declared above, and the proposition does not speak to it either way. The
remedy for the rotation case is~(v): use
$\Weff=[\,J\mid c\,]$, which is equivariant for any linear action and
which Remark~\ref{rem:zero-pixels} already recommends for the unrelated reason
that it is insensitive to vanishing input coordinates.

This is also the reason for a caveat stated later on other grounds. The
coherence $A$ of Section~\ref{sec:geometry} is invariant under a permutation
action \emph{provided the output action $\rho(g)$ is orthogonal as well} ---
Proposition~\ref{prop:equivar} does not assume this, and $A$, being built from
two Frobenius norms, needs it on both sides --- but it is not invariant under a
rotation of the input, which is exactly the pixel-basis dependence
Theorem~\ref{thm:pyth}'s discussion flags. The dependence is inherited from
$\mathrm{diag}(x)$, and by~(iv) it is not removable within the definition.
\end{remark}

\subsection{Further remarks on Section~\ref{sec:germ}}
\label{app:remarks-germ}

\begin{remark}[The activation is unconstrained; the hypothesis is $x\in X_{\mathrm{nz}}$]
\label{rem:xnz}
We impose no condition on $f$, and in particular none on $f(0)$: for sigmoid,
with $f(0)=\tfrac12$, the row sums reproduce the logits to roundoff
(\texttt{scripts/verify\_km\_claim\_checks.py}, check D) at random inputs. The one hypothesis is $x\in X_{\mathrm{nz}}$, and it is not removable by
a better guard. Where a pre-activation is exactly zero the post-activation is
$f(0)$ while $D_{qq}z_q=0$ for \emph{any} finite $D_{qq}$, so the identity fails
there by exactly $f(0)$ whatever value the guard assigns --- $0/0\mapsto0$ and
$0/0\mapsto10^6$ leave the same shortfall. Activations with $f(0)=0$ are
precisely those for which the exceptional set disappears and
\eqref{eq:km-def} holds at every input, ReLU among them; that is the role of the
condition, and it is a bonus rather than a prerequisite. Observe that
$X_{\mathrm{nz}}$ is open. For a real-analytic $f$ (sigmoid, $\tanh$, GELU,
SiLU) every pre-activation is a real-analytic function of $x$, whose zero set
is Lebesgue-null unless the pre-activation vanishes identically, so
$X_{\mathrm{nz}}$ is dense and of full measure as soon as no unit's
pre-activation vanishes identically --- a zero-weight, zero-bias unit does, for
every $f$, and is what this excludes. For the ReLU family the complement of
$X_{\mathrm{nz}}$ is the union of the realized zero sets $\{z^{(\ell)}_q=0\}$,
which lies inside the finitely many hyperplanes of Lemma~\ref{lem:null}
whenever no frozen pre-activation of that lemma's proof vanishes identically;
there, however, the question is moot, since $f(0)=0$ makes \eqref{eq:km-def}
hold at every input. One caveat is worth
recording: when $f(0)\neq0$ the quotient behaves like $f(0)/z$ near the excluded
set, so $D$ is unbounded there ($D=5.25$ at $z=10^{-1}$ rising to
$5\times10^{5}$ at $z=10^{-6}$ for sigmoid). The identity stays exact; the
entries do not stay small.
\end{remark}

\paragraph{What (LCS) is: the Cantor counterexample, the bridge to networks, and attribution.}
Part (ii) of Proposition~\ref{prop:lcs} asks $E$ to be discrete and not merely Lebesgue-null, and the
strengthening is not cosmetic; with ``discrete'' weakened to
``Lebesgue-null'' the statement is false. Let $g$ be the Cantor function
rescaled to rise from $1$ to $2$ across $[1,2]$, with $g\equiv1$ below and
$g\equiv2$ above, and set $f(z)=z\,g(z)$. Then $f$ is continuous with $f(0)=0$,
and $f(z)/z=g(z)$ is locally constant off the Cantor set --- closed,
uncountable and Lebesgue-null --- yet $f(1)=1$, $f(1.5)=2.25$ and $f(2)=4$, so
no pair $(a_+,a_-)$ fits and $f$ is outside the family
(\texttt{scripts/verify\_lcs\_vs\_pl.py}). The proof of (ii) needs each
component of $\R\setminus E$ to be an \emph{interval}, which is what
discreteness of $E$ supplies and null-ness does not.

The bridge back to networks is cheap, and we state it rather than assume it.
Suppose a network with continuous activation $f$ satisfies (LCS) at every input
outside a finite union of affine hyperplanes, and has a first-layer unit with a
nonzero weight row, so that its pre-activation $z(x)=w^{\top}x+\beta$ is an open
map onto $\R$. Fix $z_0\in\R$ whose fiber $z^{-1}(z_0)$ --- itself a hyperplane
--- is not one of those finitely many walls; then $z^{-1}(z_0)$ is not contained
in their union, so some $x$ off the walls has $z(x)=z_0$, and (LCS) at $x$ makes
$f(\cdot)/\cdot$ constant on a neighborhood of $x$, hence on an open interval
around $z_0$. The $z_0$ excluded by this argument are finitely many, so $E$ is
finite and (ii) applies: $f$ lies in the leaky-ReLU family and its only
breakpoint is $z=0$. That is what makes the
walls of Lemma~\ref{lem:null} the zero sets $\{z^{(\ell)}_{\sigma,q}=0\}$ and
nothing else, and the argument does not run through Lemma~\ref{lem:null}, so
there is no circularity.

Parts (ii) and (iii) of Proposition~\ref{prop:lcs} both restate a standard fact, and we record them only
because we need the exact class. Local constancy of
$f(z)/z$ on $\{z>0\}$ and on $\{z<0\}$ is precisely positive homogeneity of
degree one there: $f(\lambda z)=\lambda f(z)$ for $\lambda>0$ gives $f(z)=f(1)z$
for $z>0$ and $f(z)=-f(-1)z$ for $z<0$, and conversely. That continuous
positively homogeneous functions on $\R$ are exactly
$a_+\max(z,0)+a_-\min(z,0)$ is elementary and long-standing; the general theory
on $\R^n$, where the classification is genuinely substantial, is surveyed by
\citet{gorokhovik2016positively}. In the learning-theory literature the
consequence for activations --- that ReLU and leaky ReLU are the positively
homogeneous ones, and that this is what makes networks built from them
homogeneous in their parameters --- is used routinely, for instance by
\citet{neyshabur2015path}, \citet{dinh2017sharp}, and most explicitly by
\citet{lyu2020gradient}, whose analysis assumes exactly this class. What we
have not found stated elsewhere is the reformulation as a condition on the
slope diagonal $D$, and the strictness recorded in
Remark~\ref{rem:lcs-strict}: that this class is \emph{strictly smaller} than
the piecewise-linear activations, so that a PL network can have a germ
everywhere and still fail the germ identity. That strictness is therefore a
remark and not a part of the proposition --- its support is two examples and a
measurement, not a classification theorem.

\begin{remark}[(LCS) is strictly stronger than piecewise linearity]
\label{rem:lcs-strict}
Piecewise linearity of $f$ is \emph{not} (LCS), and the gap is not vacuous.
Hard-tanh $z\mapsto\mathrm{clip}(z,-1,1)$ and the shifted unit
$z\mapsto\max(z-1,0)$ are piecewise linear, so the networks they build are
piecewise affine and have a germ at every regular input, yet their slope
diagonal ranges over a continuum and $\Jwf(x)\neq D\Net(x)$. The failure is
visible in a single unit: at $z=2$ hard-tanh has $f(z)/z=1/2$ against
$f'(z)=0$, and the shifted unit has $f(z)/z=1/2$ against $f'(z)=1$. It
compounds through depth. On a $4$--$6$--$6$--$3$ network over $200$ random
inputs we measure $\max|\Jwf-D\Net|=7.8$ for hard-tanh and $7.15$ for the
shifted unit, against roundoff for ReLU, $|\cdot|$ and leaky ReLU, while the
row-sum identity \eqref{eq:km-def} is untouched throughout, holding to
roundoff for all five (\texttt{scripts/verify\_lcs\_vs\_pl.py}). What separates the two
families is Proposition~\ref{prop:lcs}(ii): (LCS) asks every linear piece of $f$
to pass through the origin, and continuity then confines the single break to
$z=0$. This is why the hypothesis of Theorems~\ref{thm:germ},
\ref{thm:maxinv} and~\ref{thm:complete} is written as (LCS) and never as ``for
piecewise-linear networks''.

The distinction is easy to lose, and has been lost.
\citet{chu2018exact} treat the whole piecewise-linear family correctly, by
carrying a per-neuron \emph{slope and intercept} rather than a slope alone ---
the intercept is precisely the term that a pieces-through-the-origin activation
does not need. \citet[Eq.~2]{wang2019bias} instead define piecewise-linear
activations by a formula in which every piece passes through the origin, and
then list hard-tanh among the activations covered; hard-tanh does not satisfy
that formula. That is a published instance of exactly the confusion this
remark is here to prevent. We record it as a remark rather than a result: its
support is two examples and a measurement, and we did not find it stated
elsewhere.
\end{remark}

\paragraph{Two readings of the boundary sentence of Theorem~\ref{thm:germ}.}
The per-region affine operator that Theorem~\ref{thm:germ} exhibits is the object
of the spline view of deep networks \citep{balestriero2018spline}, whose \S5.2
also contains the bias-by-subtraction formula $\cwf(x)=\Netx-D\Net(x)x$ used
above. Two readings of the boundary sentence must be kept apart. The row-sum
half holds for every activation, since $h_q=a_qz_q$ reads $0=0$ when $z_q=0$.
The germ-selection half does not: for ReLU, slope $0$ \emph{is} the one-sided
germ taken from the inactive side --- which is the reading used in the
sharpness example after Theorem~\ref{thm:maxinv} --- but for a general (LCS)
activation with $a_-\neq0$ (leaky ReLU, $|\cdot|$, the identity) slope $0$ is
\emph{neither} of the two one-sided slopes $a_\pm$, and the guard selects a
value that is not a germ of anything. That clause is ReLU-specific.

Observe that the germ is locally constant inside a region while $\M(x)$ is not:
the weighting $\mathrm{diag}(x)$ moves with $x$, so the
field $x\mapsto\M(x)$ varies continuously there. This matters because
\citet[Prop.~1]{srinivas2019full} prove that no saliency map $S(x)\in\R^{d}$
can, in general, be both \emph{complete} --- ``there exists a function
$\varphi$ such that $\varphi(S(x),x)=f(x)$ for all $f,x$'' (their
Definition~2) --- and \emph{weakly dependent} on the input --- for a
piecewise-linear model $f(x)=w_i^{\top}x+b_i$ on regions $U_i$, ``the
saliency map $S(x)$ restricted to a set $U_i$ is independent of $x$, and
depends only on the parameters $w_i,b_i$'' (their Definition~1). Under these
definitions $\M$ is complete and not weakly dependent, while
$\Weff=[\,J\mid c\,]$ is both: $\varphi(\Weff,x)=Jx+c$ recovers the output
and $\Weff$ is constant on each region. There is no tension, and the reason
is dimensional rather than a failure of either property: Proposition~1
concerns saliency maps valued in $\R^{d}$ --- its proof counts a map from
$(w_i,b_i)\in\R^{d+1}$ to $S\in\R^{d}$ as many-to-one --- whereas $\M$ and
$\Weff$ are $C\times(d{+}1)$, as is FullGrad's own full-gradient
representation, and it is that, not any failure of completeness, that keeps
Theorem~\ref{thm:germ} outside its scope.

\begin{remark}[Where the columns vanish]\label{rem:zero-pixels}
Column $i$ of $\M(x)$ is $J_{:,i}\,x_i$: on $\{x_i=0\}$ it vanishes
identically, so the $i$-th germ column is not recoverable from $\M(x)$ there.
This set is Lebesgue-null but has \emph{positive probability} under raw image
data; the condition must be checked on the preprocessed input the network
sees. The variant $\Weff(x)=[\,J\mid c\,]$ recovers the germ with no
condition on $x$. Invariance (Theorem~\ref{thm:maxinv}) needs no such
hypothesis; germ \emph{recovery} (Theorem~\ref{thm:complete}(i)) does.
\end{remark}

\paragraph{The chord behind Lemma~\ref{lem:quiver-inv}.}
The lemma rests on one substitution, spelled out here. Write $z_q$ for a
hidden unit's pre-activation and $h_q=f(z_q)$ for its post-activation. In
place of the derivative $f'(z_q)$ we use the \emph{secant} (chord) slope
$a_q=f(z_q)/z_q$, with the guard $a_q=0$ when $z_q=0$, so that
$h_q=a_qz_q$ holds \emph{exactly} --- an identity, not a first-order
approximation. Collecting the $a_q$ of a layer into a diagonal matrix
gives exactly the $D^{(\ell)}(x)$ of \eqref{eq:D-def}: the
secant decomposition is not an extension of \eqref{eq:D-def} but its content,
and the masked-product reading is the ReLU specialization. The chord is the
same quantity that \citet[Prop.~1]{ancona2018towards} identify as
$\varepsilon$-LRP's modified gradient, in the $\varepsilon\to0$ limit and with
the bias included in the denominators; the knowledge matrix differs in
appending a bias column, and the shortfall between $\varepsilon$-LRP's per-class
attributions and the logit is exactly that column, $\cwf(x)$. For ReLU the two agree,
since $f(z)/z=\mathbb 1[z>0]$ off $z=0$ and the guard $0/0\mapsto0$
selects the same value as the mask convention at $z=0$.

Two consequences are used repeatedly. First, the row-sum identity
$\M\mathbf 1=\Net$ is exact at every input in $X_{\mathrm{nz}}$, and at
\emph{every} input when $f(0)=0$, including points on region boundaries, where
the derivative is ambiguous but the chord is not --- and the chord is what makes
Corollary~\ref{prop:perm} hold for every $x$ rather than merely almost
every $x$. Second, the same recipe applies verbatim to smooth activations
(GELU, sigmoid, $\tanh$), for which the derivative alone does not close:
$h_q=f'(z_q)z_q$ is false, and an exact gradient-based accounting there
needs an extra implicit-bias term $f(z_q)-f'(z_q)z_q$ per unit
(\citealp[\S4]{srinivas2019full}; verified exact to roundoff on a
three-layer sigmoid network, \texttt{scripts/verify\_km\_claim\_checks.py}),
which the chord absorbs into the slope it already carries. This is not an
extension of \eqref{eq:D-def} but that definition itself, applied to a non-PL
activation, as Section~\ref{sec:attribution} records. The one requirement is $x\in X_{\mathrm{nz}}$,
and nothing is asked of $f(0)$: the chord reproduces $h_q=a_qz_q$ exactly
wherever $z_q\neq0$. Only on the null set $\{z=0\}$ does $f(0)\neq0$
matter, and there no choice of $a_q$ repairs the identity, since $a_q\cdot0=0$
for every finite $a_q$; activations with $f(0)=0$ are exactly those for
which that exceptional set is empty.

The rescaling half of the lemma holds for every $\tau\neq0$, not only $\tau>0$,
because the secant is what is conjugated: $f_\tau(\tau z)/(\tau
z)=f(z)/z$ whatever the sign of $\tau$. What changes with the sign is the
\emph{activation the rescaled unit carries}. For ReLU and $\tau>0$ one has
$f_\tau=f$, so the rescaled network is again a ReLU network; for
$\tau<0$, $f_\tau(z)=\min(z,0)$, the reflected unit, and the rescaled
layer no longer computes ReLU. Theorem~\ref{thm:complete}(iii) is about the
gauge group \emph{within} a fixed architecture and therefore needs $\tau>0$
specifically; the lemma does not. The lemma's scope is per-vertex scalar
activations: a channel permutation and a positive channel rescaling commute
with spatial max-pooling and its lowest-index tie-break, whereas a negative
$\tau$ upstream of a max-pooling window carries the maximum to a minimum
\citep[Remark~4.14]{armenta2021representation}, so the rescaling half is
stated for positive $\tau$ at such units.

\begin{remark}[A knowledge-matrix penalty is input-gradient regularization]
\label{rem:doublebackprop}
Theorem~\ref{thm:weff} has a consequence for \emph{training} that we did not set
out to obtain and do not test here, but that seems worth recording because it
lands on well-trodden ground. Penalizing the size of the knowledge matrix is
penalizing an input gradient. Explicitly, dropping the bias column,
\[
\bigl\|\M(x)\bigr\|_F^2-\bigl\|c(x)\bigr\|_2^2
\;=\;\sum_{c=1}^{C}\sum_{i=1}^{d}
\Bigl(\frac{\partial \Psi_c(W\!,f)}{\partial x_i}\,x_i\Bigr)^{\!2}
\;=\;\bigl\|\nabla_x\Net\odot x\bigr\|_F^2 ,
\]
an input-gradient$\times$input penalty. That is the input-scaled, per-class,
logit-space member of a family that is already known to work: double
backpropagation \citep{drucker1992improving}, revived as input-gradient
regularization for adversarial robustness and interpretability by
\citet{ross2018improving} and scaled up by \citet{finlay2019scaleable}. Those
penalize $\|\nabla_x\mathcal L\|^2$ or $\|\nabla_x\sum_c\log p_c\|^2$; the
knowledge-matrix form differs in weighting each coordinate by $x_i$, in
resolving the penalty by class rather than summing over classes first, and in
acting on logits rather than log-probabilities.

Two of this paper's results say what such a penalty would and would not buy.
Theorem~\ref{thm:pyth} supplies the robustness direction as an inequality rather
than an intuition: $\|\Delta\Net\|_2\le\sqrt{d{+}1}\,\|\Delta\M\|_F$, so
controlling knowledge-matrix displacement controls logit displacement, which is
the quantity an attack must move. Theorem~\ref{thm:within} says what is being
controlled within a region, $d_M^2=\sum_i\delta_i^2\|J_{:,i}\|_2^2$: a
perturbation-weighted Jacobian energy, so the penalty is anisotropic in a way a
plain $\|\nabla_x\mathcal L\|^2$ is not. It also flags the obvious failure mode
--- the $x_i$ weighting makes the penalty blind wherever $x_i=0$
(Remark~\ref{rem:zero-pixels}), which for preprocessed images is a set of
positive probability, so $\Weff$ rather than $\M$ would be the
sensible object to penalize.

We state this as a connection, not a result. We run no training experiment in
this paper, and nothing here should be read as evidence that the penalty helps;
what Theorem~\ref{thm:weff} establishes is only that it is not a new idea in
disguise --- it is double backprop with an input-scaled, class-resolved
weighting, and it inherits whatever that literature has established.
\end{remark}

\begin{remark}[The ``no pooling ties'' hypothesis]\label{rem:pooling-ties}
Max-pooling is piecewise linear but
non-differentiable wherever two or more entries of a pooling window attain the
maximum simultaneously --- a \emph{tie}. At a tie the subgradient is
set-valued: the window's output may be routed to either winning entry, so the
masked-product Jacobian, and with it $\M$ and per-class
gradient$\times$input, depend on the tie-breaking rule rather than on the
function. The tie hyperplanes are exactly the extra walls that the proof of
Lemma~\ref{lem:null} adds for max-pooling, so this is not a new hypothesis: it
is what $x\in X_{\mathrm{reg}}$ already asks of the pooling layers, spelled
out because a reader may hear $X_{\mathrm{reg}}$ as a statement about ReLU
signs alone. Ties are a Lebesgue-null event. Our convention is the one the
implementation inherits from PyTorch: the forward pass stores the
\texttt{return\_indices} argmax, which on a tie selects the lowest flat index
in the window, and every route --- masked product, probe, autograd --- reuses
that same stored index, so the three still agree with one another on the null
set (Appendix~\ref{app:engineering}). What the hypothesis rules out is the
stronger claim, that the shared value is determined by the realized function
and not by the convention.
\end{remark}

\begin{remark}[Two unrelated meanings of ``gauge'']\label{rem:gauge-collision}
We use \emph{gauge} for the parameter-space redundancy of
Section~\ref{sec:invariance}: transformations of $W$ that leave $\Net$
fixed, following the reading of parametric redundancy as gauge symmetry of
\citet{hashimoto2024unification}. In geometric deep learning the same word names
something else entirely --- the choice of local frame on a manifold, with gauge
equivariance meaning independence of that choice
\citep{cohen2019gauge,bronstein2021geometric}. The two act on different objects,
parameters in one case and tangent frames in the other, and nothing in this
paper concerns the second. We flag the collision because both usages are
established and a reader arriving from either literature will otherwise assume
the wrong one.
\end{remark}

\begin{remark}[Why the correctness checks permute channels]\label{rem:channelperm}
Proposition~\ref{prop:equivar} also explains a constraint that
Appendix~\ref{app:teleport-checks} otherwise reports as an implementation limit.
Requiring an architecture to remain $G$-equivariant restricts which quiver
isomorphisms are available: a per-neuron rescaling that varies across the
positions tied together by a convolution destroys the weight sharing, so the
rescaled layer is no longer a convolution and has left the equivariant class,
whereas a rescaling constant on each channel stays inside it. The
architecture-preserving part of the gauge group of a convolutional network is
therefore per-channel rather than per-neuron --- which is why the permutation
check acts on the $2048$-channel post-pool face and why bottleneck interiors are
not permutable in place. The knowledge matrix is invariant under the full quiver
group in any case (Lemma~\ref{lem:quiver-inv}), so nothing about its invariance
depends on this; what the restriction fixes is which transforms an experiment
can realize. The same restriction bears on the perturbation of
Theorem~\ref{thm:complete}(iv): a change of one first-layer weight row is not
a parameter change of a convolutional layer, which is why that part has a
second proof, acting on the free output layer alone.
\end{remark}

\begin{remark}[The quiver lift is scaffolding for future work, not machinery of this paper]
\label{rem:quiver-lift-scaffolding}
Every result in this paper --- the germ identity (Theorem~\ref{thm:germ}), maximal
invariance (Theorem~\ref{thm:maxinv}), completeness (Theorem~\ref{thm:complete}),
and the entire distance geometry of Section~\ref{sec:geometry} --- follows from
the germ alone, i.e.\ from the chain rule on the active region. We use the
induced representation $\kmapx$ nowhere in this paper except through its
\emph{contraction}, which is the knowledge matrix
(Proposition~\ref{prop:contraction}). The induced representation strictly
refines the matrix --- it separates functionally
identical but non-isomorphic networks that the matrix identifies --- and it is
the object of future work where that finer information is needed. In this paper
it is scaffolding for the construction, not a load-bearing tool.
\end{remark}

\subsection{Further remarks on Section~\ref{sec:geometry}}
\label{app:remarks-geometry}

\paragraph{The last clause of Definition~\ref{def:coherence}, unpacked.}
The last clause of the definition is worth unpacking. Since
$A=(d_\Psi/d_M)^2=1/r^2$ with $r=d_M/d_\Psi>0$, and $t\mapsto1/t^2$ is
strictly decreasing on $(0,\infty)$, the map from one to the other is
order-reversing pair by pair. Any statistic that sees the sample only through
the ordering of its values --- order statistics, quantiles, medians, Spearman
and Kendall coefficients, the concordance $W$ --- therefore transfers between
$r$ and $A$ with its direction flipped and nothing else changed. In
particular the attack-family ordering we report is one and the same finding
whichever of the two we tabulate, which is exactly why that ordering does not
borrow anything from $A$ having a validated meaning. The single caveat is
arithmetic rather than conceptual: an interpolated median of an even-count
sample averages two order statistics, so it commutes with the transform only
up to the gap between them --- below our quoted precision; see
Appendix~\ref{app:proofs}.

\begin{remark}[Why the coherence is the squared ratio and not the ratio itself]
\label{rem:why-squared}
There are four reasons,
of which only the first is essential. \emph{(a)} Theorem~\ref{thm:pyth} is an
\emph{energy} statement: $\|\Delta\M\|_F^2$ splits \emph{additively} into
$\|\Delta\Net\|^2/(d{+}1)$ and $\|Q\|_F^2$ because the two components are
orthogonal, and additivity across orthogonal components is a property of
squared norms, not of norms. The dimensionless quantity the theorem actually
hands us is therefore a share of squared norm, namely $\rho_{\mathrm{vis}}$;
$d_\Psi/d_M$ is not a share of anything. \emph{(b)} $A$ \emph{is} that share,
renormalized: $A=(d{+}1)\rho_{\mathrm{vis}}$, the factor $(d{+}1)$ chosen so
that the single-pixel case reads $A=1$ instead of $1/(d{+}1)$. \emph{(c)} In
the squared scale the two theorem-given reference lines are the clean numbers
$A=1$ and $A\le d$; in the unsquared scale they would be $1$ and $\sqrt d$,
and the additive split would not be legible at all. \emph{(d)} Nothing is
lost: the two are strictly monotone functions of each other, so every
rank-based conclusion is identical under either, as the previous paragraph
records.
\end{remark}

\paragraph{The symbols in the crossing term \texorpdfstring{$K$}{K} of Theorem~\ref{thm:anatomy}.}
We unpack the symbols in $K$. A \emph{dyad} is an
outer product $uv^{\top}$ of two vectors: a matrix of rank one, every column a
multiple of $u$ and every row a multiple of $v^{\top}$. Each wall crossing
contributes exactly one such term because a single unit flipping changes the
region's Jacobian by a rank-one update --- the unit has one way in from the
input and one way out to the logits. Fix the $j$-th crossing, at which unit
$k_j$ of some hidden layer $\ell^{*}$ flips (Lemma~\ref{lem:dyad},
Appendix~\ref{app:proofs}). Then $z_j=x+t_j\delta$ is the point on the segment
at which that crossing occurs, so $\mathrm{diag}(z_j)$ is the input scaling
that turns germ data into knowledge-matrix data \emph{there};
$\sigma_j\in\{+1,-1\}$ records the direction of the flip ($-\to+$ gives $+1$,
$+\to-$ gives $-1$), so that the unit's slope jumps by $\kappa\sigma_j$ with
$\kappa=a_+-a_-$ --- for ReLU, $\kappa=1$ and the jump is $\sigma_j$ itself,
which is why the ReLU reading absorbs $\kappa$ silently;
$u_{k_j}\in\R^{C}$ is the output-side vector, the column
through which unit $k_j$ reaches the $C$ logits along the downstream masked
product; and $v_{k_j}\in\R^{d}$, $\gamma_{k_j}\in\R$ are the input-side data,
defined by writing the unit's pre-activation as an affine function of the
network input, $z^{(\ell^{*})}_{k_j}(x')=v_{k_j}^{\top}x'+\gamma_{k_j}$, so
that $v_{k_j}^{\top}$ is a row of the upstream masked product and
$\gamma_{k_j}$ the bias accumulated below layer $\ell^{*}$. The wall is the
hyperplane $\{v_{k_j}^{\top}x'+\gamma_{k_j}=0\}$, which is precisely why the
dyad's row sums vanish: $[\,v^{\top}\mathrm{diag}(z_j)\mid\gamma\,]\mathbf 1
=v^{\top}z_j+\gamma=0$ at the crossing point.

\begin{remark}[The smooth part of the decomposition is integrated gradients]
\label{rem:ig-smooth}
$\bar J=\int_0^1J(x+t\delta)\,dt$ is the mean value of the Jacobian along the
segment from $x$ to $y$, and the identity $\Delta\Net=\bar J\delta$ is the mean
value theorem for vector-valued maps --- equivalently the fundamental theorem
of calculus along the segment --- which is exact here because $\Net$ is Lipschitz
and piecewise affine on it. It is also, exactly, the object behind
\emph{integrated gradients} \citep{sundararajan2017axiomatic}: the
integrated-gradients attribution of $y$ against baseline $x$ is
$(y-x)\odot\int_0^1\nabla\Net\bigl(x+t(y-x)\bigr)\,dt$, whose value for class
$c$ and coordinate $i$ is the $(c,i)$ entry of $\bar J\,\mathrm{diag}(\delta)$
--- the smooth block above. Their completeness axiom
$\sum_i\mathrm{IG}_i=\Nety-\Netx$ is the row-sum statement
$\bar J\delta=\Delta\Net$, and the axiomatic basis they invoke is the Aumann--Shapley
average-gradient cost-sharing rule. This sharpens rather than weakens the
equivalence we claim in Theorem~\ref{thm:weff}: at a single input the
knowledge matrix is per-class gradient$\times$input plus a bias column, and
\emph{between} two inputs its smooth part is integrated gradients. What the
decomposition adds is the remainder integrated gradients has no name for ---
the crossing term $K$, invisible to the endpoint logits and therefore
invisible to any completeness axiom phrased in terms of them.
\end{remark}

\paragraph{The scope of Proposition~\ref{prop:dichotomy}, and the symmetry beneath it.}
We separate two scopes of Proposition~\ref{prop:dichotomy}, which is stated
for an arbitrary hidden layer $\ell$ and applied in the experiments to the
penultimate one. Rescaling a \emph{single} unit of the layer already makes
$d_h$ unbounded above, but only on $[m,\infty)$ for some $m>0$ fixed by the
untouched coordinates; it is the uniform rescaling of the whole layer that
sweeps all of $(0,\infty)$ and so leaves no invariant function of $d_h$
alone. And the conclusion is about functions of the three stored distances
$(d_h,d_\Psi,d_M)$, which is what a stored-distance pipeline records. It is
not a claim that nothing about a hidden representation is function-level:
a statistic built from other stored quantities is outside the scope of this
argument --- the normalized distance $d_h/\|h^{(\ell)}(x)\|_2$, for
instance, is invariant under this uniform rescaling.

The fact underneath Proposition~\ref{prop:dichotomy} is not ours. That a
positive per-neuron rescaling leaves a ReLU network's function exactly
unchanged while moving its hidden activations --- so that activation norms,
and any distance built from them, can be rescaled essentially at will --- is
the positive-homogeneity (rescaling) symmetry of ReLU networks. It is used as
an optimization-geometry tool by \citet{neyshabur2015path} in Path-SGD and,
most familiarly, by \citet{dinh2017sharp} to show that the sharpness of a
minimum is reparameterization-dependent; it is the same symmetry that appears
here as the gauge group, and as the rescaling half of
Lemma~\ref{lem:quiver-inv}. What this proposition contributes is only the
consequence for the accounting of Theorem~\ref{thm:pyth}: because no function
of the three stored distances that is invariant under the layer rescaling
depends on $d_h$, the visible/invisible split has no analog for hidden
activations --- there is
no fixed, parameter-independent vector playing the role of $\mathbf 1$ on the
activation side. The symmetry is prior work; the corollary drawn from it for the
decomposition is what we claim.

\paragraph{\texorpdfstring{$C^0$}{C0} closeness, and the three tiers of Proposition~\ref{prop:visdrift}.}
In Theorem~\ref{thm:nogo}, $C^0$ closeness means closeness in \emph{value},
uniformly: we call two
maps $C^0$-close when $\sup_{x'}\|\Netxp-g(x')\|\le\varepsilon$, the metric of the
space $C^0$ of continuous functions under the supremum norm. The superscript
counts derivatives, so $C^1$ closeness would additionally require the
derivatives to agree. The theorem says the first gives no control of the second,
and the reason is that the knowledge matrix records the germ --- first-order
data. The smooth illustration $g=\Net+\varepsilon\sin(\cdot/\varepsilon^{2})$
shows the mechanism: it stays uniformly within $\varepsilon$ of $\Net$
everywhere while its derivative differs by $1/\varepsilon$. The theorem's
witness is the two-unit ReLU ramp of its proof, a same-architecture network.
This is why the teleportation check of
Appendix~\ref{app:teleport-checks} computes the matrix drift directly rather
than inferring it from the logits: agreement of the logits to any tolerance
is, by itself, no evidence at all about the matrices.

For a transform that is only approximately function-preserving,
Proposition~\ref{prop:visdrift}
separates three tiers: zero drift under an exact isomorphism; a
\emph{visible} drift pinned as an identity to the measured logit gate
$\varepsilon=\|\tilde\Psi(x)-\Netx\|$; and an \emph{invisible} drift that the
gate does not bound at all (Theorem~\ref{thm:nogo}), with a conditional bound
only when both networks are affine on a common ball.

Before the next statement we fix two words. The \emph{logit gate}
of a transform $\tau$ at $x$ is the measured discrepancy between the two
networks' outputs at the same input,
$\varepsilon=\|\tilde\Psi(x)-\Netx\|$, where $\tilde\Psi=\Psi(\tilde W\!,f)$ is
the network after the transform and $\tilde{\M}$ its matrix. A quantity is
\emph{gate-bounded} when the theory bounds it
by $\varepsilon$, and \emph{gate-pinned} when the theory forces it to
\emph{equal} $\varepsilon$ as an algebraic identity; how an experiment uses
the gate as an acceptance threshold is recorded in
Appendix~\ref{app:teleport-checks}.

\begin{proposition}[Knowledge-matrix drift under a transform with logit gate $\varepsilon$]\label{prop:visdrift}
\emph{(i)} Under an exact isomorphism, the knowledge-matrix drift is zero.
\emph{(ii)} Under a transform with logit gate $\varepsilon$, the \emph{visible}
drift equals the gate:
$\|(\tilde{\M}(x)-\M(x))\mathbf 1\|=\|\tilde\Psi(x)-\Netx\|=\varepsilon$.
\emph{(iii)} The \emph{invisible} drift admits no bound in terms of
$\varepsilon$ alone (Theorem~\ref{thm:nogo}). If both networks are affine on
a common ball of radius $r$ around $x$, then
$\|\tilde{\M}(x)-\M(x)\|_F\le(\varepsilon_r/r)\sqrt{\min(C,d)}\,\|x\|_\infty
+\varepsilon_r(1+\|x\|_2/r)$.
\end{proposition}

Part (ii) is the
gate-pinned case: the row-sum identity gives
$\|(\tilde{\M}(x)-\M(x))\mathbf 1\|=\|\tilde\Psi(x)-\Netx\|=\varepsilon$ exactly, with
no inequality anywhere, so a table of visible drift is a table of the gate
itself and says nothing about invariance beyond what the gate already says.
Part (iii) is the genuinely unbounded case: by Theorem~\ref{thm:nogo} the gate
constrains the invisible drift not at all, which is why the conditional bound
there has to carry the extra hypothesis that both networks are affine on a
common ball. Whether a given transform actually needs this three-tier
treatment is settled empirically in Appendix~\ref{sec:study1b}.

Our teleportation study (Appendix~\ref{sec:study1b}) sits at tier~(i): the
transform is an exact isomorphism, batch normalization in evaluation mode
included (Appendix~\ref{app:teleport-why}), so both components of the drift
vanish, and the check computes both rather than asserting either from a
theorem: the visible one through identity~(ii), which is an equality, and the
invisible one directly (Appendix~\ref{sec:study1b}). Tiers~(ii) and~(iii)
are stated because they
are what one needs for a transform that genuinely is approximate; the no-go of
Theorem~\ref{thm:nogo} stands regardless, being a statement about what $C^0$
closeness cannot control rather than about any particular transform. The empirical
visible/invisible split is exhibited separately, on adversarial pairs where the
gate is exact, in Study~2 (Section~\ref{sec:study2}).

%% file: sections/appendix_limitations_more.tex
\section{Further limitations}
\label{app:limitations-more}

Table~\ref{tab:limitations-more} continues Table~\ref{tab:limitations} with
the four caveats of measurement convention --- one on the attack suite, two on
the similarity panel, one on the coherence statistics --- that apply across
several studies at once and so belong to none of them.

\begingroup\small
\setlength{\LTpre}{6pt}\setlength{\LTpost}{6pt}
\begin{longtable}{@{}>{\raggedright\arraybackslash}p{0.24\linewidth}p{0.73\linewidth}@{}}
\caption{Further limitations (L9--L12): caveats of measurement convention that
apply across several studies, with the consequence of each and where it is
taken up.}\label{tab:limitations-more}\\
\toprule
Limitation & Consequence, and forward pointer \\
\midrule
\endfirsthead
\toprule
Limitation & Consequence, and forward pointer \\
\midrule
\endhead
\bottomrule
\endlastfoot
\textbf{(L9)} The perturbation budget $\varepsilon$ is never varied. &
FGSM, PGD, APGD and Square all run at the \texttt{torchattacks} default $\varepsilon=8/255$ in $\ell_\infty$ on every rater and in both pair sets, CW and DeepFool are not $\varepsilon$-budgeted, and the only per-architecture overrides are to step counts, the APGD loss and the Square query budget (Appendix~\ref{app:setup-details}). Nothing here establishes that the coherence magnitudes of Section~\ref{sec:study2}, or the attack-family ordering Kendall's $W$ summarizes, survive a change of $\varepsilon$ --- a larger budget crosses more region walls, which is exactly what $A$ measures --- so an $\varepsilon$-sweep is the most informative robustness check the design omits. \\
\addlinespace
\textbf{(L10)} The cross-architecture panel is computed on raw unequal dimensions, and no measure in it is dimension-neutral. &
The panel of Section~\ref{sec:crossarch} is computed on the $2048$- and $1024$-dimensional features with no projection and no PCA; the measures are well defined (normalized Bures reads the $n\times n$ kernel, so zero-column padding leaves it unchanged, invariance ratio $1.000000$), but on a matched-signal probe the unequal-dimension pair scores higher than the equal-dimension one by $+0.02$ to $+0.04$ across debiased CKA, Bures, RSA and distance correlation at the widths used here (\texttt{scripts/verify\_dimension\_bias\_probe.py}; the magnitude is probe-dependent, the sign is not). The bias inflates the two pairs that finish behind (ResNet-152/DenseNet-121, ResNet-152/GoogLeNet), so the reported ordering is conservative with respect to it, but it is uncorrected and no cross-dimensional similarity should be compared to an equal-dimensional one at the third decimal place. \\
\addlinespace
\textbf{(L11)} The shuffled-pair control is reported only where its null is $\approx0$. &
The Murphy control \citep{murphy2024debiased} was computed for every measure of the within-architecture panel, but the three with a null at zero --- debiased CKA, RSA, distance correlation --- are the ones the text foregrounds and the only ones the pipeline's automated gate checks. Bures's shuffled null is $0.264$--$0.294$ while the cross-architecture Bures values of Section~\ref{sec:crossarch} are $0.478$--$0.595$ --- a real margin over the null but much smaller than the raw value suggests; the null is a fidelity between two positive semidefinite kernels and depends on $n$ as well as on the spectrum, and it was computed at $n=2{,}048$ against panel values at $N=25{,}000$, so the $1.6$--$2.3\times$ ratio is indicative only, and the same caution applies to the other bounded measures whose nulls we do not quote. \\
\addlinespace
\textbf{(L12)} The coherence medians carry no intervals, and the appendix gap tests are uncorrected. &
The per-cell medians of Table~\ref{tab:coherence} (the ``median $A\le0.23$'' headline) are point estimates, the reduce behind them storing aggregates only, and only the appendix ordering panel, whose per-pair ratios are stored, carries bootstrap intervals; that panel's $15$ adjacent-gap tests ($3$ architectures $\times$ $5$ gaps, Table~\ref{tab:s3_gap_ci}) carry no family-wise correction, so under a global null the family-wise error rate at $\alpha=0.05$ approaches $1-0.95^{15}\approx54\%$. We disclose both at the point of use and adjust neither, since requiring both the independent and the paired interval to exclude zero already makes each verdict conservative in an unquantified direction, and stacking a correction on that would give a number we could not interpret. \\
\end{longtable}
\endgroup

%% file: sections/appendix_setup_details.tex
\section{Setup details}
\label{app:setup-details}

\paragraph{Distance metrics.}
For knowledge matrices we use the Frobenius distance
$\| \Mfx - \Mfxp \|_F$; for cross-architecture comparisons we report it
RMS-per-coordinate (dividing the raw Frobenius norm by $\sqrt{C(d{+}1)}$),
the canonical fair-comparison metric that strips the dimensionality
inflation. For penultimate features we report two metrics, used for
different purposes. Within a single architecture we use the absolute
$\ell_2$ distance $\| h(x) - h(x') \|_2$. For cross-architecture
comparisons (where penultimate dimensionalities differ), we use the
RMS-per-dimension $\| h(x) - h(x') \|_2 / \sqrt{D}$, with the caveat
that this is \emph{not} normalized by $\| h(x) \|_2$, so cross-arch
rankings remain confounded by feature scale. We discuss the
implications in Section~\ref{sec:limitations} (L6). For logits we use
$\ell_2$.

\paragraph{Attack budgets.}
The hyperparameters that depart from the \texttt{torchattacks}~3.5.1
defaults are DeepFool \texttt{steps}$=200$ (DenseNet-121, GoogLeNet,
ResNet-18) and \texttt{steps}$=100$ (ResNet-152, fewer steps suffice on its
deeper loss curvature); APGD \texttt{steps}$=50$ with the DLR loss; and
Square \texttt{n\_queries}$=20{,}000$. FGSM, PGD, and CW use the library
defaults everywhere. These overrides are registered per experiment and
cover only ResNet-152, DenseNet-121, GoogLeNet and ResNet-18: \emph{AlexNet
and VGG ran the \texttt{torchattacks}~3.5.1 defaults throughout} --- DeepFool
\texttt{steps}$=50$, APGD \texttt{steps}$=10$ with the CE loss, Square
\texttt{n\_queries}$=5{,}000$ --- so those two ordering-only raters
face weaker attack budgets than the other four. The full-scale pair set behind
Table~\ref{tab:s3_ordering_panel} uses its own budgets, which depart from the
defaults on five of the six attacks: PGD \texttt{steps}$=7$, CW
\texttt{steps}$=100$, DeepFool \texttt{steps}$=100$, APGD \texttt{steps}$=50$
with the DLR loss, and Square \texttt{n\_queries}$=20{,}000$; only FGSM runs
at the library default. The budgets are
recorded in the experiment-registration code of the supplementary material
(\texttt{constants/constants.py} and the per-study workers); the per-attack
result JSONs store aggregates only and do not carry them.

\paragraph{Finite arithmetic.}
\label{app:finite-arith}
Every identity in this paper --- the row sum $\M\mathbf 1=\Net$, the
invariance of the matrix under a neuron permutation or a teleportation, the
agreement of the three construction routes of Appendix~\ref{app:engineering}
--- holds in exact arithmetic, and what an implementation registers is its own
resolution, which we record here once and never read as a property of the
objects. Knowledge matrices are computed in single precision throughout the
pipeline, and distances between them in double. At the depth of ResNet-152 the
single-precision row-sum residual $\max_c|(\M(x)\mathbf 1)_c-\Netx_c|$ is not
negligible in absolute logit units: over the $11{,}024$ evaluations of its
stored pairs its median is $0.007$, its $99$th percentile $0.061$ and its
maximum $0.295$, which is also the maximum over the $35{,}024$ evaluations of
the full-scale reduce (the pipeline's correctness gate, at tolerance $0.35$,
passes every sample); recomputing the same matrices in double precision brings
the residual to $\sim10^{-14}$ (\texttt{scripts/verify\_float64\_rowsum.py}).
The permutation residual of Appendix~\ref{sec:study1a} --- $0.116$ in the mean
and $5.78$ at its maximum in Frobenius norm on ResNet-152,
Table~\ref{tab:signal-relative} --- is of the same nature and falls to
$\sim10^{-14}$ in double precision. For teleportation the resolution is set by
the accelerator. The cluster runs used TF32 convolution arithmetic on H100
GPUs, whose $10$-bit mantissa leaves logit discrepancies of
$10^{-2}$--$10^{-1}$ between a network and its teleported image, whereas the
same change-of-basis draws recomputed on a CPU give $4$--$10\times10^{-6}$ in
single precision and $\sim10^{-14}$ in double, a single-to-double ratio of
$3.3$--$7.9\times10^{8}$ against the $5.4\times10^{8}$ that pure roundoff
predicts (a change of function would give a ratio near $1$); the knowledge
matrix itself, recomputed by vector--Jacobian products on the original and the
teleported network, drifts by $0.8$--$2.5\times10^{-15}$ relative in double
precision on all three architectures
(\texttt{scripts/verify\_teleportation\_exactness.py} and
\texttt{scripts/verify\_teleportation\_km\_exactness.py}, both CPU runs of a
few minutes, with the library's three internal single-precision casts patched
out). Two implementation requirements follow. Networks must be in evaluation
mode throughout: active dropout desynchronizes the construction passes and the
row-sum identity then appears to fail at the scale of the logits themselves (on
pretrained AlexNet the residual is $0.9$ of $\max_c|\Psi_c(x)|$). And the
library's serialized change-of-basis loading path has an unrepaired defect --- a
matrix loaded through it violates $\M\mathbf 1=\Net$ structurally --- which is
why the teleported matrices are computed by vector--Jacobian products on the
teleported model.

\paragraph{Sample sets and seeds.}
Per-study sample sets are fixed and reproducible, but they are not all
the same set. Studies~1 and~3 use the first $N$ ImageNet-validation images by
sorted filename ($N=25{,}000$), and so does the full-scale pair set
behind Table~\ref{tab:s3_ordering_panel}. The $200$-pair
set behind Tables~\ref{tab:coherence}--\ref{tab:ordering} does not: it is
\texttt{random.Random(42).sample} of $200$ indices from a $25{,}000$-image
half of the validation set (split with a fixed generator seed). The two
ordering panels are therefore computed on different images, and are read as
independent cross-checks rather than as two views of one sample. Each
teleportation is generated from an explicit integer seed set on both the
\texttt{torch} and \texttt{numpy} generators before the change of basis is
sampled; the similarity panel uses the seed indices $0,\ldots,T{-}1$ directly, the
standalone teleportation run whose logit discrepancies are quoted under
``Finite arithmetic'' above uses $42,\ldots,42{+}T{-}1$, and each run's output JSON records the seed of every
teleportation it performed.

\paragraph{Library versions.}
The knowledge matrix construction is implemented in the
\texttt{knowledgematrix} library; we use a multi-architecture fork of it,
pinned by commit in the supplementary material's requirements file, which
adds the \texttt{extract\_weff}, \texttt{densenet121}, \texttt{googlenet}
and \texttt{resnet152} changes.
The \texttt{neuralteleportation} library
(\citealt{armenta2024neural}) requires patches for PyTorch~2.x
compatibility; the patch set ships with the supplementary code. Code,
scripts and the JSON artifacts behind every table are in the anonymized
supplementary material. The finite-arithmetic residuals and the eval-mode
requirement are stated under ``Finite arithmetic'' above.

\paragraph{Hardware and reproducibility.}
All experiments run on a national academic HPC allocation of H100 GPUs; the
verification scripts named in this paper run on a CPU.
The sample sets and seeds are stated under ``Sample sets and seeds'' above.

\paragraph{Signal and noise: the two reference scales.}
Several studies of the main text report a drift ``relative to
the adversarial signal''. The \emph{adversarial signal} of a representation is
the distance that representation moves between a clean input $x$ and its
adversarial counterpart $x'$ from the six-attack suite above:
$\|\Mfx-\Mfxp\|_F$ for the knowledge matrix and $\|h(x)-h(x')\|_2$ for the
penultimate features, averaged over the pairs of one (architecture, attack)
cell. It is the reference scale --- the magnitude of change the representation
is \emph{supposed} to register. Reporting some other drift as a fraction of it
makes the comparison unit-free and judges each representation in its own units,
which a direct comparison of a Frobenius norm on a $1000\times150{,}529$ matrix
against an $\ell_2$ norm on a $2048$-vector cannot do.

The \emph{permutation noise floor} is the
residual $\|\tilde{\M}(x)-\M(x)\|_F$ that survives a neuron permutation.
A permutation is a quiver isomorphism, so in exact arithmetic this residual is
$0$; what the pipeline measures is its own resolution (``Finite arithmetic''
above), the smallest drift it can register, and any measured drift at or below
it is indistinguishable from zero. On ResNet-152 it is $0.116$ in the mean and
$5.78$ at its maximum in Frobenius norm (Table~\ref{tab:signal-relative}(a),
over $5$ permutation draws $\times\,50$ knowledge-matrix samples). The
penultimate floor in the same panel is the same quantity for $h$,
$\|h_{\tilde W}(x)-h(x)\|_2$, over $5$ draws $\times\,500$ samples, and it is
a genuine motion --- penultimate activations move under a relabeling of
neurons.

Panel~(b) of
Table~\ref{tab:signal-relative} reports these two scales as a ratio, one cell
per (attack, architecture), written ``$x \mid y$'': the bar packs two numbers
into one cell and is a separator, not a division or a conditioning. Here $x$ is
the knowledge matrix's ratio and $y$ the penultimate features' own, each formed
entirely in its own units, and both are \emph{mean adversarial signal} $\div$
\emph{mean permutation noise}. The numerator is the mean over that cell's
guarded clean/adversarial pairs of the clean-to-adversarial distance
($\|\Mfx-\Mfxp\|_F$ for $x$, $\|h(x)-h(x')\|_2$ for $y$); the denominator is
the matching permutation floor of panel~(a), the mean over the permutation
draws of the clean-to-permuted distance at the same architecture
($\|\tilde{\M}(x)-\M(x)\|_F$ for $x$, $\|h_{\tilde W}(x)-h(x)\|_2$ for $y$),
which is the same all the way down an architecture's column because a
permutation floor does not depend on the attack. A ratio above $1$ therefore
says the representation registers the attack more strongly than it registers a
function-preserving relabeling of neurons; below $1$ says the reverse --- the
representation moved further when nothing about the function changed than when
the input was attacked. Worked example, ResNet-152/FGSM: $83.98/0.1157 = 726$
for the knowledge matrix and $7.600/15.05 = 0.51$ for the penultimate features.
Signal and noise come from two separate runs on the same architecture (the
attack suite and the permutation experiment), so such a cell compares two
\emph{scales}, and is not a paired per-image measurement. Panel~(a) itself is
summarized over the same draws: ``KM mean'' and ``KM max'' average and maximize
the per-draw statistics, while ``KM median (pooled)'' is the median of the
pooled per-sample distances ($250$ values per architecture).

\paragraph{Bootstrap confidence intervals.}
A $95\%$ bootstrap confidence interval is obtained by resampling the observed
units \emph{with replacement} $B$ times, recomputing the statistic on each
resample, and reading percentiles off the resulting distribution of $B$ values.
Two such intervals appear in this paper and they agree in nothing but the
$95\%$, so we state both explicitly. \emph{(i)}~The teleportation panel
(Table~\ref{tab:s1_table}) uses the \emph{percentile} bootstrap of the
\emph{mean}, $B=10{,}000$ resamples, fixed seed; the resampled unit is the
$T=50$ teleportation draws of that architecture, with the $N=25{,}000$ images
held fixed. \emph{(ii)}~The appendix ordering panel
(Table~\ref{tab:s3_ordering_panel}) uses a seeded bias-corrected and
accelerated (BCa) bootstrap of the \emph{median}, $B=20{,}000$; the resampled
unit is the adversarial pairs inside one cell. Its intervals on an
adjacent-rank \emph{difference} are percentile intervals computed twice ---
resampling the two cells independently, and resampling the matched pair indices
jointly --- and a gap counts as resolved only when both exclude zero.

What such an interval means: it is the range of values of the statistic that
resampling the observed units is consistent with. It quantifies the sampling
variability of \emph{this} estimate over \emph{that} unit, and nothing else. It
is not a range for the underlying quantity across anything the resampled unit
does not vary over: an interval over $50$ teleportations says nothing about how
the number would move on a different image set, and an interval over
adversarial pairs says nothing about how it would move on a different
architecture or a different attack budget.

\paragraph{Samples, features, and the CKA estimator.}
In the representational-similarity literature $n$ is the number of samples
entering an estimator and $p$ the dimension of the representation being
compared; $n<p$ names the \emph{high-dimensional} regime, the regime in which
the biased (V-statistic) HSIC estimator's upward bias is worst. The similarity
panel of Section~\ref{sec:study1c} is not in that regime. Its HSIC is evaluated
on the full $N=25{,}000$-sample Gram matrix --- the chunking is an I/O device,
the row blocks being concatenated before the estimator is applied --- against
penultimate dimensions $D\in\{1024,2048\}$, so here $n = 25{,}000 \gg p$: the
\emph{low}-dimensional regime. The estimator choice is unaffected by that. The
unbiased HSIC$_1$ $U$-statistic removes the $O(1/n)$ upward bias of the biased
estimator at \emph{any} ratio of $n$ to $p$, and it is the estimator the
minibatch-CKA framework we follow specifies, so it is the right choice
regardless of regime --- the high-dimensional regime is simply where the
difference matters most, and we are not in it. We check directly that it leaves
no residual upward bias: the Murphy shuffled-pair control returns debiased CKA
$\le 5.4\times10^{-4}$ (Section~\ref{sec:study1c}).

\subsection{Study~3: the scale confound, and the comparison without alignment}
\label{app:crossarch-detail}

\paragraph{What the panel's measures return, and what the knowledge-matrix distance is.}
The penultimate-feature dimensions of the three architectures used in this paper
differ ($D_{\text{RN-152}} = 2048$, $D_{\text{DN-121}} = D_{\text{GN}} = 1024$),
and the similarity panel is well defined on them as they stand: we compute it on
the raw unequal dimensions, with no projection and no padding
(Appendix~\ref{app:limitations-more}, L10). What its measures return, however ---
linear and angular CKA~\citep{kornblith2019similarity}, orthogonal
Procrustes shape distance~\citep{williams2021generalized}, Normalized
Bures Similarity~\citep{harvey2023bures}, distance
correlation~\citep{szekely2007measuring}, soft
matching~\citep{khosla2024soft}, Gromov-Wasserstein~\citep{memoli2011gromov} and
representational similarity analysis~\citep{kriegeskorte2008representational}
alike --- is a \emph{population-level} similarity score living in its own metric
space. The knowledge-matrix distance, by contrast, is a \emph{per-sample}
distance, one Frobenius distance per input, reported RMS-per-coordinate,
whose logit-visible component is exactly the logit displacement divided by
$\sqrt{d{+}1}$ (Theorem~\ref{thm:pyth}). Knowledge matrices --- and, by Theorem~\ref{thm:weff}, the
per-class gradient$\times$input data they arrange --- are simultaneously
(i)~defined uniformly across architectures
($1000 \times (d{+}1) = 1000 \times 150{,}529$ for any feedforward
network on $224\times224$ ImageNet inputs and 1000-class outputs),
(ii)~equipped with the exact row-sum accounting that splits any
displacement into logit-visible and logit-invisible parts
(Theorem~\ref{thm:pyth}), and (iii)~invariant, under (LCS) at regular
inputs, under the germ stabilizer at $x$, which contains the global
function-stabilizer
(Theorem~\ref{thm:maxinv}); among the representations compared in
Study~3, the knowledge matrix is the one that carries all three, and what is
specific to it is the fixed $C\times(d{+}1)$ arrangement.

\paragraph{The scale confound.}
We do not read the agreement between two of the columns of Table~\ref{tab:s2_table} as corroboration.
Three architectures give three pairs, so every ordering statement in the
Result paragraph of Section~\ref{sec:crossarch} rests on three points with \emph{zero} residual degrees of
freedom; no interval could be put on it that was not decoration, and we put
none on the knowledge-matrix column. Worse, two of
the three columns that agree are \emph{unnormalized} and share one nuisance
scale. DenseNet-121's penultimate representation carries $3.5$--$3.8\times$ the
energy of the other two --- centered Gram traces $2.909\times10^{6}$
(ResNet-152), $1.103\times10^{7}$ (DenseNet-121), $3.181\times10^{6}$
(GoogLeNet), and unbiased-HSIC self-terms $44.7$, $1534.0$, $95.4$ --- and
soft-matching distance is not divided by any of that. Its raw ordering
($63.6 <109.1 <124.0$) accordingly places both DenseNet-containing pairs above
the one without. Dividing each pair's soft-matching distance by
$\sqrt{\|\tilde X\|_F^2+\|\tilde Y\|_F^2}$ gives $0.0258$ (RN--GN), $0.0289$
(DN--GN), $0.0332$ (RN--DN): the closest pair survives normalization, but the
top two swap, so the agreement with the knowledge-matrix ordering
($0.0107<0.0169<0.0180$) is partial rather than exact. The knowledge-matrix
column has the same exposure and we cannot currently quantify it: the reduce
stores per-sample cross-architecture distances but no per-architecture
$\|\M(x)\|_F$ to normalize them by. We therefore make no claim that these two
statistics corroborate each other, and read that paragraph as reporting that the
penultimate measures disagree among themselves --- which is unaffected by any
common scale, since RSA-Spearman is a rank statistic and is not exposed to it.

\paragraph{Cross-architecture comparison without alignment.}
The knowledge-matrix Frobenius distance in Table~\ref{tab:s2_table} is
computed directly between architectures with no PCA projection, dimension
matching, or learned alignment. No penultimate \emph{distance} $d_h$ is reported
across architectures at all: the penultimate columns of that table are six
penultimate similarity measures and one functional baseline (square-root
output JSD on the logits) computed on the raw $2048$- and $1024$-dimensional
features, and the two that are missing --- cross-architecture Procrustes,
whose per-chunk PCA target dimension collapses in our chunked pipeline (the
shape distance itself is defined for unequal widths by zero-padding the
narrower representation, \citealp{williams2021generalized}, and needs no
PCA), and Gromov--Wasserstein, whose entropic solver returned a degenerate
plan --- are reported as omissions rather than patched into comparability. The
dimension-matched (PCA) panel was never computed
(the dimension-matched block of the reduce is empty for all three pairs), so nothing in
Section~\ref{sec:crossarch} rests on a projection. The per-pair knowledge-matrix Frobenius distances are
$0.0169$ (RN--DN), $0.0107$ (RN--GN), and $0.0180$ (DN--GN), each computed
directly on the uniform $1000\times150{,}529$ matrices with no projection,
illustrating that the knowledge matrix supports an alignment-free
cross-architecture metric where the penultimate features do not. The same cross-architecture
displacement $\M_{W_A}(x)-\M_{W_B}(x)$ is moreover subject to the exact
accounting of Theorem~\ref{thm:pyth}: it splits into the part visible to
the two networks' logit gap and an orthogonal part that records how
differently the two architectures linearize the input, with the
coherence $A=(d_\Psi/d_M)^2$ (Definition~\ref{def:coherence}) computable per
sample from one additional stored scalar, the cross-architecture logit gap
$\|\Psi_A(x)-\Psi_B(x)\|_2$, which the shipped reduce does not carry beside
the per-sample knowledge-matrix distance; the per-sample cross-architecture
$A$ is recorded as follow-up. We read this decomposition as descriptive
geometry, not as a quality ranking of the architectures.

%% file: sections/appendix_mechanism_pilot.tex
\section{The mechanism pilot behind the attack-family ordering}
\label{app:mechanism-pilot}

Theorem~\ref{thm:anatomy} shows that each wall crossing contributes to the
knowledge-matrix displacement a rank-one dyad whose row sums vanish. This
appendix states the conjecture that reads the attack-family ordering of
Study~2 through those dyads, and records the one pilot that tests its
rank-correlation consequence, with the corrections and caveats that attach to
it.

\begin{conjecture}[Mechanism of the attack-family ordering]\label{conj:mechanism}
Iterative small-step attacks cross fewer walls and align with high-energy
Jacobian columns, raising $A$; one-shot sign-based and gradient-free attacks
do the opposite. In particular, at fixed perturbation size the per-pair
coherence $A$ is negatively rank-correlated with the endpoint mask-Hamming
distance $H$.
\end{conjecture}

The second sentence is the testable consequence. A pilot on per-pair VGG data
(attack-success filtered, size-controlled) finds the predicted sign on all
three attacks tested (Table~\ref{tab:mechanism-pilot}); the
exact-$\|\delta\|$-controlled replication on the full architecture set is
future work. The conjecture speaks of high-energy Jacobian columns, which we
make precise next. The \emph{energy} of column $i$ of the
Jacobian is $\|J_{:,i}\|_2^2$: the weight with which a unit change in pixel
$i$ reaches the logits. A perturbation \emph{aligns with high-energy columns}
when its mass $|\delta_i|$ sits on the coordinates where $\|J_{:,i}\|$ is
large. The link to $A$ runs through Theorem~\ref{thm:within}: writing
$a_i=\delta_iJ_{:,i}$, we have $A=\|\sum_ia_i\|^2/\sum_i\|a_i\|^2$, so $A$
records how far the weighted columns reinforce rather than cancel --- maximal
($A=d$) when they are all equal, and $A=1$ whenever a single one stands alone,
which is the one-pixel law. Concentrating $\delta$ on a few high-energy
columns therefore suppresses the many small, mutually canceling contributions
that pull $A$ below $1$, while by Theorem~\ref{thm:anatomy} each wall crossing
avoided removes a dyad that enters $d_M$ and contributes nothing to $d_\Psi$. The conjecture
is that both halves point the same way for iterative small-step attacks; it
remains a conjecture, and the pilot tests only its rank-correlation
consequence, not the mechanism itself.

\input{tables/table_mechanism_pilot}

\paragraph{The pilot, and two refinements of its reading.}
Section~\ref{sec:study2} reports the size-controlled partial rank correlations
of the endpoint mask-Hamming distance $H$ with the coherence $A$ on the three
VGG cells that retain per-pair records, $\rho_S(H,A\mid d_\Psi)=-0.343$ (APGD),
$-0.304$ (DeepFool) and $-0.341$ (Square) (Table~\ref{tab:mechanism-pilot}).
Two points sharpen the reading of this pilot.
\emph{(i)}~The size confound is present on two of the three attacks, not
three. For DeepFool and Square, $H$ and $d_\Psi$ are strongly
\emph{positively} correlated ($\rho_S=+0.720$ and $+0.683$) and the raw
$\rho_S(H,A)$ carries the wrong sign ($+0.441$, $+0.129$), so partialling
$d_\Psi$ out is what reveals the effect. For APGD it is the other way round:
$\rho_S(H,d_\Psi)=-0.304$, the raw $\rho_S(H,A)=-0.446$ \emph{already} has the
predicted sign, and partialling moves it \emph{towards} zero. The size-confound
story is a statement about DeepFool and Square.
\emph{(ii)}~A cleaner statistic exists, and we report it too. Because
$A=(d_\Psi/d_M)^2$ \emph{contains} $d_\Psi$, conditioning on $d_\Psi$ partials
out a variable that sits inside the response. The direct form has no such
circularity: at fixed logit displacement, does crossing more walls add
knowledge-matrix mass? It does ---
$\rho_S(H,d_M\mid d_\Psi)=+0.335$ (APGD), $+0.383$ (DeepFool), $+0.373$ (Square)
--- which is the sign the crossing heuristic behind
Conjecture~\ref{conj:mechanism} predicts: Theorem~\ref{thm:anatomy} shows that
the crossing dyads contribute nothing to the logit displacement, and the
conjecture's reading is that they add knowledge-matrix mass.
One caveat this does not remove: $d_\Psi$ is a proxy for $\|\delta\|$, and
the residual $\|\delta\|$ confound points in the \emph{same} direction as the
predicted effect, so the pilot cannot exclude it; that is precisely why the
mechanism stays a conjecture. The partial defense worth stating is that for
fixed-$\varepsilon$ $\ell_\infty$ APGD, $\|\delta\|_2$ is nearly constant across
pairs, so the confound is weakest on the attack where the effect is
strongest.

This is a \emph{VGG pilot, $3$ attacks, $n=134$--$141$, $d_\Psi$-proxy-controlled};
VGG is not one of the three networks of Section~\ref{sec:setup}, and the
exact-$\|\delta\|$-controlled replication on those three (with
$\|\delta\|$ from the stored pairs) is registered as follow-up work. The
mechanism therefore remains a conjecture (Conjecture~\ref{conj:mechanism}),
supported but not established.\footnote{The VGG same-region ($d_M=0$) diagnosis itself is the
rank correlation between the endpoint ReLU-Hamming distance and $d_M$,
reported over the $n=192$ pairs with nonzero $d_M$. It is sensitive to the
attack-success filter: imposing $d_\Psi\ge1$ collapses APGD's correlation to
$\approx0.07$ while DeepFool and Square stay $\approx0.6$, confirming that
the apparent APGD relationship was carried by the same-region
attack-failure pairs.}

%% file: tables/table_mechanism_pilot.tex
\begin{table}[t]
\centering\small
\begin{tabular}{lccccccc}
\toprule
 & & & & & \multicolumn{3}{c}{median $A$ by $H$-tertile} \\
\cmidrule(l){6-8}
attack & $n$ & $\rho_S(H,d_\Psi)$ & $\rho_S(H,A)$ & $\rho_S(H,A \mid d_\Psi)$ & low & mid & high \\
\midrule
APGD & 141 & -0.30 & -0.45 & \textbf{-0.34} & 1.46 & 1.03 & 0.835 \\
DeepFool & 134 & +0.72 & +0.44 & \textbf{-0.30} & 0.0222 & 0.0851 & 0.101 \\
Square & 141 & +0.68 & +0.13 & \textbf{-0.34} & 0.0943 & 0.116 & 0.118 \\
\bottomrule
\end{tabular}
\caption{On all three attacks the size-controlled partial rank correlation of $H$ with
$A$ is negative, the sign the crossing conjecture predicts. This is the first direct
test of the smooth+crossing mechanism (per-pair VGG debug data).
The size confound is real on two of the three attacks and absent on the third, so we
report the diagnostic column $\rho_S(H,d_\Psi)$ rather than assert it. On DeepFool and
Square, pairs that cross more walls are also pairs whose logits move further
($\rho_S(H,d_\Psi)=+0.72$ and $+0.68$), and the raw $\rho_S(H,A)$ inherits that sign
($+0.44$, $+0.13$) --- the opposite of what the mechanism predicts; controlling for
perturbation size via $d_\Psi$ turns both negative. On APGD there is no such confound:
$H$ and $d_\Psi$ are \emph{negatively} associated ($-0.30$), the raw correlation already
carries the predicted sign ($-0.45$), and the control moves it \emph{toward} zero
($-0.45\to-0.34$). What the three attacks share is the size-controlled column: the
partial rank correlation is negative in every case ($-0.30$ to $-0.34$) --- crossings
reduce coherence, the sign Conjecture~\ref{conj:mechanism} predicts (crossing dyads carry
knowledge-matrix mass with zero logit displacement, the latter by
Theorem~\ref{thm:anatomy}). $H$ lower-bounds the segment crossing count. Pilot
caveats: single architecture; $d_\Psi$ is a proxy for $\|\delta\|$ (the exact control
lives in \texttt{pairs.pth}); three attacks on one network is not a corrected multiple
test and no $p$-value is claimed here; the three-network, $\|\delta\|$-controlled replication
is registered as follow-up work, and this pilot fixes its analysis plan.}
\label{tab:mechanism-pilot}
\end{table}

%% file: sections/appendix_kendall.tex
\section{Kendall's coefficient of concordance \texorpdfstring{$W$}{W}, worked through on the Study~2 panel}
\label{app:kendall}

$W$ answers one question: \emph{do several judges, each ranking the same list of
items, put them in the same order?} It is $1$ if every judge produces the
identical ordering and $0$ if they agree no more than chance would predict.

A \emph{cell} here is one (architecture, attack) pair --- one square of a $6\times6$ grid --- holding that
architecture's $\approx200$ adversarial pairs for that attack, summarized by the
median of $d_M/d_\Psi$ over them. The judges (``raters'') are the $m=6$
architectures; the items are the $n=6$ attack families. So each architecture
reads off its own row of six cell medians and ranks the six attacks by them.

Rank~$1$ is given to the \emph{largest} median $d_M/d_\Psi$, rank~$6$ to the smallest --- that is the
descending convention, and it is why rank~$1$ means ``largest ratio
$d_M/d_\Psi$, least coherent''. If two cells of one row had equal medians they would share the
average of the ranks they would have occupied (a \emph{midrank}: two items tied
for 2nd and 3rd both get $2.5$). No row of this panel has a tie, so every rank
here is a whole number. Ranking each row gives the $m\times n$ matrix $R$:

\smallskip
\centerline{\begin{tabular}{@{}lcccccc@{}}\toprule
 & FGSM & PGD & CW & DeepFool & APGD & Square \\ \midrule
ResNet-152 & 4 & 5 & 3 & 2 & 6 & 1 \\
DenseNet-121 & 4 & 6 & 2 & 1 & 5 & 3 \\
GoogLeNet & 4 & 5 & 3 & 1 & 6 & 2 \\
ResNet-18 & 4 & 6 & 3 & 1 & 5 & 2 \\
AlexNet & 4 & 6 & 2 & 1 & 5 & 3 \\
VGG & 4 & 6 & 2 & 1 & 5 & 3 \\ \midrule
column sum $C_i$ & 24 & 34 & 15 & 7 & 32 & 14 \\ \bottomrule
\end{tabular}}
\smallskip

$C_i$ is the sum down column~$i$ --- the six architectures' ranks for that one attack. If the judges agree, some
attack collects rank~$1$ from all six ($C=6$) and another collects rank~$6$ from
all six ($C=36$): the column sums are pulled far apart. If they disagree at
random, every column sum sits near its mean $\bar C=m(n{+}1)/2=21$. So the
spread of the column sums \emph{is} the agreement. Measuring that spread by
$S=\sum_i(C_i-\bar C)^2$ gives, for the matrix above,
$S=3^2+13^2+(-6)^2+(-14)^2+11^2+(-7)^2=580$, against a maximum of $630$ attained
when all six judges agree perfectly. Hence
\[
W=\frac{S}{S_{\max}}=\frac{580}{630}=0.9206,
\qquad\text{equivalently}\qquad
W=\frac{12\,S}{m^{2}\bigl(n^{3}-n\bigr)-m\,T},
\]
the second form being the usual one, since $m^{2}(n^{3}-n)/12=630$ here. The
tie correction is $T=\sum_{j}\sum_{g}(t_{jg}^{3}-t_{jg})$, summed over each
judge $j$'s groups $g$ of tied items with $t_{jg}$ the size of the group; a
group of size~$1$ contributes $1^3-1=0$, so with no ties anywhere $T=0$ and the
denominator is just $6^2(6^3-6)=7560$.

The null hypothesis of the Monte-Carlo test is that the judges are not
agreeing at all --- that each architecture's ranking is an independent
uniform random permutation of the six attacks. To see how often chance alone
reaches $W\ge0.921$, we \emph{draw} $m=6$ such random permutations, stack them
into a rank matrix, compute its $W$, and repeat $B=10^{6}$ times. That is what
is being sampled: whole random rank matrices. None of the $10^{6}$ draws reached
$0.921$, hence $p<10^{-6}$. For the three-architecture figures the null is small
enough to enumerate instead of sample --- fix the first judge's ranking and run
the other two over all $(6!)^2=518{,}400$ pairs of permutations --- so those
$p$-values are exact rather than estimated. That enumeration attains only $77$
distinct values of $W$, the closest pair $0.0127$ apart, so a third decimal
place would report a precision the statistic cannot have; we quote those $W$ to
two decimals for that reason.

Only the ranks enter --- the ranks of the six cell medians within each architecture's row --- never the medians
themselves. Two consequences. First, $W$ is blind to how far apart the families
are: the ordering is $\{$DeepFool, CW, Square$\}>$ FGSM $\gg\{$PGD, APGD$\}$,
whose two gaps are $1.13$--$1.51\times$ and $1.57$--$2.09\times$, and $W$ would
be unchanged if both gaps were a hundred times larger or a hundred times
smaller. Those magnitudes are reported separately, in the text above. Second,
$W$ is unchanged by any strictly increasing transformation of $d_M/d_\Psi$, which
is what makes it immune to this paper's raw-versus-RMS unit choice
(Appendix~\ref{app:setup-details}): passing to root-mean-square-per-coordinate divides
$d_M/d_\Psi$ by $\sqrt{d{+}1}$, the same positive constant for all six attacks of a
given architecture, so no row's ordering moves and $W$ is identical either way.

A high $W$ is not a statement that the ordering is interesting. Concordance
measures agreement among raters, not the interest of what they agree on, and it is \emph{maximized} by quantities that
have nothing to do with the networks. Any attack-level constant --- the
$\varepsilon$ of the ball, the iteration budget, the year the attack was
published --- is assigned identically by every rater, so every rater's row is the
same row and $W=1$ exactly, while the statistic says nothing whatever about any
network. A high $W$ therefore establishes that the ordering is an
architecture-independent property of the \emph{attacks}; it is evidence
\emph{against} the ordering being an idiosyncrasy of one network, and evidence of
nothing else. ``More concordant'' is not ``better'', which is why we compare
$W$ values below only to say which quantity is architecture-independent, never
to rank the quantities by quality.

\paragraph{The ratio and its denominator.}
The statistic ranked here is $d_M/d_\Psi$, and the obvious objection is that its
ordering is a shadow of its denominator's. On the six-rater
$200$-pair panel that objection largely lands.
Ordering the same $36$ cells by median logit displacement $d_\Psi$ alone gives
$W=0.9746$, \emph{higher} than the ratio's $0.9206$;\footnote{The coincidence
of this value with the full-scale ratio concordance of the three networks,
$W(d_M/d_\Psi)=0.9746$ below is numerical: the two are different statistics
on different pair sets.} ordering them by median
penultimate displacement $d_h$ gives $0.9714$. On four of the six raters
(DenseNet-121, ResNet-18, AlexNet, VGG) the ratio ordering is the exact
rank-for-rank reversal of the $d_\Psi$ ordering, rank Spearman $-1.000$, with
$-0.943$ on GoogLeNet and $-0.771$ on ResNet-152. On this pair set the ratio
adds nothing beyond reversing $d_\Psi$ on four raters, and something only on
ResNet-152 and GoogLeNet.

Three facts stop that from being the whole story, and all three are computed on
data already reported here. \emph{First}, the ordering inverts at full scale. On
the three networks at full scale --- $512$--$1000$ pairs per cell against $200$, on a different
image set, with per-pair records --- the ratio is the \emph{most} concordant of
the three: $W(d_M/d_\Psi)=0.9746$ against $W(d_\Psi)=0.8222$ and
$W(d_M)=0.7968$, identical under both filtering conventions. A statistic that
was merely a function of its denominator could not be more concordant across
raters than that denominator itself is. \emph{Second}, the ratio ordering is the one that reproduces
across the two image sets: comparing each architecture's ranking on the $200$-pair set
against its ranking on the full-scale set, the ratio reproduces at rank Spearman
$+0.943$ on all three networks, whereas $d_\Psi$'s own ordering reproduces at
$+1.000$, $+0.943$ and only $+0.600$ (ResNet-152). \emph{Third}, the concordance
survives partialling the denominator out. Regressing $\log d_M$ on $\log d_\Psi$
within each architecture and ranking the six attacks by the residual --- the
part of the matrix displacement that the logit displacement does not explain ---
gives a \emph{partial concordance} of $W=0.8857$ over the four full-budget raters
(exact permutation $p=2.4\times10^{-5}$), $W=0.7587$ over all six (Monte-Carlo
$p=7\times10^{-6}$, $10^{6}$ seeded draws), and $W=0.9111$ on the three networks at full scale (exact permutation $p=4.9\times10^{-4}$). The residual ordering is led by \emph{Square} on four of
the six raters and on all four full-budget ones: at matched logit displacement,
the gradient-free random-search attack moves the knowledge matrix most. This is
the sense in which $d_M$ is not decorative --- what it contributes is a
consistently ranked residual, not a rescaling of $d_\Psi$.

\paragraph{Two qualifications on the resolved gaps of the appendix panel.}
Two qualifications attach to the sentence of Section~\ref{sec:study2} that reads the panel's bootstrap: within the matched panel the leading ResNet-152 cells are not separated by the data, while on DenseNet-121 and GoogLeNet every adjacent gap is resolved at $95\%$.
\emph{(i)}~The unresolved ResNet-152 set straddles the family boundary: it is
$\{$DeepFool, Square, CW, \emph{FGSM}$\}$, and FGSM is not a member of the
leading family --- the unresolved CW--FGSM gap (difference $0.5829$; independent
interval $[-0.254,1.755]$, paired interval $[-0.080,1.542]$) is exactly the
boundary that \emph{defines} the grouping. So on ResNet-152 the panel does not
resolve the family grouping itself at $95\%$; only the point estimate orders the
families there, and the resolved grouping is a DenseNet-121 and GoogLeNet
result.
\emph{(ii)}~The panel runs $15$ simultaneous resolution tests without
correction: three architectures $\times$ five adjacent gaps, each requiring two
$95\%$ intervals to exclude zero, with no Bonferroni, Holm or false-discovery
adjustment anywhere in this paper. Under a global null the family-wise error
rate of a $15$-test family at $\alpha=0.05$ approaches $1-0.95^{15}\approx54\%$,
so ``every adjacent gap is resolved'' on DenseNet-121 and GoogLeNet --- ten
simultaneous resolutions --- should be read as an uncorrected nominal-$95\%$
statement. We report it that way rather than adjusting it, because the
requirement that \emph{both} the independent and the paired interval exclude zero
already makes each individual verdict conservative in an unquantified direction,
and stacking an uncorrected conservatism against an uncorrected multiplicity
would give a number we could not interpret (Appendix~\ref{app:limitations-more},
L12).

\paragraph{What the within-region theorem does not explain.}
A density reading of the ordering of Section~\ref{sec:study2} --- that Theorem~\ref{thm:within}(iii)
ranks attacks by the sparsity of $\delta$ --- is tempting and wrong, and we
record why: it does not survive contact with either the theorem or
the data. Part~(ii) of that theorem gives a
$1$-sparse within-region perturbation coherence \emph{exactly} $1$, which is the
\emph{top} of the scale, whereas every cell of the three networks here sits at
$A\in[0.010,0.228]$; part~(iii) is a non-tight \emph{upper} bound $A\le k$, and
at $k=d=150{,}528$ (the matrix has $d{+}1=150{,}529$ columns) it is vacuous. So
the theorem places sparse perturbations
high and says nothing about dense ones, which is the opposite of that ordering.
The data refute the density reading directly as well: FGSM, PGD and APGD
all run at the identical default $\varepsilon=8/255$
(Appendix~\ref{app:limitations-more}, L9) with FGSM's $\delta$ the
densest of the three, yet on ResNet-152 $A=0.0303$ (FGSM), $0.0998$ (PGD),
$0.1130$ (APGD) --- a $3$--$4\times$ spread at matched density. And of the six
attacks only Square produces a genuinely localized $\delta$; DeepFool and
CW-$\ell_2$ produce dense, full-support perturbations. The boundary-distance
reading in the Result paragraph of Section~\ref{sec:study2} is what the data support, and
Theorem~\ref{thm:within} is used in this study only as the pair of reference
lines $A=1$ and $A\le d$ against which the magnitudes are read.

\subsection{Population and filtering details}
\label{app:population-filtering}

\paragraph{The two filters.}
The stricter attack-success filter
$d_\Psi \ge 1$, $d_M > 0$ is a different, far more aggressive cut than the
division guard $d_\Psi > 10^{-6}\max_i d_{\Psi,i}$ of Section~\ref{sec:study2},
and is used
only in the two analyzes that hold per-pair records: the mechanism pilot and
the appendix panel. On VGG the guard drops $2$ of $200$ pairs on each of
DeepFool, APGD and Square and none at all on FGSM, PGD and CW, whereas
$64/200$ DeepFool and $57/200$ APGD pairs sit in $0 < d_\Psi < 1$ and would be
cut by the stricter filter.

\paragraph{The headline is a maximum of cell medians.}
``Median $A\le0.23$'' is a maximum of medians, over two nested
populations, and never a median of pooled data. The inner population is one
(architecture, attack) \emph{cell}: within a cell we take the median, over that
cell's $110$--$200$ guarded adversarial pairs, of the per-pair coherence
$A=(d_\Psi/d_M)^2$. The outer population is the $18$ cells of the three networks ($3$ architectures
$\times$ $6$ attacks); $0.23$ is the largest of those $18$ cell medians,
attained at GoogLeNet/APGD ($0.228$, Table~\ref{tab:coherence}). Pairs are
never pooled across attacks, and cells are never pooled across architectures,
so the statement is the stronger ``no cell of the three networks has a median above $0.23$''
rather than a statement about one pooled median. Each cell entry is computed as
$1/\mathrm{median}(d_M/d_\Psi)^2$ from the stored median ratio; that equals the
median of $A$ exactly when the cell's valid-pair count is odd, hence the
interpolation caveat in the caption of Table~\ref{tab:coherence}.

\paragraph{Sensitivity to the filtering convention.}
These medians are point estimates without bootstrap CIs: the reduce behind
Table~\ref{tab:coherence} kept only per-cell aggregates, so putting CIs on
them means re-running the pairs, not resampling a stored file (the appendix
panel, whose per-pair ratios \emph{are} stored, does carry them;
Appendix~\ref{app:limitations-more}, L12). They are
also sensitive to the filtering convention. On the three cells that retain
per-pair records --- VGG/APGD, VGG/DeepFool, VGG/Square --- replacing the
division guard by the stricter attack-success filter ($d_\Psi\ge1$, $d_M>0$)
moves the median $d_M/d_\Psi$ by $-12\%$, $-39\%$ and $-12\%$ respectively,
enough to swap VGG's top two (CW ahead of DeepFool) and to carry VGG/APGD
across the $A=1$ line ($0.84 \to 1.09$); the family-level grouping is
unchanged. We report the tables under one convention and state the
sensitivity rather than re-cutting them (Section~\ref{sec:limitations}, L7).

\paragraph{The matched population of the appendix panel.}
We do not read the panel as agreeing \emph{better} than the rows of the three networks of
Table~\ref{tab:ordering}: three
raters give $W$ almost no resolving power against a difference this size, the
two are computed on different image sets, and the per-cell sampling
uncertainty below is wide enough to move either. That agreement depends on
matching the cells to a common image population: the ResNet-152 DeepFool run
stopped at $512$ of $1000$ pairs, and comparing the six ResNet-152 cells at
face value --- across different image populations --- moves CW by $+12.9\%$
and Square by $+8.5\%$ against untruncated cells that move by under $1\%$,
which is enough to invert CW and FGSM spuriously. Within the matched panel
the leading ResNet-152 cells are not separated by the data: a seeded
bootstrap over pairs reads DeepFool $\approx$ Square $\approx$ CW $\approx$
FGSM $>$ PGD $>$ APGD there, while on DenseNet-121 and GoogLeNet every
adjacent gap is resolved at $95\%$; the two qualifications on that reading are
the paragraph ``Two qualifications on the resolved gaps of the appendix
panel'' above.

%% file: sections/appendix_negatives_detail.tex
\section{The honest negatives in detail}
\label{app:negatives-detail}

This appendix carries the setup, the full table and figure, and the reading of
the two honest negatives that Section~\ref{sec:negatives} summarizes.

\subsection{The detector bake-off}
\label{app:negatives-bakeoff}

\paragraph{Setup and procedure.}
The bake-off is a $6$-detector $\times$ $3$-representation grid on
\emph{one} configuration: \emph{AlexNet trained on CIFAR-10}, evaluated
against $16$ \texttt{torchattacks} attack families. It is \emph{not} run on the
three ImageNet networks of Sections~\ref{sec:study1}--\ref{sec:crossarch}, and we state that
scope wherever the negative is cited. The numbers below are read from the rendered tables of the retired detection pipeline (\texttt{tables\_alexnet\_cifar10/\{representation\_comparison, per\_attack\_auroc\_alexnet\_cifar10, svd\_ablation\_alexnet\_cifar10, lee2018\_comparison\}.tex}, shipped in the supplementary material); the per-(detector, representation, attack) AUROC files that produced them were not preserved, so the bake-off is reported from those tables, not regenerated.
\begin{enumerate}
\item Fix three representations of each sample: the penultimate features, the
concatenation of \emph{all} hidden-layer activations, and the knowledge matrix.
(The third arm is the all-layer concatenation, not the logits.)
\item Fix six off-the-shelf detector configurations --- Mahalanobis, $k$-NN,
KDE, GMM, one-class SVM, and Isolation Forest.
\item For each (representation, detector) pair, fit the detector on clean
samples, score clean against adversarial, and average AUROC over the $16$
attacks.
\item Count, per detector, which representation wins. Report a properly
calibrated external baseline alongside: the multi-layer Mahalanobis detector
with per-attack logistic regression of \citet{lee2018simple}.
\end{enumerate}

\paragraph{Result.}
\begin{table}[htbp]
\centering\small
\caption{Penultimate features win five of the six detectors on AlexNet/CIFAR-10, and
the external baseline beats all three representations. The detector bake-off in full,
with its true scope: mean AUROC over $16$
\texttt{torchattacks} attack families, \textbf{AlexNet / CIFAR-10} --- one architecture and one
dataset, not the three ImageNet networks the rest of this paper uses. The three representations are the
penultimate features, the concatenation of all hidden-layer activations, and the knowledge
matrix; \emph{the third arm is the all-layer concatenation, not the logits}. Bold marks the best
representation per detector. Penultimate features win $5$ of $6$; the exception, one-class SVM,
puts the knowledge matrix at $0.550$ against $0.512$, both essentially at the $0.5$ chance line
on the one detector that works for nothing, so we do not describe it as a knowledge-matrix win.
The Lee et al.\ row is an external baseline, not one of the three arms: a multi-layer Mahalanobis
detector with per-attack logistic regression, which beats all three.}
\label{tab:bakeoff}
\begin{tabular}{@{}lccc@{}}
\toprule
Detector & Penultimate & All-layer & Knowledge matrix \\
\midrule
Mahalanobis      & \textbf{0.917} & 0.208 & 0.638 \\
$k$-NN           & \textbf{0.904} & 0.165 & 0.798 \\
KDE              & \textbf{0.909} & 0.289 & 0.788 \\
GMM              & \textbf{0.847} & 0.176 & 0.374 \\
one-class SVM    & 0.512          & 0.356 & \textbf{0.550} \\
Isolation Forest & \textbf{0.853} & 0.175 & 0.518 \\
\midrule
\citet{lee2018simple} baseline & \multicolumn{3}{c}{$0.937$} \\
\bottomrule
\end{tabular}
\end{table}

\emph{Penultimate features win $5$ of the $6$ detector configurations}
(Table~\ref{tab:bakeoff}), and the sixth is not a knowledge-matrix win in any
useful sense: one-class SVM scores $0.550$ for the knowledge matrix against
$0.512$ for the penultimate features, both close to the $0.5$ chance
line, on the one detector that fails for every representation. Two further cuts
of the same run point the same way. Broken out per attack, taking each
representation's best of the six detectors, the penultimate features beat the
knowledge matrix on \emph{all $16$} attacks (average $0.918$ against $0.798$),
and the external Lee et al.\ baseline beats both at $0.937$. And an SVD rank
ablation on the Mahalanobis detector --- which equalizes the dimensionality that
a charitable reading might have blamed --- has the penultimate features ahead at
\emph{every} rank from $16$ to $512$; at rank $16$ they score $0.925$ against the
knowledge matrix's $0.817$, and the knowledge matrix gets monotonically
\emph{worse} as more rank is restored ($0.817\to0.590$). The knowledge matrix
does not provide a detection advantage on this task under any of these cuts.
The adversarial-detection claim of \citet{leblanc2024hidden} is therefore \emph{not made here},
and the bake-off is reported as an honest negative.

\paragraph{Reading.}
This negative is consistent with the theory rather than in tension with
it. Detection asks which representation best separates two finite
empirical samples under a chosen classifier; it is a statistical-power
question about a particular discriminator, and nothing in
function-determination (Theorem~\ref{thm:maxinv}) or in the displacement
decomposition (Theorem~\ref{thm:pyth}) predicts that the canonical
representation should also be the most \emph{separable} one for an
off-the-shelf detector. What the theory does buy --- invariance under
the germ stabilizer at $x$ (which contains the global function-stabilizer),
exact row-sum accounting, alignment-free
cross-architecture comparison --- is orthogonal to detection performance.
The knowledge matrix is the right object for the canonical-representation
questions of Studies~1--3 and the wrong object for this detector bake-off.

\subsection{The single-region LP-counterfactual}
\label{app:negatives-lp}

For a source class $s$ and target class $t\neq s$, the
\emph{LP-counterfactual direction} at $x$ is the $\ell_1$-minimum input
perturbation $\delta\in\R^d$ that, \emph{within the linearization around
$x$}, pushes the target--source logit gap past a margin $m>0$:
\begin{equation}
\min_{\delta} \|\delta\|_1
\quad\text{s.t.}\quad
\big(\Weff(x)_{t,\cdot}-\Weff(x)_{s,\cdot}\big)\cdot\delta
\;\ge\;
m-\bigl(\Netx_t-\Netx_s\bigr),
\label{eq:lp-cf}
\end{equation}
subject to a per-coordinate box $\delta_i\in[-x^{\text{px}}_i/\sigma_c,
(1-x^{\text{px}}_i)/\sigma_c]$ that keeps the pixel-space image in
$[0,1]$ ($x^{\text{px}}_i=x_i\sigma_c+\mu_c$ recovers the pixel
intensity from the ImageNet-normalized input). Without the box the LP
\eqref{eq:lp-cf} is solved in closed form by the single most cost-effective coordinate
($k=\arg\max_i|c_i|$ for $c=\Weff(x)_{t,\cdot}-\Weff(x)_{s,\cdot}$);
with the box it becomes a saturation greedy that fills coordinates in
order of $|c_i|$ until the margin is met. Both are $O(n\log n)$, and the greedy is $\ell_1$-optimal for the
one-constraint box LP because the cost per unit of swing is $1/|c_i|$
\emph{independently} of the per-coordinate bound, so filling coordinates in
descending $|c_i|$ is an exchange argument: any feasible $\delta$ can be
rewritten with weakly smaller $\ell_1$ norm by moving swing onto a
cheaper-per-unit coordinate that is not yet saturated.

\paragraph{Setup and result.}
For each architecture in $\{$ResNet-152, DenseNet-121, GoogLeNet$\}$ we
run $3$ source images $\times\,3$ source classes $\times\,2$ target
classes ($18$ LPs per architecture, $54$ in all) with margin $m=0.1$,
and for each we record a Boolean \texttt{region\_ok} flag --- True only
when $x$ and $x+\delta$ realize identical activation patterns at every
ReLU and identical argmax indices at every pooling, i.e.\ when the
linearization the LP solved against still governs the network at
$x+\delta$. \emph{On $54$ of $54$ LPs, \texttt{region\_ok} is False},
and accordingly the actual (non-linearized) logits at $x+\delta$ never
flip:
\begin{itemize}
\item \texttt{region\_ok}: $0/54$;
\item target logit $>$ source logit at $x+\delta$: $0/54$.
\end{itemize}
We always report \texttt{region\_ok} next to any $\delta$-norm or
new-logit number: when it is False the linear-program guarantee is void,
because the perturbation derived from $\Weff(x)$ is large enough to flip
many ReLUs and pooling argmaxes downstream, landing $x+\delta$ in a
\emph{different} linear region governed by a different $\Weff$. This is
the geometry the theory predicts (Theorem~\ref{thm:within} holds
\emph{within} a region; nothing extends it across walls), so the
negative is a corollary, not a surprise. Figure~\ref{fig:lp-pathology}
shows one worked example.

\begin{figure}[h]
\centering
\includegraphics[width=\linewidth]{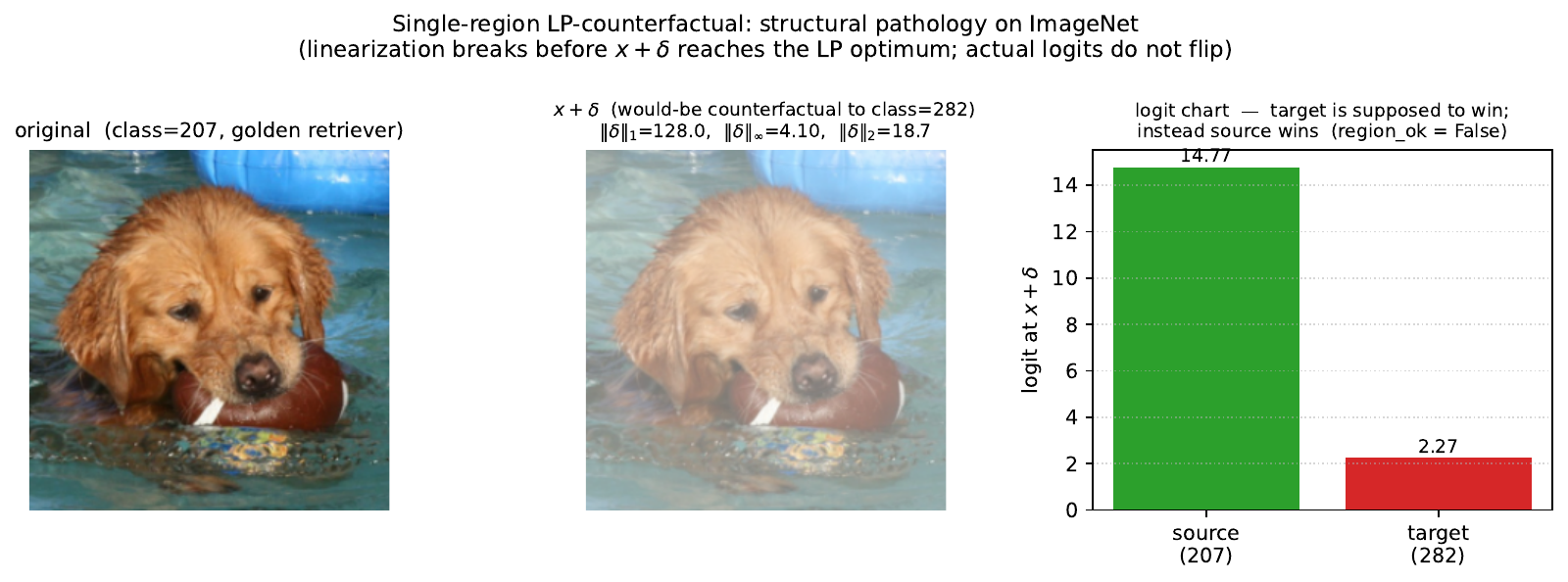}
\caption{The LP perturbation leaves the linear region it was solved in, and
the source class still wins at $x+\delta$. One worked instance of the
single-region LP-counterfactual structural failure, source class $207$ (golden
retriever) $\to$ target class $282$ (tiger cat). Left: original image $x$. Center: the would-be
counterfactual $x+\delta$ in ImageNet-normalized space, at
$\|\delta\|_1=128.0$, $\|\delta\|_\infty=4.10$, $\|\delta\|_2=18.7$. Right:
at the actual (un-clamped) $x+\delta$ the source class still wins by a large
margin, $14.77$ against $2.27$; \emph{region\_ok = False}, i.e.\ the region the
LP solved against does not contain $x+\delta$. Every one of the $54$
source/target pairs across the three architectures fails identically
(\emph{region\_ok} $=0/54$, Limitation~L3); this panel's $\|\delta\|_\infty$
sits at the per-channel saturation cap $1/\sigma_c\approx4.4$. \emph{Provenance:}
the figure is the surviving record of this LP ---
the per-LP $\delta$ and \emph{region\_ok} dumps behind the $0/54$ count are not
preserved in the supplementary material, and the run's architecture is not recorded in the
figure, so we do not attribute the panel to a specific network.}
\label{fig:lp-pathology}
\end{figure}

\noindent
The LP-counterfactual is a \emph{theoretical} matrix-direction object,
not an adversarial perturbation: its magnitudes are far outside any
standard budget. The negative
is that the matrix direction does not transfer to the model's actual
logits once it leaves the source region; a multi-region or
continuation-based reformulation is left as future work
(Section~\ref{sec:limitations}, L3).

%% file: sections/vocab/appendix_vocab_I.tex
\providecommand{\vocabdir}{../vocab}
\input{\vocabdir/vocab-macros}
\providecommand{\vocablabel}{app:vocab}

\begingroup
\renewcommand{\KM}{\mathrm{M}}
\renewcommand{\vx}{x}
\renewcommand{\vone}{\mathbf{1}}

\section{Vocabulary for algebraists}
\label{\vocablabel}

\noindent This appendix collects, for a reader who knows linear algebra, representation theory and quotients but is meeting the vocabulary of machine learning and statistics for the first time, the terms this paper uses. Each entry has four fixed fields: a \emph{definition}; an \emph{algebraic reading}; how the quantity is \emph{computed here}, with the routine or study that computes it; and how it is \emph{validated here}. Each entry is stated in this paper's notation. Notation: $\KM(\vx)\in\R^{C\times(d+1)}$ is the knowledge matrix of the input $\vx$, $\KM^{\circ}$ its class-centered form, $\vone_k$ the all-ones vector of $\R^k$, and $\langle\cdot,\cdot\rangle_F$, $\|\cdot\|_F$ the Frobenius inner product and norm.

\subsection*{Objects of training}
\input{\vocabdir/entries/network-as-map}
\input{\vocabdir/entries/parameters-vs-architecture}
\input{\vocabdir/entries/activation-function}
\input{\vocabdir/entries/logits-softmax-cross-entropy}

\subsection*{Knowledge-matrix objects}
\input{\vocabdir/entries/km}
\input{\vocabdir/entries/row-sum-invariant}
\input{\vocabdir/entries/class-centring}
\input{\vocabdir/entries/frobenius-inner-product}

\subsection*{Baselines}
\input{\vocabdir/entries/penultimate-features}
\input{\vocabdir/entries/cka}
\input{\vocabdir/entries/svcca}
\input{\vocabdir/entries/invariances-of-measures}
\input{\vocabdir/entries/discriminability}

\subsection*{Interventions and evaluation}
\input{\vocabdir/entries/adversarial-pair-attack-family}
\input{\vocabdir/entries/counterfactual}
\input{\vocabdir/entries/held-out-validation}
\endgroup

%% file: sections/vocab/vocab-macros.tex
\providecommand{\vocabusbreak}{\vocaborigus\allowbreak}
\providecommand{\vocabentry}[5]{\par\medskip\noindent\begingroup\sloppy\let\vocaborigus\_\let\_\vocabusbreak\textbf{#1.}\ \emph{Definition.} #2\ \emph{Algebraic reading.} #3\ \emph{Computed here.} #4\ \emph{Validated here.} #5\par\endgroup}
\providecommand{\R}{\mathbb{R}}
\providecommand{\vx}{\boldsymbol{x}}
\providecommand{\vone}{\boldsymbol{1}}
\providecommand{\KM}{\mathcal{M}}
\providecommand{\vocabsee}[1]{(see \emph{#1})}

%% file: sections/vocab/entries/network-as-map.tex
\vocabentry{Network as a map $\Psi(W\!,f)$}
{A feedforward neural network is a function $\Psi(W\!,f):\R^d\to\R^C$ obtained by alternating affine maps $z\mapsto W^{(\ell)}z+b^{(\ell)}$ with a fixed scalar nonlinearity $f$ applied coordinatewise: $\Psi(W\!,f)(\vx)=W^{(L)}f\bigl(W^{(L-1)}\cdots f(W^{(1)}\vx+b^{(1)})\cdots\bigr)+b^{(L)}$. Here $W$ collects every weight and bias and $f$ is the activation, which is why the map is written $\Psi(W\!,f)$ and never $f$ itself: on ImageNet $d=150{,}528$ and $C=1000$. A convolutional network is the special case in which the $W^{(\ell)}$ are banded with shared entries.}
{For $f=\mathrm{ReLU}$, $z\mapsto\max(z,0)$, the map $\Psi(W\!,f)$ is continuous and piecewise affine: the zero sets of the pre-activations --- affine hyperplanes for the first layer, piecewise-affine hypersurfaces for the deeper layers (Lemma~\ref{lem:null}) --- cut $\R^d$ into finitely many convex polyhedral regions (activation regions), on each of which $\Psi(W\!,f)$ is one affine map $\vx\mapsto J\vx+c$. The germ of $\Psi(W\!,f)$ at a generic $\vx$ is that affine map, and this paper shows the knowledge matrix is a function of it (Theorem~\ref{thm:germ}) and is determined by nothing less (Corollary~\ref{cor:stab-eq}). Two parameter collections are indistinguishable for everything in this paper when they define the same map.}
{Pretrained \texttt{torchvision} ImageNet networks --- ResNet-152, DenseNet-121 and GoogLeNet, with ResNet-18, AlexNet and VGG as ordering-only raters (Section~\ref{sec:setup}) --- used in \texttt{eval} mode and never trained.}
{The map is checked against its own knowledge matrix on every run through the row-sum identity $\KM(\vx)\vone_{d+1}=\Psi(W\!,f)(\vx)$ \vocabsee{The row-sum invariant $\KM\vone=\Psi(W\!,f)$}.}

%% file: sections/vocab/entries/parameters-vs-architecture.tex
\vocabentry{Parameters versus architecture (the quiver)}
{The architecture is the shape of a network: the number of layers, the width $n_\ell$ of each, which entries of each $W^{(\ell)}$ are free, tied together, or fixed at zero, and the nonlinearity. The parameters $W$ are the numerical values of the free entries --- weights and biases together, which is the collection this paper writes $W$ throughout. Training changes $W$ and never the architecture.}
{Following \citet{armenta2021representation}, an architecture is a quiver $Q$ --- a directed graph with one vertex per neuron and one arrow per weight --- with an activation attached to each hidden vertex, and a parameter collection is a representation of $Q$: a real number on every arrow, the biases being the weights on the arrows leaving the bias vertices; the representation is thin --- dimension vector $(1,\dots,1)$ --- and the width of a layer is its number of vertices. The quiver isomorphisms fixing the input and output vertices --- permutations of the hidden vertices of a layer and, for ReLU, the positive rescalings $\lambda_v$ at a hidden vertex $v$ (incoming arrows times $\lambda_v$, outgoing arrows times $\lambda_v^{-1}$) --- act on representations without changing $\Psi(W\!,f)$; This paper calls the largest group that fixes $\Psi(W\!,f)$ the global function-stabilizer; it is contained in the germ stabilizer at $\vx$, and, under (LCS) at regular $\vx$, Theorem~\ref{thm:maxinv} shows the knowledge matrix is invariant under the whole germ stabilizer at $\vx$, not only under the quiver isomorphisms. Two architectures have non-isomorphic quivers, so no group relates their parameters; the knowledge matrix compares them anyway because its shape depends only on $(C,d)$.}
{This paper holds the architecture fixed and moves the input: it compares three pretrained ImageNet networks whose quivers are not isomorphic (Section~\ref{sec:crossarch}).}
{The implementation checks of Appendix~\ref{app:teleport-checks} verify the quiver-isomorphism case, Lemma~\ref{lem:quiver-inv}, in software, under a neuron permutation and under teleportation; the cross-architecture content of Theorem~\ref{thm:maxinv} has no transform-based test (L1).}

%% file: sections/vocab/entries/activation-function.tex
\vocabentry{Activation function}
{The fixed scalar nonlinearity $f:\R\to\R$ applied coordinatewise between affine layers. Every network in this paper uses ReLU, $f(z)=\max(z,0)$ --- the three networks of Section~\ref{sec:setup} and the three ordering-only raters alike --- so the piecewise-affine theory of Section~\ref{sec:germ} applies throughout; the smooth activations GELU, $f(z)=z\,\Phi(z)$ with $\Phi$ the standard normal distribution function, and $\tanh$ appear here only as the contrasting cases. Without $f$ the composition of layers would collapse to one affine map.}
{ReLU is positively homogeneous, $f(\lambda z)=\lambda f(z)$ for $\lambda>0$, which is exactly what makes the rescaling gauge a symmetry and $\Psi(W\!,f)$ piecewise affine; GELU and $\tanh$ are not homogeneous, so for them the rescaling is not a symmetry of the fixed-activation network --- it is a symmetry only when the activation is carried with the unit as $f_\tau(z)=\tau f(z/\tau)$ (Lemma~\ref{lem:quiver-inv}) --- and $\Psi(W\!,f)$ is smooth rather than piecewise affine. The knowledge matrix carries, at each hidden neuron, the chord $f(z)/z$ --- the slope of the line from the origin through $(z,f(z))$ --- which for ReLU is the $0/1$ activation indicator and for a smooth $f$ differs from the tangent $f'(z)$. No condition whatever is imposed on $f$: the row-sum identity holds for \emph{every} activation at every input of $X_{\mathrm{nz}}$, the set where no hidden pre-activation vanishes, because $f(z)=(f(z)/z)\cdot z$ is an identity for $z\neq0$. What $f(0)=0$ buys is only that $X_{\mathrm{nz}}$ becomes the whole space; where a pre-activation is exactly zero the identity fails by exactly $f(0)$, and no choice of guard repairs that, since a finite diagonal entry times $0$ is $0$. Sigmoid, with $f(0)=\tfrac12$, reproduces the logits to roundoff at random inputs. The honest caveat is that for $f(0)\neq0$ the chord behaves like $f(0)/z$ near the excluded set, so the diagonal is unbounded there ($5.25$ at $z=10^{-1}$, rising to $5\times10^{5}$ at $z=10^{-6}$ for sigmoid): the identity stays exact, the entries do not stay small.}
{ReLU throughout, with the matrix built by the \texttt{knowledgematrix} library at the multi-architecture fork pinned in Section~\ref{sec:setup}.}
{Proposition~\ref{prop:lcs} classifies the activations for which the chord is locally constant --- the two-parameter leaky-ReLU family, and no other continuous $f$ --- and the germ identity of Theorem~\ref{thm:germ} is stated under that condition (LCS), not under piecewise linearity of $f$, which is strictly weaker.}

%% file: sections/vocab/entries/logits-softmax-cross-entropy.tex
\vocabentry{Logits, softmax and cross-entropy}
{The output $\Psi(W\!,f)(\vx)\in\R^C$ is the vector of logits. The softmax $p_c=e^{\Psi_c}/\sum_k e^{\Psi_k}$ turns it into a probability vector on the $C$ classes, and the predicted class is $\arg\max_c \Psi_c$. The cross-entropy loss of a labeled example $(\vx,y)$ is $-\log p_y(\vx)$; the training loss is its average over the training examples.}
{Softmax is constant on the cosets of the line $\R\vone_C$: $p(\Psi+\kappa\vone_C)=p(\Psi)$, so the loss sees only the class of $\Psi$ in $\R^C/\R\vone_C$, and the argmax sees only a fan of cones. Because $\KM\vone_{d+1}=\Psi$ is linear in $\KM$, the shift $\Psi\mapsto\Psi+\kappa\vone_C$ lifts to $\KM\mapsto\KM+\vone_C v^\top$ with $v^\top\vone_{d+1}=\kappa$, an element of $\vone_C\otimes\R^{d+1}$: this is the loss-gauge that class-centering quotients out. Softmax is also equivariant under permutations of the classes, which the class-shuffle null exploits.}
{This paper trains nothing: it reads logits off pretrained networks in \texttt{eval} mode, and the cross-entropy loss enters only through the attacks, which differentiate it to build their perturbations.}
{The row-sum identity ties the two together and is checked at every call \vocabsee{The row-sum invariant $\KM\vone=\Psi(W\!,f)$}.}

%% file: sections/vocab/entries/km.tex
\vocabentry{Knowledge matrix $\KM(\vx)$}
{For a network $\Psi(W\!,f)$ and an input $\vx\in\R^d$, the knowledge matrix is the $C\times(d+1)$ real matrix obtained by contracting the quiver representation $\phi(W\!,f)(\vx)$ that the network $(W\!,f)$ induces on the input $\vx$, in which an arrow leaving a hidden neuron carries its weight times the chord $f(z)/z$ of that neuron, an arrow leaving an input vertex its weight times the input coordinate, and an arrow leaving a bias vertex its weight: column $i\le d$ collects the contribution of the input coordinate $x_i$ to each of the $C$ outputs, column $d+1$ that of the biases. For a ReLU network it equals $[\,J(\vx)\,\mathrm{diag}(\vx)\mid c(\vx)\,]$, where $\Psi(W\!,f)=J\vx+c$ on the activation region of $\vx$. The induced representation is that of \citet{armenta2021representation}, its contraction to one matrix that of \citet{armenta2022double}, and the name is from \citet{leblanc2024hidden}.}
{$\KM(\vx)$ is a linear-algebraic shadow of the representation seen through one input: an element of $\R^C\otimes\R^{d+1}$ whose shape is fixed by the data space $(C,d)$, not by the quiver, so networks of different width and architecture produce elements of one space. When the slope diagonal is locally constant (Definition~\ref{def:lcs}; by Proposition~\ref{prop:lcs} exactly the leaky-ReLU family, ReLU among them) this paper proves it is a function of the germ of $\Psi(W\!,f)$ at $\vx$ (Theorem~\ref{thm:germ}) --- hence invariant, at regular $\vx$, under the germ stabilizer at $\vx$, which contains the global function-stabilizer (Theorem~\ref{thm:maxinv}) --- and that the penultimate activations are not (Theorem~\ref{thm:complete}). Outside that class the matrix is still defined, and still sums to the logits, but the germ reading is unavailable: for a smooth $f$ the chord is not the derivative, so $J$ is not the Jacobian. On ImageNet, $\KM$ is $1000\times150{,}529$.}
{The effective-weight route of the \texttt{knowledgematrix} library at the multi-architecture fork pinned in Section~\ref{sec:setup}, on pretrained ImageNet networks --- equivalently, and $50$--$150\times$ cheaper at that scale, by $C$ vector--Jacobian products, the three routes being cross-checked against each other in Appendix~\ref{app:engineering}.}
{The row-sum identity at every call \vocabsee{The row-sum invariant $\KM\vone=\Psi(W\!,f)$} (what a finite-arithmetic implementation registers is recorded in Appendix~\ref{app:finite-arith}), and the implementation checks of Appendix~\ref{app:teleport-checks}, which exercise the invariance in software under a neuron permutation and under teleportation.}

%% file: sections/vocab/entries/row-sum-invariant.tex
\vocabentry{The row-sum invariant $\KM\vone=\Psi(W\!,f)$}
{For every input in $X_{\mathrm{nz}}$ --- the set at which no hidden pre-activation vanishes --- $\KM(\vx)\,\vone_{d+1}=\Psi(W\!,f)(\vx)$: summing each row of the knowledge matrix returns the corresponding logit, exactly. This holds for \emph{any} activation $f$, with no condition on $f$ and in particular none on $f(0)$ (Section~\ref{sec:germ}): sigmoid, with $f(0)=\tfrac12$, reproduces the logits to roundoff. Activations with $f(0)=0$ --- ReLU, used throughout this paper, among them --- are precisely those for which $X_{\mathrm{nz}}$ is everything and the hypothesis is vacuous, which makes $f(0)=0$ a bonus rather than a prerequisite; at an exactly vanishing pre-activation the identity fails by exactly $f(0)$, and no guard repairs it, since $D_{qq}\cdot0=0$ for every finite $D_{qq}$. It is the one identity every stored knowledge matrix must satisfy, and it is checked, never assumed.}
{The row sum is the linear map $\mathrm{Id}_C\otimes\vone_{d+1}^{\top}:\R^C\otimes\R^{d+1}\to\R^C$. Its kernel $\R^C\otimes\vone_{d+1}^{\perp}$ is the part of $\KM$ invisible to the logits, its orthogonal complement $\R^C\otimes\vone_{d+1}$ the constant-row matrices, and the two are orthogonal for the Frobenius inner product; the visible/invisible decomposition $\|\Delta\KM\|_F^2=\|\Delta\Psi\|_2^2/(d{+}1)+\|Q\|_F^2$ (Theorem~\ref{thm:pyth}) is Pythagoras for this splitting. Class-centering acts on the other tensor factor, as $P\otimes\mathrm{Id}_{d+1}$, so it commutes with the row sum and $\KM^{\circ}\vone_{d+1}=P\Psi$: the identity survives centering in the form softmax sees. Because $\vone$ is a fixed vector independent of the parameters, no stored activation has an analogous accounting.}
{As a gate at the boundary of the pipeline rather than an assumption: the correctness gate of the full-scale reduce passes every sample, and what a finite-arithmetic implementation registers is recorded once in Appendix~\ref{app:finite-arith}. The gate requires \texttt{eval} mode: active dropout desynchronizes the construction passes and the identity then appears to fail at the scale of the logits themselves.}
{The displacement decomposition --- the visible part equals the logit displacement --- is checked per pair on the teleportation check, and on adversarial pairs it is the row-sum identity applied twice; Appendix~\ref{app:engineering} records the agreement of the three construction routes.}

%% file: sections/vocab/entries/class-centring.tex
\vocabentry{Class-centering}
{For $\KM\in\R^{C\times(d+1)}$ with class rows $\KM[c,:]$, subtract the mean row $\bar m=\tfrac1C\sum_c \KM[c,:]$ from every row: $\KM^{\circ}=\KM-\vone_C\bar m^\top=(I_C-\tfrac1C\vone_C\vone_C^\top)\KM$.}
{Write $\R^{C\times(d+1)}=\R^C\otimes\R^{d+1}$. The operator is $P\otimes\mathrm{Id}$ with $P$ the orthogonal projector onto $\vone_C^\perp$; its kernel is $\mathrm{span}(\vone_C)\otimes\R^{d+1}$, the matrices whose $C$ rows coincide, and its image is $\vone_C^\perp\otimes\R^{d+1}$, the matrices whose entries in each column sum to zero over the class index. Softmax is invariant under $\Psi\mapsto\Psi+\kappa\vone_C$, and that shift moves the KM by an element of the kernel; class-centering is the quotient by this loss-gauge, the one symmetry the invariance theorems of this paper leave standing --- it changes the function everywhere, so it lies in neither the global function-stabilizer nor the germ stabilizer at $\vx$ of Theorem~\ref{thm:maxinv}. It also gives $\KM^{\circ}\,\vone_{d+1}=(I_C-\tfrac1C\vone\vone^\top)\Psi$: the row-sum identity survives in the form softmax sees.}
{Not used in this paper, which compares raw knowledge matrices and reports Frobenius distances between them; the entry is here because the gauge it quotients is the one the invariance theorems do not remove.}
{Not validated here, since it is not used; the algebra above is immediate.}

%% file: sections/vocab/entries/frobenius-inner-product.tex
\vocabentry{Frobenius inner product}
{For $A,B\in\R^{C\times(d+1)}$, $\langle A,B\rangle_F=\mathrm{tr}(A^\top B)=\sum_{c,i}A_{ci}B_{ci}$, with norm $\|A\|_F=\langle A,A\rangle_F^{1/2}$ and distance $\|A-B\|_F$. It is the inner product behind every knowledge-matrix comparison in this paper.}
{The standard inner product of $\R^{C(d+1)}$ transported along the vectorization $\mathrm{vec}:\R^{C\times(d+1)}\to\R^{C(d+1)}$, equal to the tensor product of the standard inner products of $\R^C$ and $\R^{d+1}$; hence an orthogonal projector on either factor ($P\otimes\mathrm{Id}$ for class-centering, $\mathrm{Id}\otimes\vone\vone^\top/(d{+}1)$ for the visible part) is orthogonal for it and Pythagoras applies. Both the Frobenius and the operator norm are invariant under rotations of either factor; what is tied to the input basis is the matrix itself, through $\mathrm{diag}(\vx)$ --- a rotation of $\R^d$ changes $\KM$, whose columns are indexed by the input coordinates --- which is why the coherence $A$ (Definition~\ref{def:coherence}) is a pixel-basis descriptor rather than a basis-free one.}
{$d_M=\|\KM(\vx')-\KM(\vx)\|_F$ per adversarial pair, reported raw within an architecture and RMS-per-coordinate --- divided by $\sqrt{C(d{+}1)}$ --- for cross-architecture comparison, so that dimensionality alone does not inflate the distance (Section~\ref{sec:setup}).}
{The visible/invisible decomposition (Theorem~\ref{thm:pyth}) is Pythagoras for this inner product; its equality --- the visible part equals the logit displacement --- is checked per pair on the teleportation check, and on adversarial pairs it is the row-sum identity applied twice.}

%% file: sections/vocab/entries/penultimate-features.tex
\vocabentry{Penultimate features}
{The vector $h(\vx)\in\R^D$ of hidden activations just before the last affine layer, so that $\Psi(W\!,f)(\vx)=W^{(L)}h(\vx)+b^{(L)}$; its dimension $D$ is the width of the last hidden layer ($2048$ for ResNet-152, $1024$ for DenseNet-121 and GoogLeNet). It is the representation most similarity methods compare, and the hidden layer this paper's experiments take as representative when comparing the knowledge matrix against hidden activations; the theorems about hidden activations hold for any hidden layer.}
{$h$ lives in a space whose dimension depends on the architecture and whose coordinates carry the hidden-neuron gauge: under the gauge group (permutations and, for ReLU, positive rescalings of the hidden neurons) $h$ is covariant, not invariant --- it moves while $\Psi(W\!,f)$ does not --- and under an architecture change it has no transformation law at all; Theorem~\ref{thm:complete}(iv) sharpens covariance to incompleteness: two networks with the same germ can differ in $h$. Comparing two networks' $h$ therefore needs either a gauge-invariant statistic (CKA, SVCCA) or an alignment; and nothing like the row-sum accounting exists for it: what turns $h$ back into the logits is the last weight matrix $W^{(L)}$ itself, which depends on the parameters and whose shape depends on the width, whereas the knowledge matrix's accounting vector is the constant $\vone_{d+1}$, the same for every network.}
{Feature hooks on the pretrained \texttt{torchvision} networks, read in \texttt{eval} mode; the drift of $h$ is reported in absolute $\ell_2$ within an architecture and RMS-per-dimension across architectures, the second carrying the feature-scale caveat recorded in Section~\ref{sec:limitations} (L6).}
{The implementation checks of Appendix~\ref{app:teleport-checks} measure the permutation and teleportation drift of $h$ against that of $\KM$ on the same pairs, and Section~\ref{sec:study1c} adjudicates it with the nine-measure panel.}

%% file: sections/vocab/entries/cka.tex
\vocabentry{Linear CKA}
{Centered kernel alignment \citep{kornblith2019similarity} between two feature matrices $A\in\R^{N\times m_a}$ and $B\in\R^{N\times m_b}$ on the same $N$ inputs: with the columns centered, $\mathrm{CKA}(A,B)=\|B^\top A\|_F^2/(\|A^\top A\|_F\,\|B^\top B\|_F)\in[0,1]$. The debiased form replaces the plug-in estimator by the unbiased Hilbert--Schmidt independence criterion estimator of \citet{song2012feature}; this paper uses the debiased form.}
{With the centered Gram matrices $K_A=AA^\top$ and $K_B=BB^\top$ ($N\times N$), $\mathrm{CKA}=\langle K_A,K_B\rangle_F/(\|K_A\|_F\|K_B\|_F)$: the cosine of two Gram matrices, hence defined for any $m_a\neq m_b$ and invariant under orthogonal transformations and isotropic scaling of either feature space, but not under general invertible maps. The plug-in estimator is biased upward toward $1$ when $m/N$ is large \citep{murphy2024debiased}, and a permuted-probe null measures that bias. The panel of this paper is not in that regime, at $N=25{,}000$ samples against widths $D\in\{1024,2048\}$ (Section~\ref{sec:setup}); it uses the unbiased estimator anyway, because the choice is right at any ratio, and checks the residual bias directly rather than arguing it away.}
{The debiased linear CKA column of the nine-measure panel (Section~\ref{sec:study1c}) and of its cross-architecture extension (Section~\ref{sec:crossarch}), computed with the unbiased HSIC estimator inside the minibatch-CKA framework.}
{Study~1, whose reading is that every one of the eight panel measures that returned a value registers some teleportation drift and none quotients out the full change of basis (raw CCA and PWCCA, which are not in the panel, would); the shuffled-pair control returns debiased CKA $\le5.4\times10^{-4}$, so the estimator carries no upward bias, and the three readings available --- $0.90$--$0.94$ for a network and its teleported copy, $0.36$--$0.54$ for two different architectures, $0.01$--$0.06$ for a network and its random initialization --- place the exact-function pair far above every other pair; whether an untrained network and its own teleported image would score as high is not tested.}

%% file: sections/vocab/entries/svcca.tex
\vocabentry{SVCCA}
{Singular-vector canonical correlation analysis \citep{raghu2017svcca}: reduce each centered feature matrix to the top singular directions carrying $99\%$ of its variance, then compute the canonical correlations between the two reduced feature sets and report their mean, in $[0,1]$.}
{Canonical correlations are the singular values of $\Sigma_A^{-1/2}\Sigma_{AB}\Sigma_B^{-1/2}$, i.e.\ the cosines of the principal angles between the column spaces of the whitened features; \emph{they} are invariant under any invertible linear map of either feature space, which is why the SVD truncation must come first (without it, two generic subspaces of $\R^N$ of dimensions $m_a,m_b$ with $m_a+m_b>N$ intersect and correlations of $1$ appear for free). The truncation buys well-posedness at the price of the invariance: keeping the top $99\%$ of the variance is not equivariant under an invertible map, since an anisotropic rescale changes the variance spectrum and hence which directions are retained. SVCCA is therefore \emph{not} invariant under the per-channel rescaling $h\mapsto\tau\odot h$ that a teleportation induces on penultimate features, even though raw CCA and PWCCA are exactly invariant under it --- a distinction worth keeping straight, because the claim of Section~\ref{sec:study1c} is about the eight measures of its panel that returned a value and not about every measure that can be built from canonical correlations. The measure is in any case ill-posed unless $N$ is comfortably larger than the retained dimensions.}
{Not a member of the panel of Section~\ref{sec:study1c}; it enters this paper only through the invariance boundary above.}
{The invariance boundary above is checked directly: under a teleportation-style per-channel rescale $h\mapsto\tau\odot h$ the raw canonical correlations stay at $1$ to roundoff --- with dead units present, with $\tau_{\min}\approx3\times10^{-5}$, and even when $N<p$ --- while SVCCA moves to $0.97$--$0.99$ on the same draws (Section~\ref{sec:study1c}).}

%% file: sections/vocab/entries/invariances-of-measures.tex
\vocabentry{What each measure is invariant to}
{For a group $G$ acting on representations, a measure is $G$-invariant when it takes the same value on $g\cdot A$ as on $A$. Knowledge-matrix measures --- the Frobenius distance $\|\cdot\|_F$ and the normalized similarities built from it --- are invariant under hidden-neuron permutation (any activation) and under ReLU positive rescaling, because $\KM$ itself is (Lemma~\ref{lem:quiver-inv}, Corollary~\ref{prop:perm}); under softmax translation once the matrices are class-centered; and, for those that divide by the norms, under a positive rescaling of the logits (temperature). They are not invariant under rotations of the input coordinates or general reparameterizations of the data. Linear CKA: orthogonal maps and isotropic scaling of the features. Raw CCA, and PWCCA with it: every invertible linear map of either feature space, the per-channel rescaling $h\mapsto\tau\odot h$ that a teleportation induces included, which they absorb exactly. SVCCA: not that class, despite being assembled from canonical correlations, because its variance-energy truncation runs first \vocabsee{SVCCA}. The penultimate Frobenius distance: none of the hidden gauge.}
{Each measure factors through a quotient: the knowledge matrix already factors, under (LCS) at regular inputs, through $\{W\}/\mathrm{Stab}(\Psi)$ at the level of the object, and the normalized similarities through the further quotient by $\R_{>0}$; CKA factors through $\R^{N\times m}/(O(m)\times\R_{>0})$, i.e.\ through the Gram matrix; the canonical correlations factor through $\R^{N\times m}/GL(m)$, i.e.\ through the column space --- but SVCCA composes them with a truncation that is not $GL(m)$-equivariant, so SVCCA factors through no group quotient at all. The larger the group, the less a measure can distinguish, so an invariance is a design choice with a cost: SVCCA all but quotients out the anisotropic scaling that CKA registers, and neither sees the row-sum structure the knowledge matrix carries. That structure is not the knowledge matrix's alone --- by Theorem~\ref{thm:weff}, under (LCS), the same content is carried by per-class gradient$\times$input together with the exact bias attribution, and this paper claims no uniqueness. The consequence it does draw: the alignment freedom of the activation-based measures is exactly the ill-posedness the fixed-shape comparison avoids --- comparing two networks through their activations means comparing feature spaces of different dimension, so one must pick a quotient or search for an alignment, and the answer depends on the choice, whereas the knowledge matrices of any two networks on this data already lie in one $\R^{C\times(d+1)}$, with nothing to align.}
{Not a computed quantity but the reading key for the nine-measure panel of Section~\ref{sec:study1c}, which runs the representational-similarity literature's canonical measures on pairs whose ground truth is known exactly --- a network and its teleported copy compute the same function --- and asks which measures recover the invariance and which do not.}
{Section~\ref{sec:study1c}: every one of the eight measures of that panel that returned a value registers some teleportation drift and none reaches exact invariance (Gromov--Wasserstein did not converge and is not counted), while the knowledge matrix does not move (Lemma~\ref{lem:quiver-inv}; the implementation check is Appendix~\ref{app:teleport-checks}).}

%% file: sections/vocab/entries/discriminability.tex
\vocabentry{Discriminability}
{The ability of a similarity or distance measure to \emph{separate} two populations of pairs --- pairs that a known relation holds between, against pairs it does not --- rather than the level it reports: a measure that returns $0.9$ on both populations is uninformative however high the number. It is quantified by the gap between the populations relative to their spread, or by a margin over a permutation null.}
{A one-dimensional signal-detection quantity: with means $\mu_1,\mu_0$ and common standard deviation $\sigma$, $d'=(\mu_1-\mu_0)/\sigma$; a measure with a large invariance group may collapse the two populations (small $d'$) even when both levels are high --- which is why the invariance table \vocabsee{What each measure is invariant to} and this one have to be read together. It is also the reading of this paper's detector bake-off, where penultimate features separate clean from adversarial samples better than knowledge matrices in five of the six detector configurations (Section~\ref{sec:study3-bakeoff}): separability under a chosen classifier is a question of statistical power, to which the invariance theorems do not speak.}
{The six-detector, three-representation bake-off of the honest-negatives section, and the two nulls that accompany the nine-measure panel --- the random-network control and the shuffled-pair control (Section~\ref{sec:study1c}). Discriminability across independently trained networks, a statistic over a population of trainings rather than over inputs at one fixed network, is not attempted here.}
{The nulls are the validation: the random-network control returns debiased CKA $0.010$--$0.058$ and the shuffled-pair control $\le5.4\times10^{-4}$, far below the $0.90$--$0.94$ agreement the panel reports on a network and its teleported copy and below the $0.36$--$0.54$ of two different architectures, so the measure separates the exact-function pair from every other pair (Section~\ref{sec:study1c}); what the random-network control does not test is whether an untrained network and its own teleported image would score as the trained pair does.}

%% file: sections/vocab/entries/adversarial-pair-attack-family.tex
\vocabentry{Adversarial pair and attack family}
{An adversarial pair is an input $\vx$ together with a perturbed copy $\vx'=\vx+\delta$, with $\delta$ small in a chosen norm ($\ell_\infty$ or $\ell_2$, within a budget) and constructed so that the network's prediction changes; an attack is the algorithm that constructs $\delta$, and attacks group into families by mechanism: one-shot gradient sign (FGSM), iterative $\ell_\infty$ (PGD, APGD), margin-optimizing $\ell_2$ (CW), minimal-norm boundary (DeepFool), gradient-free score-based (Square).}
{For a network satisfying (LCS) --- ReLU here --- the pair either stays inside one activation region, where $\KM(\vx')-\KM(\vx)=[\,J\,\mathrm{diag}(\delta)\mid0\,]$ exactly (the bias column cancels, Theorem~\ref{thm:within}), or crosses region walls, where the matrix jumps. This paper reads each pair through the visible/invisible decomposition (Theorem~\ref{thm:pyth}) and the coherence $A=(d_\Psi/d_M)^2$ (Definition~\ref{def:coherence}), with $d_\Psi=\|\Psi(W\!,f)(\vx')-\Psi(W\!,f)(\vx)\|_2$ and $d_M=\|\KM(\vx')-\KM(\vx)\|_F$, bounded by $A\le d$ within a region and equal to $1$ for a one-pixel change; attack families are then ordered by their median coherence and the order is compared across architectures (Kendall's $W$).}
{\texttt{torchattacks} with the six attacks above on ResNet-152, DenseNet-121 and GoogLeNet, plus ResNet-18, AlexNet and VGG as ordering-only raters, six raters in all for the concordance (Section~\ref{sec:setup}). Two filtering conventions are used and must not be conflated: the headline coherence and ordering tables apply a per-cell relative division guard $d_\Psi > 10^{-6}\max_i d_{\Psi,i}$, which discards only pairs whose logits did not move; the appendix full-scale ordering panel and the mechanism pilot apply the stricter attack-success filter $d_\Psi\ge1$, $d_M>0$.}
{Study~2 (coherence and attack-family ordering) and the six-detector, three-representation bake-off of the honest-negatives section, on AlexNet/CIFAR-10; the within-region identity is checked numerically.}

%% file: sections/vocab/entries/counterfactual.tex
\vocabentry{Counterfactual}
{An intervention that changes one factor of a system, holds the rest fixed and observes the output, so as to attribute the change to that factor; its relatives in the interpretability literature are ablation (set a component to zero or a baseline) and activation patching (copy a component's value from a second forward pass). The matrix-direction counterfactual of this paper asks the local linear model $W_{\mathrm{eff}}(\vx)$ behind the knowledge matrix for the smallest input change that flips the class.}
{Within one activation region the network is affine, so the smallest $\ell_1$ perturbation $\delta$ that pushes the target-minus-source logit gap past a margin $m$ is a linear program: $\min\|\delta\|_1$ subject to $(W_{\mathrm{eff}}(\vx)_{t,\cdot}-W_{\mathrm{eff}}(\vx)_{s,\cdot})\,\delta\ge m-(\Psi_t-\Psi_s)$ and a per-coordinate box keeping the image inside $[0,1]$; without the box the single most cost-effective coordinate solves it, with the box a saturation greedy in order of $|c_i|$ does ($O(n\log n)$, $\ell_1$-optimal). The guarantee is conditional on $\vx+\delta$ staying in the region of $\vx$: the within-region anatomy of Theorem~\ref{thm:within} says nothing across walls.}
{Three architectures, three source images, three source classes, two target classes ($18$ linear programs per architecture, $54$ in all), margin $m=0.1$, each with a Boolean \texttt{region\_ok} flag (ReLU patterns and pooling argmaxes identical at $\vx$ and $\vx+\delta$) recorded beside every $\delta$-norm.}
{The negative is the validation: \texttt{region\_ok} is False on $54$ of $54$ and the actual logits never flip ($0/54$; every $\|\delta\|_\infty$ at the per-channel saturation cap of about $4.4$) --- the matrix direction leaves the region it was computed in, exactly as the within-region geometry predicts; reported as Limitation~L3 (Section~\ref{sec:limitations}), with a multi-region reformulation left open.}

%% file: sections/vocab/entries/held-out-validation.tex
\vocabentry{Held-out set and validation}
{Data the training iteration never sees. Validation is the use of held-out data to choose between models or hyperparameters; test accuracy is held-out accuracy reported after every choice has been made. This paper selects nothing --- it trains no network and tunes no hyperparameter --- and evaluates on the ImageNet validation images throughout: the first $N$ by sorted filename for Studies~1 and~3 and for the full-scale pair set, while the $200$-pair set of Study~2 is instead a seeded draw from a split half, so the two ordering panels sit on different images and are read as independent cross-checks (Section~\ref{sec:setup}).}
{Training minimizes an empirical average over the training split, so a quantity computed on that split is a biased estimate of its population value (the optimizer has adapted to those points), while on an independent split it is unbiased. Nothing of this bites here, where no quantity is fitted.}
{The image indices behind every study are fixed and recorded, and the two sampling conventions above are not interchangeable --- a table computed on one set is not a second view of the other.}
{This paper quotes the sample set alongside every panel, and the two attack-family ordering panels --- computed on different images under different attack budgets --- reproduce the same family-level grouping, which is exactly why they are read as independent cross-checks rather than as two views of one sample.}